%% file: main.tex
\documentclass[12pt,a4paper,oneside]{amsart}

\input{header.tex}

\date{\today}
\title[Non-parametric online market regime detection]{Non-parametric online market regime detection and regime clustering for multidimensional and path-dependent data structures}
\author{Blanka Horvath$^{1, 2, 3}$}
\author{Zacharia Issa$^{*4}$}
\email{zacharia.issa@kcl.ac.uk}
	
\begin{document}

    \begingroup
        \renewcommand*{\thefootnote}{\fnsymbol{footnote}}
        \footnotetext[1]{Corresponding author: \href{mailto:zacharia.issa@kcl.ac.uk}{\texttt{zacharia.issa@kcl.ac.uk}}}
    \endgroup
    \footnotetext[1]{Department of Mathematics, Oxford University, Oxford, United Kingdom.}  
    \footnotetext[2]{The Oxford Man Institute, Oxford, United Kingdom.} 
    \footnotetext[3]{The Alan Turing Institute, London, United Kingdom.} 
    \footnotetext[4]{Department of Mathematics, King's College London, London, United Kingdom.} 
	
	\begin{abstract}
	   In this work we present a non-parametric online market regime detection method for multidimensional data structures using a path-wise two-sample test derived from a maximum mean discrepancy-based similarity metric on path space that uses rough path signatures as a feature map. The latter similarity metric has been developed and applied as a discriminator in recent generative models for small data environments, and has been optimised here to the setting where the size of new incoming data is particularly small, for faster reactivity. 
        On the same principles, we also present a path-wise method for regime clustering which extends our previous work \cite{horvath2021clustering}. The presented regime clustering techniques, as in \cite{horvath2021clustering}, were designed as ex-ante market analysis tools that can identify periods of approximatively similar market activity, but the new results also apply to path-wise, high dimensional-, and to non-Markovian settings as well as to data structures that exhibit autocorrelation. 
        We demonstrate our clustering tools on easily verifiable synthetic datasets of increasing complexity, 
        and also show how the outlined regime detection techniques can be used as fast on-line automatic regime change detectors or as outlier detection tools, including a fully automated pipeline. Finally, we apply the fine-tuned algorithms to real-world historical data including high-dimensional baskets of equities and the recent price evolution of crypto assets, and we show that our methodology swiftly and accurately indicated historical periods of market turmoil. \\
   
	\end{abstract}
	\maketitle
	\begin{spacing}{0.5}
	\tableofcontents
	\end{spacing}

    \section{Introduction and problem setting}\label{sec:intro}

    \input{section1/11introduction.tex}

    \input{section1/12contribution.tex}

    \input{section1/13pathwiseregimes.tex}

    \input{section1/14market_regime_classification_and_detection.tex}

    \section{Preliminaries} \label{sec:prelims}
	
	Here we give the theoretical background and tools used in our experiments. We cover the essentials here and refer the reader to the appropriate section of the Appendix for further details.

    \input{section2/21path_signatures.tex}

    \input{section2/22streaming_data.tex}

    \input{section2/23rkhs.tex}

    \section{Experimental methods}\label{sec:experiments}

    \input{section3/31path_data_partitioning.tex}

    \input{section3/32mrdp_methods.tex}

    \input{section3/33mrcp_methods.tex}

    \input{section3/34experimental_overview.tex}

    \section{Online regime detection on path space}\label{sec:mrdp}

    \input{section4/41toy_example.tex}

    \input{section4/42multiclass.tex}

    \input{section4/43non_ensemble.tex}

    \input{section4/44rough_higher_dim.tex}

    \input{section4/45mmd2_v_mmd1.tex}

    \input{section4/46comparisons.tex}

    \input{section4/47nonmarkovian.tex}

    \section{The regime classification problem}\label{sec:mrcp}

    In this section we give results for our approach to the market regime classification problem (MRCP) outlined in Section \ref{subsec:mrcpmethods}. 

    \input{section5/51toy_example.tex}

    \section{Real data experiments}\label{sec:realdata}

    In this section we provide experiments on real market data using our anomaly detection technique. We seek to leverage the ease at which we can study higher-dimensional paths by tackling both the MRDP and MRCP on a basket of equities and cryptocurrency pairs respectively. We conclude the section with an example of how one can employ our approach as part of a wholly data-driven detection pipeline, operating on a path-by-path basis.

    \input{section6/61basket_equities.tex}

    \input{section6/62crypto.tex}

    \input{section6/63total_pipeline.tex}

    \section{Conclusion}\label{sec:conclusion}

    In this work, we show that the (higher-rank) signature maximum mean discrepancy can be applied to a wide variety of online regime detection problems, and offline regime clustering problems. We showed that online regime detection could be performed in a parametric, beliefs-based setting and in a non-parametric setting. Common issues faced when using the MMD as a detection tool (ensemble evaluation, non-Markovianity of input data) were handled via the use of signature kernel scoring rules, and an associated similarity function. We showed that our methods worked well on both synthetic and real market data. 

    Topics for future research include pursing a deeper understanding of the two-sample test associated to $\mathcal{D}^r_{\text{sig}}$. In particular it is of interest to determine how the terms of the signature contribute to the value of the MMD between two sets of paths $\mathbb{P}, \mathbb{Q}$. If $\mathbb{P}, \mathbb{Q}$ are comprised of only a small number of path objects, sampling variance can cause a higher incidence of Type II error than would be expected. This has implications not only to regime detection, but to any work which utilizes the signature kernel MMD as a two-sample testing tool (generative modelling, calibration, and so on). In this work we have seen that this issue can be circumvented by using an appropriate path transformation or path scaling ($\sigma$ bandwith parameter in the static RBF kernel). It is thus of interest to understand how each of these concepts are related to each other. 
    
    {\renewcommand{\addcontentsline}[3]{}
    \section*{Acknowledgements}
    ZI was supported by EPSRC grant EP/R513064/1.
    }
    \newpage
	
	\bibliographystyle{halpha-abbrv}
	\bibliography{references}

    \newpage

    \appendix
    \section*{Appendix}
    \addtocontents{toc}{\protect\setcounter{tocdepth}{0}} 
    \input{appendix.tex}
    
\end{document}

%% file: header.tex
\usepackage{amsmath}
\usepackage{amssymb}
\usepackage{amsthm}
\usepackage{graphicx}
\usepackage{caption}
\usepackage{subcaption}
\usepackage{ esint }
\usepackage{enumerate}
\usepackage{enumitem}
\usepackage{epstopdf}
\usepackage[utf8]{inputenc}
\usepackage[ruled,vlined]{algorithm2e}
\usepackage{xcolor}
\usepackage{tabularx}
\usepackage{geometry}
\usepackage{float}
\usepackage[T1]{fontenc}
\usepackage{xfrac}
\usepackage{bbm}
\usepackage{cancel}
\usepackage{parskip}
\usepackage{relsize}
\usepackage[normalem]{ulem}
\usepackage[english]{babel}
\usepackage{textcomp}
\usepackage{mathrsfs}
\usepackage{upgreek}
\usepackage{stackengine}
\usepackage{verbatim}
\usepackage{hyperref}
\usepackage[OT2,T1]{fontenc}
\usepackage{titletoc}

\hypersetup{
	colorlinks,
	citecolor=black,
	filecolor=black,
	linkcolor=black,
	urlcolor=black
}

\usepackage{setspace}
\usepackage{indentfirst}
\usepackage{rotating,booktabs}
\usepackage[T2A,T1]{fontenc}
\makeatletter
\renewcommand{\@seccntformat}[1]{\csname the#1\endcsname.\quad}
\makeatother

\newcommand{\norm}[1]{\left\lVert#1\right\rVert}

\newcommand{\ex}{\mathbb{E}}
\newcommand{\fil}{\mathcal{F}}

\newcommand{\real}{\mathbb{R}}

\DeclareMathOperator{\Sha}{\sqcup\kern-0.18em\sqcup}

\DeclareMathOperator*{\argmin}{arg\,min}

\newtheorem{theorem}{Theorem}[section]
\newtheorem{proposition}[theorem]{Proposition}

\theoremstyle{remark}
\newtheorem{remark}[theorem]{Remark}

\theoremstyle{definition}
\newtheorem{definition}[theorem]{Definition}
\newtheorem{problem}[theorem]{Problem}

\makeatletter 
\providecommand{\paragraph}{} 
\renewcommand{\paragraph}{
  \@startsection{paragraph}{4}{\z@}%
                {1.5ex \@plus 0.5ex \@minus 0.2ex}%
                {-1em}%
                {\normalsize\bfseries}
}
\makeatother 

\allowdisplaybreaks

%% file: section1/11introduction.tex
Time series data derived from asset returns are known to exhibit certain properties, termed \emph{stylised facts}, that are consistently prevalent
across asset classes and markets. To name a few, asset price returns are widely accepted to be   
non-stationary and to exhibit volatility clustering.  We refer the reader to \cite{cont2001empirical} for a thorough discussion of such properties. In this article we will turn our attention to one property in particular, the heteroscedastic nature of financial time series, since it is of imminent practical relevance to financial analysts and quants
for a multitude of practical applications. In this context, one may be interested in whether a given asset returns series---or a set of series, in case of multiple assets---can be divided into periods in which the (random) asset price dynamics can be attributed to the \emph{same} (up to a small margin of imprecision) underlying distribution. Such periods are often referred to as \emph{market regimes}. The task of grouping these regimes the \emph{market regime clustering problem (MRCP)}. This article is devoted to the \emph{online detection} of changes in such regimes, i.e. to developing tools that help us recognise in real time (as data comes in) if a shift in the underlying regime is happening. We will refer to this as \emph{(online) market regime detection problem (MRDP)}.

\subsection{Relation to existing literature:}\label{sec:lit}
In classical statistics the problem we address in this work is related to problems studied under the umbrella of \emph{change point detection} (CPD) methods. Among these, our setting is closest to on-line change point detection, although our algorithm has a somewhat different flavour than many similar methods in the literature. The term refers to the localization of \emph{change points}, that is to the detection of instances of \emph{abrupt distributional changes} in ordered observations, such as time series. CPD has been studied over the last several decades in numerous scientific fields including statistics, computer science, and data mining. Applications include a broad range of important real-world problems, from medical conditioning and speech recognition to climate change detection. While each of these problems has its bespoke challenges, our method is inspired by financial applications (such as optimal investment and (deep) hedging and pricing) and is intended as an early (on-line) indicator of distributional changes in financial time series. We have fine tuned our algorithm to such applications, potentially at the expense of features that might be relevant for other fields (e.g. speech recognition).
For a strenuously thorough and comprehensive account of associated methods, see the 2016 survey \cite{aminikhanghahi2017survey}, and for a more recent review of statistical techniques related to CPD, see \cite{truong2020review}. Finally, for some of the most recent related literature we refer the reader to the literature review in \cite{londbuhlkov2022changepoint}. 

In the classical statistics literature, change point detection methods broadly fall in one of the two categories: Parametric, and non-parametric change point methods:
(i) Parametric change point methods typically assume that the observations in between change points stem from a finite dimensional family of parametric (often Gaussian or exponential) distributions.  
(ii) Non-parametric change point detection methods use measures that do not rely on parametric forms of the distribution or the nature of change.
\\
From a different point of view, change point detection algorithms can be classified as (a)  “offline” or (b) “online”. (a) Offline algorithms consider the entire data set at once, and look back in time to recognize where the change occurred.  In contrast, (b) online (or real-time), algorithms run concurrently with the process they are monitoring, processing each data point as it becomes available, with a goal of detecting a change point as soon as possible after it occurs.
In practice, no change point detection algorithm operates in perfect real time because it must inspect new data before determining if a change point occurred between the old and new data points. Different (online) algorithms require different amounts of new data before change point detection can occur. Based on this, \cite{aminikhanghahi2017survey} defines the term ``$\epsilon$ -real time algorithm'' to refer to an online algorithm which needs at least $\epsilon$ data samples in the new batch of data to detect change points. Smaller $\epsilon$ values imply more responsive (stronger) change point detection algorithms.

\textit{Multivariate setups} of change-point detection algorithms are generally considered to be  challenging both in (i) parametric scenarios, but even more so in (ii) non-parametric scenarios, see \cite{londbuhlkov2022changepoint}.  There is a limited number of contributions in the multivariate non-parametric setup available. The non-parametric multivariate change point methods we are aware of are based on one of the following approaches: 
Ranks (\cite{lungyutfong2015homogeneity}), distances (\cite{matteson2014nonparametric, chenzhang2015graph, zhang2021chen}) and some are based on kernel methods, such as kernel distances \cite{arlot2019kernelcpd, garreauarlot2018cpdkernels},  and kernel densities.\\
Classically, the idea behind these methods is a form of hypothesis testing based on a test statistic that measures the dissimilarity of distributions. 
The method we develop in this paper is similar--in its core--to this line of thoughts, but with some key differences (which we detail in section \ref{sec:Contribution}) and with the twist that we make use of path-based methods for this task.\\
In the more recent era of statistics \emph{with machine learning} (ML), a large range of non-parametric methods are available to learn complex conditional classes of probability distributions (see for example recent works on conditional Market Generators \cite{buehler2020data, issa2023non, wiese2019simulation}).
Making use of these ML-based non-parametric methods for change point detection purposes have proven (in direct comparison with the aforementioned classical methods) to often outperform simple rank and distance-based methods: 
\cite{friedman2004report} proposed to use binary classifiers for two- sample testing, \cite{heidiger2022twosample, londbuhlkov2022changepoint} applied this framework in combination with  with random forests and \cite{lopezpaz2017oquab} with deep neural networks. \\
In fact, once we have delved into the techniques used on the newly emerging field of \textit{Market Generation} or \textit{Market Simulation} (two twin-terms that were coined in the works of Kondratyev and Schwarz \cite{kondratyev2019market} and B\"uhler et al. \cite{buehler2020data}), it is clear that the very same tool---the one we use there to verify that the simulated distribution is indistinguishable from the historical one---can be used for the opposite task as well: to indicate if two samples are very \emph{unlikely} to have been drawn from the same distribution, which would indicate a \emph{regime change}. In \cite{buehler2020data}, a pathwise signature-MMD has been used to evaluate the ``quality'' of the generative model's synthetic market paths. We observe that in fact with the same reasoning as above, a ``good'' and easily computable metric on path space is a central object for a number of key financial applications, including not only market generation (and discrimination) but regime classification and clustering problems, outlier detection and--given sufficient accuracy--even calibration problems as well as an array of further applications.

Of course, the above reasoning about different applications which sounds straightforward in theory poses some intricate modelling challenges if the algorithm is tasked with picking up distributional changes in real time via a sliding window of a financial time series. As we will see later on, the challenge becomes even more intricate if the time series is not an output of a synthetic dataset (generated by the same distribution) but rather a real (historical) time series that is inherently heteroscedastic and non-stationary.

%% file: section1/12contribution.tex
\subsection{Our contribution in this work}\label{sec:Contribution}
A key difference between our work and the non-parametric method by \cite{londbuhlkov2022changepoint} and the contributions presented in the review article \cite{truong2020review}, lies in our pronounced aim for real-time applicability of the method. While the aforementioned methods focus on offline problems, that is, their aim is to retrospectively detect changes after all samples have been observed, our present work---in contrast--develops an algorithm that can be used as an early indicator of shifts in market regimes in real time (i.e. online), which relies on a satisfactory accuracy of the algorithm for a relatively moderate number of observations. The latter property ensures that  our algorithm can be tuned to be more \emph{reactive}, i.e. mark \emph{potential} changes in the underlying distribution of the data within a short time frame, that is, as based on a relatively low number of new incoming observations. The "price" to pay for a reactive methodology comes in terms of its accuracy\footnote{As usual, the accuracy of any such methodology can be expected to improve as the sample size (of observations) increases.}.  
One of the difficulties mentioned in the literature is that several non-parametric change point methods only detect changes that are deemed "large enough" considering the noise level and available data. This requirement is necessary to prevent numerous false positives, see \cite{arlot2019kernelcpd}.
However, with view to the applications we have in mind, we can be somewhat more lenient about ``false positives" than the majority of traditional change point detection methods: Since our goal is an automated early indicator of shifts market regimes, it is more important to pick up on \emph{potential} changes---even if they happen to ease back into the previous state as time passes---rather than to prioritise the elimination of false positives at all costs.
At the same time, we should point out that while signal-to-noise concerns are indeed (justifiably) a limiting factor in most other methods, signal strength is particularly good in our framework, which is a strength of (pathwise) signature-kernel methods in general, but it is a particularly optimised feature of our specific model design and kernel choice in this paper.
On a slightly different note, we also argue, that for a number of financial applications (deep pricing and hedging or deep trading or portfolio optimization) it is indeed advantageous to focus our regime detection efforts on \emph{large changes} (rather than small ones) and to equip the algorithms with robustification tools (to account for small levels of model uncertainty) to ensure that they remain valid in the face of small changes in the data-- for example changes that are within a noise and small-data-induced threshold.

In our previous work \cite{horvath2021clustering}, we gave an algorithm which attempts to solve the MRCP by implementing a modified version of the classical $k$-means algorithm on the space of distributions $\mathcal{P}(\mathbb{R})$. This was done by representing asset price paths by the empirical measure associated to segments of the asset's corresponding path of log returns. Equipped with a suitable metric, this methodology was able to detect changes in regime from both synthetically generated and real data. However, this approach is somewhat limited in two aspects: first, auto-correlative effects are discarded when conducting inference on distributions of returns, and second, extending to higher-dimensional price paths poses a computational challenge, as discussed in our previous work \cite{horvath2021clustering} which discusses techniques presented in \cite{nadjahi2019asymptotic} and \cite{rabin2011wasserstein} on this matter.

In this paper, we outline a different approach to market regime detection: We study distributions on path space as opposed to those on $\mathbb{R}$, and use path signatures as features together with a version of the \emph{signature maximum mean discrepancy distance} (sig-MMD) introduced in \cite{chevyrev2018signature}. Utilising the path signature in our inference approach means we are able to capture time-dependent effects which were previously not possible.

With this in mind, we aim to fine-tune and improve our sig-MMD distance as much as possible: as one can uniformly approximate functions on path space by linear functionals acting in signature space (see for instance \cite{lemercier2020distribution}, Theorem 2.1) the more signature terms one takes, the more nuanced picture we have about the observed paths. However, with the truncated signature approach, one quickly faces a trade-off between taking more terms in the truncated signature and facing the curse of dimensionality. Instead, in this paper we take a leap towards kernel methods with a \emph{kernel trick} (cf. \cite{salvi2021signature, lemercier2020distribution, salvi2021higher}) to study distributions on pathspace and thus overcome the truncation problem associated to the path signature (cf. \cite{salvi2021signature}) and to make a number of further improvements to the original sig-MMD algorithm used in \cite{buehler2020data}.	Specifically, using a variant of the kernel methods introduced in \cite{salvi2021higher} we are also able to extend dimensionality rather easily, and in fact obtain stronger results in higher-dimensional cases without facing too heavy a computational penalty (see section \ref{subsec:mmd}). In passing, we also fix a problem of the original truncated sig-MMD approaches connected to conditional distributions: As conditional distributions are closely liked with the filtration of the process, one needs a version of the sig-MMD that is fine enough to carry the filtration information of the stochastic evolution of the process in order to generate (and evaluate) conditional distributions. Without such a finer sig-MMD distance, any conditional analysis (and thus any analysis of path-wise temporal effects) is confined to the Markovian case and strictly speaking one is left to good faith when leaving the Markovian regime. As most settings in real life datasets and financial time series are typically non-Markovian by nature, this would a priori be quite a severe limitation.

To summarize, our techniques are superior to any method using the original truncated sig-MMD for the following reasons:

\begin{enumerate}
\item \textbf{Non-truncation}: Since one can employ a ``kernel trick'' to explicitly arrive at the signature kernel without needing to truncate the signature mapping at any point\footnote{Often in the literature it is assumed that due to the factorial decay of the signature mapping, very little relevant information is lost at truncation. This implicitly draws a parallel between the information contained in the higher-order signature terms and their numerical size, but the latter does not necessarily translate to their relevance for distinguishing between distributions. This is the topic of a current working paper.}.
\item \textbf{High dimensionality}: The computational cost for calculating the signature kernel between collections of paths is linear in the path dimension, making inference on higher dimensional objects possible and computationally feasible.
\item \textbf{Filtration information}: By choosing the rank $2$ maximum mean discrepancy associated to the signature kernel, we can incorporate filtration information into the inference process. Such information is crucial when wishing to ascertain more than just the finite-dimensional distributions between sets of paths.
\end{enumerate}	 

Finally, we also show that our approach is suitable to tackle both the \emph{market regime classification problem} (MRCP) and the \emph{market regime detection problem} (MRDP). The MRDP can be generalised to a one-class (contemporary) classification procedure, and thus can be framed as a method for anomaly detection. We compare our results to current pathwise techniques in the literature and show that our approach is powerful even in low data environments. 

The paper is structured as follows. In the rest of this section, we give an outline of the MRCP and MRDP. In Section \ref{sec:prelims}, we give a brief overview of path signatures (including those of higher rank), streamed data, and the signature maximum mean discrepancy. In Section \ref{sec:experiments}, we outline the objects and methodologies used in our experiments. Section \ref{sec:mrdp} and Section \ref{sec:mrcp} covers the MRDP and MRCP respectively, where paths are generated synthetically. Finally, Section \ref{sec:realdata} gives real data experiments, including a study of our approach on both a basket of equities and on cryptocurrency data.

%% file: section1/13pathwiseregimes.tex
\subsection{Identifying and distinguishing regimes in a path-wise setting}\label{subsec:notation}
Let us start by recalling the terminology of  \emph{regimes} and \emph{change-points} of regimes in the classical (and in particular return-based) sense. Then, we proceed to recast these concepts into a path-wise setting in such a way, that reflects the essence of the return-based formulation. We start by recalling basic definitions and notations that will be central to the paper, presented in a way that they can easily be recast in the next step as path-wise operations.

\begin{definition}[Time-series data streams and sets of data streams]\label{def:set ofstreams}Let $E$ be a vector space over $\mathbb{R}$. 
Classically, a \emph{time series data stream} is an infinite sequence of elements 
\begin{equation}\label{eqn:timeseries}
    \{x_{t_1}, \dots, x_{t_i}, \ldots\}  \qquad  x_{t_i} \in E \text{ for }\quad i \in \mathbb{N}.
\end{equation}
For a finite integer $N \in \mathbb{N}$, we call a sequence
\begin{equation*}
    \mathsf{x_N} = (x_{t_1}, \dots, x_{t_N}), \quad x_{t_i} \in E \text{ for }i=1,\dots, N
\end{equation*}
an element of the \emph{set of data streams} $\mathcal{S}_N(E)$ over $E$ (cf. Def. \ref{def:streamofdata}).  If there is no danger of confusion, for ease of notation we will simply use $\mathcal{S}(E)$ and index the observations as
\begin{equation}\label{eqn:path}
    \mathsf{x} = (x_{1}, \dots, x_{N}), \quad x_{i} \in E \text{ for } i=1,\dots, N.
\end{equation}
\end{definition}
Throughout the paper, objects of the type $\mathsf{x}$ will be the main sources of data that we perform analysis on. Essentially, $\mathsf{x}$ can be thought of as a segment of a (financial) time series observed over the period $[0,T]$ with $t_0=0$ and $t_N=T$. This is typically the price evolution of an asset or index in our examples, formally it can be seen as one realisation of a random path observed at discrete time-points $t_i$,  $i=1,\dots, N$ that originates from some underlying distribution, which \emph{may or may not have changed} throughout the observation period $[0,T]$ with $t_0=0$ and $t_N=T$.

In order to formalise what we mean by a \emph{regime} or a \emph{change of regime} (that may or may not be present in such an object $\mathsf{x}$), let us first recall the following classical definitions. Recall that a \emph{stationary} time series on the horizon $[0,T]$ is a finite variance process whose statistical properties are all constant over time. This definition assumes in particular that the mean value function $\mu_t:=\mathbb{E}[x_t]$ is constant on $t\in[0,T]$ and that the auto-covariance function $\mathrm{cov}(x_s, x_t) = \mathbb{E}[(x_s - \mu_s)(x_t - \mu_t)]$ depends on the timestamps $s,t$ only through their time difference, or $|t-s|$, but not through their position within $[0,T]$.
Given an element $\mathsf{x}$ of the set of data streams of fixed length $N\in \mathbb{N}$, a stream of all possible sub-sequences of length $n$ can be built by moving a sliding window of size $n$ across the data set. A change point then represents a transition in the process that generates the time series data. For more details on the classical setting, see \cite{aminikhanghahi2017survey}.

\begin{definition}[Change point detection problems: classical setting]\label{def:regimechangedetecionproblem}
     The corresponding \emph{change point detection problem} is the hypothesis test given by 
    \begin{align*}
        H_0&: \mathbb{P}_{\mathsf{X}_1} = \mathbb{P}_{\mathsf{X}_2} = \dots = \mathbb{P}_{\mathsf{X}_M}, \text{ versus} \\
        H_1&: \text{there exists }1 \le k^* \le M \text{ st. } \mathbb{P}_{\mathsf{X}_1} = \mathbb{P}_{\mathsf{X}_2} \dots = \mathbb{P}_{\mathsf{X}_{k^*}} \ne \mathbb{P}_{\mathsf{X}_{k^*+1}} = \dots = \mathbb{P}_{\mathsf{X}_M}.
    \end{align*}
\end{definition}

At this point, we begin to start deviating somewhat from the classical nomenclature and from the classical setting to serve our purposes regarding applications better. In the specific settings we would like to use regime detection for, small fluctuations of the underlying dynamics should be treated differently than significant and lasting changes in the underlying dynamics.
As outlined in \cite{horvath2021clustering}, we are often interested in a representative (in \cite{horvath2021clustering} represented as a "barycenter") of distributions that are similar enough to one another from the perspective of our application.
The applications we primarily have in mind are quantisation,  (robust) deep hedging, and risk management considerations. In the context of deep hedging, we would like to make sure that small deviations in the data generating process used for training does not lead to large errors in the algorithm or the hedging performance, and to this end, we introduce a robustification to the hedging objective.
At the same time, if deviations from our reference model become significant and persist over time, then a retraining of the hedging engine is called for. With the tools presented here, we provide a tool-set to distinguish between these two cases.

\begin{definition}[Regime]\label{def:regime}
    Let $\mathsf{x} \in \mathcal{S}(\mathbb{R}^d)$ be a stream  over $\mathbb{R}^d$. Let $\mathsf{X}_i$ denote the subsequence associated to a \emph{windowing operation} of size $n \in \mathbb{N}$, $n \ll N$ across $\mathsf{x}$, i.e. $\mathsf{X}_i := (x_i, x_{i+1}, \ldots , x_{i+n})$ for some $i$ on the observation horizon $i=1,\dots, N$. 
     We denote by $\mathbb{P}_{\mathsf{X}_i}$ the underlying probability density of observations $\mathsf{X}_i$ associated to the window $i$
     .
     A \emph{regime} (starting at $i$) is then the longest consecutive run of indices $k$, for which it holds that\footnote{Reversely, an occurrence $k^*$ where equality does not hold is classically referred to as a \emph{change point}. We chose to deviate from this nomenclature in this paper, in order to be able to distinguish between different types of changes: significant and lasting changes, changes of  small magnitude, and changes of significant magnitude but very short duration.}
     $\mathbb{P}_{\mathsf{X}_i} = \mathbb{P}_{\mathsf{X}_{i+1}} = \dots = \mathbb{P}_{\mathsf{X}_{i+k}}$.  Significant and lasting changes will be referred to as instances of \emph{regime change} $\mathbb{P}_{\mathsf{X}_i}= \dots =\mathbb{P}_{\mathsf{X}_{i+k^*}} \ll \footnote{Or similarly for $\gg$.} \  \mathbb{P}_{\mathsf{X}_{i+k^*+1}}= \dots =\mathbb{P}_{\mathsf{X}_{i+k^*+l}}$, with $k,k^*,l\gg 0$. Changes of (very) large magnitude but very short duration will be referred to as \emph{spikes or outliers}, and we will permit ourselves the discretion for changes of sufficiently small magnitude (estimation errors, or changes within a predifined tolerance\footnote{In the context of \emph{robust deep hedging} this tolerance would be a region where the hedging engine does not yet need retraining.}) to be disregarded.\end{definition}

For simplicity, we will focus on \emph{non-overlapping windowing operations} in the sequel. Let $M \in \mathbb{N}$ denote the maximum number of non-overlapping windows of size $n$ that can be extracted from the $N-$stream $\mathsf{x}$ in \eqref{eqn:path}. Then we can associate to $\mathsf{x}$ the $M$ sequence of $n$-windows, which we denote by $\mathsf{X} = (\mathsf{X}_1, \dots, \mathsf{X}_M)$.

Viewing regimes in terms of $\mathsf{X}$ means that they are represented as $\mathbb{P}_{\mathsf{X}_i} \in \mathcal{P}(E)$, where $\mathcal{P}(E)$ denotes the space of Borel probability measures over $E$. This was the approach taken in \cite{horvath2021clustering} for the case $E = \mathbb{R}$. As mentioned in Section \ref{sec:intro} this approach necessarily discards temporal effects due to aggregation; moreover, it does not scale easily to higher dimensions (and can become computationally intractable, see \cite{kolouri2019generalized}). It turns out that we can remedy both of these issues by considering elements of $\mathsf{X}$ as \emph{path segments} instead of elements of $\mathcal{S}_M(\mathbb{R}^d)$. The pathwise formulation will allow us to employ an array of recently developed powerful tools from rough path theory and (signature-) kernel methods. From now on, instead of referring to a regime as a collection of $E$-valued random variables as in Definition \ref{def:regime} above, we will observe regimes as a series of consecutive path segments of an observed realisation of a process. In particular, for a given time interval $I = [a, b]$, $a < b$, we assume that $\mathsf{x}$ is observed at points on a partition $\Delta$ of $I$, whereby
\begin{equation*}
    \Delta = \{a < t_1 < \dots < t_N = b\}, \qquad a, b \in [0, +\infty).
\end{equation*}
We can embed $\mathsf{x}$ into path space via an interpolation operation, which we denote by $\pi_\Delta$. Write $X = \pi_\Delta(\mathsf{x})$, so $X \in C_1(I, E)$, the space of continuous paths of bounded variation over $I$ into $E$. Similar in spirit to the windowing operation $\mathsf{X}$ from Definition \ref{def:regime}, it will be useful to define the set of data available to be extracted from $X$ over $I$.

\begin{definition}[Set of sub-paths]\label{def:setofsubpaths}
    Let $E$ be a vector space and $I = [a, b]$. Suppose $X \in C_1(I, E)$ is an embedding of a discretely observed path $\mathsf{x} \in \mathcal{S}(E)$  via $\pi_\Delta$. Then, the set of \emph{sub-paths} $\Pi_\Delta(\hat{\mathsf{x}})$ associated to $\mathsf{x}$ is given by
    \begin{equation}\label{eqn:subpaths}
        \Pi_\Delta(\mathsf{x}):= \left\{X_{|[u, v]} \in C_1([u, v], E) : [u, v] \subseteq [a, b], X = \pi_\Delta(\mathsf{x})\right\}.
    \end{equation} 
\end{definition} 
\begin{remark}
    Suppose that $J = [a, b']$ where $b' \ge b$. Associate to $J$ a partition $\Delta'$, a discretely-observed path $\mathsf{y} \in \mathcal{S}(E)$, and the embedded path $Y \in C_1(J, E)$. Suppose that $\Delta \subset \Delta'$, that is, all partition points in $\Delta$ and also in $\Delta'$. Further suppose that $Y_{t_i} = X_{t_i}$ for all $t_i \in \Delta$. Then, we have that $\Pi_\Delta(\mathsf{x}) \subseteq \Pi_{\Delta'}(\mathsf{y})$. 
\end{remark}
\begin{figure}[ht]
    \centering
    \includegraphics[scale=0.5]{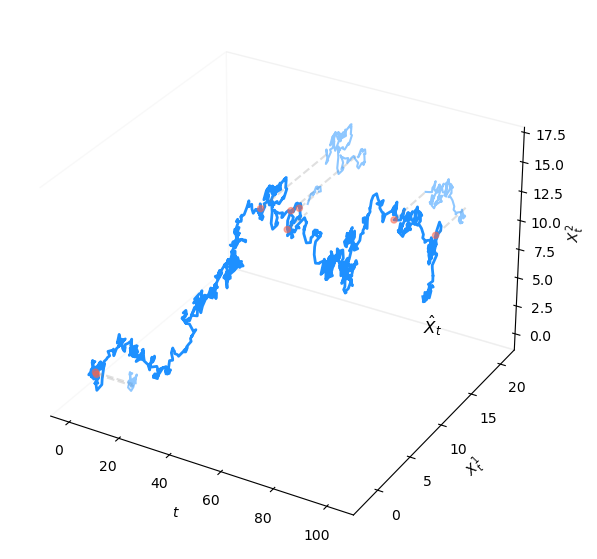}
    \caption{Illustration of extraction of sub-paths from a path evolving in $\mathbb{R}^2$. Given an interval $[u, v] \subset [0, T]$, one can always extract the path corresponding to $X$ restricted to $[u, v]$.}
\end{figure}

Collections of sub-paths found in $\Pi_\Delta(\mathsf{x})$ over the interval partitioned by $\Delta$ can be thought of as samples drawn from a regime $\mathbb{P}$, which prescribes dynamics over $I$. Moreover, any collection $\mathcal{X} \subset \Pi_\Delta(\mathsf{x})$ is a compact subset of $C_1(I, E)$. Thus, for a given $\mathcal{X}$ of cardinality $N \in \mathbb{N}$, one can define the empirical measure
\begin{equation}\label{eqn:empiricalmeasuresubpaths}
    \delta_\mathcal{X} = \frac{1}{N}\sum_{j=1}^{N}\delta_{X^j}, \quad X^j \in \mathcal{X},
\end{equation}
so in particular $\delta_\mathcal{X} \in \mathcal{P}(\mathcal{X})$. In this way, we can re-frame the regime detection problem as one over collections of path objects, which can be viewed as compactly supported measures on path space.

%% file: section1/14market_regime_classification_and_detection.tex
\subsection{Market regime detection and clustering  problems}\label{subsec:mrcpmrdp}

Let $I = [0, T]$ and $\Delta= \{0 = t_0 < t_1 <\dots< t_N = T\}$ be a partition of $I$, representing observation instances of a data stream $\mathsf{x} \in \mathcal{S}(E)$ Let $X \in C_1(I, E)$ be the path obtained via interpolating $\mathsf{x}$ through $\pi_\Delta$. Finally, suppose $(\tau_i)_{i=1}^M$ is a sequence of times with $\tau_0 = 0$ and $\tau_M = T$ with $M \ll N$. Write $\mathrm{T}_i = (\tau_{i-1}, \tau_i]$.

One can consider a sequence of collections $\mathcal{X}^i \subset \Pi_{\Delta_{|\mathrm{T}_i}}(\mathsf{x})$, where each $\mathcal{X}^i$ is comprised of (sub)-path samples drawn from the restriction of $X$ to $\mathrm{T}_i$. Write $\mathcal{X} = \cup_{i\ge 1}\mathcal{X}^i$. Associate to each $\mathcal{X}^i$ the empirical measure $\delta_{\mathcal{X}^i}$ for $i=1,\dots, N$, which we call $\delta^i$ for short. The sequence of measures $(\delta^i)_{i=1}^M$ can be thought of as noisy estimates of regime dynamics over $\mathrm{T}_i$. Hence, we are interested in the following questions:
\begin{enumerate}
    \item \textbf{Market regime clustering and classification:} How many distinct groups (clusters) are present within the set $\{\delta^1, \dots, \delta^M\}$? And given new data, which of these clusters can it be assigned to?
    \item \textbf{Market regime detection:} Suppose we observe $\delta^1, \dots, \delta^M$ sequentially, in an online setting. Then, given $\delta^j$, have we observed a regime change (in the sense of Definition \ref{def:regime})?
\end{enumerate}

We start with a discussion of the latter, which as mentioned in Section \ref{sec:intro}, if we are only interested in detecting two kinds of states: ``normal'' and ``different''. 

The set-up is as follows. Suppose we observe $\mathsf{x} \in \mathcal{S}(E)$ and construct a sequence of sets of paths $\mathcal{X}^1, \dots, \mathcal{X}^M$ with associated empirical distributions $\delta^1, \dots, \delta^M$. As we observe each measure, we want to know (in an on-line setting) whether each observation is anomalous or not. This requires us to define what we consider normal, and what we consider to be anomalous. We codify this with the following. 

\begin{definition}[Beliefs]
    Let $I=[0, T], T > 0$ and $E \subset \mathbb{R}^d$. We call a compact set of paths $\mathfrak{P} \subset C(I, E)$ an agent's \emph{beliefs}. 
    
    Furthermore, suppose $(U, \Sigma, \mathbb{U})$ is a probability space and $G_\theta: U \to C(I, E)$ is a generator with $\theta \in \Theta$, and $\Theta \subset \mathbb{R}^d$. Write $\mathbb{P}_\theta  = {G_\theta}_{\#}\mathbb{U}$ for the push-forward measure of $\mathbb{U}$ through $G_\theta$. Then, if $\mathfrak{P}$ is comprised of samples $p \sim \mathbb{P}_\theta$, we say the beliefs $\mathfrak{P}$ are \emph{parametric}. Else, they are \emph{non-parametric}. 
\end{definition}

A natural way to derive beliefs is by setting the generator $G_\theta$ to be the action of an SDE solver over a given stochastic differential equation parameterized by $\theta \in \Theta$, which are often calibrated according to observed market data. A less prescriptive, but more computationally expensive, procedure involves modelling beliefs as the output of a generative machine learning model \cite{buehler2020data, ni2021sigwasserstein, wiese2019simulation, issa2023non}. Here, the generator has been trained on observed market data, synthetic data, or a combination of the two. Note that beliefs obtained from this method could also be prescribed as ``parametric'' since the generator architecture is often comprised of highly-parameterized neural networks. Thus these approaches are often called ``data-driven'' to distinguish them from the classical approach. Non-parametric beliefs can be obtained by sampling from observed market data. 

We now turn to how beliefs are used to detect changes in regime. As mentioned at the start of Section \ref{subsec:mrcpmrdp}, this problem can be thought of on the space of probability measures on path space. What is of concern, then, is deciding how to determine if (in an online setting) whether each $\delta^1, \dots, \delta^M$ is different to any of our prescribed beliefs. We formalise this statement with the following.

\begin{problem}[Market regime detection problem]\label{prob:regimedetectionproblem}
    Let $\alpha \in [0, 1)$ be an acceptance threshold. Let $I = [0, T]$ and suppose $\mathsf{x} \in \mathcal{S}(E)$ is a stream of data observed over a partition $\Delta$ of $I$. 
    
    Let $\Pi_{\Delta}(\mathsf{x})$ be the set of all sub-paths extracted from $\mathsf{x} \in \mathcal{S}(E)$. Define $\mathrm{T}_i = (\tau_{i-1}, \tau_i]$ where $(\tau_i)_{i=0}^M$ is a sequence of times with $\tau_0 = 0, \tau_M = T$. Let $\mathcal{X}^i\subset \Pi_{\Delta_{|\mathrm{T}_i}}(\mathsf{x})$, and write $\mathcal{X} = \cup_{i\ge 1} \mathcal{X}^i$. Denote by $\delta^i$ the empirical measure associated to $\mathcal{X}^i$. Let $\mathcal{K} \subset C_1(I, E)$ be compact.

    Then, given beliefs $\mathfrak{P} = (\mathfrak{P}_1,\dots,\mathfrak{P}_k)$, the \emph{market regime detection problem (MRDP)} is to find a statistic  
    
    \begin{equation*}
        \Gamma: \mathcal{P}(\mathcal{K}) \times \mathcal{P}(\mathcal{K}) \to \mathbb{R}
    \end{equation*}

    that minimises the average Type II error $\beta_\Gamma: \mathcal{P}(\mathcal{K})  \times \mathcal{P}(\mathcal{K}) \to [0, 1]$ along the sequence $\delta^1, \dots, \delta^M$ for each of our beliefs $\mathfrak{P}_j \in \mathfrak{P}$. That is, we wish to find

    \begin{equation*}
        \Gamma^* = \argmin_{\Gamma} \frac{1}{M}\sum_{i=1}^M\mathbb{E}_{\mu \sim \mathfrak{P}_j}[\beta_\Gamma(\mu, \delta^i)] \quad \text{for }j=1,\dots,k.
    \end{equation*}
    
\end{problem}

We conclude this section with a formulation of the market regime clustering problem, 
which (as an unsupervised learning problem) does not require a consideration of one's beliefs.

\begin{problem}[Market regime clustering problem]\label{prob:regimeclassificaitonproblem}
    With the notation the same as Problem \ref{prob:regimedetectionproblem}, the \emph{market regime clustering problem} is to find a function $F: \mathcal{P}(\mathcal{K}) \to \{1, \dots, k\}$ which suitably clusters the measures $\delta^1, \dots, \delta^M$ into $\mathcal{C}_1,\dots, \mathcal{C}_k$.
\end{problem}

Problem \ref{prob:regimeclassificaitonproblem} does contain a rather large ambiguity - the definition suitability in regards to the choice of the function $F$. This is a classic problem in unsupervised learning, as without data labels, goodness-of-fit evaluations are often made either qualitatively or with metrics not used in the learning procedure. Oftentimes the output of a given clustering algorithm can be thought of as assigning a \emph{class label} to the elements of a given cluster. This approach makes sense if the clusters themselves have sufficient interpretability. Yet, cluster assignments are relative by construction, and more work is required to validate whether the clusters themselves are distinct from each other whilst remaining \emph{self-similar}. In Section \ref{sec:mrcp}, we evaluate our clustering methodology on labelled data so as to verify our class assignments before moving onto unlabelled data.

%% file: section2/21path_signatures.tex
\subsection{The signature of a path}\label{subsec:pathsignatures}

In this section we introduce the path signature and state its main properties. We skip much of the rigorous detail, which can be found in Appendix \ref{appendix:signatures}.

In what follows we take $(\Omega, \mathcal{A}, \mathbb{P})$ to be the underlying probability space. For $0\le a < b$, set $I = [a, b]$ and suppose that $X: I \to E$ is a path evolving in a Banach space $E$ (usually, we take $E = \mathbb{R}^d$). In what follows let
\begin{equation*}
    T((E)) := \prod_{k\ge 0} E^{\otimes k}
\end{equation*}
denote the algebra of formal tensor series over $E$. Later, it will be relevant to consider objects in the space 
\begin{equation*}
    T^N(E) := \{(a_i)^{N}_{i=0} : a_i \in E^{\otimes i} \},
\end{equation*}
which we call the truncated tensor algebra of order $N \in \mathbb{N}$. For convention set $E^{\otimes 0} = \mathbb{R}$.

Denote by $C_p(I, E)$ the space of continuous paths of finite $p$-variation (see Definition \ref{def:pvariationappendix}). We now introduce the \emph{signature} map $S: C_p(I, E) \to T((E))$, which one can think of as the canonical feature map over path space.

\begin{definition}[Signature, \cite{chen1957integration, lyons1998differential}, Definition 2.1]\label{def:signature}
    Let $1 \le p <  2$ and suppose $I=[a, b]$. Then, for $X \in C_p(I, E)$, the \emph{signature} $S(X)_{[a, b]} \in T\left((E)\right)$ of $X$ is given by
    \begin{equation}\label{eqn:signature}
        S(X)_{a, b} := (1, \mathbb{X}^1_{[a,b]}, \dots, \mathbb{X}^N_{[a,b]}, \dots),
    \end{equation}
    where
    \begin{equation}\label{eqn:signatureintegral}
        \mathbb{X}^k_{[a,b]} = \left(\ \ \idotsint\limits_{a<t_1<\dots< t_k < b} dX^{i_1}_{u_1} \otimes \dotsm \otimes dX^{i_k}_{u_k}\right)_{(i_1, \dots, i_k) \in \{1, \dots, d\}^k}.
    \end{equation}
    Thus, each $\mathbb{X}^k_{[a,b]} \in E^{\otimes k}$. Furthermore, the \emph{truncated signature} $S^N(X) \in T^N(E)$ is given by
    \begin{equation*}
        S^N(X)_{[a, b]} := (1, \mathbb{X}^1_{[a,b]}, \dots, \mathbb{X}^N_{[a,b]}, 0, 0, \dots).
    \end{equation*}
\end{definition}
\begin{remark}[Existence of signature integrals]\label{rmk:signatureintegrals}
    The condition on the regularity of the inputs paths at the beginning of Definition \ref{def:signature} is necessary as it determines the theory of integration required to define the iterated integrals that comprise the signature. For instance, if $p \in [1, 2)$ then the integrals in eq. (\ref{eqn:signatureintegral}) can be thought of in the sense of Young. Larger values of $p$ require the theory of rough paths. 
    
    If $X$ is a path of bounded variation, the integrals can be understood in the Riemann-Stieltjes sense. Given that the paths we consider are streamed datum (which can always be embedded in $C_p(I, E)$ via interpolation), we are always in this setting and thus the regularity requirements are satisfied.
\end{remark}

If the interval over which the signature of a path $X$ is being calculated is obvious from the context, we suppress the subscript $[a, b]$ and will write $S(X)$. We now briefly outline some of the key analytic properties of the signature mapping.

\begin{proposition}[Signature properties, \cite{hambly2010uniqueness, levin2013learning, lemercier2020distribution}]\label{prop:signatureproperties}
   Suppose $S(X)_{[a, b]}$ is the path signature of $X: [a, b] \to E$ where $X \in C_p([a, b], E)$ for $1 \le p < 2$. Then, we have the following: 
   \begin{itemize}
       \item \textbf{Uniqueness.} $S(X)_{[a,b]}$ is unique up to the equivalence relation $\sim_{\tau}$ which denotes \emph{tree-like equivalence} (equality up to retracings), 
       \item \textbf{Invariance to time reparameterisation.} For any non-decreasing surjection $\lambda: [c, d]\to [a, b]$, one has that $S(X)_{[a,b]} = S(X \circ \lambda)_{[c, d]}$, and 
       \item \textbf{Universality.} Suppose $\mathcal{K} \subset C_p(I, E)$ is compact and $f: \mathcal{K} \to \mathbb{R}$ is continuous. Then, for every $\varepsilon > 0$ there exists $N \in \mathbb{N}$ and a linear functional $L \in T((E))^*$ such that $\norm{f(X) - \langle L, S^N(X) \rangle_{T(E))}}_\infty < \varepsilon$ for all $X \in \mathcal{K}$.
   \end{itemize}
\end{proposition}

The first property is the key result which motivates our use of the signature mapping as a way of encoding path data. Again, paths we consider will be linear interpolants of streamed data with an added monotonically increasing channel (corresponding to time), and thus the uniqueness property of $S(X)_{[a,b]}$ will always hold, as we are working exclusively with tree-like reduced paths. The second property is crucial as it removes dependence on sampling rate, a key consideration when working with financial data which is often available at different tick frequencies. Related to universality of the signature is injectivity of following mapping, which we introduce now.

\begin{definition}[Expected signature, \cite{lemercier2020distribution}, Definition 3.1]\label{def:expectedsignature}
    Let $1 \le p < 2$ and suppose $\mathcal{K} \subset C_p([a, b], E)$ is a compact set containing tree-like reduced paths. Suppose $\mathbb{P} \in \mathcal{P}(\mathcal{K})$. Then, the mapping
    
    \begin{equation}\label{eqn:expectedsignature}
        \mathbb{E}S: \mathcal{P}(\mathcal{K}) \to T\left((E)\right), \quad \mathbb{E}S(\mathbb{P}) = \left(\mathbb{E}_\mathbb{P}[\mathbb{X}^k_{[a, b]}]\right)_{k\ge 0} 
    \end{equation}
    is called the \emph{expected signature}.
\end{definition}

The expected signature mapping is the analogue (on path space) to the moment map $x \mapsto (\mathbb{E}[x^{\otimes m}/m!])_{m\ge 0}$ for random variables in $\mathbb{R}^d$, as the following Theorems will illustrate.

\begin{theorem}[\cite{lemercier2020distribution}, Theorem 3.1]\label{thm:expectedsignature}
    Let $1 \le p < 2$. Suppose that $\mathcal{K} \subset C_p(I, E)$ be a compact set of tree-like reduced paths. Then, for any $\mathbb{P} \in \mathcal{P}(\mathcal{K})$, the mapping $\mathbb{P} \mapsto \mathbb{E}S(\mathbb{P})$ is injective.
\end{theorem}

\begin{theorem}[\cite{chevyrev2016characteristic}, Proposition 6.1]\label{thm:expectedsignaturelaws}
    Denote by $\mathbb{P}_Z = \mathbb{P} \circ Z^{-1}$ as the law of a random variable $Z$. Suppose $X, Y \in C_p(I, \mathcal{K})$ are path-valued random variables taking values in a compact set $\mathcal{K}\subset \mathbb{R}^d$. Then, $\mathbb{P}_X = \mathbb{P}_Y$ if and only if $\mathbb{E}S(\mathbb{P}_X) = \mathbb{E}S(\mathbb{P}_Y)$. 
\end{theorem}

This result is akin to characteristicness of $S$ on compact sets $\mathcal{K} \subset C_p(I, E)$.

%% file: section2/22streaming_data.tex
\subsection{Streamed data}\label{subsec:streamed_data}

As mentioned in Subsection \ref{subsec:notation}, the main objects of focus in this work are collections of discretely observed time series data. We provide a formal definition here.

\begin{definition}[Set of data streams, \cite{cochrane2020anomaly}, Definition 2.1]\label{def:streamofdata}
    Let $E$ be a non-empty set. The set of \emph{streams of data} $\mathcal{S}$ over $E$ is given by 
    \begin{equation}\label{eqn:setofstreams}
        \mathcal{S}(E) = \{\mathsf{x} = (x_1, \dots, x_n) : x_i \in E, n \in \mathbb{N} \}.
    \end{equation}
\end{definition}

In practice, elements $\mathsf{x} \in \mathcal{S}(E)$ come with an associated sequence of timestamps $t = (t_i)_{i=1}^n$ at which each $x_i$ is observed. This leads to the following definition. 
\begin{definition}[Time-augmented stream]
    Let $\mathsf{x} \in \mathcal{S}(E)$ be a stream of data over $E$. Further, let 
    \begin{equation*}
        \Delta = \{0 = t_1 < \dots < t_n = T\}
    \end{equation*}
    be a partition of $I=[0, T]$ corresponding to the observation times associated to each element of $\mathsf{x}$. The set of \emph{time-augmented streams} over the interval $I$ with respect to $\Delta$ is given by
    \begin{equation}\label{eqn:timeaugmentedstream}
        \mathcal{T}_\Delta (I, E) = \{\hat{\mathsf{x}} = \{(t_i, x_i)\}_{i=1}^n: x_i \in E, n\in\mathbb{N}\}.
    \end{equation}
\end{definition}
As mentioned in Section \ref{subsec:pathsignatures}, we can directly embed elements of $\mathcal{T}_\Delta (I, E)$ into a continuous path $X \in C(I, E)$ of bounded variation via linear interpolation, which we will denote the operation of as $\pi_\Delta: \mathcal{S}(E) \to C(I, E)$. If the grid over which the interpolation is taken is obvious from the context, we suppress the subscript $\Delta$. Refer to Appendix \ref{appendix:streamingdata} for more details. 

Often it is advantageous to first transform a stream of data $\hat{\mathsf{x}} \in \mathcal{T}_\Delta(I, \mathbb{R}^d)$ before calculating its corresponding signature. For example, in financial applications, one is often more concerned with \emph{relative} changes as opposed to absolute ones, implying a type of path normalisation is necessary when conducting inference with path signatures. In the literature, path transformations have been used to correctly evaluate discrete-time stochastic integrals in the It\^o sense \cite{Flint_2016}, directly scale path values \cite{cass2021general} or provide sensitivity to translation after applying the signature \cite{morrill2020generalised}. Briefly, the most common choice of transformations used in our experiments are as follows:

\begin{itemize}
    \item \textbf{Time normalisation transform}: If the raw value of the time channel is not of importance (for example, just placeholders representing the order of observations), then the transformation $$\phi_{\text{time}}: \mathcal{T}_\Delta([a, b], E) \to \mathcal{T}_\Delta([0, 1], E)$$ can be used. 
    \item \textbf{State space normalisation transform}: If the values a path takes are only important relative to where the path began, one can apply the transformation $$\phi_{\text{norm}}: \mathcal{T}_\Delta([a, b], E) \to \mathcal{T}_\Delta([a, b], E)$$ whereby $\phi_{\text{norm}}(\hat{\mathsf{x}})_i = (t_i, x_i/x_0)$ for $i=1,\dots,n$.
    \item \textbf{Absolute increment transformation}: If one wishes to capture path volatility as a first-order effect, the transformation $$\phi_{\text{incr}}: \mathcal{T}_\Delta([a, b], E) \to \mathcal{T}_\Delta([a, b], E)$$ can be used, whereby $\phi_{\text{norm}}(\hat{\mathsf{x}})_i = (t_i, \sum_{j=1}^{i}|x_j - x_{j-1}| + x_1)$ for $i=2,\dots,n$, and $\phi_{\text{norm}}(\hat{\mathsf{x}})_1 = (t_1, x_1)$.
    \item \textbf{Scaling transformation}. Often it is useful to directly scale the (non-time) components of a path. Given $\lambda \in E$, the function $$\phi^\lambda_{\text{scale}}: \mathcal{T}_\Delta([a, b], E) \to \mathcal{T}_\Delta([a, b], E),$$ given by $\phi^\lambda_{\text{scale}}(\hat{\mathsf{x}})_i = (t_i, \lambda_i \odot x_i)$ for $i=1,\dots,n$ scales each of the non-time components by the components of the vector $\lambda$. Here, $\odot: \mathbb{R}^d \to \mathbb{R}^d$ denotes the Hadamard product.
\end{itemize}

We leave the more detailed definitions of the above (other useful transformations) in Appendix \ref{appendix:streamingdata}. It is straightforward to consider applying multiple (ordered) stream transformations to some $\hat{\mathsf{x}} \in \mathcal{T}_\Delta(I, \mathbb{R}^d)$.

\begin{definition}[Stream transformer]\label{def:streamtransformer}
    Suppose $\hat{\mathsf{x}} \in \mathcal{T}_\Delta(I, \mathbb{R}^d)$ and let $(\phi_i)_{i=1}^N$ be a sequence of stream transformations. For $p \in \mathbb{N}$, define the \emph{stream transformer}
    \begin{equation}\label{eqn:streamtransformer}
        \Phi: \mathcal{S}(\mathbb{R}^d) \to \mathcal{S}(\mathbb{R}^p), \quad \Phi(\hat{\mathsf{x}}) = \left(\phi_1 \circ \dots \circ \phi_N\right)(\hat{\mathsf{x}}).
    \end{equation}
\end{definition}

\begin{remark}
    We note that $\mathcal{T}_\Delta([0, T], \mathbb{R}^d) \cong \mathcal{S}(\mathbb{R}^{d+1})$, so in particular all compositions of stream transformations are well-defined. This does not imply that the order in which the transformations are taken is not important.
\end{remark}

%% file: section2/23rkhs.tex
\subsection{Inference in reproducing kernel Hilbert spaces}\label{subsec:mmd}

In this section we introduce the basics of inference with kernels. We define the signature kernel (including its higher-rank extension) and their associated maximum mean discrepancies (MMD). Finally, we introduce the signature kernel scoring rule, which can be used to perform inference in a set-to-point experimental framework. Again, we leave the majority of the details to Appendix \ref{appendix:rkhsmaterial}. For further details on kernel methods we direct the reader to \cite{gretton2006kernel, gretton2009fast, gretton2012kernel}.

\subsubsection{Reproducing kernel Hilbert spaces and the maximum mean discrepancy}

We begin this section with a brief introduction to kernels and reproducing kernel Hilbert spaces. In this section, $\mathcal{X}$ denotes a non-empty set and $\mathcal{H} = \{f: \mathcal{X} \to \mathbb{R}\}$ a Hilbert space of $\mathbb{R}$-functions on $\mathcal{X}$. In general $\mathcal{X}$ is not required to have any structure (it need not even be a topological space). The following quantity does give us a notion of distance between Borel measures on $\mathcal{X}$.

\begin{definition}[Maximum mean discrepancy, \cite{gretton2012kernel}, Definition 2]\label{def:maximummeandiscrepancy}
    Suppose $\mathcal{X}$ is a topological space. Let $\mathcal{F}$ be a class of bounded, measurable functions $f: \mathcal{X} \to \mathbb{R}$ and let $\mathbb{P}, \mathbb{Q} \in \mathcal{P}(\mathcal{X})$. Then, the \emph{maximum mean discrepancy} (MMD) between $\mathbb{P}$ and $\mathbb{Q}$ is given by
    \begin{equation}\label{eqn:mmd_general}
        \mathcal{D}^\mathcal{F}(\mathbb{P}, \mathbb{Q}) := \sup_{f \in \mathcal{F}}\Bigg(\ex_{x \sim \mathbb{P}}[f(x)] - \ex_{y \sim \mathbb{Q}}[f(y)] \Bigg).
    \end{equation}
    If $x = (x_1, \dots, x_n)$ and $y = (y_1, \dots, y_m)$ are i.i.d samples, $x_i \sim \mathbb{P}$ and $y_j \sim \mathbb{Q}$, then a \emph{biased empirical estimate} of the MMD is given by
    \begin{equation}\label{eqn:biased-mmd-general}
        \mathcal{D}_b^\mathcal{F}(x, y) := \sup_{f \in \mathcal{F}}\Bigg[\frac{1}{n}\sum_{i=1}^n f(x_i) - \frac{1}{m}\sum_{j=1}^m f(y_j) \Bigg].
    \end{equation}
\end{definition}

Equation (\ref{eqn:mmd_general}) refers to a more general class of distances, the integral probability metrics, which naturally depend on the choice of $\mathcal{F}$. For example, when $\mathcal{X} = \mathbb{R}^d$, identifying $\mathcal{F}$ with the class of continuous, bounded functions on $\mathbb{R}^d$ such that $\text{Lip}(f) \le 1$ recovers the 1-Wasserstein distance. 

Since we are looking to extend this definition to measures on compact sets $\mathcal{K} \subset C_p(I, E)$, we require a different approach. In particular we would like to choose $\mathcal{F}$ to be rich enough to induce metric-like properties on $\mathcal{D}^\mathcal{F}$ without sacrificing tractability, in the sense that evaluating the supremum on the RHS of eq. (\ref{eqn:mmd_general}) is computationally feasible. We begin with a fundamental definition.
\begin{definition}[Reproducing kernel Hilbert space, \cite{aronszajn1950theory}, Section 1.1]\label{def:reproducingkernelhilbertspace}
    Suppose $\mathcal{X}$ is a non-empty set and let $(\mathcal{H}, \langle \cdot, \cdot \rangle_\mathcal{H})$ be a Hilbert space of functions $f: \mathcal{X} \to \mathbb{R}$. We call a positive definite function $\kappa: \mathcal{X} \times \mathcal{X} \to \real$ a \emph{reproducing kernel} of $\mathcal{H}$ if 
    \begin{enumerate}[label=(\roman*)]
        \item For all $x \in \mathcal{X}$, we have that $\kappa(\cdot, x) \in \mathcal{H}$, and
        \item For all $x \in \mathcal{X}$ and $f \in \mathcal{H}$, one has that 
        \begin{equation}
            f(x) = \langle f(\cdot), \kappa(\cdot, x) \rangle_\mathcal{H},
        \end{equation}
        referred to as the \emph{reproducing property}.
    \end{enumerate}
    We call the Hilbert space $\mathcal{H}$ associated to $\kappa$ a \emph{reproducing kernel Hilbert space} (RKHS). 
\end{definition}
We can associate to each RKHS $\mathcal{H}$ the \emph{canonical feature map} $\varphi(x) = \kappa(\cdot, x)$. By the reproducing property one has that
\begin{equation*}
    \kappa(x,y) = \langle \kappa(\cdot, x), \kappa(\cdot, y) \rangle_\mathcal{H} = \langle \varphi(x), \varphi(y) \rangle_\mathcal{H} \qquad \text{for all }x, y \in \mathcal{X}.
\end{equation*}

The canonical feature map $\varphi$ gives a representation of elements in $\mathcal{X}$ in the linear inner product space $\mathcal{H}$. A simple extension allows one to represent measures $\mathbb{P} \in \mathcal{P}(\mathcal{X})$ via the \emph{mean embedding} $m(\mathbb{P}) = \mathbb{E}_{\mathbb{P}}[\varphi(X)]$. 

Naturally this is only a useful representation if the function $m: \mathcal{P}(\mathcal{X}) \to \mathcal{H}$ is injective. If so, a candidate distance between measures on $\mathcal{P}(\mathcal{X})$ becomes apparent: the distance (induced by the respective inner product) between their mean embeddings on $\mathcal{H}$. This injectivity property is referred to as \emph{characteristicness}, which is a fundamental requirement in this setting for the following reason.

\begin{theorem}[\cite{fukumizu2007kernel, gretton2012kernel}, Theorem 5]\label{theorem:mmdmetric}
    Let $\mathcal{F}$ be the unit ball of a RKHS $(\mathcal{H}, \kappa)$ comprised of $\mathbb{R}$-functions on a compact space $\mathcal{X}$. Assume $\kappa$ is characteristic to the space of Borel probability measures on $\mathcal{X}$. Then, for all $\mathbb{P}, \mathbb{Q} \in \mathcal{P}(\mathcal{X})$, one has that $$\mathcal{D}^\mathcal{F}(\mathbb{P}, \mathbb{Q}) = 0 \qquad \text{if and only if} \qquad \mathbb{P} = \mathbb{Q}.$$
\end{theorem}

If we choose $\mathcal{F}$ as in Theorem \ref{theorem:mmdmetric}, we can write the squared population MMD between two measures $\mathbb{P}, \mathbb{Q} \in \mathcal{P}(\mathcal{X})$ as (dropping the superscript $\mathcal{F}$ and instead replacing with $\kappa$ for emphasis):

\begin{equation}\label{eqn:squaredmmd}
    \mathcal{D}^\kappa(\mathbb{P}, \mathbb{Q})^2 = \norm{m(\mathbb{P}) - m(\mathbb{Q})}^2_\mathcal{H} = \langle m(\mathbb{P}) - m(\mathbb{Q}), m(\mathbb{P}) - m(\mathbb{Q}) \rangle_{\mathcal{H}}.
\end{equation}

Thus, given i.i.d samples $X = (x_1, \dots, x_n)$ and $Y = (y_1, \dots, y_m)$ where $x_i \sim \mathbb{P}$ for $i=1,\dots,n$ and $y_j \sim \mathbb{Q}$ for $j =1,\dots,m$, a biased estimate of eq. (\ref{eqn:squaredmmd}) is given by

\begin{equation}\label{eqn:biasedsamplemmd}
    \mathcal{D}^\kappa_b(\mathbb{P}, \mathbb{Q})^2 = \frac{1}{n^2}\sum_{i, j = 1}^n \kappa(x_i, x_j) - \frac{2}{mn}\sum_{i=1}^n\sum_{j=1}^m\kappa(x_i, y_j) + \frac{1}{m^2}\sum_{i, j=1}^m\kappa(y_i, y_j).
\end{equation}

and an unbiased estimator is given by 

\begin{equation}\label{eqn:unbiasedsamplemmd}
   \mathcal{D}^\kappa_u(\mathbb{P}, \mathbb{Q})^2 = \frac{1}{n(n-1)}\sum_{i\ne j = 1}^n \kappa(x_i, x_j) - \frac{2}{mn}\sum_{i=1}^n\sum_{j=1}^m\kappa(x_i, y_j) + \frac{1}{m(m-1)}\sum_{i \ne j}^m\kappa(y_i, y_j).
\end{equation}

Asymptotic consistency of these estimators is shown in \cite{gretton2012kernel}, Lemma 6.

\subsubsection{The two sample test, and consistency}

In the literature, the MMD has been used in the context of change-point detection \cite{https://doi.org/10.48550/arxiv.1202.3878, sinn2012detecting}, minimum distance estimation  \cite{briol2019statistical, cherief2022finite}, and generative modelling \cite{li2017mmd, mroueh2021convergence}. Each of these implementations utilizes the power of the MMD as a generalized two-sample test between empirical measures. 
\begin{definition}[MMD two sample test]
     Suppose $X = (x_1, \dots, x_n), x_i \sim \mathbb{P}$ for $i=1,\dots, n$ and $Y = (y_1, \dots, y_m), y_j \sim \mathbb{Q}$ for $j=1,\dots, m$ are i.i.d. Then, the \emph{MMD two-sample test problem} is to design an appropriate threshold $c_\alpha$ such that if $\mathcal{D}^\kappa(X, Y) > c_\alpha$, one can reject the null hypothesis $H_0$ that $\mathbb{P} = \mathbb{Q}$ at a confidence level $\alpha$.
\end{definition}
\begin{remark}
    Regarding the i.i.d. assumption: when studying (real) path data, or in certain synthetic cases, one cannot reasonably assume that successive samples are independent of each other. We refer the reader to \cite{cherief2022finite}, Section 3 for results regarding the convergence in probability of the sample MMD to the population MMD in the case that sames are not strictly independent. In particular, Theorem 3.2. gives a (distribution-free) bound which uses a version of McDiarmind's \cite{mcdiarmid1989method} inequality to bound the rate of convergence in probability in the case where samples are dependent.
\end{remark}

A consistent two-sample test achieves a Type I error of $\alpha$ (rejecting $H_0$ when it is true) and an asymptotic Type II error (accepting $H_0$ when $\mathbb{P} \ne \mathbb{Q}$) of $0$. It was shown in \cite{gretton2012kernel}, Theorem 7 that empirical (sample) MMD converges to the population MMD at a rate $\mathcal{O}((m+n)^{-1/2})$, where $n,m \in \mathbb{N}$ are the number of samples comprising $X$ and $Y$ respectively. As a consequence, a test at the level $(1-\alpha)$ for the biased empirical statistic has associated distribution-free bound $\sqrt{2K/n}(1+2\sqrt{\log \alpha^{-1}})$ (when $n=m$). However, this bound is much too conservative for sample sizes used in practice, and data-dependent bounds are often required to achieve a powerful test (and a less conservative $c_\alpha$).

To this, there are several choices one can use, see \cite{gretton2006kernel}, Theorem 8, or Appendix \ref{appendix:rkhsmaterial}. The most useful for our purposes involves estimating the null distribution via a Gamma approximation, which is $\mathcal{O}(n^2)$ to calculate in the number of empirical samples $n$. However, in our setting we are often confined to using very few samples and thus a simple bootstrapping technique is often sufficient. We outline further in Section \ref{subsec:mrcpmethods}.

\subsubsection{The (higher-rank) signature MMD}

In this section we introduce a kernel defined on compact sets $\mathcal{K} \subset C(I, E)$. In particular we would like to leverage the signature map $S$ from Definition \ref{def:signature}. Given that $S$ can be thought of as the canonical feature map on path space, it makes sense to define a kernel $k(\cdot, x) = S(x)$ through $S$ in the following manner.
\begin{definition}[Signature kernel, \cite{chevyrev2018signature}, Theorem 7.7]\label{def:signaturekernel}
    Suppose $1 \le p < 2$. Recall $C_p(I, E)$ is the set of paths of bounded $p$-variation evolving in a Hilbert space $V$ over the interval $I = [a, b]$. We call the bounded, positive semi-definite function
    \begin{equation}\label{eqn:signaturekernel}
        k_{\text{sig}}: C_p(I, V) \times C_p(I, V) \to \mathbb{R}, \qquad k_{\text{sig}}(x, y) = \langle S(x)_{[a,b]}, S(y)_{[a, b]}\rangle_{T((V))}
    \end{equation}
    the \emph{signature kernel}. For $\mathcal{K} \subset C_p(I, E)$ compact, $k_{\text{sig}}$ is universal to $C_b(\mathcal{K}, \mathbb{R})$ equipped with the strict topology. It is characteristic to finite, signed Borel measures on $\mathcal{K}$.
\end{definition}

Since the signature kernel (\ref{eqn:signaturekernel}) is characteristic (and, equivalently, universal), the MMD with respect to the RKHS $(\mathcal{H}, k_{\text{sig}})$ will be a metric. The authors in \cite{chevyrev2018signature} included a tensor normalisation in order to extend the definition to non-compact spaces. This is not required in our setting and we omit the inclusion to simplify the definition.

In order to define an MMD through $(\mathcal{H}, k_\text{sig})$, we must be able to evaluate the signature kernel $k_{\text{sig}}(x,y)$ where $x, y \in C_p(I, V)$. As of writing, until recently there has not been an explicit kernel trick associated to $k_{\text{sig}}$. Instead, the \emph{truncated signature kernel} $k_{\text{sig}}^N(x,y) = \langle S^N(x), S^N(y) \rangle_{T^N(V)}$ was used as an approximation to the true value of $k_\text{sig}(x, y)$. This technique proved popular, motivated by the factorial decay property exhibited by terms in the signature, see eq. \eqref{eqn:factorialdecay} in the Appendix.

However, in truncation one necessarily loses information embedded in the path that may be relevant. The fact that signature terms exhibit factorial decay does not alleviate this. In our calculations, we employ the following kernel trick which gives an explicit evaluation of the signature kernel between two paths without requiring truncation. 
\begin{theorem}[\cite{salvi2021signature}, Theorem 2.5]\label{thm:kerneltrick}
    Suppose $V$ is a $d$-dimensional Banach space. Let $I = [0, T]$ and $J = [0, T]$ be compact intervals and let $X, Y \in C_1(I, V)$. The signature kernel $k_{\text{sig}}$ from Definition \ref{def:signaturekernel} is the solution of the following linear, second order hyperbolic PDE
    \begin{equation}\label{eqn:sigpde}
        k_{\text{sig}}(x,y) = f(T, T), \qquad \text{where } f(s, t) = 1 + \int_0^s \int_0^t f(u, v) \, \langle dx_u, dy_v \rangle_V. 
    \end{equation}
\end{theorem}
In this way, the kernel trick for the signature kernel is reduced to numerically solving the hyperbolic PDE (\ref{eqn:sigpde}). As $k_{\text{sig}}$ is bounded (\cite{salvi2021signature}, Lemma 4.5) the associated MMD two-sample test is consistent by Theorem \ref{thm:mmdconsistency}. 

Since $k_{\text{sig}}$ is approximated via a finite-difference scheme, there exists a discrepancy between the resulting solution $\tilde{k}_{\text{sig}}(x, y)$ and the true value $k_{\text{sig}}(x, y)$. The global error of approximation is given in \cite{salvi2021signature} as being $\mathcal{O}(2^{-2\lambda})$, where $\lambda$ is a chosen level of dyadic refinement over the time grid $I \times I$. The coarsest grid is given by $P_0 := I \times I$. As noted, for suitably normalised paths $X, Y$, even coarse partitions give well-approximated results.

Finally, as noted in \cite{chevyrev2018signature}, Section 5.5, one may wish to lift paths evolving in a general topological space $\mathcal{T}$ to those in a linear space $V$ via a feature map $\varphi: \mathcal{T} \to V$. Such transformations form the fundamental basis for kernel methods and are known to assist the process of learning non-linear relationships on sets of data. In our experiments we often make use of the following lifted kernel, corresponding to $\varphi$ being the canonical feature map associated to the radial basis function (RBF) kernel on $\mathbb{R}^d$.
\begin{definition}[Lifted signature kernel, \cite{lemercier2020distribution}, Theorem 3.3]\label{def:liftedsignaturekernelrank1}
    Let $X, Y \in C(I, E)$ and suppose $\sigma > 0$. Denote by $\varphi(x) = \kappa_\sigma(x, \cdot)$ as the canonical feature map associated to the RBF kernel with scaling parameter $\sigma$. 
    
    One has that the \emph{RBF-lifted signature kernel}
    \begin{equation}\label{eqn:liftedsignaturekernel}
        k_{\text{sig}}^\varphi(x,y) = \langle S\circ \varphi(x), S\circ \varphi(y) \rangle_{T((E))}
    \end{equation}
    is universal, that is, for compact sets $\mathcal{X} \subset C(I, E)$ the associated RKHS $(H, k^\varphi_{\text{sig}})$ is dense in the space of bounded continuous functions from $\mathcal{X}$ to $\mathbb{R}$. 
\end{definition}
\begin{remark}
    Because we only ever work directly with pairwise evaluations of paths, infinite-dimensional feature maps associated to static kernels $\kappa$ do not need to be calculated (due to the kernel trick).
\end{remark}

The fact that the MMD corresponding to $(\mathcal{H}, k_{\text{sig}}^\varphi)$ is a metric on $\mathcal{P}(\mathcal{K})$ is a direct consequence of Theorem 3.3 in \cite{lemercier2020distribution}. In what follows we will suppress the dependence on the lift and explicitly state which we are applying given the context, which will be either the identity mapping (corresponding to the classical signature kernel) or the RBF-lifted kernel.

In the final part of this section, we briefly introduce the concept of the \emph{higher-rank} signature and the associated \emph{higher-rank signature kernel}. The majority of the details can be found in the Appendix \ref{appendix:adaptedprocesses}, and further specifics can be found in \cite{salvi2021higher}. We summarize the main ideas here. From Theorem \ref{thm:expectedsignaturelaws}, given a filtered probability space $(\Omega, \mathcal{F}, (\mathcal{F}_t)_{t \in I}, \mathbb{P})$ and a $\mathcal{F}$-adapted stochastic process $X$, the expected signature of $X$ is enough to characterize its law $\mathcal{L}(X)$. It turns out that this is not the case for \emph{conditional} distributions of $X$ with respect to $\mathcal{F}_t$. For a filtered process $\boldsymbol{X} = (\Omega, \mathcal{F}, (\mathcal{F}_t)_{t\in I},\mathbb{P}, X)$, the object we are interested in studying is the \emph{prediction process} 

\begin{equation}\label{eqn:predictionprocess}
    \hat{X}_t = \mathbb{P}\left[X \in \cdot | \mathcal{F}_t \right],
\end{equation}

or the distribution of $X$ conditional on $\mathcal{F}_t$ at time $t \in I$. One can repeat this process of iterated conditioning to build a prediction process of the prediction process, and so on. The number of iterates determines the \emph{rank} of the prediction process. We call eq. \eqref{eqn:predictionprocess} the \emph{rank 1} adapted process with associated law $\mathcal{L}(\hat{X}^1)$. The classical law of $X$ can be identified with the rank 0 adapted process, $\mathcal{L}(X) = \mathcal{L}(\hat{X}^0)$. For two filtered processes $\boldsymbol{X}, \boldsymbol{Y}$, we write $\boldsymbol{X} \sim_r \boldsymbol{Y}$ if $\mathcal{L}(\hat{X}^r) = \mathcal{L}(\hat{Y}^r)$, and $\boldsymbol{X} \sim \boldsymbol{Y}$ if $\boldsymbol{X} \sim_r \boldsymbol{Y}$ for all $r \ge 0$. Again one can associate the classical (weak) equality in law between two stochastic processes $X, Y$ with $\boldsymbol{X} \sim_0 \boldsymbol{Y}$.

We know that the mean embedding of the signature mapping $S$ is able to distinguish between compactly supported laws of stochastic processes, i.e., we have that $$\mathbb{E}S(X) = \mathbb{E}S(Y) \iff \boldsymbol{X} \sim_0 \boldsymbol{Y}.$$ The classical signature map is referred to as the \emph{rank 1 signature}, and we write $S^1$. We can thus term the expected signature as the \emph{rank 1 kernel mean embedding (KME)}. If we wish to be able to conclude when we have that $\boldsymbol{X} \sim_1 \boldsymbol{Y}$ (i.e., equality in conditional law), we are necessarily required to use a higher-order KME, and thus a higher-order signature mapping (see Definition \ref{def:rank2signature}), which we call $S^2$. It turns out that the difference in the KME between two filtered processes $\boldsymbol{X}, \boldsymbol{Y}$ with respect to the rank 2 signature is enough to conclude whether $\boldsymbol{X} \sim_1 \boldsymbol{Y}$. Thus in much the same way as we were able to define the (rank 1) signature kernel $k_{\text{sig}}^1:= k_{\text{sig}}$, we can introduce the \emph{rank 2} equivalent 
\begin{equation*}
    k_{\text{sig}}^2(x, y) = \langle S^2(x), S^2(y) \rangle_{T^2((E))},
\end{equation*}
where $T^2((E)) = T((T((E))))$. Again, a kernel trick does exist for $k_{\text{sig}}^2$ in much the same way that it did for $k_{\text{sig}}^1$, see \cite{salvi2021higher}, Theorem 1. We formalise the statement of maximum mean discrepancy considered in this paper with the following.
\begin{definition}[Rank-$r$ signature MMD, \cite{salvi2021higher}]\label{def:rankrmmd}
    For $1 \le p < 2$, let $\mathcal{K} \subset C_p(I, V)$ be compact. Suppose that $\mathbb{P}, \mathbb{Q} \in \mathcal{P}(\mathcal{K})$ are empirical measures. Then, the \emph{rank $r$ maximum mean discrepancy} is given by 
    \begin{equation}\label{eqn:rankrmmdrbf}
        \mathcal{D}^r_{\text{sig}}(\mathbb{P}, \mathbb{Q}) = \norm{\mathbb{E}S^r(\mathbb{P}) - \mathbb{E}S^r(\mathbb{Q})}_{T^r((E))}.
    \end{equation}
\end{definition}

In the same way that we composed the rank 1 signature kernel with state-space kernels, we can compose each iterated rank with a static kernel $\kappa^r(x, y)$. In our case, this will again be the RBF kernel, and thus there are $r$ smoothing parameters $\sigma \in \mathbb{R}^r$ to be chosen when this approach is taken.

Finally, as alluded to in the preamble to Definition \ref{def:rankrmmd}, we have the following, crucial result.

\begin{theorem}[\cite{salvi2021higher}, Theorem 4]\label{thm:highrankmmd}
    The rank $r$ MMD metrizes the rank $r-1$ adapted topology. That is for filtered processes $\boldsymbol{X}, \boldsymbol{Y}\in \mathcal{FP}_I$,
    \begin{equation}
        \mathcal{D}^r_{\text{sig}}(X, Y) = 0 \iff \boldsymbol{X} \sim_{r-1}\boldsymbol{Y}.
    \end{equation}
\end{theorem}

In this way, the rank 2 signature kernel MMD $\mathcal{D}^2_\text{sig}$ can be used in a two-sample testing framework if we wish to incorporate filtration information. In fact, it follows that if $\mathcal{D}^2_\text{sig}(X, Y) = 0$, one must have that $\mathcal{D}^1_\text{sig}(X, Y) = 0$ (\cite{salvi2021higher}, Theorem 2), making the rank 2 signature kernel MMD a strictly better metric to use when considering compactly supported measures on path space. The main drawback (which will be shown in Subsection \ref{subsec:higherrankdetection}) is the computational time associated to the rank 2 MMD is some levels of magnitude greater than its rank 1 equivalent.

\subsubsection{Signature kernel scores}

In this section we introduce a quantity which allows us to perform set-to-point comparisons in path space, that is, comparisons from measures on compact sets of paths to single paths objects.

The class of \emph{scoring rules} have typically been employed in the context of evaluating a probabilistic forecast $\mathbb{P} \in \mathcal{P}(\mathcal{X})$ against a single evaluation $y \in\mathcal{X}$, see \cite{gneiting2007strictly} for an introduction and \cite{merkle2013choosing, gneiting2005weather, pacchiardi2021probabilistic} for applications in economics, weather forecasting, and supervised learning. 

\begin{definition}[Scoring rule, \cite{gneiting2007strictly}]
    Let $\mathcal{P}$ be a convex class of measures on a probability space $(\mathcal{X}, \mathcal{A})$. A \emph{scoring rule} $s: \mathcal{P} \times \mathcal{X} \to [-\infty, \infty]$ is any function such that $s(\mathbb{P}, \cdot)$ is $\mathcal{P}$\emph{-quasi integrable} for all $\mathbb{P} \in \mathcal{P}$. 
\end{definition}
The quasi-integrable condition is necessary to ensure the existence of the \emph{expected scoring rule}
\begin{equation*}
    s(\mathbb{P}, \mathbb{Q}) = \int_{\Omega} s(\mathbb{P}, \omega)\, d\mathbb{Q}(\omega) \qquad \text{for }\mathbb{P}, \mathbb{Q} \in \mathcal{P},
\end{equation*}
which represents the average penalty applied to the prediction $\mathbb{P} \in \mathcal{P}$ assuming the true underlying measure is $\mathbb{Q}$. An important property that a given scoring rule must possess is the following. 
\begin{definition}[Properness, \cite{gneiting2007strictly}]
    A scoring rule $s: \mathcal{P} \times \mathcal{X} \to [-\infty, +\infty]$ is called \emph{proper} (relative to the class $\mathcal{P}$) if $s(\mathbb{P}, \mathbb{P}) \le s(\mathbb{Q}, \mathbb{P})$ for all $\mathbb{P}, \mathbb{Q} \in \mathcal{P}$. It is called \emph{strictly proper} if $\mathbb{Q}=\mathbb{P}$ is the unique minimiser.
\end{definition}
If a scoring rule is strictly proper, one can define a statistical divergence $D(\mathbb{P}||\mathbb{Q}) = s(\mathbb{Q},\mathbb{P}) - s(\mathbb{P}, \mathbb{P})$. For a general strictly proper scoring rule, this divergence is not guaranteed to be a metric - it is not even necessarily symmetric. However, we will focus on a special class of scores where this is indeed the case. 
\begin{definition}[Kernel scores, \cite{gneiting2007strictly}, Section 5]
    Suppose $k: \mathcal{X} \times \mathcal{X} \to \mathbb{R}$ is a continuous, positive semi-definite kernel on $\mathcal{X}$. The associated \emph{kernel scoring rule} 
    \begin{equation*}
        s_k(\mathbb{P}, y) = \mathbb{E}_{X, X' \sim \mathbb{P}}[k(X, X')] - 2\mathbb{E}_{X \sim \mathbb{P}}[k(X, y)]
    \end{equation*}
    is proper relative to $\mathcal{P}$, the set of of Borel probability measures on $\mathcal{X}$ for which $\mathbb{E}_{\mathbb{P}}[k(X, X')]$ is finite for all $\mathbb{P} \in \mathcal{P}$.
\end{definition}
If we choose $k = k^r_{\text{sig}}$ to be the rank-$r$ signature kernel, where $\mathcal{X} \subset C_p([0, T]; \mathbb{R}^d)$ is compact, then we arrive at the following. 
\begin{definition}[\cite{issa2023non}, Proposition 3.3]\label{def:signaturekernelscore}
    Let $\mathcal{K}\subset C_p([0, T];V)$ be a compact set of paths. The \emph{signature kernel score} $s_{k^r_{\text{sig}}}: \mathcal{P}(\mathcal{K}) \times \mathcal{K}\to \mathbb{R}$ is strictly proper relative to $\mathcal{P}(\mathcal{K})$. 
\end{definition}

Much the same as the MMD, given observations $\{x_i\}_{i=1}^N$ where $x_i \sim \mathbb{P}$ and $y\sim \mathbb{Q}$, an unbiased estimator for $s_k(\mathbb{P}, y)$ is given by 

\begin{equation}\label{eqn:scoringunbiased}
    \hat{s}_k(\mathbb{P}, y) = \frac{1}{N(N-1)}\sum_{i\ne j} k(x_i, x_j) - \frac{2}{N}\sum_{i=1}^N k(x_i, y).
\end{equation}

For a proof of this result, see \cite{issa2023non}, Proposition 3.4.

%% file: section3/31path_data_partitioning.tex
\subsection{Partitioning of path data}\label{subsec:partitioning}

In this section we outline the data preprocessing procedure required to perform our experiments. In what follows we suppose $\hat{\mathsf{s}} \in \mathcal{T}_\Delta(I, \mathbb{R}^d)$ is a time-augmented $d$-dimensional stream of data over an interval $[0, T]$. For simplicity we assume this is representing (close) prices of $d$ financial instruments. We note here that our methodology is not restricted to just prices: in theory, one could incorporate any data deemed relevant to the regime detection or classification problem at hand.

The unbiased estimator of the maximum mean discrepancy (cf. eq (\ref{eqn:unbiasedsamplemmd}) evaluates sets of i.i.d samples $(x_1, \dots, x_n)$ and $(y_1, \dots, y_m)$, where  $x_i \sim \mathbb{P}$ and $y_m \sim \mathbb{Q}$. Recalling Problem \ref{prob:regimedetectionproblem}, we wish to compare data extracted from $\hat{\mathsf{s}}$ against either a set of prior beliefs $\mathfrak{P} = (\mathfrak{P}_1, \dots, \mathfrak{P}_k)$, or directly against the observed path itself. The method of extraction needs to be order-preserving. Thus, we seek to decompose $\hat{\mathsf{s}}$ into a collection of \emph{ordered sub-paths} $(s_j)_{j\ge 0} \subset \Pi_{\Delta}(\hat{\mathsf{s}})$, (recall from Def. \ref{def:setofsubpaths} that $\Pi_{\Delta}(\mathsf{s})$ is the set of all collections of sub-paths obtainable from the embedded observation $S = \pi(\mathsf{s})$ over the mesh grid $\Delta$) in the following manner.
\begin{definition}[Sub-paths]\label{def:subpaths}
    Let $\hat{\mathsf{s}} \in \mathcal{T}_\Delta(I, \mathbb{R}^d)$ and let $h = (h_1, h_2)$ be a vector of hyperparameters. Then, the sequence of \emph{ordered sub-paths} $\mathcal{SP}_h(\hat{\mathsf{s}})$ associated to $\hat{\mathsf{s}}$ is given by
    \begin{equation}\label{eqn:subpathsexperiment1}
        \mathcal{SP}_h(\hat{\mathsf{s}}) = \left(s_j = (\hat{\mathsf{s}}_{jh_1}, \dots, \hat{\mathsf{s}}_{h_1(j + 1)-1})\right)_{j=0}^{N_1} \subset \Pi_{\Delta}(\hat{\mathsf{s}}),
    \end{equation}
    where $N_1 = \lfloor N/h_1\rfloor$ is the maximum number of sub-paths of length $h_1$ that can be generated from a time-augmented path comprised of $N$ observations. 
\end{definition}
The set $\mathcal{SP}_h(\hat{\mathsf{s}}) \in \mathcal{S}\left(\mathcal{T}_\Delta(I, \mathbb{R}^d)\right)$ can be viewed as an (ordered) stream of paths, much in the same spirit as the windowing operation $\mathsf{X}$ from Definition \ref{def:regime}, except we re-sample from $\mathsf{X}$ so as to exclude all overlapping segments.

Should we wish to apply the MMD, we now need to aggregate these path samples into ensembles, which we do via the following.
\begin{definition}[Ensemble paths]\label{def:ensemblepaths}
    With the same notation as Definition \ref{def:subpaths}, the set of \emph{(ordered) ensemble paths} $\mathcal{EP}_h(\hat{\mathsf{s}})$ relative to $\mathcal{SP}_h(\hat{\mathsf{s}})$ is given by		
    \begin{equation}\label{eqn:ensemblepathsexperiment1}
        \mathcal{EP}_h(\hat{\mathsf{s}}) = \left\{\boldsymbol{s}^k = \left(s_j\right)_{j=k}^{k+h_2} : s_j \in \mathcal{SP}^\Phi_h(\hat{\mathsf{s}}), \, k = 0, \dots, N_2\right\},
    \end{equation}
    where $N_2= N_1 - h_2$ is the maximum number of collections of size $h_2$ that can be extracted from $\mathcal{SP}_h(\hat{\mathsf{s}})$. 
\end{definition}
\begin{remark}
    Computationally, the set $\mathcal{SP}_h(\hat{\mathsf{s}})$ is represented as a $(N_1 \times h_1 \times d+1)$ tensor, and $\mathcal{EP}_h(\hat{\mathsf{s}})$ as a $(N_2 \times h_2 \times h_1 \times d+1)$ tensor.
\end{remark}
We incorporate the application of successive path transformations in the following manner. Recalling Definition \ref{def:streamtransformer}, for a given stream transformer $\Phi: \mathcal{S}(\mathbb{R}^d) \to \mathcal{S}(\mathbb{R}^{p})$, we can obtain the set of \emph{transformed sub-paths} 
\begin{equation}\label{eqn:transformedsubpaths}
    \mathcal{SP}^\Phi_h(\hat{\mathsf{s}}) = \left\{\Phi(s) : s \in \mathcal{SP}_h(\hat{\mathsf{s}}) \right\}.
\end{equation}
Applying no transformations is equivalent to setting $\Phi \equiv \mathrm{Id}$. Using (\ref{eqn:transformedsubpaths}) we can then naturally calculate the set of transformed ensemble paths $\mathcal{EP}^\Phi_h(\hat{\mathsf{s}})$. We will drop the superscript $\Phi$ if no stream transformations are applied or if it is obvious from the context.
\begin{remark}[Choice of $h$]
    We note here that the choice of $h = (h_1, h_2)$ is largely dependent on the characteristics of the data one is working with. The choice of $h_1$, for instance, might be limited due to the frequency of available tick data. In general, the lower the value of $h_1$ (assuming all price data is equally spaced), the more reactive the detection method will be, at the cost of increased sensitivity to noise. In \cite{morrill2020generalised}, several partitionings based on hierarchical dyadic windowing was suggested, which can be a useful framework for determining robustness to scale of the detection and clustering algorithms. Finally we note that in theory $h_1$ need not be constant: paths can be comprised of many evaluations as long as the period between the first and last observation remains the same. 

    Regarding the parameter $h_2$ determining the size of the path ensembles: larger values reduce the variance of the associated estimator, at the cost of being less reactive. New, potentially anomalous samples form a lesser portion of the evaluated ensemble, meaning that short-lived regime ``spikes'' may not be captured if they persist for only a short period of time. Again, the choice of $h_2$ is dependent on the modeller's preference. 
\end{remark}

%% file: section3/32mrdp_methods.tex
\subsection{Methods for the market regime detection problem}\label{subsec:mrdpmethods}

\subsubsection{$k$-class detection with prior beliefs}\label{subsubsec:kclassprior}

In this section, we outline our experimental methods in the context of the MRDP. In particular, given $\hat{\mathsf{s}} \in \mathcal{T}_\Delta([0, T]; \mathbb{R}^d)$, $h \in \mathbb{N}^2$ and a stream transformer $\Phi: \mathcal{S}(\mathbb{R}^{d+1}) \to \mathcal{S}(\mathbb{R}^{p})$, we give details as to how one evaluates elements drawn from the set of (transformed) ensemble paths $\mathcal{EP}^\Phi(\hat{\mathsf{s}})$. 

The first case we explore is where the modeller holds (parametric) beliefs $\mathfrak{P} = (\mathfrak{P}_1, \dots, \mathfrak{P}_k)$. In this setting, each belief class $\mathfrak{P}_i$ is represented via a model $\mathbb{P}_i$ with associated parameter vector $\theta_i \in \mathbb{R}^{d_i}$ for $i=1,\dots, k$. We call $(\mathbb{P}_i, \theta_i)$ a \emph{model pair}, and sometimes write $\mathbb{P}_{\theta_i}$. 

For each model pair, we construct $\mathfrak{P}_i$ by simulating $N \in \mathbb{N}$ time-augmented paths (denoted by $x$, for consistency of notation with Definition \ref{def:subpaths}) of length $h_1$ over the interval $[0, T^*]$, where $T^* = Th_1|\Delta|$. Choosing $T^*$ in this way ensures that the time-step used to simulate paths $p \sim (\mathbb{P}, \theta)$ corresponds to that from $\Delta$. Naturally we are also required to apply the same stream transformer $\Phi$ over each path in $x^i \in \mathfrak{P}_i$ in the same way that we did to the set of sub paths. In practice, this means that each $\mathfrak{P}_i$ is represented as a $(N \times h_1 \times d')$ tensor, and $\mathfrak{P}$ is thus a $(k \times N \times h_1 \times d')$ tensor.

Now that we have constructed $\mathfrak{P}$, we need to build a statistical threshold (within each class) in order to determine if our beliefs are being violated or not. For a given confidence level $\alpha \in [0, 1]$, we do so by estimating the $(1-\alpha)$ quantile(s) $c^i_\alpha$ of the null distribution associated to the rank $r$ maximum mean discrepancy $\mathcal{D}^r_{\text{sig}}$ for each set of beliefs $\mathfrak{P}_i, i=1,\dots, k$. One method for obtaining an estimation of the null distribution (and thus, the associated quantiles) is via a bootstrapping method which we outline in Definition \ref{def:bootstrappedmmd}.
\begin{definition}[Bootstrapped $\mathcal{D}^r_{\text{sig}}$ distribution under $\mathfrak{P}$]\label{def:bootstrappedmmd}
    Suppose $I = [0, T]$ and let $\mathfrak{P} \subset \mathcal{T}_\Delta(I, E)$ be a collection of time-augmented paths of size $N \in \mathbb{N}$. Write $\boldsymbol{x}^{\mathsf{i}} = (x^{\mathsf{i}_1}, \dots, x^{\mathsf{i}_{h_2}})$ to denote a set of samples extracted from $\mathfrak{P}$ according to the indexes $\mathsf{i} = (\mathsf{i}_1, \dots, \mathsf{i}_{h_2})$, where each $\mathsf{i}_i \sim U_d(N)$ are drawn from the discrete uniform distribution over $[0, N]$ without replacement. Thus, for $M \in \mathbb{N}, M << N$, we can construct the (random) set of \emph{pairwise path samples}
    \begin{equation*}
        \mathsf{P} = \left\{\left(\boldsymbol{x}^{\mathsf{i}^i}, \boldsymbol{x}^{\mathsf{j}^j}\right)  : \mathsf{i}^i, \mathsf{j}^j \sim U_d(N)^{h_2}, \quad i, j=1,\dots, M\right\},
    \end{equation*}
    and the associated \emph{bootstrapped distribution $\mathfrak{D}$ of $\mathcal{D}_\text{sig}^r$ under $\mathfrak{P}$}
    \begin{equation*}
        \mathfrak{D} := \left\{\mathcal{D}^r_{\text{sig}}(\boldsymbol{x}^i, \boldsymbol{y}^i) : (\boldsymbol{x}^i, \boldsymbol{y}^i) \in \mathsf{P}, i = 1,\dots, M \right\}.
    \end{equation*}
\end{definition}

For $i=1,\dots, k$, we can obtain a bootstrapped distribution $\mathfrak{D}_i$ for each class $\mathfrak{P}_i$. Then, if $\alpha \in (0, 1)$ is a given confidence level, we can extract the associated critical value $c_\alpha^i$ by considering the $(1-\alpha)\%$ quantile of $\mathfrak{D}_i$. It is this value that we will use to gauge how anomalous observed paths from the set $\mathcal{EP}^\Phi_h(\hat{\mathsf{s}})$. Write the collection of bootstrapped distributions as $\mathfrak{D} = (\mathfrak{D}_1, \dots, \mathfrak{D}_k)$. 

Bootstrapping can become computationally expensive as the hyperparameters $h_1, h_2$ and $M$ become large: as mentioned in Section \ref{subsec:mmd}, the order of computation of the rank 1 MMD is quadratic in both $h_1$ and $h_2$. Another (faster) approach is to fit a parametric distribution to each $\mathfrak{P}_i$ for each $i=1,\dots,k$ as given by Definition \ref{def:gammaapprox}. 

Whether we bootstrap the null distribution of the test statistic or approximate it via a gamma distribution, we can derive the critical value(s) $c_\alpha = (c_\alpha^1, \dots, c_\alpha^k)$ and use them to detect changes in regime according to the following.

\begin{proposition}[MMD detector]
    Let $\mathfrak{P} = (\mathfrak{P}_1, \dots, \mathfrak{P}_k)$ be beliefs with associated null distributions $\mathfrak{D} = (\mathfrak{D}_1, \dots, \mathfrak{D}_k)$ constructed under the test statistic $\mathcal{D}_\text{sig}^r$. Associate to $\mathfrak{D}$ the vector of $(1-\alpha)\%$ critical values $c_\alpha = (c_\alpha^1, \dots, c_\alpha^k)$.
    
    Given a stream of data $\mathsf{x} \in \mathcal{S}(E)$ and its associated time augmentation $\hat{\mathsf{x}} \in \mathcal{T}_\Delta([0, T]; E)$, we call the object which compares sets of paths $\boldsymbol{x} \in \mathcal{EP}^\Phi_h(\hat{\mathsf{x}})$ to $\mathfrak{P}$ via $\mathfrak{D}$ the \emph{MMD-Detector} (MMD-DET).
\end{proposition}

Given a collection of transformed sub-paths $\mathcal{EP}^\Phi_h(\hat{\mathsf{s}})$ and a chosen MMD statistic $\mathcal{D}_\text{sig}^r$, we evaluate collections $\boldsymbol{s} \in \mathcal{EP}^\Phi_h(\hat{\mathsf{s}})$ by calculating the \emph{score vector}
\begin{equation}\label{eqn:scorevector}
    \alpha(\boldsymbol{s}) = \left[\frac{1}{n}\sum_{l=1}^n \mathcal{D}^r_{\text{sig}}(\boldsymbol{s}, \boldsymbol{x}^i_l) \right]_{i=1,\dots,k},
\end{equation}
where $\boldsymbol{x}^i \subset \mathfrak{P}_i$ is a collection of path samples, and $n\in\mathbb{N}$ is the number of repeat evaluations. Writing the ordered set of (transformed) ensemble paths $\mathcal{EP}^\Phi_h(\hat{\mathsf{s}})$ as $(\boldsymbol{s}^1, \dots, \boldsymbol{s}^{N_2})$, we can obtain the progressive \emph{score matrix} $\Lambda \in \mathbb{R}^{k \times N_2}$, which is given by
\begin{equation}\label{eqn:scorematrix}
    \Lambda(\hat{\mathsf{s}}) = \begin{bmatrix}
        \alpha(\boldsymbol{s}^1)_1 & \dotsm & \dotsm & \alpha(\boldsymbol{s}^{N_2})_1 \\
        \vdots & \ddots & & \vdots \\
        \vdots & & \ddots & \vdots \\
        \alpha(\boldsymbol{s}^1)_k & \dotsm & \dotsm & \alpha(\boldsymbol{s}^{N_2})_k \\
    \end{bmatrix}
\end{equation}
Thus, real-time detection is a matter of comparing the most recent score vector value to the vector of critical values $c_\alpha = (c_\alpha^1, \dots, c_\alpha^k)$. If there exists at least one $c^\alpha_j$ such that $\alpha(\boldsymbol{x}^i)_j < c^\alpha_j$, then we can claim that the path ensemble $\boldsymbol{x}^i$ is not anomalous for our given confidence level. If we wish to make a statement about $k$-class conformity, we can obtain the quantile value of the score $\alpha(\boldsymbol{x}^i)_j$ by applying the appropriate quantile function $F^{-1}_{\mathfrak{D}_j} : \mathbb{R} \to [0, 1]$; lower values are indicative of higher degrees of conformance to beliefs $\mathfrak{P}_j$. 

\begin{remark}
    As mentioned in Section \ref{subsec:mrcpmrdp}, detection in a single belief setting is a candidate generalized framework for anomaly detection. One chooses the model class $\mathfrak{P}$ to be sufficiently general so as to build a robust estimate of what constitutes an anomaly. 
\end{remark}
\begin{remark}
    The response time of the regime detector is the amount of time required to pass til a new observation $s \in \mathcal{SP}_h(\hat{\mathsf{s}})$ is observed. Due to time reparamitarisation invariance of the signature mapping, one can theoretically increase the response time with no theoretical drawbacks. For simplicity in this work, we do not discuss these modifications.
\end{remark}

\subsubsection{Non-parametric evaluation}\label{subsubsec:nonparametricevaluation}

In this section we outline how one can perform online regime detection in absence of any model. This involves building confidence thresholds from observed data, which requires defining an analogue to the set $\mathfrak{D}$ (so as to obtain critical values), and by extension to the score vector (\ref{eqn:scorevector}). We first introduce the data-driven analogue to the score vector and set of MMD values $\mathfrak{D}$.

\begin{definition}[$L$-lag auto evaluation score]\label{def:autoevaluationscore}
    Let $h \in \mathbb{R}^2$ be a vector of hyperparameters, $\hat{\mathsf{s}} \in \mathcal{T}_\Delta(I, \mathbb{R}^d)$ a time-augmented path, and $\Phi: \mathcal{S}(\mathbb{R}^{d+1}) \to \mathcal{S}(\mathbb{R}^{p})$ a stream transformer. Associate to $(\hat{\mathsf{s}}, h, \Phi)$ the set of transformed ensemble paths $\mathcal{EP}_h^\Phi(\hat{\mathsf{s}})$ with $N_2 = |\mathcal{EP}_h^\Phi(\hat{\mathsf{s}})|$. Let $\mathcal{D}_\text{sig}^r$ be the $r$-rank maximum mean discrepancy.
    
    Then, for $L \subset \{1, \dots, N_2\}$ and a vector of weights $w = (w_{l})_{l \in L}$, the $L$-\emph{lag auto evaluation score} associated to the $i^{\text{th}}$ path ensemble $\boldsymbol{s}^i \in \mathcal{EP}_h^\Phi(\hat{\mathsf{s}})$ is given by
    \begin{equation}\label{eqn:autoevaluationscore}
        A_L(\hat{\mathsf{s}})_i = \sum_{l \in L}w_{l}\mathcal{D}^r_{\text{sig}}\left( \boldsymbol{s}^{i-l}, \boldsymbol{s}^i\right) \qquad \text{for }n= N_2 - \max(L), \dots, N_2. 
    \end{equation}
\end{definition}

The simplest auto-evaluation score is the $1$-lagged version, given by
\begin{equation}\label{eqn:onelaggedscore}
    A_{\{1\}}(\hat{\mathsf{s}}) = \left(\mathcal{D}^r_{\text{sig}}(\boldsymbol{s}^{i-1}, \boldsymbol{s}^i)\right)_{i=2}^{N_2} \qquad\text{for } \boldsymbol{s}^i \in \mathcal{EP}^\Phi_h(\hat{\mathsf{s}}),
\end{equation}

which is quicker to calculate at the cost of providing a potentially noisier local estimate than one averaged over several lag levels. One can derive an explicit formula for the right-hand side of eq. (\ref{eqn:onelaggedscore}) for a fixed $i \in \{2, \dots, N_2\}$ and $r=1$. Writing $\boldsymbol{s}^{i-1} = (s_1, \dots, s_N)$ and $\boldsymbol{s}^i = (s_2, \dots, s_{N+1})$ one has 
\begin{equation*}
    \mathcal{D}^1_{\text{sig}}(\boldsymbol{s}^{i-1}, \boldsymbol{s}^i) = \frac{2}{N^2(N-1)}\left(\sum_{i=1}^N\sum_{j=2}^{N+1}k_{\text{sig}}(s_i, s_j) - N\left( k_{\text{sig}}(s_1, s_{N+1}) + \sum_{i=2}^N k_{\text{sig}}(s_i, s_i)\right) \right),
\end{equation*}
where $k_{\text{sig}}$ is the rank-1 signature kernel (corresponding to the rank-1 signature mapping). 

To determine whether a new path ensemble $\boldsymbol{s} \in \mathcal{EP}^\Phi_h(\hat{\mathsf{s}})$ can be deemed anomalous in an online setting (absent of any prior beliefs held by the modeller), we can study the past distribution of $A_L(\hat{\mathsf{s}})$ over a sliding window $W \in \mathbb{N}$. New observations can be evaluated relative to this distribution which we call $\overline{\mathfrak{D}}_t$, where $t \in [0, T]$ ranges over the course of the path. Naturally a burn-in period would be required to initially populate $\overline{\mathfrak{D}}_t$.

The distribution of $\overline{\mathfrak{D}}_t$ can be approximated via a parametric distribution, which as per Definition \ref{def:gammaapprox} we take to be a Gamma which for $t \in [0, T]$ is parameterised by the location and scale quantities
\begin{equation*}
    \alpha_t = \frac{\mathbb{E}[\overline{\mathfrak{D}}_t]^2}{\mathrm{Var}(\overline{\mathfrak{D}}_t)} \text{ and }\beta_t = \frac{W\mathrm{Var}(\overline{\mathfrak{D}}_t)}{\mathbb{E}[\overline{\mathfrak{D}}_t]}.
\end{equation*}
The appropriate critical threshold is again taken to be the $(1-\alpha)\%$ quantile of the corresponding approximated distribution. We can then define the following. 

\begin{definition}[Auto evaluator]\label{def:autoevaluator}
    Suppose $\hat{\mathsf{s}} \in \mathcal{T}_\Delta([0, T], \mathbb{R}^d)$ is a time-augmented stream of data. We call the object that detects regime changes in $\hat{\mathsf{s}}$ via the score function from eq. (\ref{eqn:autoevaluationscore}), and obtains a critical threshold at the $(1-\alpha)\%$ confidence level $\overline{c}^\alpha_t$ from the temporal prior distribution $\overline{\mathfrak{D}}_t$ the \emph{auto evaluator}.
\end{definition}

Again as in the parametric setting, paths which score higher than $\overline{c}^\alpha_t$ are considered anomalous.

\subsubsection{Switching at random times}\label{subsec:randomtimeswitch}

The first experiment we are interested in is a study of the MRDP in the case where regime changes happen at random times $(\tau_i)_{i=0}^M$ over an interval $I = [0, T]$ for $T > 0$. This experiment can be performed on both real and synthetic data; here, we will work within the context of synthetic data, where we explicitly know the regime-switching times, before extending the approach to real data, which can be found in Section \ref{sec:realdata}.

We begin our experiment by defining the intervals $\mathrm{T}_i = (\tau_{i-1}, \tau_i]$, with $\tau_0 := 0$. Associate to each $\mathrm{T}_i$ a model $\mathbb{P}_{\theta_i}$ for $i=1,\dots, M$. In order to perform our experiments, for a given partition $\Delta$ we wish to generate a path
\begin{equation}\label{eqn:stockpriceprocesstime}
    \hat{\mathsf{s}} = \left\{(t_i, s_i): s_i \in \mathbb{R}^d, \ \ i=1, \dots, N\right\}, 
\end{equation}  
where $s_0 = 1$. First, we specify how the sequence of stopping times $(\tau_i)_{i=0}^M$ are generated. This can be done in one of two ways. The first is by randomly sampling points from $\Delta$ and prescribing the regime change persists for a fixed period of time $a \in \mathbb{R}_+$. In this way $\mathrm{T}_i = [\tau_i, \tau_i + a)$ and we have that $\tau_i = \tau_{i-1} + a$, except for the last interval $\mathrm{T}_M = [\tau_M, T]$. 

Else, we allow the regime change (i.e., the length of $\mathrm{T}_i$) to persist for a random period of time. We iterate over each point on the mesh $\Delta$ in the following way: we start at $t=0$ with model dynamics $\mathbb{P}_{\theta_1}$. Then, after $h_1$ time points, we first draw a random variable $Z_1 \sim \mathrm{Po}(\lambda_1)$ which tells us if we change regime or not. If $Z_2 > 0$, then we set $\tau_i := t_k$ and flag that we are in a regime change. We then simulate the path  until we reach the next time-step that is a multiple of $h_1$, plus one. If we are not in a regime change then we draw $Z_1 \sim \mathrm{Po}(\lambda_1)$ again; otherwise, simulate a second random variable $Z_2 \sim \mathrm{Po}(\lambda_2)$ which tells us whether we exit the regime change or not. If $Z_2 > 0$ we set $\tau_{i+1}:= t_{k} + h_1 + 1$; otherwise, we move to the next time step. We end up with a sequence of times $(\tau_i)_{i\ge 0}$ whereby $|\tau_i - \tau_{i-1}|$ is a random variable. Moreover, the number of times is also random. If the number of regime changes is greater than $M$, we simply begin with $\mathbb{P}_{\theta_1}$ again.

\begin{remark}
    The choice to institute regime changes at multiples of $h_1$ ensures that each sub-path within $\mathcal{EP}^\Phi_h(\hat{\mathsf{s}})$ is generated by just one model pair, as opposed to being a concatenation of two (or more) pairs.
\end{remark}

Once we have built our regime change points and chosen our models $(\mathbb{P}_{\theta_i})_{i=1}^M$, we build $\hat{\mathsf{s}}$ by solving the corresponding SDE over $\mathrm{T}_i$, stitching the solutions together. Once we have simulated our regime changed path $\hat{\mathsf{s}}$, given a vector of hyperparameters $h\in\mathbb{R}^2$ and a path transformer $\Phi$ we create the sequence of sub paths $\mathcal{SP}_h(\hat{\mathsf{s}})$ from (\ref{eqn:subpaths}), transform them to obtain $\mathcal{SP}^\Phi_h(\hat{\mathsf{s}})$ from (\ref{eqn:transformedsubpaths}), and then build the set of transformed ensemble paths $\mathcal{EP}_h^\Phi(\hat{\mathsf{s}})$. 

Regarding evaluating elements of $\boldsymbol{s}_i \in \mathcal{EP}_h^\Phi(\hat{\mathsf{s}})$ for $i=1,\dots,N_2$: in the parametric (beliefs) case, after specifying $\mathfrak{P} = (\mathfrak{P}_1, \dots, \mathfrak{P}_k)$ we first build the associated empirical distributions $\mathfrak{D} = (\mathfrak{D}_1, \dots, \mathfrak{D}_k)$ and obtain the vector of $(1-\alpha)\%$ quantiles $c_\alpha = (c_\alpha^1, \dots, c_\alpha^k)$ via the estimation or bootstrapping method. Then, for a given number of evaluations $n \in \mathbb{N}$ we progressively calculate the score matrix from eq. (\ref{eqn:scorematrix}) for a given number of evaluations $n\in\mathbb{N}$. Detection is then a matter of comparing elements of $\alpha(\boldsymbol{s}^i)$ with their corresponding $(1-\alpha)\%$ critical values. For $j=1,\dots, k$, if $\alpha(\boldsymbol{s}^i)_j < c^\alpha_j$, then the path collection $\boldsymbol{s}^i$ conforms to the model pair represented by the beliefs $\mathfrak{P}_j$. If $\alpha(\boldsymbol{s}^i)$ fails to conform to any of our beliefs, then we can interpret the given path segment as being anomalous. 

In the non-parametric case, we calculate the (weighted) $L$-lag auto evaluation score (\ref{eqn:autoevaluationscore}) and report this. Theory tells us that the MMD score between path ensembles drawn from the same distribution will be asymptotically $0$, and conform to a null distribution for a given number of finite samples. During a regime change, newly distributed samples enter the most recent ensemble, causing the MMD score to become larger relative to previous values. It will then decay to a lower level as each path ensemble again contains paths drawn from the new regime. 

\subsubsection{Non-ensemble evaluation}\label{subsubsec:nonensembleevaluation}

In this section, we give an experimental methodology for the market regime detection problem in the case where only a single path sample is to be evaluated. That is, we are interested in performing regime detection over the set of transformed sub paths $\mathcal{SP}_h^\Phi(\hat{\mathsf{s}})$ where $\hat{\mathsf{s}}$ is a time-augmented discrete asset price path. As we cannot directly apply the MMD in this situation, we instead seek to use the rank-$r$ signature kernel scoring rule $s_{k_{\text{sig}}^r}$ as a way of detecting conformance from paths $s \in \mathcal{SP}_h^\Phi(\hat{\mathsf{s}})$ to our beliefs $\mathfrak{P}$. As we are not able to directly apply detection based on two-sample testing, we require a slightly different approach which makes use of the following function. 

\begin{definition}[Similarity score]\label{def:conformancescore}
    Suppose $\mathcal{K} \subset C(I, E)$ is a compact set of bounded variation paths. Let $s_{k_{\text{sig}}^1}: \mathcal{P}(\mathcal{K}) \times \mathcal{K} \to \mathbb{R}$ be the \emph{signature kernel scoring rule} (cf. Def \ref{def:signaturekernelscore}). Write $s:= s_{k^1_{\text{sig}}}$. Suppose $\mathbb{P}, \mathbb{Q} \in \mathcal{P}(\mathcal{K})$. Then, we call the function 
    \begin{equation}\label{eqn:similarityscore}
        \Sigma^{\mathbb{P}, \mathbb{Q}} : \mathcal{K} \to \mathbb{R}, \quad \Sigma^{\mathbb{P}, \mathbb{Q}}(x) = s(\mathbb{P}, x) - s(\mathbb{Q}, x)
    \end{equation}
    the \emph{(signature) similarity score} function between the measures relative to $\mathbb{P}, \mathbb{Q}\in \mathcal{P}(\mathcal{K})$. 
\end{definition}
Directly evaluating $\Sigma^{\mathbb{P}, \mathbb{Q}}$ requires using the unbiased estimator from eq. (\ref{eqn:scoringunbiased}), which means we need to specify the number of paths $N \in \mathbb{N}$ sampled from both $\mathbb{P}$ and $\mathbb{Q}$. 

We now introduce some important facts regarding the conformance scoring function. 
\begin{proposition}
    Suppose that $\mathbb{Z} \in \mathcal{P}(\mathcal{K})$. Then, we have that 
    \begin{equation}\label{eqn:expectedconformancescore} 
        \mathbb{E}_{\mathbb{Z}}[\Sigma^{\mathbb{P}, \mathbb{Q}}(X)]  = \mathcal{D}^r_{\text{sig}}(\mathbb{P}, \mathbb{Z})^2 - \mathcal{D}^r_{\text{sig}}(\mathbb{Q}, \mathbb{Z})^2.
    \end{equation}
    It follows directly from the above that 
        \begin{enumerate}
        \item $\mathbb{E}_{\mathbb{P}}[\Sigma^{\mathbb{P}, \mathbb{Q}}(X)] \le 0$, and
        \item $\mathbb{E}_{\mathbb{Q}}[\Sigma^{\mathbb{P}, \mathbb{Q}}(X)] \ge 0$.
    \end{enumerate}
\end{proposition}

\begin{proof}
     To prove eq. (\ref{eqn:expectedconformancescore}), apply the definition of the signature kernel scoring rule: 
    \begin{align*}
        \mathbb{E}_{\mathbb{Z}}[\Sigma^{\mathbb{P}, \mathbb{Q}}(X)]  &= s(\mathbb{P}, \mathbb{Z}) - s(\mathbb{Q}, \mathbb{Z}) \\
        &= \mathbb{E}_\mathbb{P}[k_{\text{sig}}^r(X, X')] - 2\mathbb{E}_{\mathbb{P}, \mathbb{Z}}[k_{\text{sig}}^r(X, Z)] - \mathbb{E}_\mathbb{Q}[k_{\text{sig}}^r(Y, Y')] + 2\mathbb{E}_{\mathbb{Q}, \mathbb{Z}}[k_{\text{sig}}^r(Y, Z)] \\ 
        &=\left(\mathbb{E}_\mathbb{P}[k_{\text{sig}}^r(X, X')] - 2\mathbb{E}_{\mathbb{P}, \mathbb{Z}}[k_{\text{sig}}^r(X, Z)] + \mathbb{E}_\mathbb{Z}[k_{\text{sig}}^r(Z, Z')] \right) \\
        &\qquad- \left(\mathbb{E}_\mathbb{Q}[k_{\text{sig}}^r(Y, Y')] - 2\mathbb{E}_{\mathbb{Q}, \mathbb{Z}}[k_{\text{sig}}^r(Y, Z)] + \mathbb{E}_\mathbb{Z}[k_{\text{sig}}^r(Z, Z')]\right) \\
        &= \mathcal{D}^r_{\text{sig}}(\mathbb{P}, \mathbb{Z})^2 - \mathcal{D}^r_{\text{sig}}(\mathbb{Q}, \mathbb{Z})^2.
    \end{align*}
\end{proof}
(1) and (2) follow directly from eq. (\ref{eqn:expectedconformancescore}). We highlight them separately because they are key edge cases explored in the upcoming experiments. Consider the following example: From eq. (\ref{eqn:expectedconformancescore}), we know that if $\mathbb{E}_\mathbb{Z}[\Sigma^{\mathbb{P}, \mathbb{Q}}] \ge 0$, then $\mathcal{D}^r_{\text{sig}}(\mathbb{P}, \mathbb{Z})^2 \ge \mathcal{D}^r_{\text{sig}}(\mathbb{Q}, \mathbb{Z})^2$. The interpretation of this is the observation $\mathbb{Z}$ is closer (conforms more) to $\mathbb{P}$ than to $\mathbb{Q}$. 

In the context of regime detection, if $z_1, \dots, z_{N_1} \sim \mathbb{Z}$ are observed market trajectories extracted from $\mathcal{SP}_h(\hat{\mathsf{z}})$, then studying the value of the score function relative to two measures $\mathbb{P}, \mathbb{Q}$ gives an indication as to how much the data conforms to either $\mathbb{P}$ or $\mathbb{Q}$. Negative values indicate greater similarity to $\mathbb{P}$, whereas positive values indicate greater similarity to $\mathbb{Q}$. In our setting, the measures $\mathbb{P}$ and $\mathbb{Q}$ are component samples from a set of modelling beliefs $\mathfrak{P} = (\mathfrak{P}_1, \dots, \mathfrak{P}_k)$ as per Section \ref{subsubsec:kclassprior}. Given $s \in \mathcal{SP}_h^\Phi(\hat{\mathsf{s}})$, we can calculate the $\mathbb{R}^{k\times (k-1)}$ similarity matrix 
\begin{equation}\label{eqn:conformancematrix}
    \Sigma^{\mathfrak{P}}(s) = \begin{bmatrix}
        \Sigma^{\mathfrak{P}_1, \mathfrak{P}_2}(s) & \dots & \dots & \Sigma^{\mathfrak{P}_1, \mathfrak{P}_k}(s) \\
        \vdots & \ddots & & \vdots \\ 
        \vdots & & \ddots &\vdots \\ 
        \Sigma^{\mathfrak{P}_k, \mathfrak{P}_1}(s) & \dots  & \dots & \Sigma^{\mathfrak{P}_k, \mathfrak{P}_{k-1}}(s)
    \end{bmatrix}
\end{equation}
and study the sign of each of the terms. We present three examples using this method: a synthetic, Markovian example (with $k=2$) in Subsection \ref{subsec:toysinglepath}, a non-Markovian detection problem in Subsection \ref{subsec:nonmarkovian}, and in Subsection \ref{subsec:realdatapipeline} a real-data, data-driven pipeline, showing how $\Sigma^{\mathfrak{P}}$ can be used to test for bull/bear regimes, or more bespoke market conditions.

%% file: section3/33mrcp_methods.tex
\subsection{A method for the market regime classification problem}\label{subsec:mrcpmethods}

Suppose $(E, d)$ is a metric space and suppose $\mathsf{s} \in \mathcal{S}(E)$ is a stream of data over $E$ comprised of $N \in \mathbb{N}$ elements. The problem of grouping data points $x_i \in X, i=1,\dots,N$ into a fixed number of partitions $\mathcal{C}_l, l \in\mathbb{N}$ such that elements of each $\mathcal{C}_l$ share similar characteristics (usually defined under $d$) is referred to as \emph{clustering} the vector $X$. This process is a classic unsupervised learning technique which has many popular approaches: $k$-means \cite{macqueen1967some}, $k$-medoids \cite{kaufman2009finding}, fuzzy clustering and its relatives \cite{dunn1973fuzzy} or density-based models such as DBSCAN \cite{ling1972theory}. Each of these methods has their own benefits and drawbacks, which depending on the context make some algorithms more suitable than others. We refer the reader to \cite{horvath2021clustering} for a more in-depth summary of various clustering techniques and in particular their application to mathematical finance and regime classification. 

Another classical technique is called \emph{hierarchical clustering}, first given in \cite{johnson1967hierarchical}. The algorithm employs either a top-down (divisive) or bottom-up (agglomerative) approach to partitioning data into $k \in \mathbb{N}$ distinct clusters where $k$ is chosen by the user. In the top-down approach, data is assumed to belong to one cluster, whereas the agglomerative approach assumes the opposite: all values $x_i \in X$ belong to their own cluster and are progressively agglomerated. In both cases, the algorithm requires a pairwise distance matrix $D \in \mathbb{R}^{N \times N}$, with elements given by
\begin{equation}\label{eqn:distancematrix}
    D_{ij} = d(x_i, x_j) \qquad \text{for }i, j  = 1,\dots, N.
\end{equation}

We refer the reader to Appendix \ref{appendix:agglomerativeclustering} for further details. Importantly, we are able to leverage this technique in our framework by setting our clustering space to be $\mathcal{P}(\mathcal{K})$. In our case this will be given by the set of (transformed) sub paths $\mathcal{SP}_h^\Phi(\hat{\mathsf{s}})$ associated to the piecewise-linearly interpolated path $S$ obtained from a time-augmented stream of data $\hat{\mathsf{s}} \in \mathcal{T}_\Delta(I, E)$. 

We can equip $\mathcal{P}(\mathcal{K})$ with the rank-$r$ MMD $\mathcal{D}^r_{\text{sig}}$ and thus create the matrix $D$ via evaluating elements of the set of transformed ensemble paths $\mathcal{EP}_h^\Phi(\hat{\mathsf{s}})$. It is this stream of paths that we are aiming to cluster. We then calculate the pairwise distance matrix $D \in \mathbb{R}^{N_2\times N_2}$ where
\begin{equation}\label{eqn:pairwisedistancematrixmmd}
    D_{ij} = \mathcal{D}^r_\text{sig}(\boldsymbol{s}^i, \boldsymbol{s}^j), \qquad i,j=1,\dots,N_2.
\end{equation}
Once we have obtained $D$, for a given $k \in \mathbb{N}$ we pass $(D, k)$ to an agglomerative hierarchical clustering algorithm and (with a specified linkage criterion). After the algorithm has completed, we obtain labels for each element of $\mathcal{EP}_h^\Phi(\hat{\mathsf{s}})$ which we take to be the associated class classification. Sub-paths $s \in \mathcal{SP}_h(\hat{\mathsf{s}})$ can be assigned a cluster label, which is given by the average of each label assigned to each set of ensemble paths $\boldsymbol{s}$ that $s$ was a member of.

%% file: section3/34experimental_overview.tex
\subsection{Overview of experiments}\label{subsec:experimentsoverview}

The rest of the paper is devoted to experiments showing how one can successfully apply the different versions of the MMD\footnote{cf. Figures 15, 16, and 17.} (truncated MMD, kernel MMD, and in particular the higher-rank MMD which is our method of choice), and the signature kernel scoring rule, to regime detection and classification problems on path space. Here we will briefly summarize each of the experiments and in particular how we evaluate the performance of our approach. 

Section \ref{sec:mrdp} is devoted to the market regime detection problem and its variants. In the entire section we focus synthetically generated data, before moving on to real data experiments in Section \ref{sec:realdata}. We do this separation so that we can first establish the accuracy and validity of our methods, by evaluating our algorithms output against known (synthetic) solutions. \\
Subsection \ref{subsec:toyexample} showcases our approach on a first simple toy example that can be easily validated: There we consider a regime-switching problem where the underlying process is a univariate geometric Brownian motion (gBm) and the different regimes correspond to different parameter choices, which are flipped at random times. Within this framework we first provide examples of experiments in the case where the modeller holds a given (parametric) belief $\mathfrak{P}$, and then proceed to case where the modeller holds a priori no beliefs, but infers these from observations of the market in an evolving manner (this is done by evaluating the observed path against itself as per Definition \ref{def:autoevaluationscore}). This first experiment demonstrates that regime changes at random times can be detected with a high confidence.  \\
In order do add further information about \emph{direction} and \emph{magnitude} of such changes we consider a case of multiclass prior beliefs which is demonstrated in Subsection \ref{subsec:toymulticlass}. Here we extend this first toy example to the modeller holding multiple beliefs $\mathfrak{P} = (\mathfrak{P}_1, \dots, \mathfrak{P}_k)$ to add granularity about the direction and magnitude of changes. \\
An extreme case of the the first simple toy example would be if we want to detect regime changes based on single-path observations. To demonstrate that this is possible in our framework, we demonstrate in Subsection \ref{subsec:toysinglepath} an experiment, that takes as a starting point the first motivating toy example. However, in this new setting we restrict the detection problem to a path-by-path setting, as opposed to using ensembles. 
Initial experiments have shown that naive methods (such as truncation) may struggle to yield reliable detection performance on non-Markovian data, therefore we show that our method is designed to work for such data as well.
In order to showcase that our method also works on data that exhibits an autocorrelative structure and non-Markovian properties, we demonstrate experiments on synthetic data generated by rough stochastic volatility models in Subsection \ref{subsec:rbergomi}.  Within the same section also demonstrate in passing how the detection problem can be extended to multidimensional processes. \\ Subsection \ref{subsec:higherrankdetection} shows how utilizing the higher-rank signature kernel MMD method can lead to better accuracy in regime detection (than its lower rank counterparts), however this comes at a computational cost and we conclude the section by leaving it to the modeller to choose the method according to her preferences between accuracy and speed of the method. Subsection \ref{subsec:mrdpcomparisons} showcases how our method can be applied as an anomaly detection tool. We benchmark our technique against an existing signature-based method for anomaly detection developed in \cite{cochrane2020anomaly}. Finally, Subsection \ref{subsec:nonmarkovian} gives an experiment where non-Markovianity of sample trajectories is handled via the similarity score function from Definition \ref{def:conformancescore}.

Section \ref{sec:mrcp} is concerned with the market regime classification problem. In contrast to the previous section, here we are interested in finding regimes on path space via clustering (ensembles of) paths extracted from a (simulated) path observation.
Again, we focus on synthetic data examples in the entire section and refer to Section \ref{sec:realdata} for real data experiments. As before, the starting point is a simple gBm example with two regimes (that change at random times), shown in Section \ref{sec:mrcp}.  We draw the comparison of this regime classification method to our previous results (see \cite{horvath2021clustering}) and demonstrate how our kernel-based method helps go beyond dimensionality limitations that we were facing in the returns-based setting, previously in \cite{horvath2021clustering}. For this we compare three clustering methods against each other: the Wasserstein $k$-means approach explored in \cite{horvath2021clustering}, a naive truncated-signature-based algorithm, and our kernel-based clustering, see Figure \ref{fig:toyclustering}. 
Specifically, within a high-dimensional setting, we show how our kernel-based method extends the other two considered clustering methods, and in particular, outperforms the truncated signature-MMD based method.

Finally, after establishing the validity of both our MRDP and MRCP methods on verifiable synthetic data in the previous two Sections (\ref{sec:mrdp} and \ref{sec:mrcp}), we demonstrate the performance of our method on a variety of real data examples in the final Section \ref{sec:realdata}.
The real data examples we consider are organised as follows: Section \ref{subsec:basket} explores both the MRDP and MRCP for a multidimensional setting comprised of a basket of equities. We verify that our methodologies successfully identify periods of known historical market turmoils (such as the recent coronavirus pandemic or the \textrm{dot-com} crisis). We also observe that the obtained MMD scores tracks the VIX fairly accurately, which further highlights the the speed of detection ability at a good accuracy.\\
Subsection \ref{subsec:crypto} applies our regime detection- and regime classification methods on (multidimensional) cryptocurrency data. Again, known periods of instability are tracked efficiently and accurately though our MMD scores even in high-dimensional cases. It is worth noting that accuracy is not compromised by higher dimensional settings, which is expected in our methodological design.\\
Finally, Subsection \ref{subsec:realdatapipeline} showcases again an extreme example in a single-path-observation setting: This example demonstrates how our MRD and MRC methodologies can be combined in a setting where the practitioner holds non-parametric beliefs, to create a wholly data-driven detection pipeline, evaluated (path-by-path) on single-observations.

All experiments were run on a NVIDIA GeForce RTX 3070 Ti GPU with 8GB of GPU memory. Code is available at \href{https://github.com/issaz/signature-regime-detection/}{\texttt{https://github.com/issaz/signature-regime-detection/}}.

%% file: section4/41toy_example.tex
\subsection{Toy example: switching at random times}\label{subsec:toyexample}

In the first motivating example we implement simple regime switching dynamics from within the same model family, being a univariate geometric Brownian motion (gBm) governed by the stochastic differential equation 
\begin{equation}\label{eqn:gbm}
    dX_t = \mu X_t dt + \sigma X_t dW_t, \quad X_0  = 1.
\end{equation}
Moreover, we write $\theta = (\mu, \sigma) \in \mathbb{R}^2$. 

We begin by describing how we generate the regime-switching path. For simplicity here, we assume only two possible regimes: either a low or high volatility period. This defines our sequence of model pairs $\mathbb{M} = (\mathbb{P}_{\theta_1}, \mathbb{P}_{\theta_2})$, with $\theta_1 = (0, 0.2)$ and $\theta_2 = (0, 0.3)$. For a given $T > 0$ and partition $\Delta$ of $[0, T]$, we generate a path $\hat{\mathsf{s}} \in \mathcal{T}_\Delta([0, T], \mathbb{R})$ where at the random times $\tau_1, \dots, \tau_M$ regime dynamics cycle between elements of $\mathbb{M}$ as outlined in Section \ref{subsec:mrdpmethods}. Initially we take $T=4$ and set $|\Delta| = \tfrac{1}{252\times 7}$, so each path value roughly represents hourly tick data. Our regime switching times are determined by the random variable $Z_1 \sim \mathrm{Po}(2)$ and regime exit times by $Z_2 \sim \mathrm{Po}(1/49)$. An example of such a regime-changed path is given in Figure \ref{fig:regimechangedpathtoyexample} and the corresponding path of log-returns is given alongside it. 

\begin{figure}[ht]
    \centering
    \begin{subfigure}{0.5\linewidth}
        \centering
        \includegraphics[width=\textwidth]{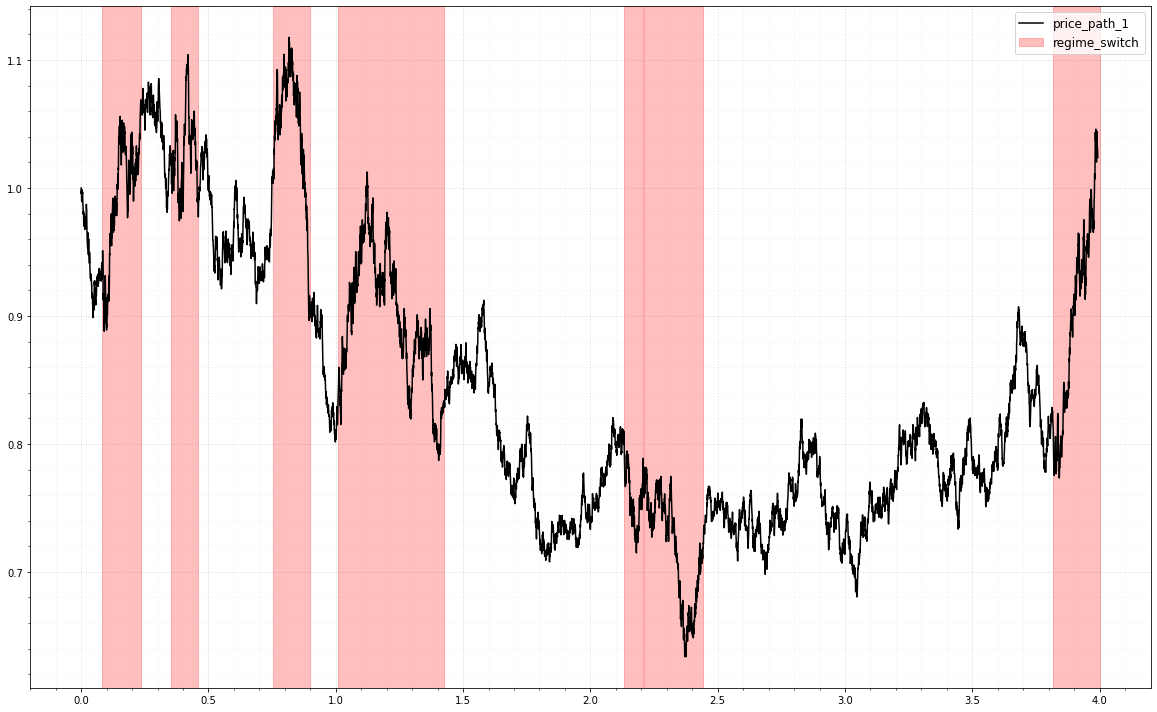}
        \caption{Price path.}
        \label{fig:toyexamplepath}
    \end{subfigure}%
    \begin{subfigure}{0.5\linewidth}
        \centering
        \includegraphics[width=\textwidth]{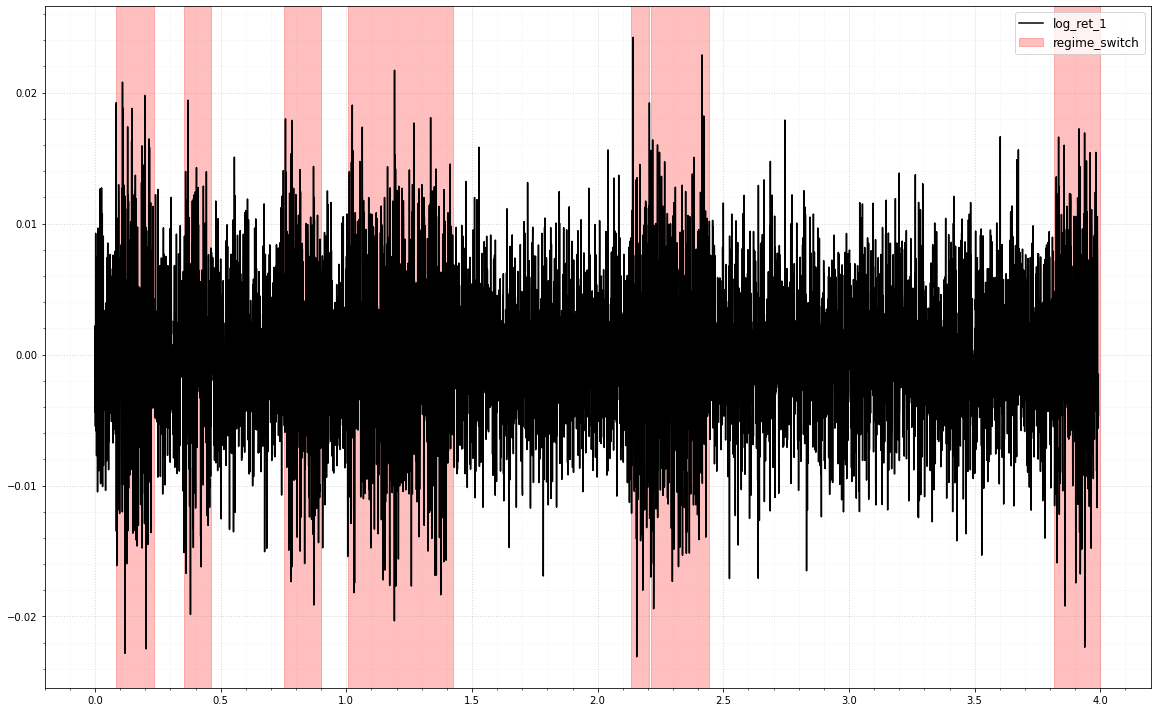}
        \caption{Log returns.}
        \label{fig:toyexamplelogret}
    \end{subfigure}
    \caption{Sample regime-changed path, toy example. Red shaded areas indicate periods of regime change, where $\mathbb{P}_{\theta_2}$ is the governing model.}
    \label{fig:regimechangedpathtoyexample}
\end{figure}
    
Our hyperparameter choice $h$ is given by $(h_1, h_2) = (7, 10)$, which partitions paths into roughly a day's worth of returns in progressive ensembles of $10$. Given the regime-changed path $\hat{\mathsf{s}}$, we calculate the set of sub-paths $\mathcal{SP}_h(\hat{\mathsf{s}})$ and transform them via the stream transformer $\Phi = \phi_{\text{incr}} \circ \phi_{\text{time}} \circ \phi_\text{norm}$. This allows us to then calculate the set of transformed ensemble paths $\mathcal{EP}^\Phi_h(\hat{\mathsf{s}})$ relative to $\mathcal{SP}^{\Phi}_h(\hat{\mathsf{s}})$. 

We evaluate the resulting path assuming that the modeller holds one set of static beliefs. In particular we set $\mathfrak{P}$ to be paths simulated from $\mathbb{P}_{\theta_1}$, so the modeller's beliefs match the standard regime. We populate $\mathfrak{P}$ with $N = 100,000$ solutions to the SDE from eq. (\ref{eqn:gbm}) over the mesh $\Delta = \{ndt : n = 0, \dots, 6\}$. The bootstrapped distribution $\mathfrak{D}$ of $\mathcal{D}^1_{\text{sig}}$ under $\mathfrak{P}$ from Definition \ref{def:bootstrappedmmd} was generated with $M = 1000$ pairwise ensembles of size $10$ drawn from $\mathfrak{P}$. We obtain the associated critical value $c_\alpha$ to $\mathfrak{D}$ with $\alpha = 0.05$. 

Regarding the choice of metric, the rank $1$ MMD is appropriate for this example due to Theorem \ref{thm:rankmarkov}. Utilizing a higher-rank version of the MMD would incur increased computational cost for marginal improvements in inference. Finally, we applied the RBF lift in the calculation of the signature kernel and chose the smoothing hyperparameter $\sigma = 0.025$. The associated maximum mean discrepancy is given by eq. (\ref{eqn:rankrmmdrbf}). We determined $\sigma$ via a bootstrapping technique and note here that the process of determining an optimal scaling in a finite-sample setting is the topic of current research. 

\begin{figure}[ht]
    \centering
    \begin{subfigure}{0.5\linewidth}
        \centering
        \includegraphics[width=\textwidth]{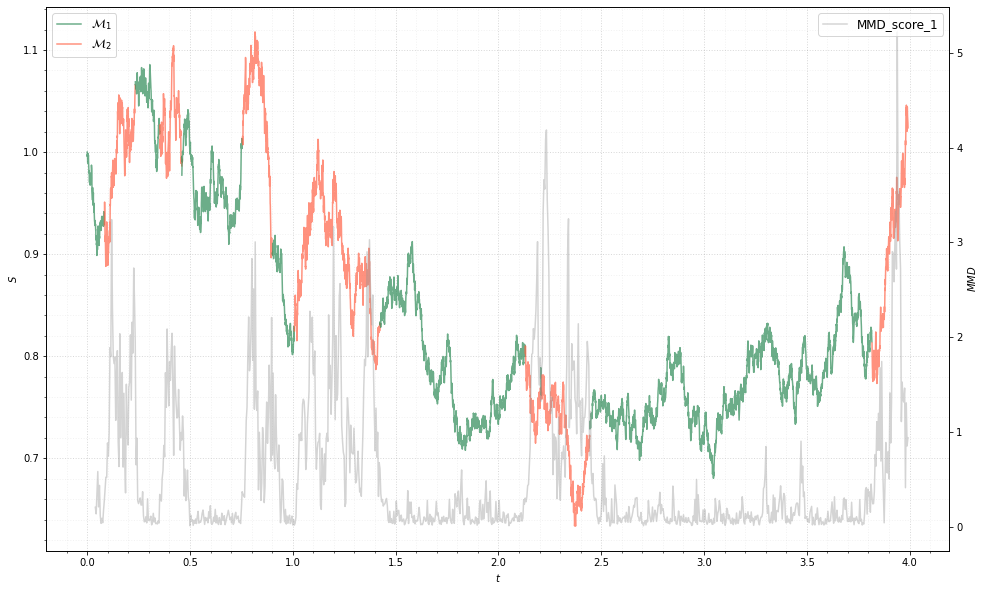}
        \caption{Regime-changed path and score vector $\Lambda(\hat{\mathsf{s}})$.}
        \label{fig:mmdscoretoy}
    \end{subfigure}%
    \begin{subfigure}{0.5\linewidth}
        \centering
        \includegraphics[scale=0.2275]{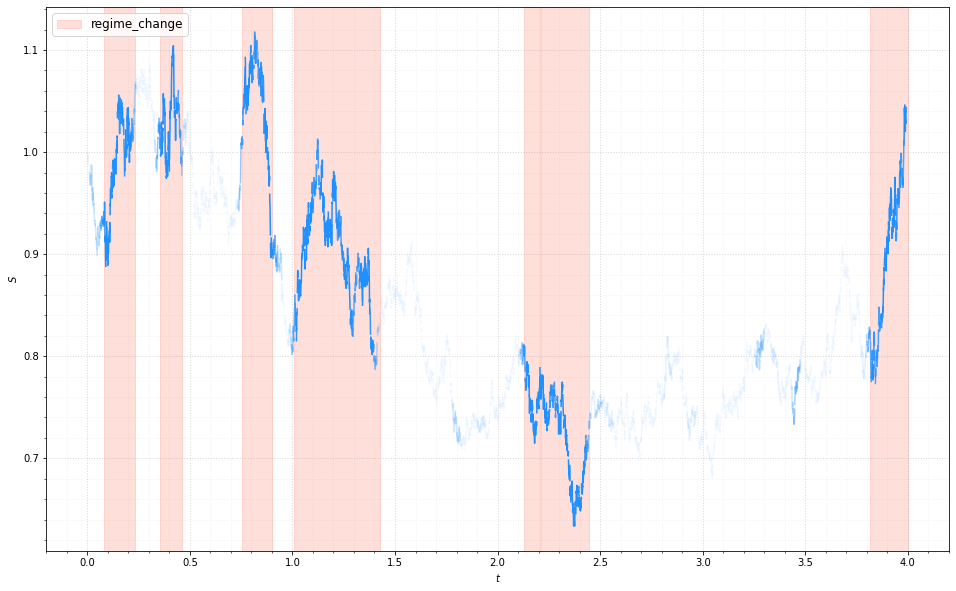}
        \caption{$\mathcal{D}_\text{sig}^1 > c_\alpha$ plot.}
        \label{fig:alphascoretoy}
    \end{subfigure}
    \caption{MMD score and accompanying threshold plot, sample run of toy example.}
    \label{fig:toyexamplemmd}
\end{figure}

We take a moment here to explain the plots shown in Figure \ref{fig:toyexamplemmd} which we will continue to display over the course of the paper. In Figure \ref{fig:mmdscoretoy} we have the same regime-changed path from Figure \ref{fig:toyexamplepath}. Green shading indicates that section of the path was generated according to $\mathbb{P}_{\theta_1}$, and red shading corresponds to $\mathbb{P}_{\theta_2}$. In grey we have the score matrix $\Lambda(\hat{\mathsf{s}})$ from eq. (\ref{eqn:scorematrix}) which in this case (as we have just one set of beliefs) is the vector $(\alpha(\boldsymbol{s}^i))_{i=1}^{N_2}$. Figure \ref{fig:alphascoretoy} quantifies the degree of regime change by individually plotting each element $s \in \mathcal{SP}_h(\hat{\mathsf{s}})$. The intensity of shading is given by the percentage of the time a given path $s$ was part of an ensemble $\boldsymbol{s} \in \mathcal{EP}_h(\hat{\mathsf{s}})$ such that $\alpha(\boldsymbol{s}) > c_\alpha$, i.e., part of an ensemble such that the two-sample test concluded that $\boldsymbol{s}$ is distributed differently to the underlying beliefs of the modeller.

Hence, in Figure \ref{fig:mmdscoretoy}, one can visually see that the parametric MMD detector is able to discriminate between paths drawn from the first model pair and those from the second. This is verified by the $\alpha$-plot Figure \ref{fig:alphascoretoy}. One can see that in the red shaded regime-changed regions the detector is able to identify periods of anomalous behaviour. 

We close this section with an illustration of detection with non-parametric beliefs. Recalling the definition of the auto evaluator from Definition \ref{def:autoevaluationscore}, we provide a plot of the 1-lag unweighted auto-evaluator score $A_{\{1\}}(\hat{\mathsf{s}})$ from (\ref{eqn:onelaggedscore}). Here, we see how the MMD score sharply increases at the regime-change points $(\tau_i)_{i\ge 0}$. To illustrate the effect of taking more lags, we plot the vector of scores associated to $A_{\{1,2,3,4,5\}}(\hat{\mathsf{s}})$ in Figure \ref{fig:5lagtoy}. Here, the averaging process has the effect of smoothing the score vector at the cost of less prominent change point signals.

\begin{figure}[ht]
    \centering
    \begin{subfigure}{0.5\linewidth}
        \centering
        \includegraphics[width=\textwidth]{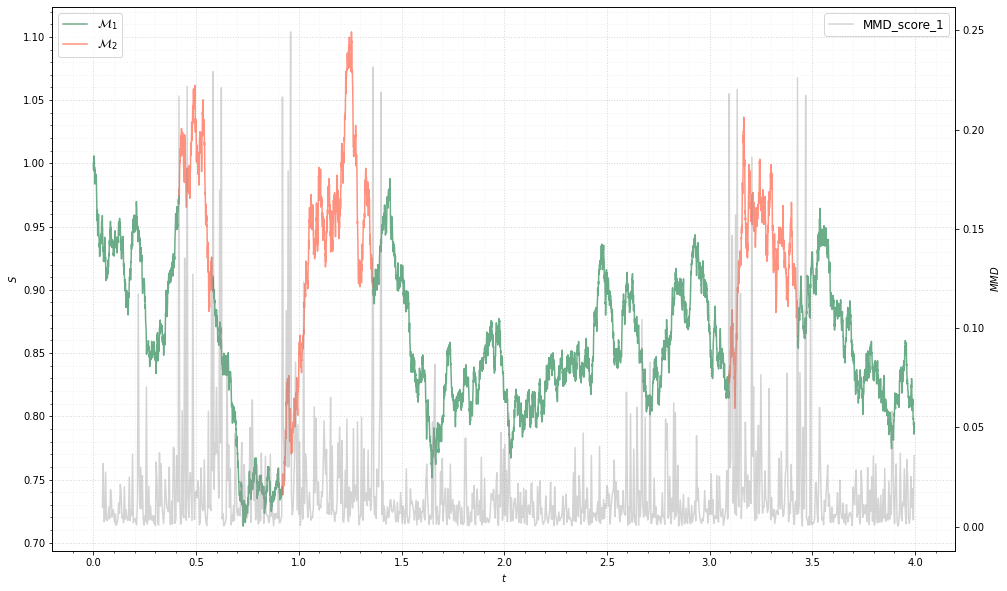}
        \caption{1-lag.}
        \label{fig:1lagtoy}
    \end{subfigure}%
    \begin{subfigure}{0.5\linewidth}
        \centering
        \includegraphics[width=\textwidth]{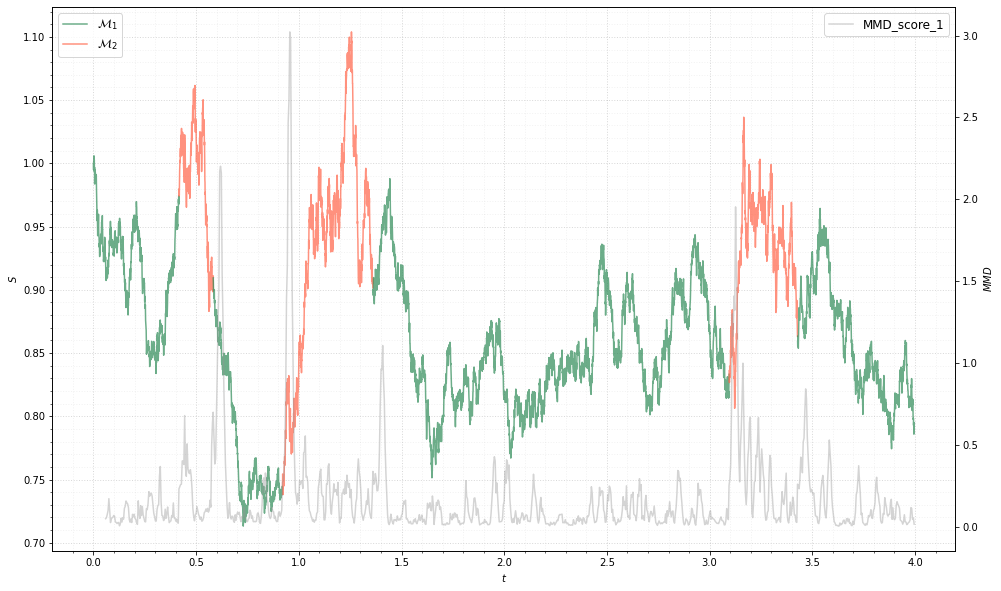}
        \caption{5-lag.}
        \label{fig:5lagtoy}
    \end{subfigure}
    \caption{Sample auto-evaluator scores, toy example. Using more lags results in a smoother rolling MMD score. However, signals can appear less prominent, as evidenced by the final regime change.}
    \label{fig:toyexampleautoevaluator}
\end{figure}

%% file: section4/42multiclass.tex
\subsection{Multiclass beliefs}\label{subsec:toymulticlass}

In most instances, holding one set of parametric beliefs will not provide the modeller a descriptive enough picture of the componentry surrounding a 
regime change, should one occur. In practice one might choose to hold multiple beliefs $\mathfrak{P} = (\mathfrak{P}_1, \dots, \mathfrak{P}_k)$ and evaluate observations against each of these in order to characterize a regime change should one occur. 

Here, we illustrate an example where we again restrict ourselves to the modelling universe of geometric Brownian motion for simplicity. Model pairs are given by $\mathbb{P}_\theta$ for $\theta = (\mu, \sigma)$. To generate the regime-changing path, we define the model pair sequence $\mathbb{M} = (\mathbb{P}_{\theta_1}, \dots, \mathbb{P}_{\theta_{20}})$ where $\theta_i = (0, 0.2)$ for $i=1,3,\dots,19$ and $\theta_j = (0, 0.1 + (j-2)/60)$ for $j=2, 4, \dots, 20$. In this way, the regime switching path alternates between a base state and an increasingly volatile state. We set our beliefs $\mathfrak{P} = (\mathfrak{P}_1, \mathfrak{P}_2, \mathfrak{P}_3)$ to be generated by a geometric Brownian motion, with parameters given by $(0, \sigma_i)$ where $\sigma_1 = 0.1, \sigma_2 = 0.2$ and $\sigma_3 = 0.4$. All other parameters remain the same as per Subsection \ref{subsec:toyexample}. 

Figure \ref{fig:multiclassregimechanged} gives a plot of a sample regime changed path $\hat{\mathsf{s}} \in \mathcal{T}_\Delta([0, T]; \mathbb{R})$. We first present a plot of the regime change path along with the corresponding log-returns. One can clearly observe the increasing volatility with each successive regime change. 

\begin{figure}[ht]
    \centering
    \begin{subfigure}{0.5\linewidth}
        \centering
        \includegraphics[width=\textwidth]{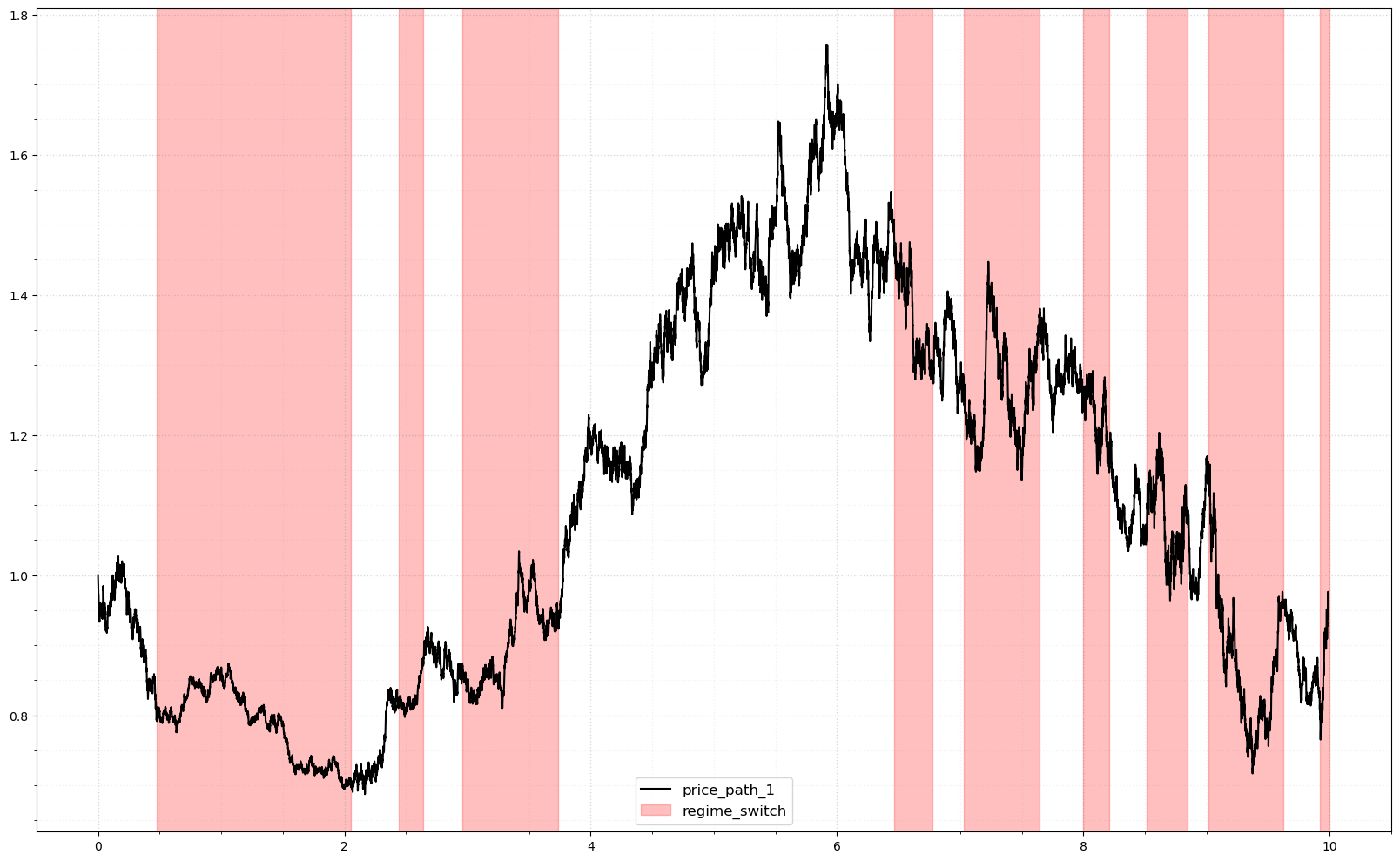}
        \caption{Price path.}
        \label{fig:multiclasspath}
    \end{subfigure}%
    \begin{subfigure}{0.5\linewidth}
        \centering
        \includegraphics[width=\textwidth]{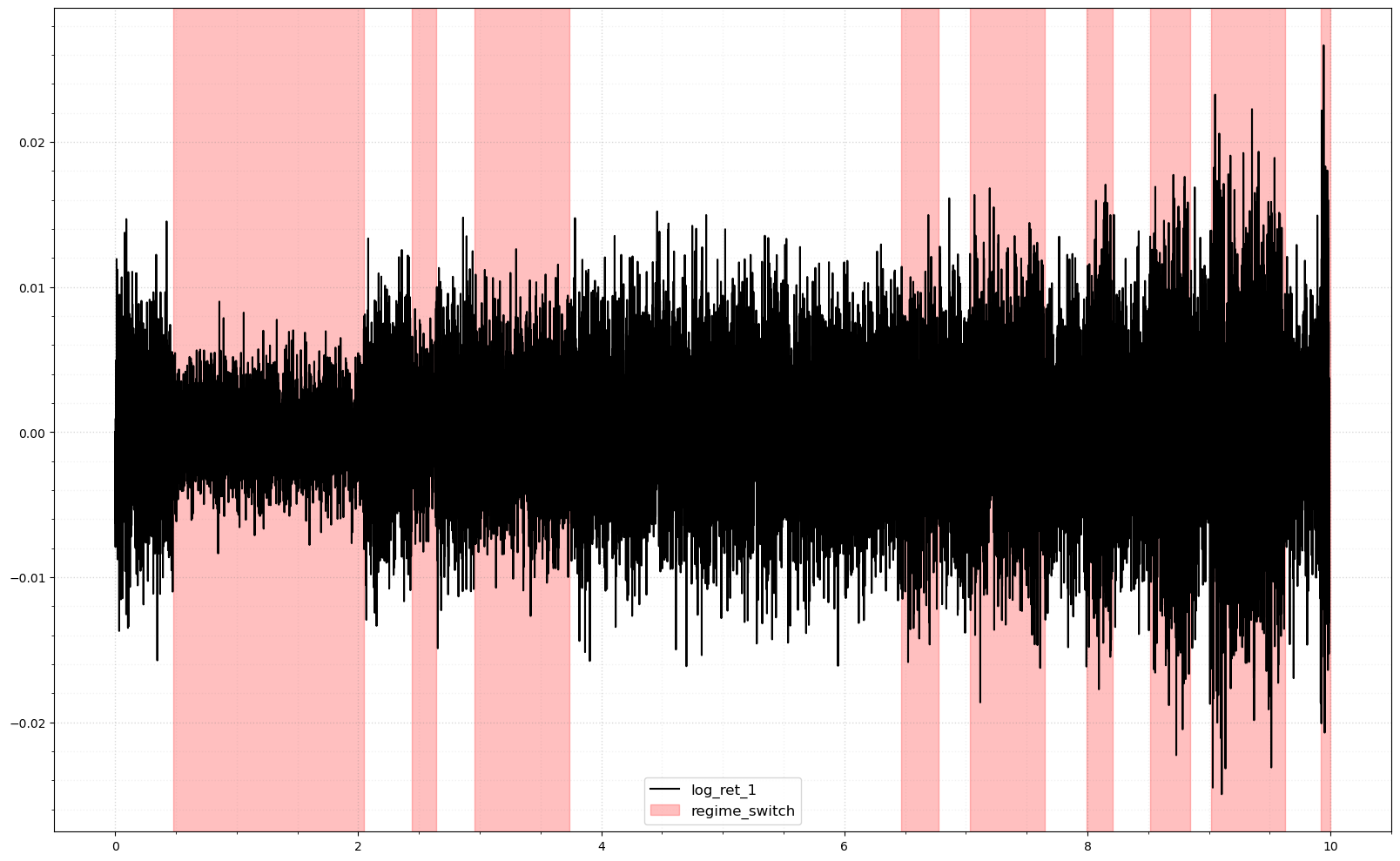}
        \caption{Log returns.}
        \label{fig:multiclassrets}
    \end{subfigure}
    \caption{Sample regime changed path, multiclass example. Each successive regime change (given by the shaded region in red) results in increasingly volatile path values, which is displayed clearly on the log returns plot.}
    \label{fig:multiclassregimechanged}
\end{figure}

We choose the hyperparameter vector $h = (8, 16)$ and instantiate our metric as the rank $1$ signature kernel MMD with the linear kernel, as per Definition \ref{def:rankrmmd}. We use the same path transformations as per the previous section and calculate the score matrix from eq. (\ref{eqn:scorematrix}) from the set of transformed ensemble paths relative to the generated regime-change path. 

\begin{figure}[ht]
    \centering
    \begin{subfigure}{0.33\linewidth}
        \centering
        \includegraphics[width=\textwidth]{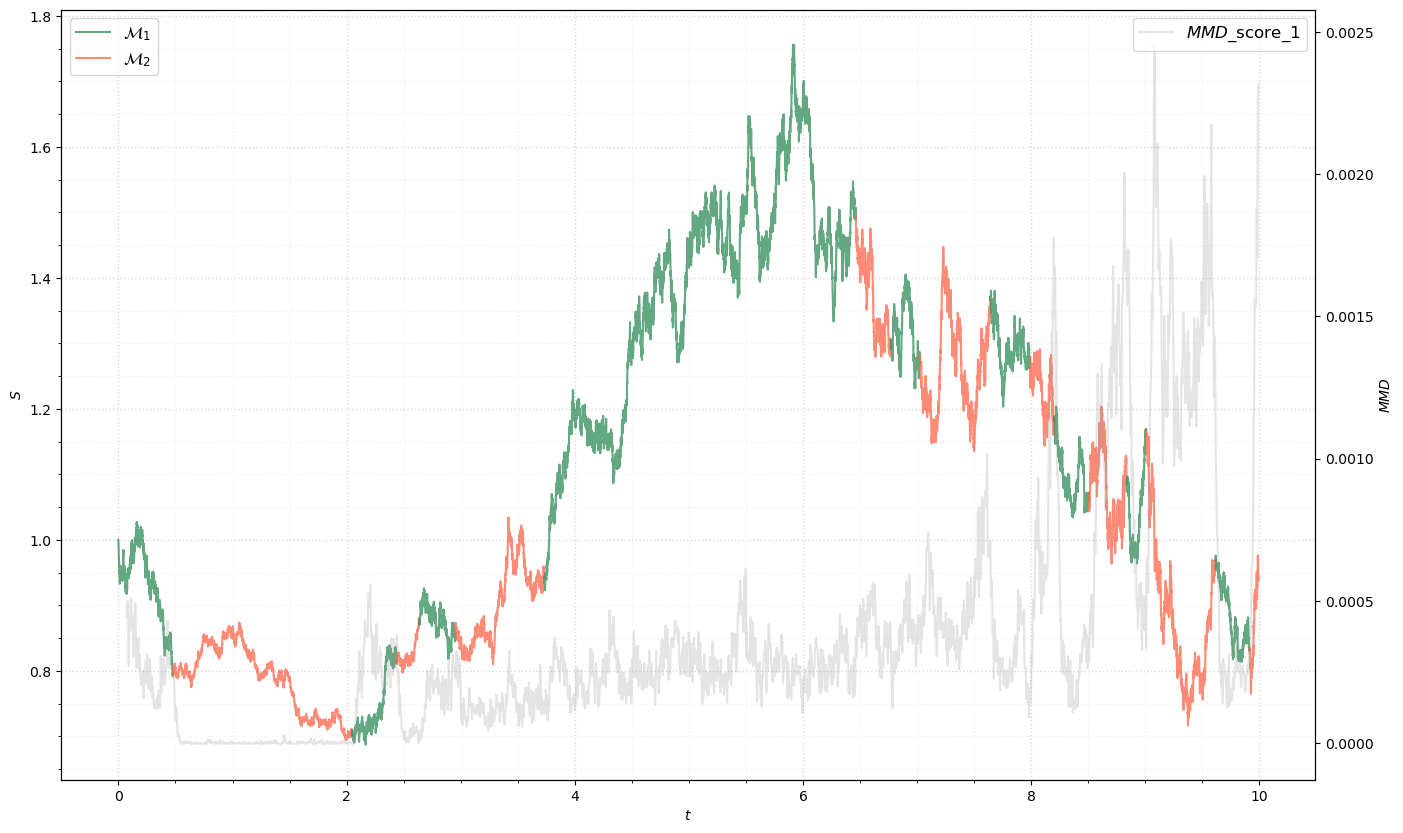}
        \caption{$\mathfrak{P}_1: \sigma=0.1$.}
        \label{fig:multiclassscore1}
    \end{subfigure}%
    \begin{subfigure}{0.33\linewidth}
        \centering
        \includegraphics[width=\textwidth]{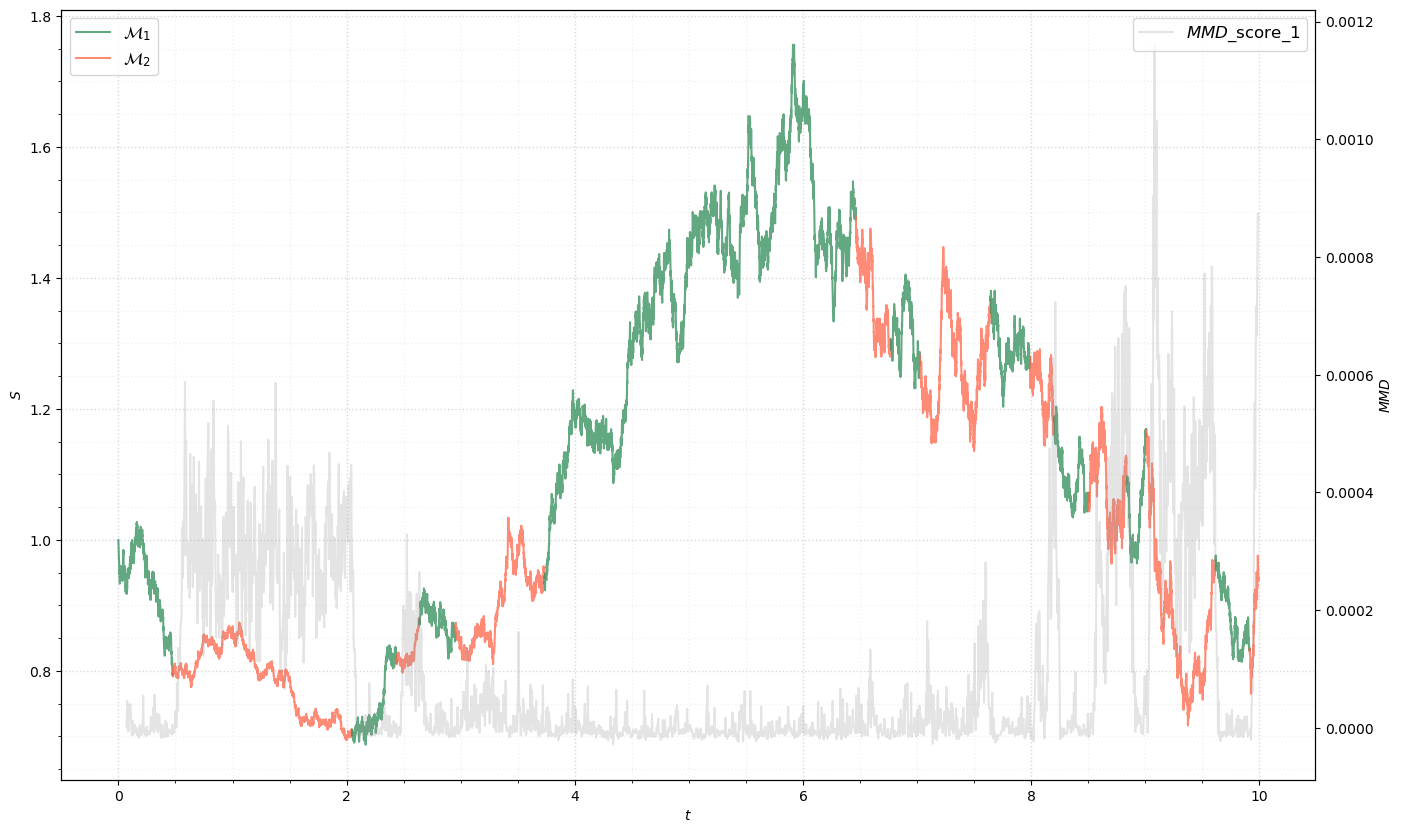}
        \caption{$\mathfrak{P}_2: \sigma =0.2$.}
        \label{fig:multiclassscore2}
    \end{subfigure}
    \begin{subfigure}{0.33\linewidth}
        \centering
        \includegraphics[width=\textwidth]{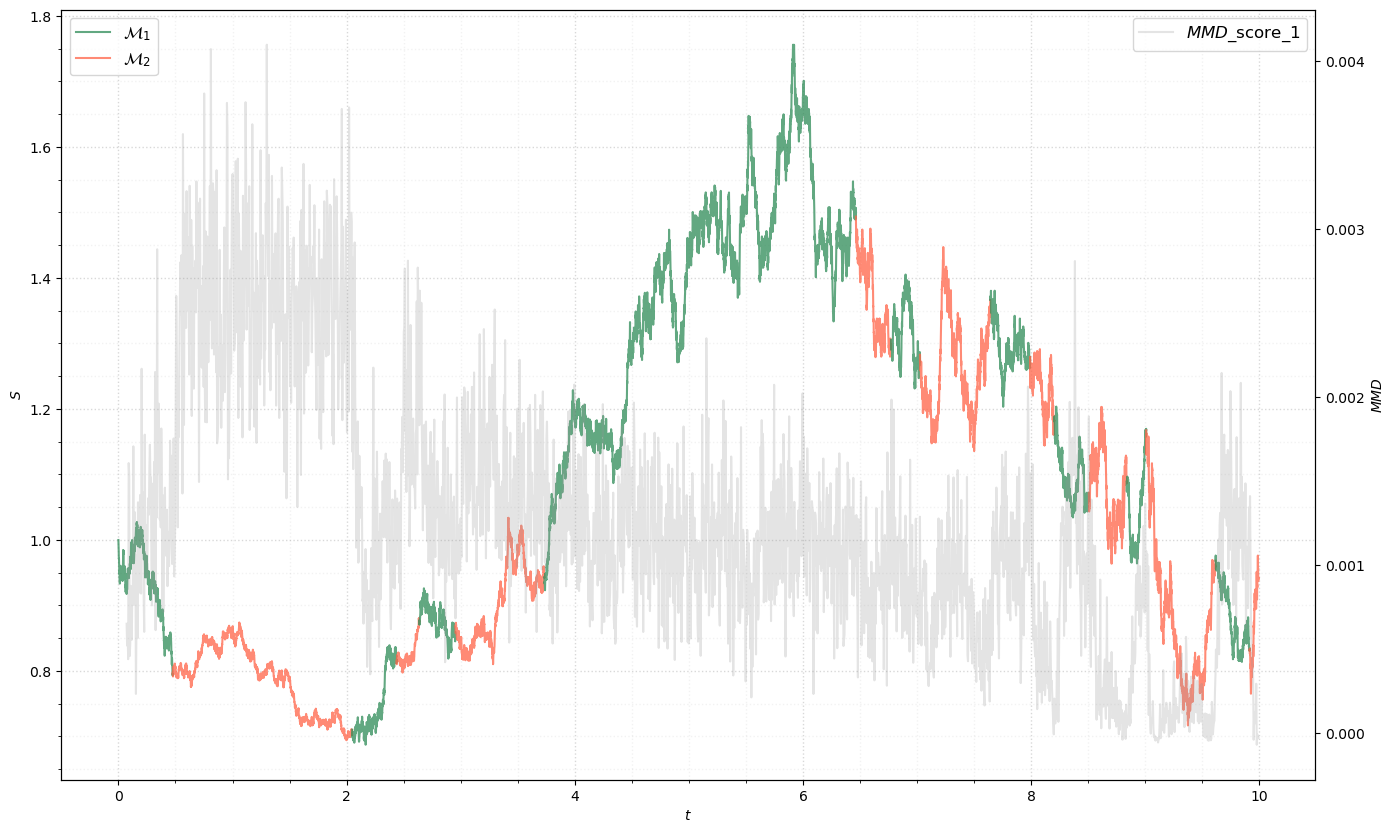}
        \caption{$\mathfrak{P}_3 : \sigma=0.4$.}
        \label{fig:multiclassscore3}
    \end{subfigure}
    \caption{MMD scores associated to each belief, gBm.}
    \label{fig:multiclassmmdscores}
\end{figure}

The MMD scores evolve as one expects: in the lowest volatility case, they increase during each successive regime change, and the reverse is true for the belief corresponding to the highest volatility. The middle ``base case'' belief exhibits a U-shaped MMD score as one would expect. This illustrates how holding multiple beliefs can be useful in determining directional change in regimes. According to the first two beliefs, the latter intervals constitute a regime change, though it may be initially unclear as to what that regime change represents. By including a high-volatility belief, we can see that the associated score vector during these latter intervals is low, characteristic the corresponding regime.

%% file: section4/43non_ensemble.tex
\subsection{Single path (non-ensemble) evaluation}\label{subsec:toysinglepath}

In this section we present a toy example where online regime detection can be performed on a path-by-path basis, as opposed to ensembles of paths against ensembles of paths. This may be required in instances where data of the desired structure is scarce, or a more sensitive detection algorithm is required.

In our framework, for a given path $\hat{\mathsf{s}} \in \mathcal{T}_\Delta([0, T]; \mathbb{R}^d)$, single path evaluation requires working successively on the space of sub-paths $\mathcal{SP}(\hat{\mathsf{s}})$ as opposed to $\mathcal{EP}_h(\hat{\mathsf{s}})$. Thus we require a function that operates directly on $\mathcal{SP}_h(\hat{\mathsf{s}})$. This necessitates the use of the similarity score $\Sigma^{\mathbb{P}, \mathbb{Q}}$ (cf Def. \ref{def:conformancescore}) where $\mathbb{P}, \mathbb{Q}$ are distributions on path space. Recall that $\Sigma^{\mathbb{P}, \mathbb{Q}}$ gives a measure of conformance a path $s \in \mathcal{SP}_h(\hat{\mathsf{s}})$ has to either $\mathbb{P}$ or $\mathbb{Q}$. Negative values indicate greater conformance to $\mathbb{P}$; positive values indicate greater conformance to $\mathbb{Q}$.

\begin{figure}[h]
    \centering
    \begin{subfigure}{0.5\linewidth}
        \centering
        \includegraphics[width=\textwidth]{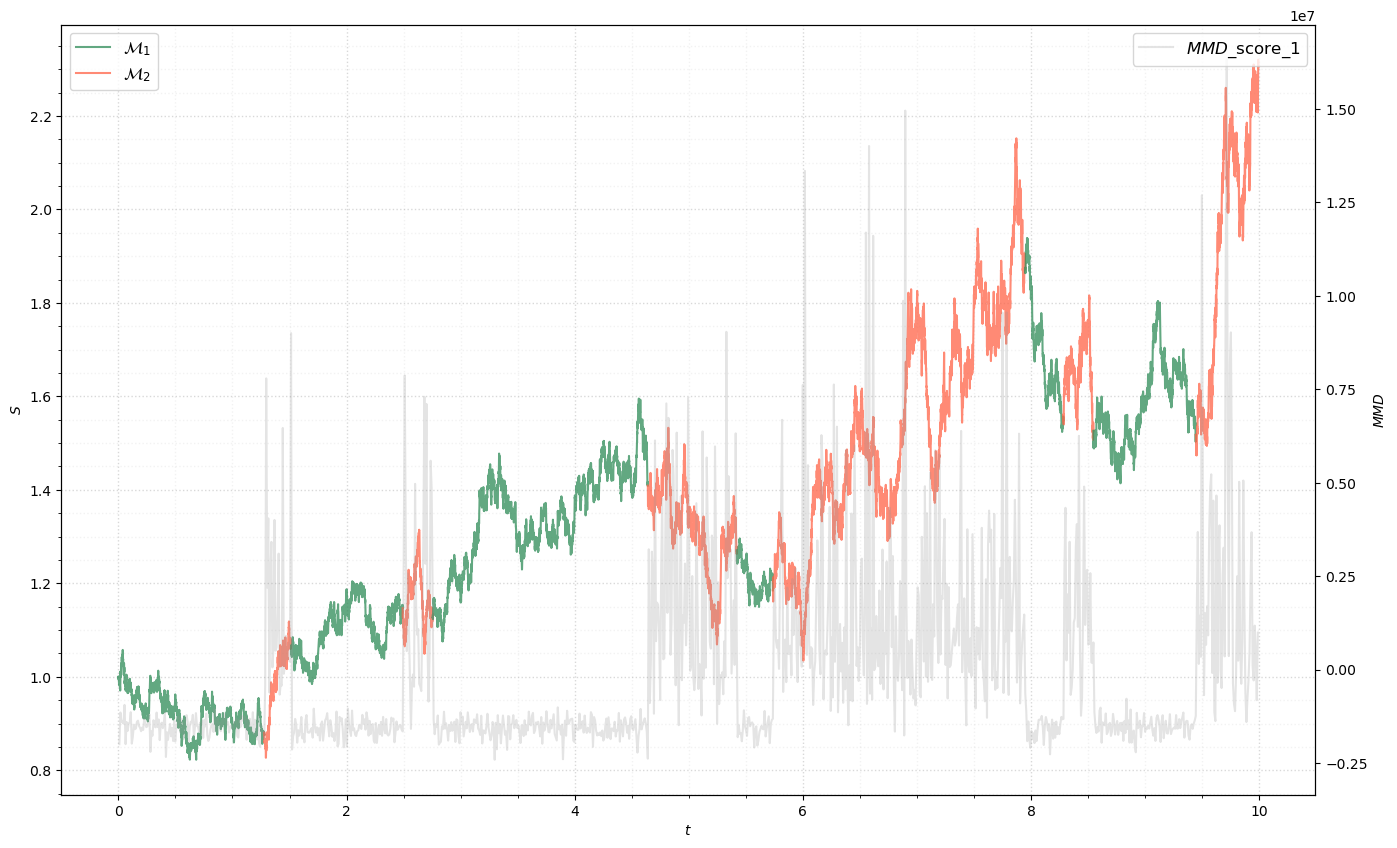}
        \caption{Value of $\Sigma^{\mathfrak{P}_1, \mathfrak{P}_2}(s)$.}
        \label{fig:scoringvals}
    \end{subfigure}%
    \begin{subfigure}{0.5\linewidth}
        \centering
        \includegraphics[scale=0.2275]{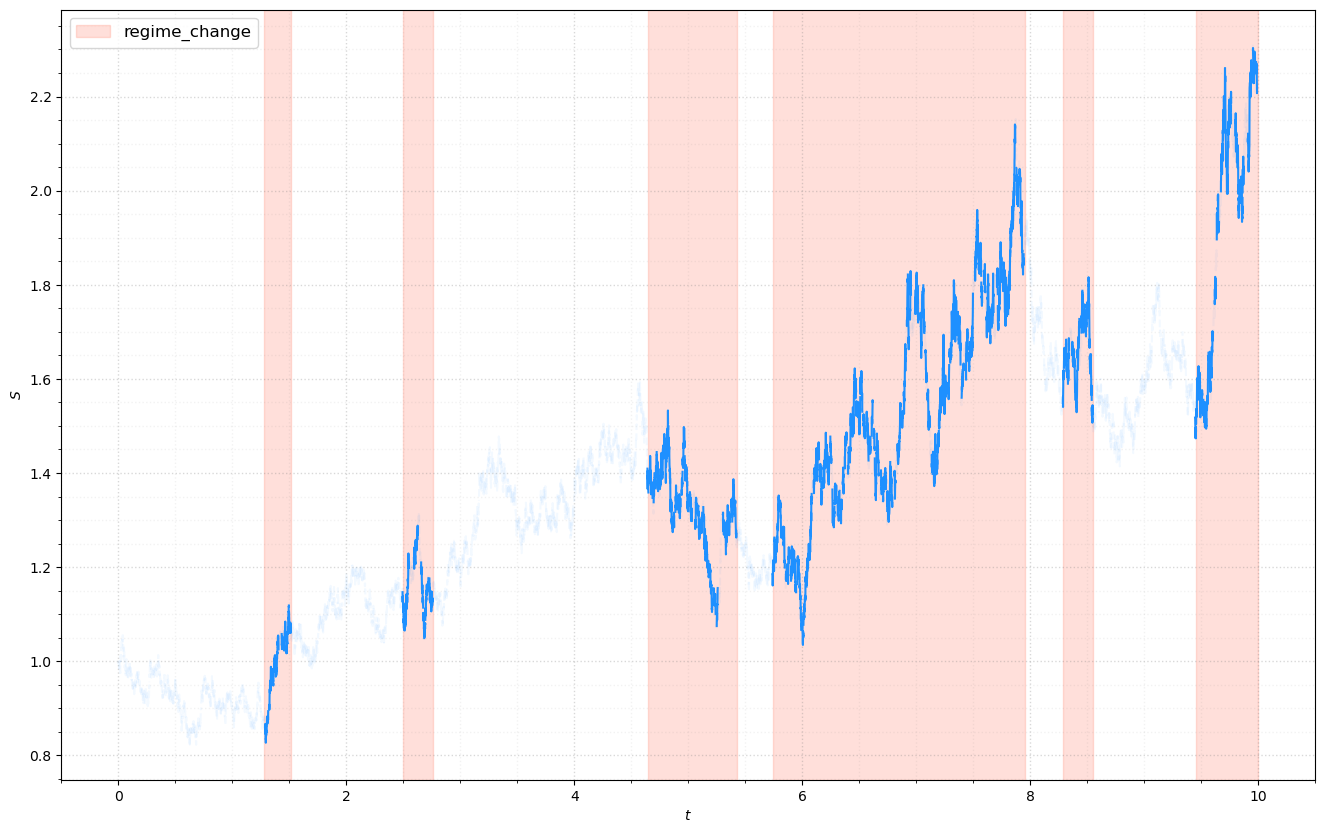}
        \caption{$\Sigma^{\mathfrak{P}_1, \mathfrak{P}_2}(s) \ge 0$ plot.}
        \label{fig:scoringthresh}
    \end{subfigure}
    \caption{Regime detection with conformance scoring function, toy example.}
    \label{fig:scoringexample}
\end{figure}

In what follows we use the same experimental setup as Subsection \ref{subsec:toyexample}, including the same choice of hyperparameters for division and transformation of an asset price path $\hat{\mathsf{s}} \in \mathcal{T}_\Delta([0, T]; \mathbb{R})$ which is again a regime-switching geometric Brownian motion. We set our beliefs to be $\mathfrak{P} = (\mathfrak{P}_1, \mathfrak{P}_2)$ where the generating model for each belief is a geometric Brownian motion with parameters $(\mu_1, \sigma_1) = (0, 0.2)$ and $(\mu_2, \sigma_2) = (0, 0.3)$. For evaluation, we now calculate the similarity score matrix $\Sigma^{\mathfrak{P}}(s)$ from eq. (\ref{eqn:conformancematrix}) associated to each sub-path extracted from the regime change path. We sample 64 paths from each belief to calculate the unbiased estimators of the constituent scoring rules.

In Figure \ref{fig:scoringexample}, we see that the similarity scoring function behaves as expected: when we are in the first regime, the sample score is negative, indicating proximity to the first belief over the second. When the regime changes, the sample score increases, indicating closer conformance to the second belief pair. \\
We note that this example is highly idealized as our beliefs again perfectly reflect the two types of regimes observed. We will see in Section \ref{sec:realdata} that this not need be true in order to utilise this approach for regime detection with a real data example.

%% file: section4/44rough_higher_dim.tex
\subsection{Rough and higher-dimensional processes}\label{subsec:rbergomi}

The general detection problem associated to the toy example presented in Section \ref{subsec:toyexample} can be solved with more rudimentary tools that those proposed by us. This section showcases that the method can be used in more realistic (synthetic) models which include autocorrelative effects such as the rough Bergomi model. It also showcases that the the method can be used both for regime detection and market regime clustering in high-dimensional examples (dimension $d=20$).\\
Here, we give an example where the underlying model is given by the strong solution to the SDE 
\begin{equation}\label{eqn:RoughBergomi}
    dX_t = -\frac{1}{2}V_t dt + \sqrt{V_t} dW_t \quad \text{where} \quad d\xi_t^u = \xi_t^u \eta \sqrt{2\alpha + 1}(u-t)^\alpha dB_t, \quad X_0 = 1.
\end{equation}
Eq. (\ref{eqn:RoughBergomi}) was introduced in \cite{bayer2016pricing} and is called the \emph{rough Bergomi (rBergomi) model}. It is parametrized by the vector $\theta = (\xi_0, \nu, \rho, H)$, where $\xi_0$ is the initial (forward) variance, $\nu$ controls the level of the at-the-money (ATM) skew of the associated volatility surface, $\rho$ the correlation between volatility and price moves, and $H$ the Hurst exponent governing the roughness of the volatility process. 

In the example that follows, regime-switching dynamics are given by  
\begin{equation*}
    \mathbb{M}_1 = \left(\text{rBergomi}, \theta_1\right), \text{ and }\mathbb{M}_2 = \left(\text{rBergomi}, \theta_2\right),
\end{equation*}
with
\begin{equation*}
    \theta_1 = (0.03, 0.5, -0.7, 0.4), \quad \theta_2 = (0.03, 0.5, -0.7, 0.3).
\end{equation*}
In this way we seek to detect a change in the Hurst exponent. Again we set our beliefs to be one-class, so samples in $\mathfrak{P}$ are generated according to $\mathbb{P}_{\theta_1}$. Path samples drawn from a given model pair include both the stock $\mathsf{s}$ and volatility processes $\mathsf{v}$. Our regime-changed path is thus given by the pair $\hat{\mathsf{s}}_v = (t, \mathsf{s}, \mathsf{v}) \in \mathcal{T}_\Delta([0, T]; \mathbb{R}^{2d})$ where $d \in \mathbb{N}$ is the dimensionality of the stock price process. Initially we take $d=1$. Again we define the mesh so that each time step roughly represents an hourly return. 

\begin{remark}[Observation of volatility process]
    In practice $\mathsf{v}$ is not directly observable; however, there exist proxies which can be used instead, see for instance the Oxford Man Institute's realised volatility library\footnote{The original library, at time of publication, seems to have been discontinued, although it was available at the time of writing.}.
\end{remark}

The regime-change dynamics are again given as in Section \ref{subsec:toyexample}: our regime change signal is given by $Z_1 \sim \mathrm{Po}(2)$ and our exit simple by $Z_2 \sim \mathrm{Po}(1/30)$. We give here an example of such a path in Figure \ref{fig:rbergomi1dpath}.

\begin{figure}[h]
    \centering
    \begin{subfigure}{0.5\linewidth}
        \centering
        \includegraphics[width=\textwidth]{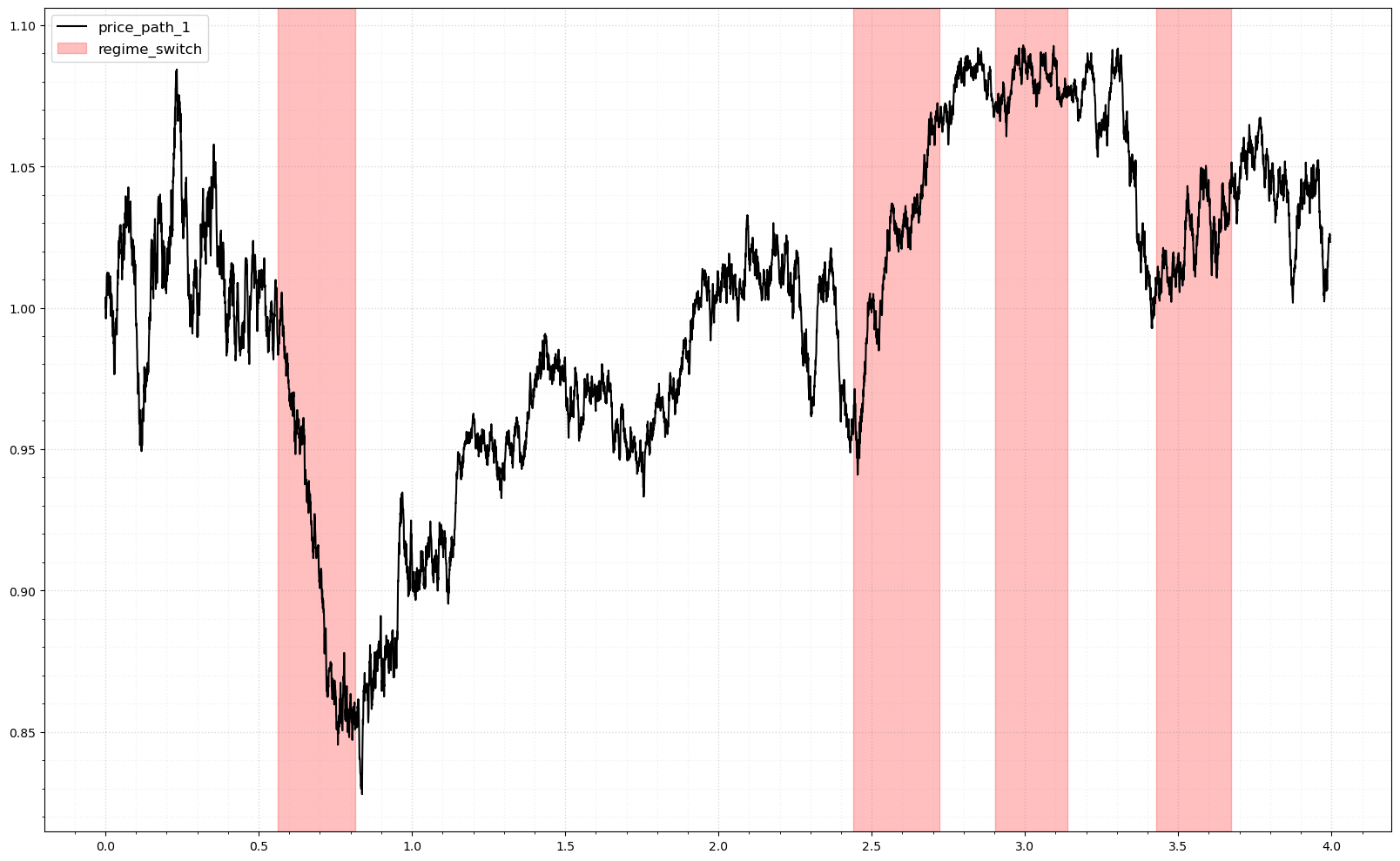}
        \caption{Price path.}
        \label{fig:rbergomi1dpath}
    \end{subfigure}%
    \begin{subfigure}{0.5\linewidth}
        \centering
        \includegraphics[width=\textwidth]{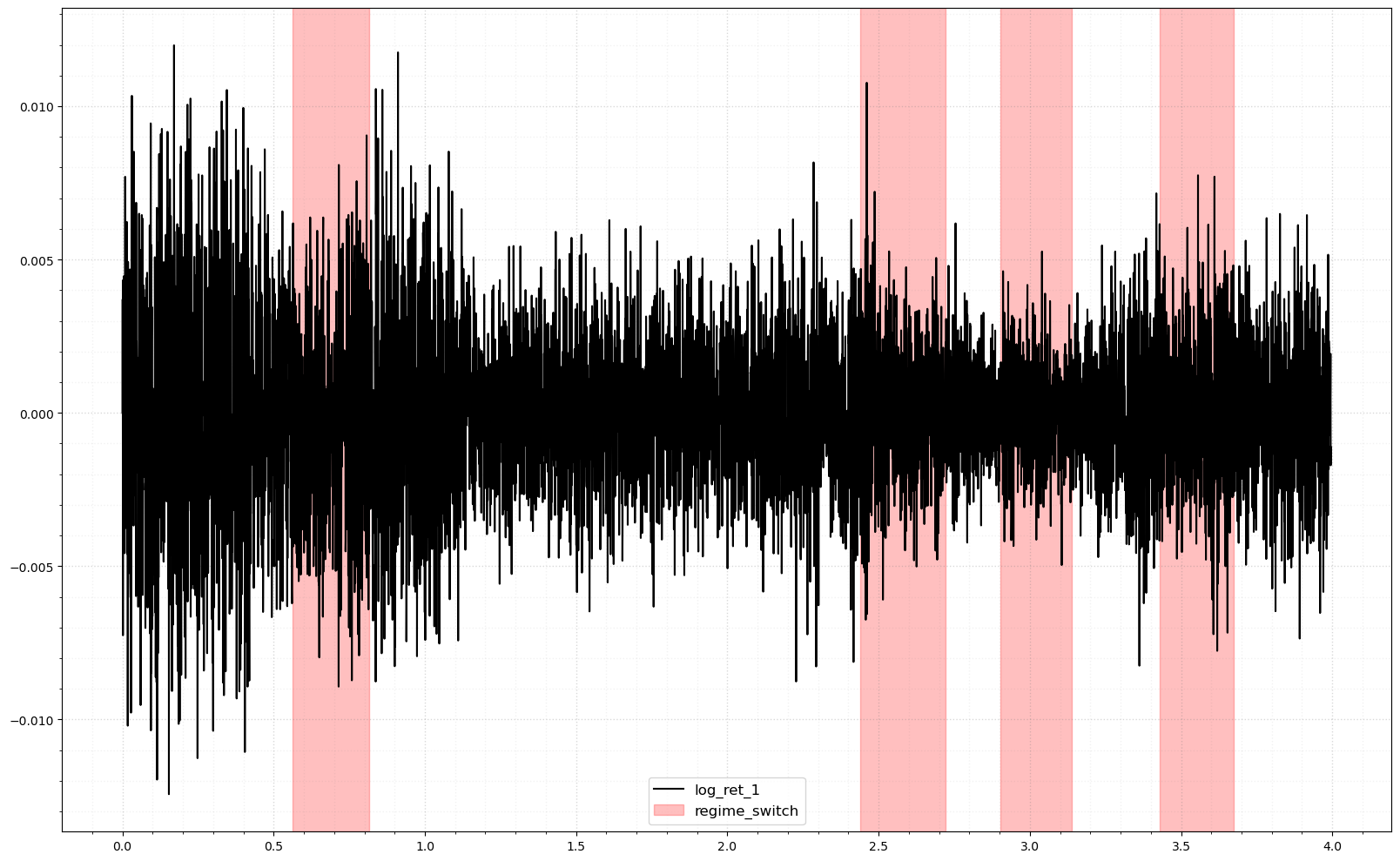}
        \caption{Log returns.}
        \label{fig:rbergomi1dret}
    \end{subfigure}
    \caption{Sample regime changed path, rBergomi model.}
    \label{fig:rbergomipathandret}
\end{figure}

We process a given regime-changed path $\hat{\mathsf{s}}_v$ by choosing $h=(16, 8)$ in order to generate $\mathcal{SP}_h(\hat{\mathsf{s}}_v)$. We take the path transformer $\Phi = \phi_{\mathrm{norm}} \circ \phi_{\mathrm{time}}$ to obtain $\mathcal{SP}_h^\Phi(\hat{\mathsf{s}}_v)$ and thus $\mathcal{EP}^\Phi_h(\hat{\mathsf{s}}_v)$. As filtration information is relevant, we choose our evaluation metric to be the rank $1$ MMD $\mathcal{D}^1_{\text{sig}}$ with associated RBF kernel smoothing hyperparameter $\sigma_1 = 0.025$. A sample result is given in Figure \ref{fig:rbergomi1dresults}, where one can visually identify the MMD score increasing during periods of regime change and decreasing when the path returns to conforming to our specified beliefs. The threshold plot confirms our visual intuition, with the MMD score being greater than the prior threshold more often during periods of regime change. 

\begin{figure}[h]
    \centering
    \begin{subfigure}{0.5\linewidth}
        \centering
        \includegraphics[width=\textwidth]{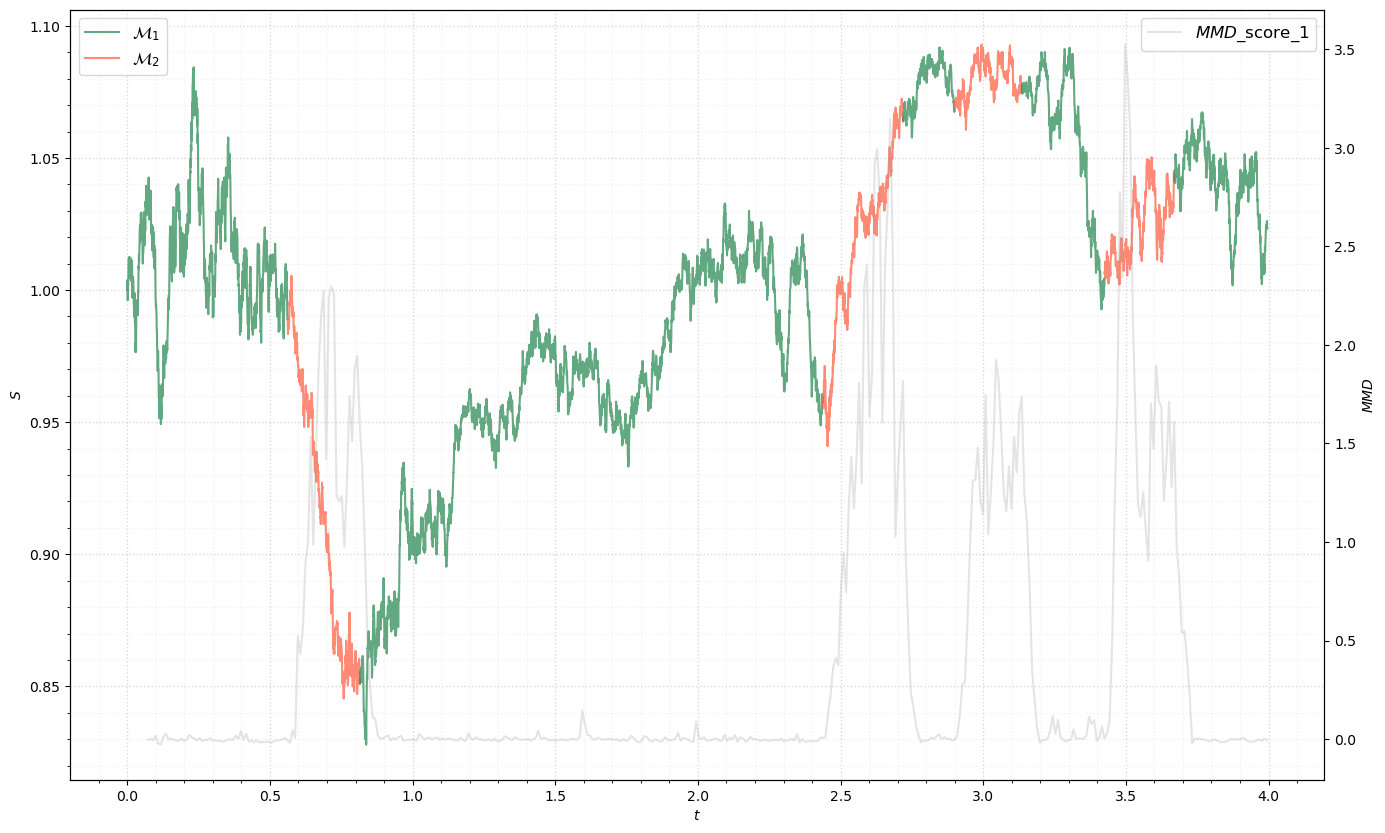}
        \caption{$\mathcal{D}_\text{sig}^1$ score.}
        \label{fig:mmdscorerbergomi1d}
    \end{subfigure}%
    \begin{subfigure}{0.5\linewidth}
        \centering
        \includegraphics[scale=0.2275]{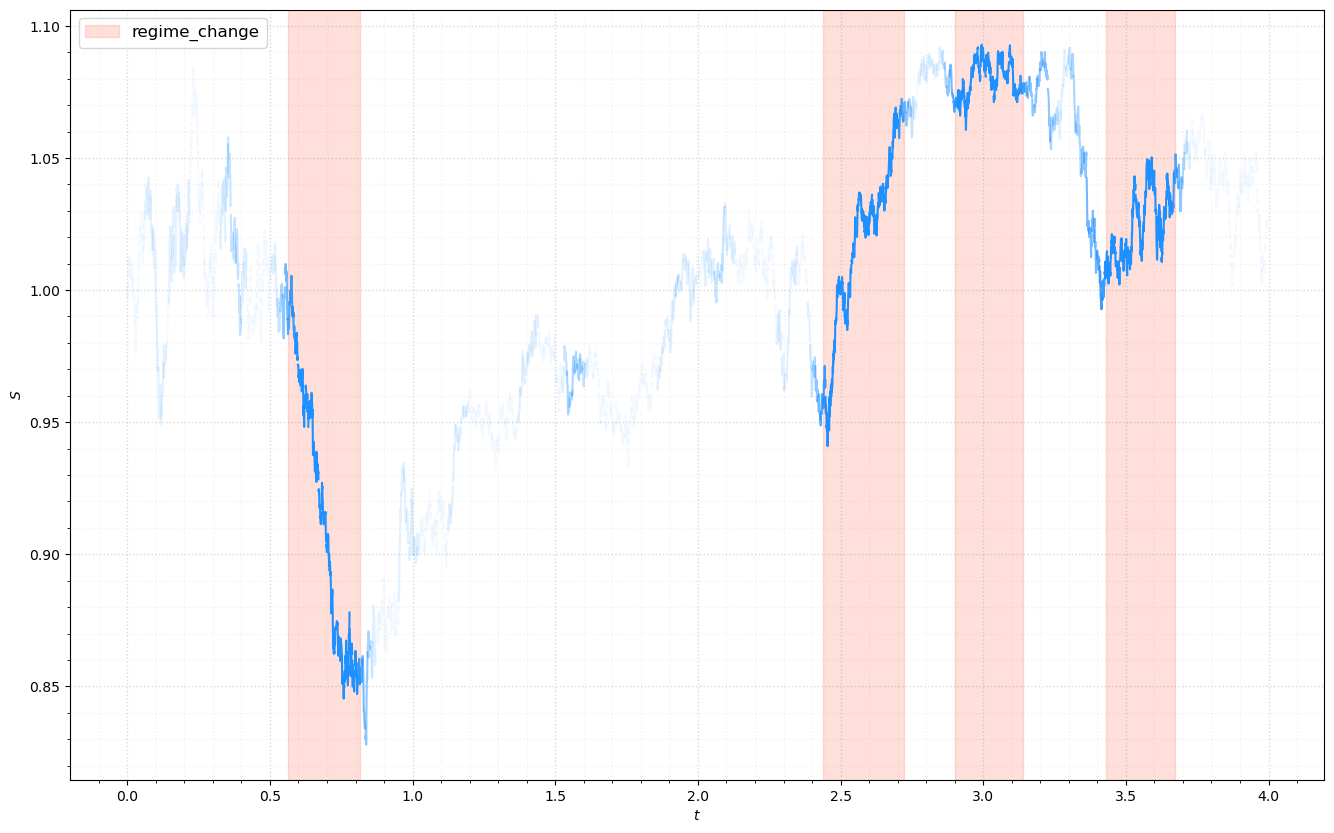}
        \caption{$\mathcal{D}_\text{sig}^1 > c_\alpha$ plot.}
        \label{fig:alphascorerbergomi1d}
    \end{subfigure}
    \caption{MMD score and accompanying threshold plot, rBergomi model, $d=1$.}
    \label{fig:rbergomi1dresults}
\end{figure}

As outlined in Section \ref{subsec:mmd}, the cost of calculating the MMD under the signature kernel $k_{\text{sig}}$ from (\ref{eqn:signaturekernel}) is linear in the state-space dimension of $\hat{\mathsf{s}}_v$. Thus it is computationally feasible to perform the same experiment as above with rBergomi paths of considerably higher dimensionality. In this way we leverage the fact that we can both perform higher-dimensional MMD computations with ease. In what follows we set $d=10$ so our regime-changed path $\hat{\mathsf{s}}_v \in \mathcal{T}_\Delta([0, T], \mathbb{R}^{20})$ is comprised of $10$ rBergomi stock price processes with their accompanying volatility processes. Each stock-volatility pair $(\mathsf{s}^i, \mathsf{v}^i), i=1,\dots,10$ is initially distributed according to $\theta_1$ and is suitably regime-changed as outlined at the beginning of the section. Again we choose $h=(16, 8)$ and our metric is given by $\mathcal{D}_\text{sig}^1$, with $\sigma_1 = 0.025$.
 
\begin{figure}[h]
    \centering
    \begin{subfigure}{0.5\linewidth}
        \centering
        \includegraphics[width=\textwidth]{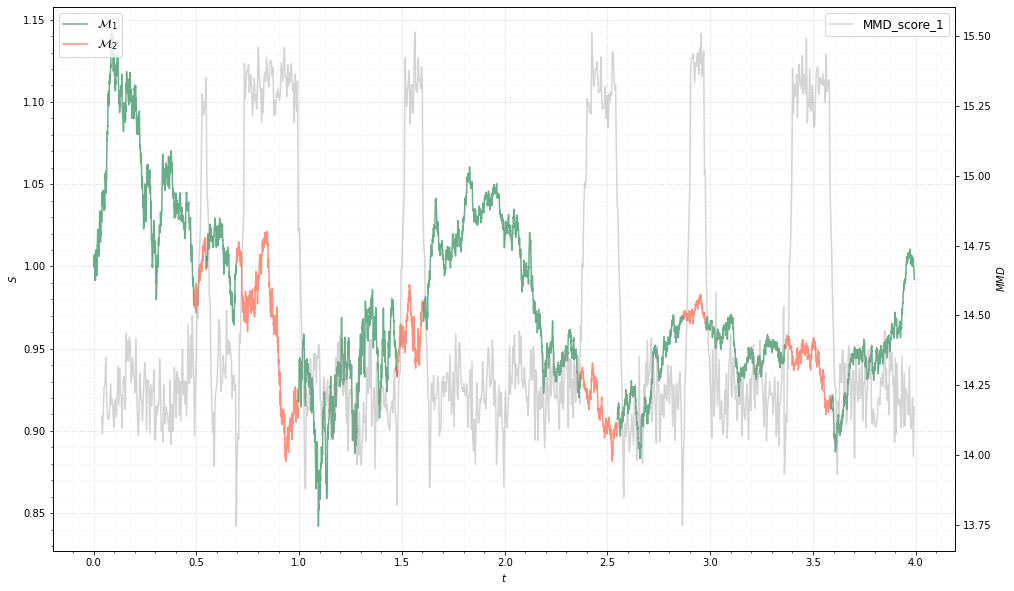}
        \caption{$\mathcal{D}_\text{sig}^1$ score.}
        \label{fig:mmdscorerbergomi10d}
    \end{subfigure}%
    \begin{subfigure}{0.5\linewidth}
        \centering
        \includegraphics[scale=0.2275]{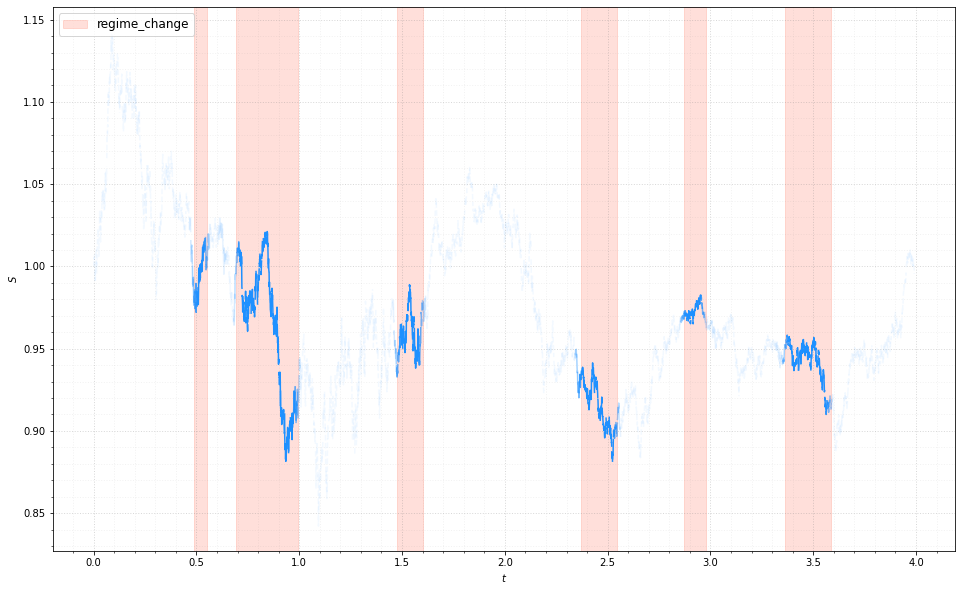}
        \caption{$\mathcal{D}_\text{sig}^1 > c_\alpha$ plot.}
        \label{fig:alphascorerbergomi10d}
    \end{subfigure}
    \caption{MMD score and accompanying threshold plot, rBergomi model, $d=10$.}
    \label{fig:rbergomi10dresults}
\end{figure}

In Figure \ref{fig:rbergomi10dresults}, we give a sample stock price process $\mathsf{s}^1$ from the regime-changed path $(\mathsf{s}, \mathsf{v})$. We again display the running MMD score and the threshold plot. As one can see from Figure \ref{fig:mmdscorerbergomi10d}, the addition of more paths does indeed have the effect of making the regime change signals stronger than compared to those seen in Figure \ref{fig:mmdscorerbergomi1d}. This is verified by the threshold plot seen in Figure \ref{fig:alphascorerbergomi10d} whereby periods of regime change are isolated and identified much more clearly than those in Figure \ref{fig:alphascorerbergomi1d}. 

%% file: section4/45mmd2_v_mmd1.tex
\subsection{The higher rank MMD}\label{subsec:higherrankdetection}

As given in Theorems \ref{theroem:mmdmetric} and \ref{thm:rankmarkov} in the Appendix, we have seen from a simple synthetic example with geometric Brownian motions that the rank 1 MMD is capable of capturing differences between Markovian processes. Due to Theorem 2 in \cite{salvi2021higher}, we know the rank 2 MMD is capable of distinguishing between such sets of paths as well. From Subsection \ref{subsec:rbergomi}, we have also seen that the rank 1 MMD is also able to distinguish between non-Markovian processes $X, Y$ if their marginal distributions are sufficiently different. However, if differences between the marginal distributions of $X$ and $Y$ are more subtle, and if there is some information contained within $\mathcal{F}_X$ or $\mathcal{F}_Y$ (or both), then one would expect the regime detector equipped with the rank 2 MMD to outperform that with the rank 1 MMD.

To illustrate this, we provide an experiment where a regime-switching detection problem with prior beliefs is first performed using the rank-1 MMD, and then with the rank-2 MMD associated to the rank-2 signature mapping from Definition \ref{def:rank2signature}. Here our two cycling measures are given by a geometric Brownian motion with parameters $\theta_1 = (0, 0.2)$ and an rBergomi model with parameters $\theta_2 = (0.1, 0.1, -0.7, 0.3)$. We set our beliefs in both settings to that given by the geometric Brownian motion. The hyperparameters for the rank-1 MMD detector was chosen to be $\sigma = 0.5$, where for the rank-2 MMD detector (MMD2-DET), we set $\sigma_1 = 0.5, \sigma_2 = 1.$. These parameters were determined via bootstrapping. Null distributions $\mathfrak{D}^1, \mathfrak{D}^2$ were obtained by further bootstrapping with $(h_1, h_2) = (21, 10)$. Regarding the regime-changed path: we set the grid mesh to mirror daily returns, so $|\Delta| = 1/252$. We simulated a path $X \in \mathcal{T}_\Delta([0, T]; \mathbb{R}^d)$, where here we took $d=5$ and $T=20$. Regime changes were given by $Z_1 \sim \text{Po}(0.25)$ and exits given by $Z_2 \sim \text{Po}(5)$. 

\begin{figure}[h]
    \centering
    \begin{subfigure}[t]{0.4\linewidth}
        \centering
        \includegraphics[width=\textwidth]{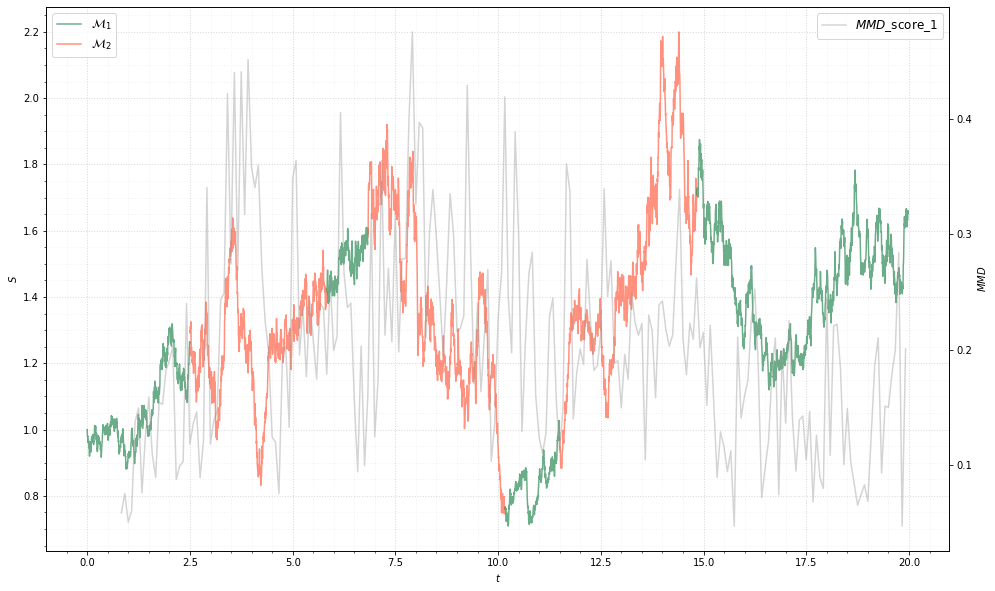}
        \caption{$\mathcal{D}^1_\text{sig}$ scores.}
        \label{fig:mmd0scores}
    \end{subfigure}
    \begin{subfigure}[t]{0.4\linewidth}
        \centering
        \includegraphics[width=\textwidth]{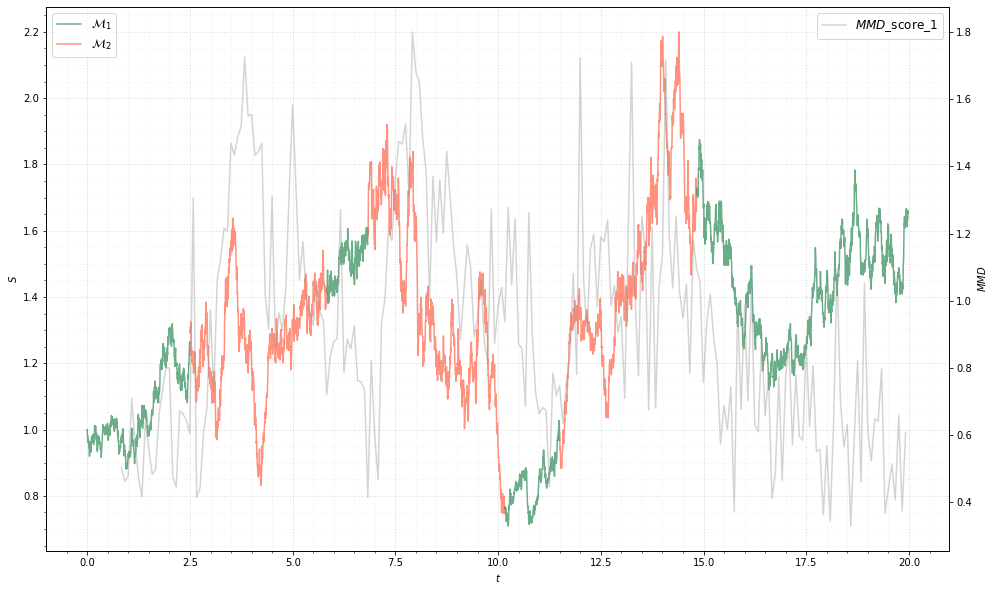}
        \caption{$\mathcal{D}^2_\text{sig}$ scores.}
        \label{fig:mmd1scores}
    \end{subfigure}
    \medskip
    \begin{subfigure}[t]{0.4\linewidth}
        \centering
        \includegraphics[width=\textwidth]{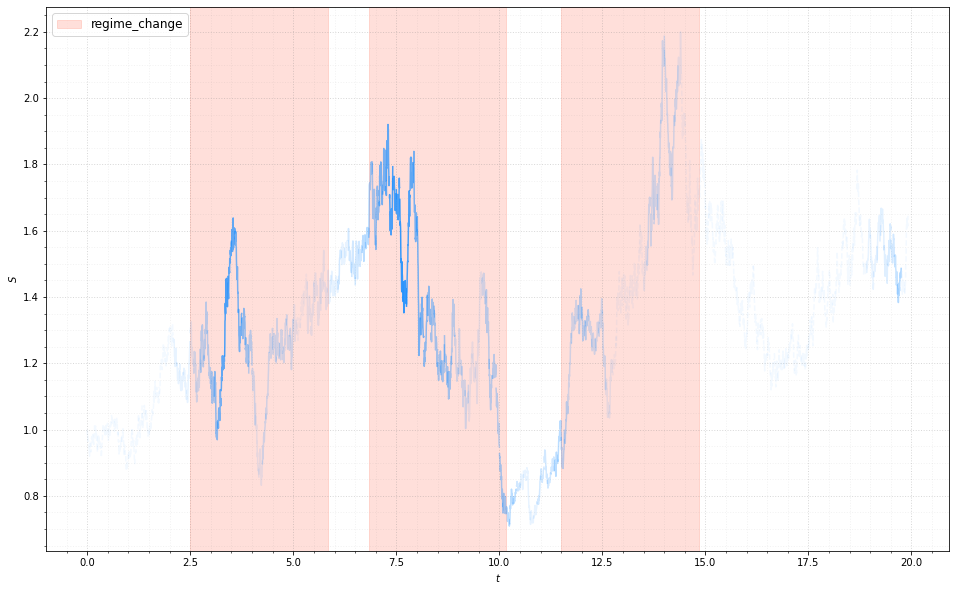}
        \caption{$\mathcal{D}^1_\text{sig} > c_\alpha$.}
        \label{fig:mmd0alphas}
    \end{subfigure}
    \begin{subfigure}[t]{0.4\linewidth}
        \centering
        \includegraphics[width=\textwidth]{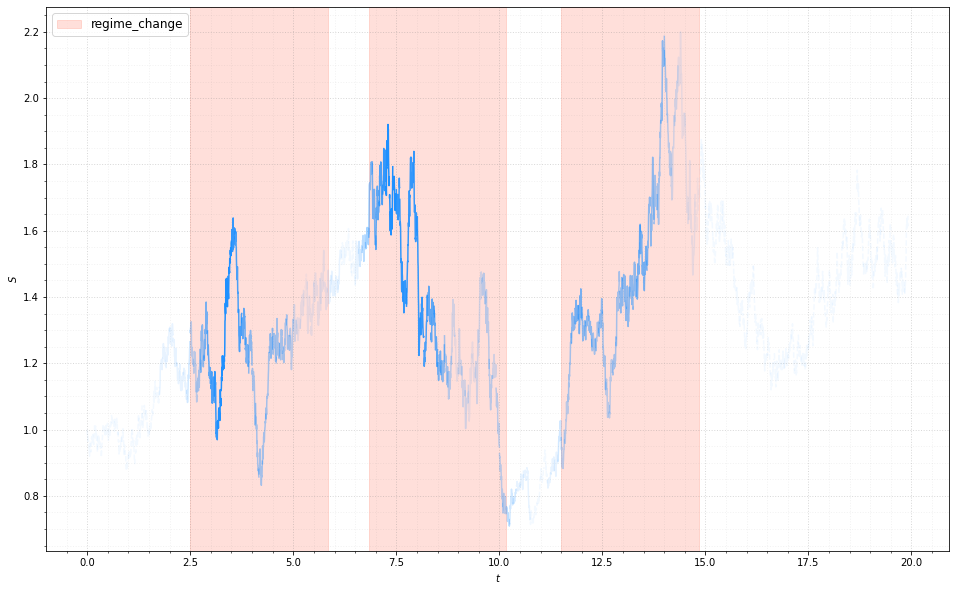}
        \caption{$\mathcal{D}^2_\text{sig} > c_\alpha$.}
        \label{fig:mmd1alphas}
    \end{subfigure}
    \caption{MMD scores and associated critical threshold plots, $\mathcal{D}_{\text{sig}}^1$ vs $\mathcal{D}_{\text{sig}}^2$. On average the online detector using $\mathcal{D}_{\text{sig}}^2 $ performs better; notice from the sample run here that the final regime change is better detected by the rank-2 detector.}
    \label{fig:mmd1mmd2results}
\end{figure}

Figure \ref{fig:mmd1mmd2results} gives an example run of both detectors on the given switching problem. Although both are somewhat able to distinguish each regime, it is clear that the MMD2-DET is better able to identify the regime changes and their subsequent cessations. Again we provide the score plot and the $(1-\alpha)\%$ threshold plots. Finally, Table \ref{tab:mmd0mmd1} gives a summary of how each detector performed over $n=100$ runs. We directly report the accuracy, which is measured in the following way: for each path $s \in \mathcal{SP}_h(\hat{\mathsf{s}})$, assign the label $0$ if $s$ was generated by $\mathbb{P}_{\theta_1}$, and $1$ if it was generated by $\mathbb{P}_{\theta_2}$. For each $s \in \mathcal{SP}_h(\hat{\mathsf{s}})$, the detector receives a score between $0$ and $1$, given by the percentage of the time $s$ belong to an ensemble which failed the two-sample test. We condition the performance in three categories: Regime on, during the standard regime; regime off, during the regime change, and finally the total accuracy. We also report ROC AUC scores (given we hav labelled data) and the total algorithm run time.

\begin{table}[ht]	
    \centering
    \footnotesize
    \begin{tabularx}{\textwidth}{cccccc}
        \toprule
        \textbf{Algorithm}                 &                        & Accuracy               &                        & ROC AUC             & Runtime            \\ 
        & Regime on              & Regime off             & Total                  &                     &                         \\ \midrule
        $\mathcal{D}^1_\text{sig}$                    & $45.5\pm 9.7$\% & $\boldsymbol{98.4\pm 1.3}$\% & $77.9\pm 5.5$\% & $0.934 \pm 0.043$ & $\boldsymbol{0.962\pm 0.007}$s \\ \addlinespace
        $\mathcal{D}^2_\text{sig}$                    & $\boldsymbol{62.9\pm 10.1}$\% & $97.4\pm 4.1$\% & $\boldsymbol{83.3\pm 6.6}$\% & $\boldsymbol{0.952 \pm 0.078}$ & $81.356\pm 0.439$s \\  \midrule
    \end{tabularx}
    \setlength\tabcolsep{4pt}
    \caption{Algorithm performances, MMD1-DET versus MMD2-DET, $n=100$ runs. MMD1-DET more often commits Type II error during periods of regime change, making MMD2-DET the better detector. However, it is nearly 100x slower than MMD1-DET.}
    \label{tab:mmd0mmd1}
\end{table}

%% file: section4/46comparisons.tex
\subsection{Comparison to existing techniques}\label{subsec:mrdpcomparisons}
In this section we showcase that our method can be used for anomaly detection in streamed data. We also compare the performance of our algorithm against a existing (signature-based) anomaly detection algorithms and find that it is competitive with existing methods, and even outperforms them in some case.

As mentioned in Subsection \ref{subsec:notation}, the MRDP is closely linked to the classical anomaly detection problem in time series data in the single belief setting. In this section we compare our technique to existing ones in the literature. The most relevant comparison to make is with the techniques given in \cite{cochrane2020anomaly} which also seeks to perform anomaly detection on path space. We also show how our approach compares to using the truncated signature kernel $k^N_{\text{sig}}$ and the associated MMD. We reserve the explanation of the \emph{signature conformance} (SIG-CON) method from \cite{cochrane2020anomaly} to Appendix \ref{appendix:conformance}.

We outline results of each of the algorithms on an anomaly detection problem. Again, stratified accuracy and ROC AUC scores are reported for both the MMD-DET and the two comparison methods (SIG-CON and the truncated MMD, which we call MMD-T). We note here that our method of assigning accuracy as outlined at the end of Subsection \ref{subsec:higherrankdetection} is not a perfect like-for-like comparison to the method provided in \cite{cochrane2020anomaly} as (overall) scores are not assigned ``on the fly'' - every ensemble that a given sub-path $s \in \mathcal{SP}_h(\hat{\mathsf{s}})$ was a part of must receive a score before the anomaly score associated to $s$ can be supplied. We argue that any benefits incurred from our experimental setup is an argument for studying path ensembles rather than individual paths in the context of anomaly detection; see \cite{zhang2021understanding} for more details.

The anomaly detection problem we study is again the toy example from Section \ref{subsec:toyexample}, except we study paths of dimensionality $d=5$. We set $\mathbb{P}_{\theta_1} = (\mathrm{gBm}, (0, 0.2))$ and $\mathbb{P}_{\theta_2} = (\mathrm{gBm}, (0, 0.3))$ and $\Delta = \{0 = t_0 < \dots < 4\}$ where $|\Delta| = 7 \times 252$. We denote the regime-changed path by $\hat{\mathsf{s}} \in \mathcal{T}_\Delta([0, T]; \mathbb{R}^5)$. Regarding hyperparameter selections: we set $h = (8, 16)$, so sub-paths contained $8$ observations. We used the rank $1$ signature for the detector associated to the signature kernel $k_\text{sig}$, and thus the associated metric was given by the rank 1 (RBF-lifted) MMD from eq. (\ref{eqn:rankrmmdrbf}). For MMD-T, we used the same parameters and truncated the signature mapping at $N=5$. For MMD-DET, we set the smoothing parameter $\sigma=0.0025$. Our beliefs are again distributed according to samples from $\mathbb{P}_{\theta_1}$, and we bootstrap the null distribution $\mathfrak{D}$ with $N_1 = 512$ samples. 

For the conformance-based method, we also choose the corpus of beliefs to be distributed according to $\mathbb{P}_{\theta_1}$. We truncated the signature at rank $N=2$. Larger levels of truncation were not able to be tested due to memory constraints (highlighting again the issue of directly evaluating the signature map). The null distribution is bootstrapped with $|\mathfrak{P}^1| = |\mathfrak{P}^2| = 1000$ paths. All methods had paths first transformed via the stream transformer $\Phi = \phi_{\text{norm}} \circ \phi_{\text{time}}$. 

\begin{figure}[ht]
    \centering
    \begin{subfigure}[t]{0.33\linewidth}
        \centering
        \includegraphics[width=\textwidth]{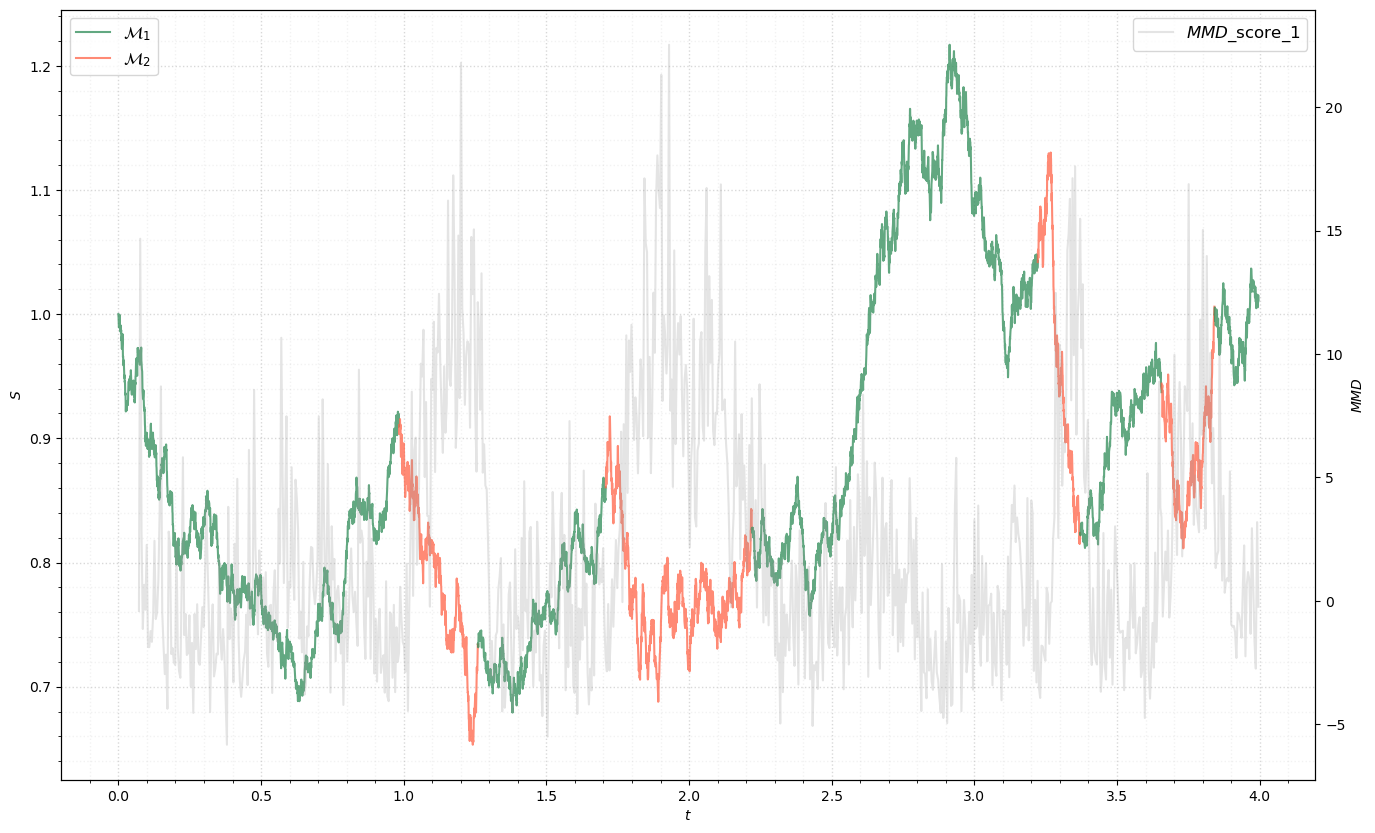}
        \caption{$\mathcal{D}^1_\text{sig}$ scores.}
        \label{fig:sigkeranomalyscores}
    \end{subfigure}%
    \begin{subfigure}[t]{0.33\linewidth}
        \centering
        \includegraphics[width=\textwidth]{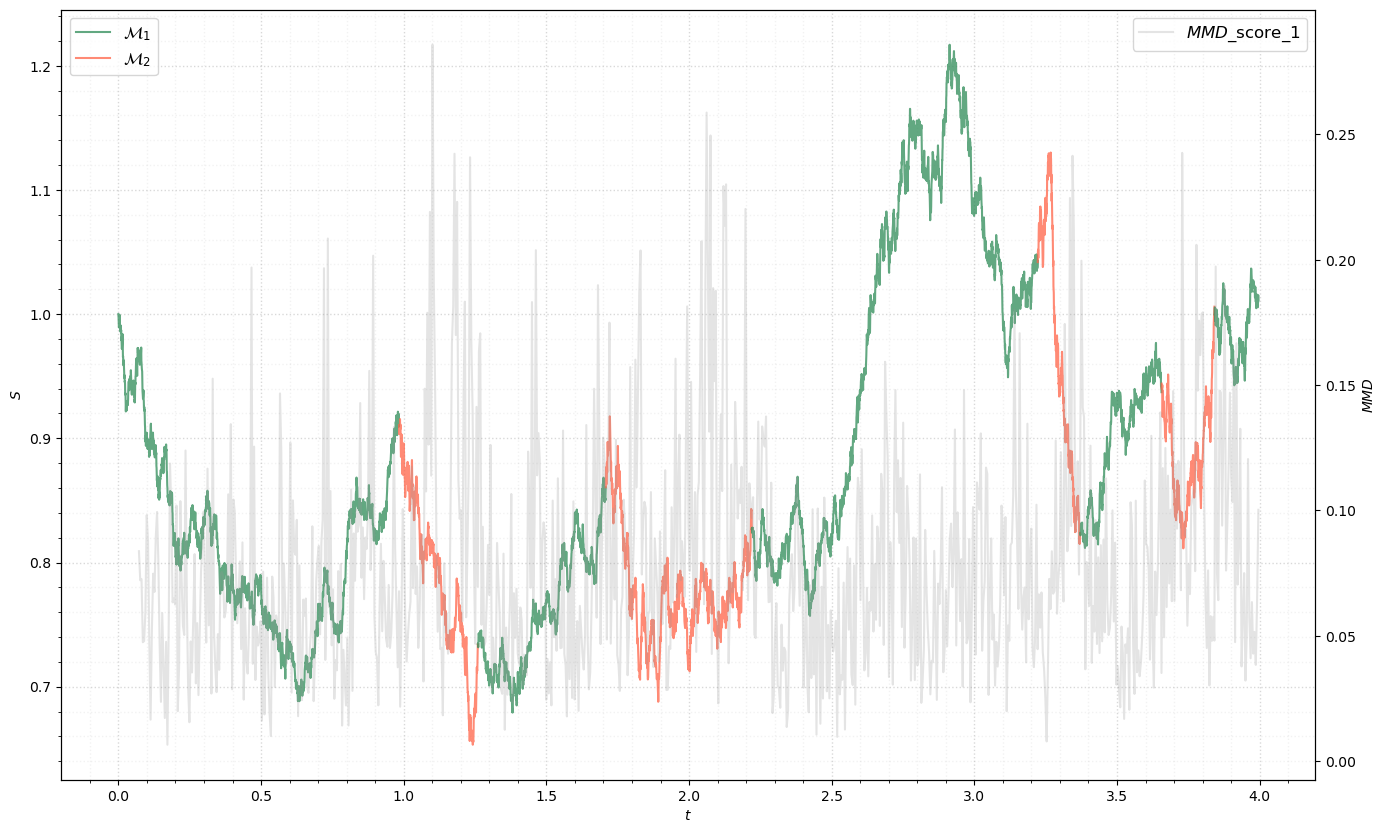}
        \caption{$\mathcal{D}^N_\text{sig}$ scores.}
        \label{fig:truncatedanomalyscores}
    \end{subfigure}%
    \begin{subfigure}[t]{0.33\linewidth}
        \centering
        \includegraphics[width=\textwidth]{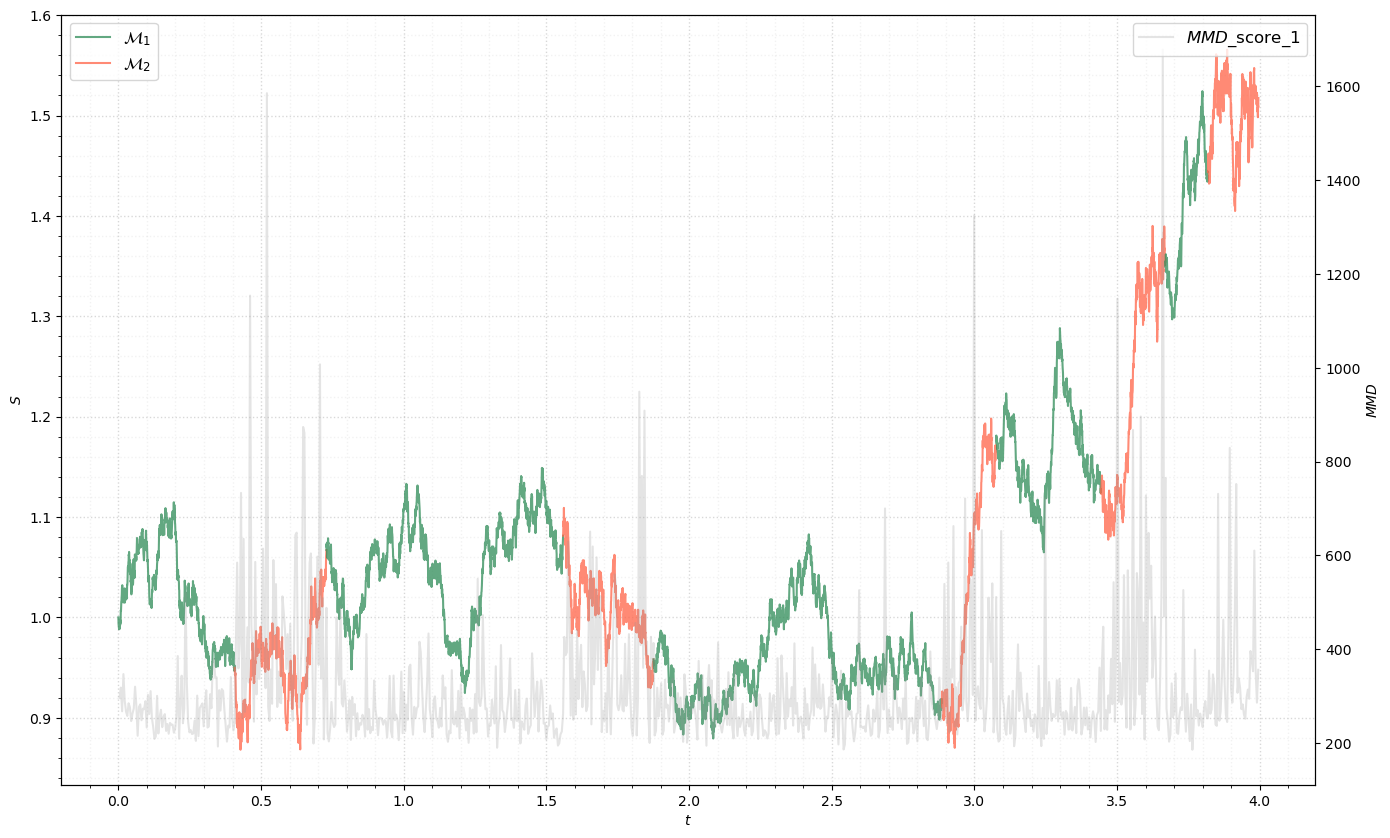}
        \caption{SIG-CON scores.}
        \label{fig:conformanceanomalyscores}
    \end{subfigure} %
    \begin{subfigure}[t]{0.33\linewidth}
        \centering
        \includegraphics[width=\textwidth]{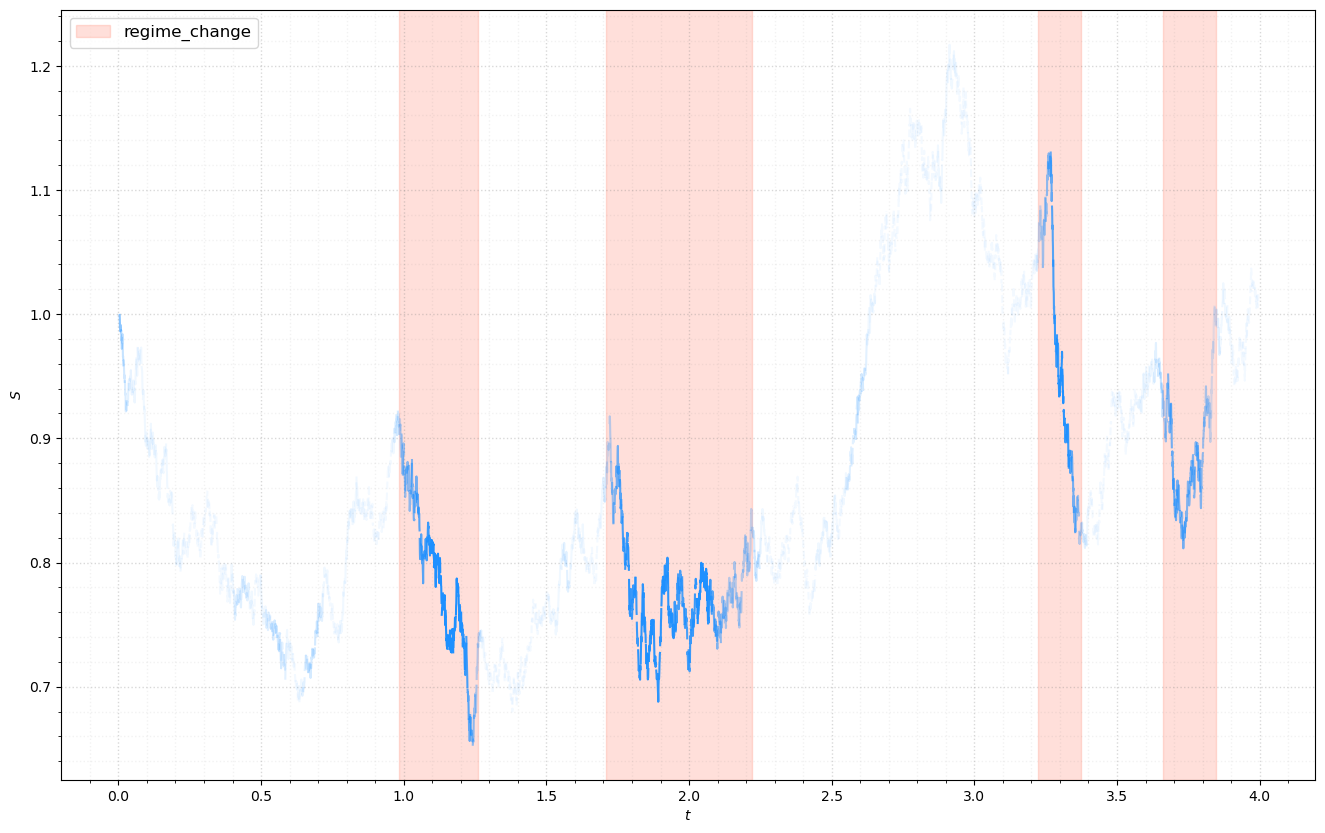}
        \caption{$\mathcal{D}^1_\text{sig} > c_\alpha$ plot.}
        \label{fig:mmddetectoranomalytoythreshold}
    \end{subfigure}%
    \begin{subfigure}[t]{0.33\linewidth}
        \centering
        \includegraphics[width=\textwidth]{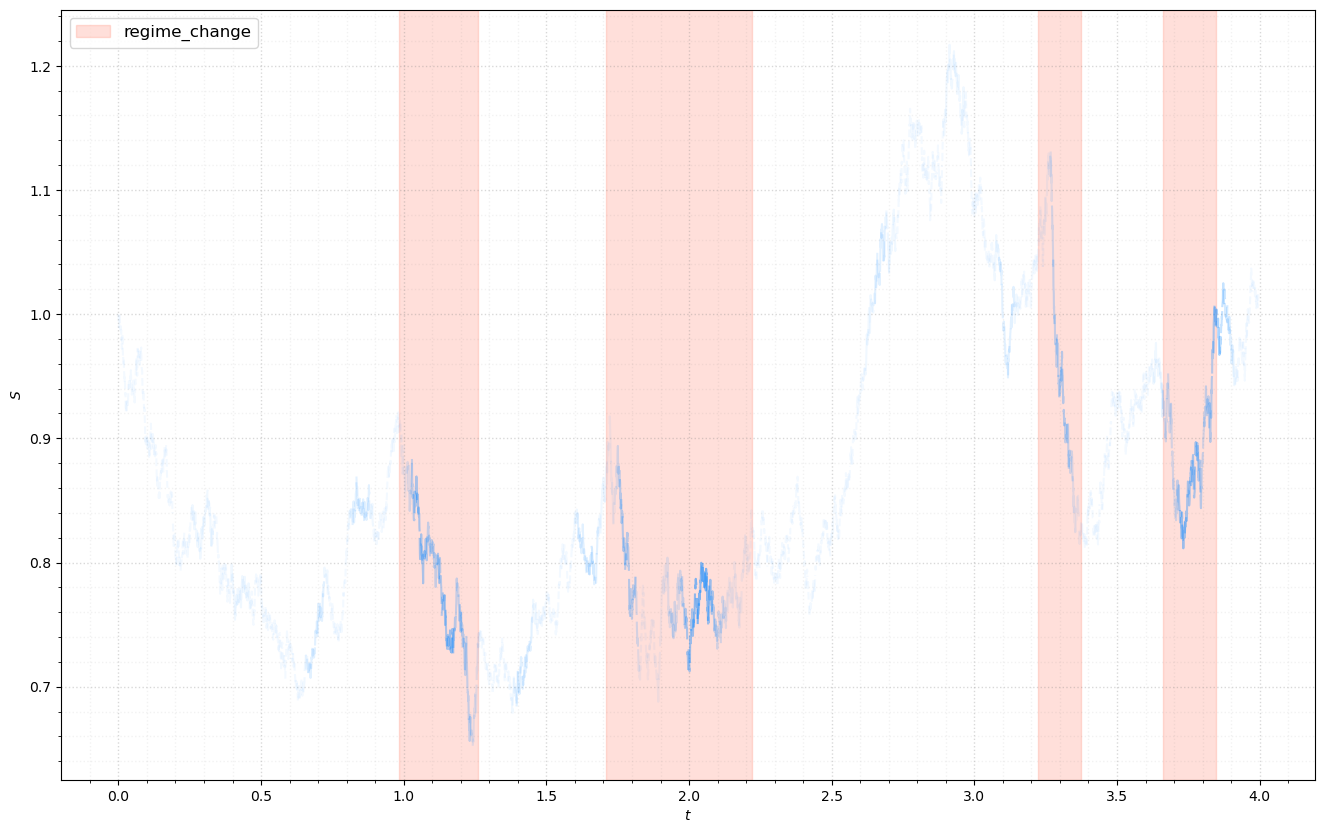}
        \caption{$\mathcal{D}^N_\text{sig} > c_\alpha$ plot.}
        \label{fig:truncatedmmdthreshold}
    \end{subfigure}%
    \begin{subfigure}[t]{0.33\linewidth}
        \centering
        \includegraphics[width=\textwidth]{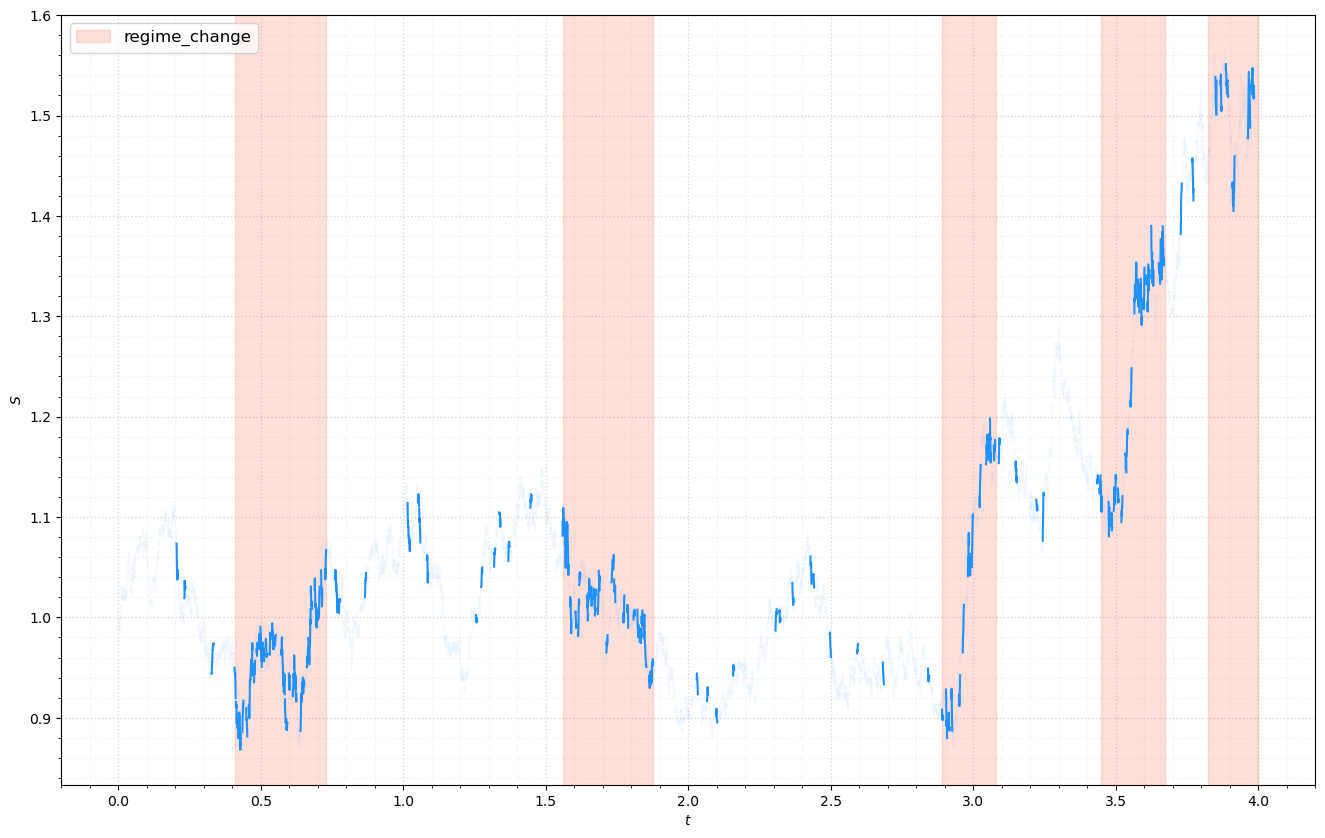}
        \caption{SIG-CON threshold plot.}
        \label{fig:conformanceanomalythreshold}
    \end{subfigure}
    \caption{An example run with each of the three detection algorithms. Both the non-truncated signature kernel and signature conformance method perform well, with the non-truncated signature method performing slightly better. The truncated signature method fails to detect changes in regime.}
    \label{fig:anomalytoyresults}
\end{figure}

Visual results are given in Figure \ref{fig:anomalytoyresults}. By considering path ensembles, the MMD-based anomaly detector gives more ``graded'' results than the conformance-based method. This is made clear in the threshold plots from Figures \ref{fig:mmddetectoranomalytoythreshold} and \ref{fig:conformanceanomalythreshold}. Furthermore, it is interesting to note that even with a relatively large level of truncation $N=5$, the truncated MMD is not able to distinguish between regimes at all. This is likely due to the fact that necessary information is embedded in higher-order terms which are not considered by construction. We summarize the results in Table \ref{tab:anomalytable}. 
    
\begin{table}[ht]	
    \centering
    \begin{center}
        \footnotesize
        \begin{tabular}{cccccc}
            \toprule
            \textbf{Algorithm}                 &                        & Accuracy               &                        & ROC AUC             & Runtime            \\ 
                                               & Regime on              & Regime off             & Total                  &                     &                         \\ \midrule
            $\mathcal{D}^N_{\text{sig}}$ & $9.2\pm 2.7$\% & $\boldsymbol{99.8\pm 0.2\%}$ & $72.7
            \pm 3.3$\% & $0.875\pm 0.019$ & $5.135\pm 0.207$s \\ \addlinespace
            SIG-CON & $50.0\pm 1.5$\% & $94.5\pm 3.4$\% & $82.2\pm 1.5$\% & $0.835 \pm 0.006$ & $45.004\pm 1.457$s \\ \addlinespace
            $\mathcal{D}^1_{\text{sig}}$                   & $\boldsymbol{77.1\pm 3.5\%}$ & $99.7\pm 0.2$\% & $\boldsymbol{93.6\pm 1.0\%}$ & $\boldsymbol{0.992 \pm 0.002}$ & $\boldsymbol{3.888\pm 0.346}$s \\ \midrule
        \end{tabular}
    \end{center}
    \setlength\tabcolsep{4pt}
    \caption{Algorithm performances, toy example, $n=100$ runs. The detector using the maximum mean discrepancy associated to $k^1_{\text{sig}}$ performs the best. It is also the fastest, since the kernel evaluation times are linear in $d$.}
    \label{tab:anomalytable}
\end{table}

We conclude this section with an explicit comparison between the true signature kernel $k^1_\text{sig}$ and the truncated variant $k^N_{\text{sig}}$. We do so by studying a process with jumps, which will introduce increased kurtosis on the marginal distributions of the path. Higher order terms of the signature will be required to effectively detect changes in regime (it would require a signature up to order 4 at the very least), and in theory the MMD associated to $k^N_{\text{sig}}$ cannot pick up on such differences. We are again in the random-time regime switching problem with two possible states: $\mathbb{P}_{\theta} = \mathrm{gBm}(\theta)$ and $\mathbb{P}_{\phi} = \mathrm{MJD}(\phi)$, where MJD (Merton jump diffusion) is the model corresponding to the stochastic differential equation

\begin{equation*}
    dX_t = \mu X_t dt + \sigma X_t dW_t + S_{t-}dJ_t, \quad X_0 = 1,
\end{equation*} 

where $J_t = \sum_{j=1}^{N_t} (V_j - 1)$, $N_t \sim \text{Po}(\lambda)$ and $\ln(1+V_j) \sim N(\gamma, \delta^2)$. In this way $\phi = (\mu, \sigma, \lambda, \gamma, \delta)$ is the associated parameter vector. As always we have that $\theta \in \mathbb{R}^2$ is the parameter vector corresponding to the drift and diffusion coefficients in the geometric Brownian motion model. Here we set $\theta = (0, 0.2)$ and $\phi = (0, 0.05, 100, 0, 0.025)$. Beliefs $\mathfrak{P}$ were (in both cases) given by $\mathbb{P}_\theta$, and null distributions were bootstrapped with $1000$ evaluations of the corresponding rank 1 MMD under $H_0$, and critical values were chosen at the $95\%$ test threshold. Regarding the metric associated to each detector, the MMD associated to $k_{\text{sig}}$ was chosen to be the RBF-lifted rank 1 variant from eq. (\ref{eqn:rankrmmdrbf}). We chose the RBF smoothing hyperparameter to be $\sigma = 0.1$. The truncated MMD was defined with $N=3$, and the associated smoothing parameter was given by $\sigma = 10$. Both parameters were bootstrapped to be optimal given the problem setting.

Regarding other experimental parameters: we choose $T=20$, $\Delta = 1/252$, and $h=(21, 10)$. In this way each path roughly represents a month's worth of price data. We then generate stopping times $\tau_1, \dots, \tau_M$ over the interval $[0, T]$ corresponding to when regime-switching events occur, with entry times given by $X \sim \text{Po}(0.25)$ and exit times by $Y \sim \text{Po}(5)$. Finally we set $\Phi = \varphi_{\text{time}} \circ \varphi_{\text{norm}}$, the standard path transformations.  
\begin{figure}[ht]
    \centering
    \begin{subfigure}{0.5\linewidth}
        \centering
        \includegraphics[width=\textwidth]{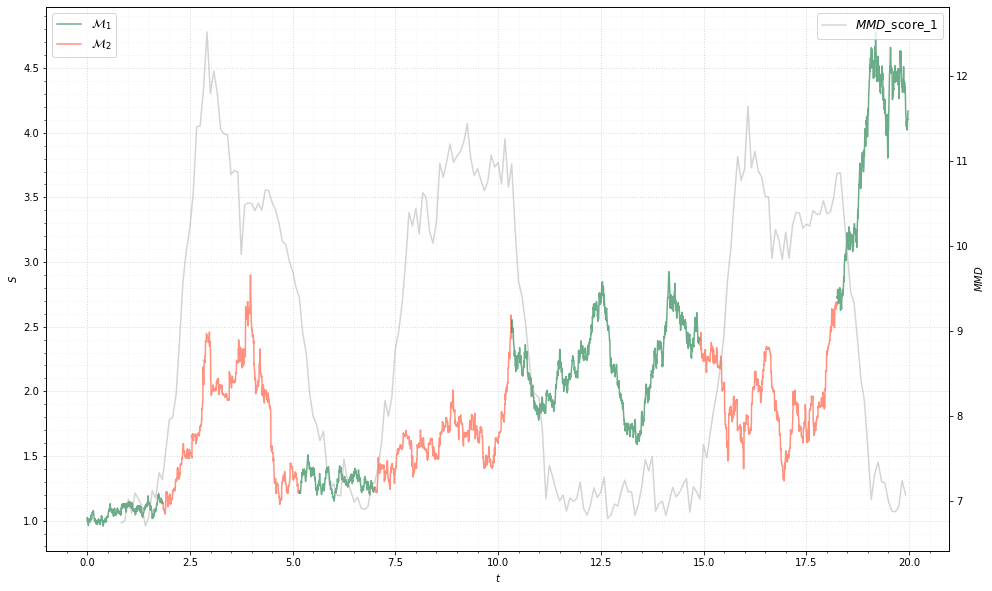}
        \caption{$\mathcal{D}^1_\text{sig}$ scores.}
        \label{fig:truncvstotal_total_mmd}
    \end{subfigure}%
    \begin{subfigure}{0.5\linewidth}
        \centering
        \includegraphics[width=\textwidth]{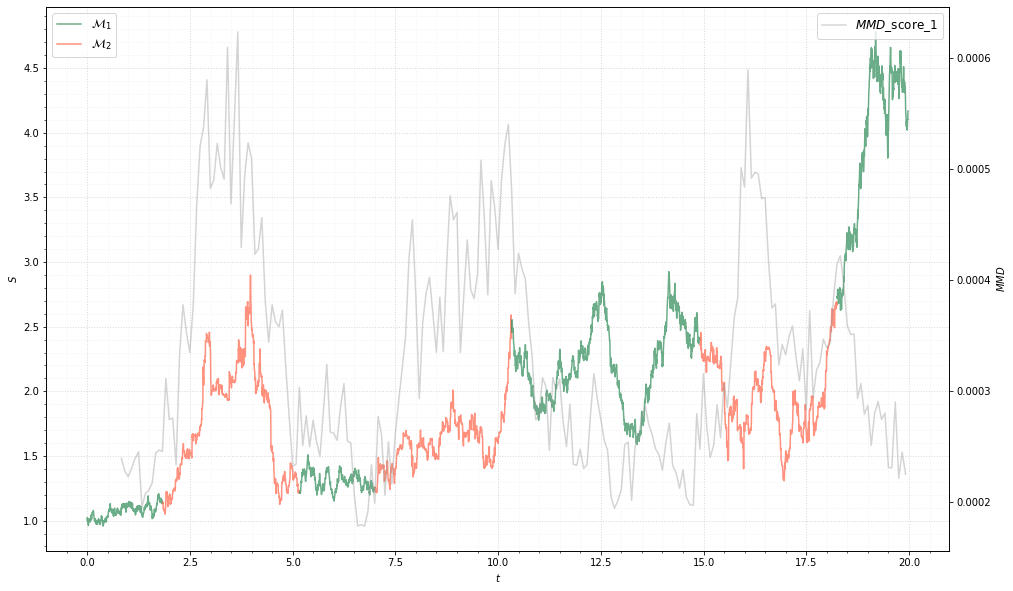}
        \caption{$\mathcal{D}^N_\text{sig}$ scores.}
        \label{fig:truncvstotal_trunc_mmd}
    \end{subfigure}
    \caption{Example run, gBm vs MJD, MMD scores. Both algorithms are largely able to detect regime changes.}
    \label{fig:truncvstotalmmd}
\end{figure}

\begin{figure}[ht]
    \centering
    \begin{subfigure}{0.5\linewidth}
        \centering
        \includegraphics[width=\textwidth]{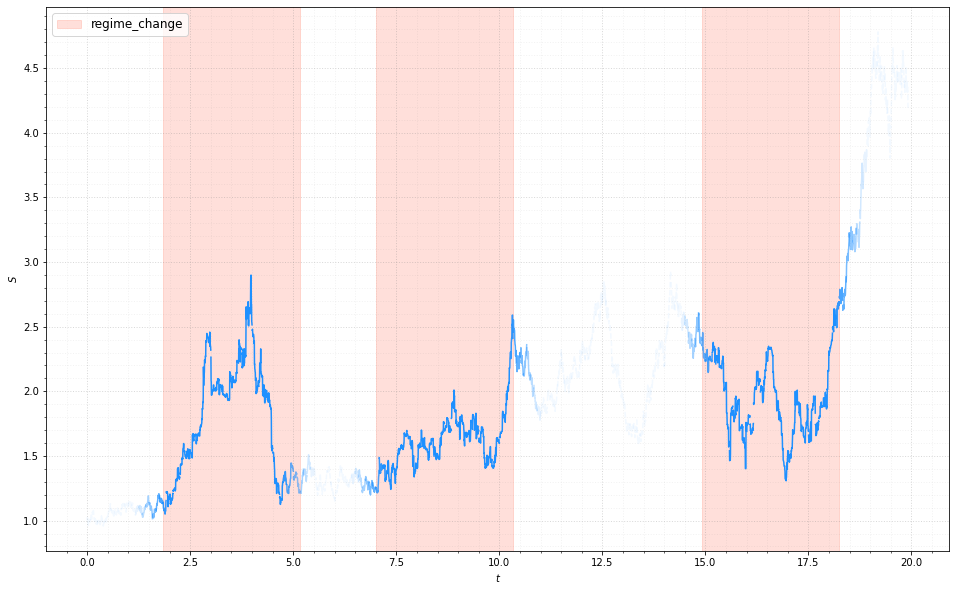}
        \caption{$\mathcal{D}^1_\text{sig} > c_\alpha$.}
        \label{fig:truncvstotal_total_alpha}
    \end{subfigure}%
    \begin{subfigure}{0.5\linewidth}
        \centering
        \includegraphics[width=\textwidth]{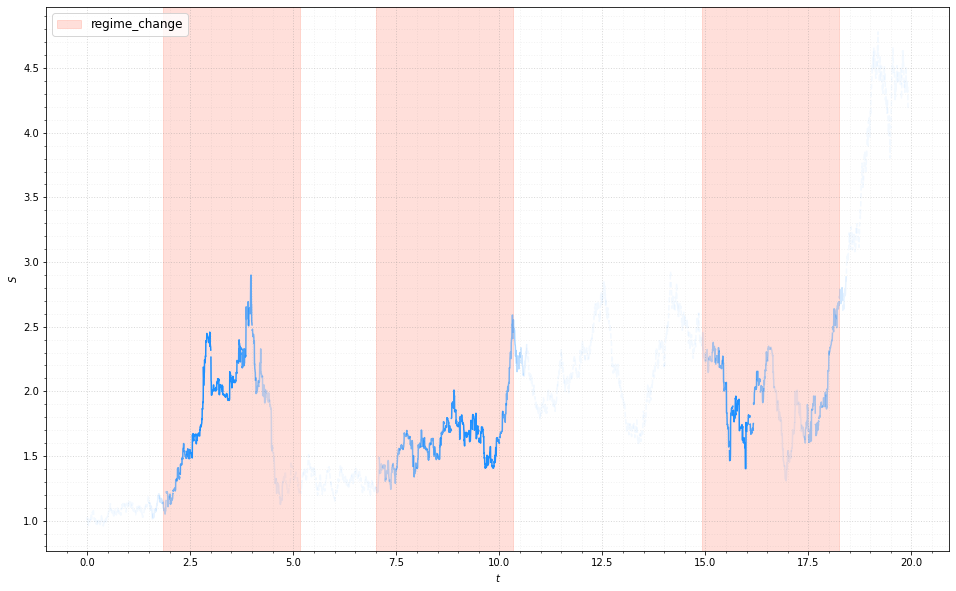}
        \caption{$\mathcal{D}^N_\text{sig} > c_\alpha$.}
        \label{fig:truncvstotal_trunc+_alpha}
    \end{subfigure}
    \caption{Example run, gBm vs MJD, $\mathrm{MMD} > c_\alpha$ plots. Here we can see that the final regime change was (incorrectly) not considered significant enough under $\mathcal{D}^N_{\text{sig}}$ to warrant a regime change signal.}
    \label{fig:truncvstotalalpha}
\end{figure}

A sample run with attached MMD scores is given in Figure \ref{fig:truncvstotalmmd}. Visually one can verify that although both detectors performed well, the true signature kernel technique is better able to distinguish between the two different regimes. This may be due to the fact that the marginal distributions of stochastic processes associated to $\mathbb{P}_\phi$ exhibit more kurtosis than those associated to $\mathbb{P}_\theta$. Weaker performance could be due to the truncated signature kernel not ``seeing'' fourth-order effects with $N=3$. For further evidence we give a summary of $n=100$ runs in Table \ref{tab:truncvstotal}. 

\begin{table}[ht]	
    \begin{center}
        \footnotesize
        \begin{tabular}{cccccc}
            \toprule
            \textbf{Algorithm}  &                        & Accuracy               &                     & ROC AUC             & Runtime           \\ 
                                & Regime on              & Regime-off             & Total               &                     &                        \\ \midrule
            $\mathcal{D}^N_{\text{sig}}$               & $62.4\pm 6.8$\% & $\boldsymbol{96.3\pm 4.5\%}$ & $81.5\pm 5.5\%$ & $0.946 \pm 0.062$ & $\boldsymbol{9.823\pm 0.077}$s \\ \addlinespace
            $\mathcal{D}^1_{\text{sig}}$             & $\boldsymbol{99.3\pm 2.6\%}$ & $86.2\pm 3.9\%$ & $\boldsymbol{92.9\pm 2.7\%}$ & $\boldsymbol{0.973 \pm 0.028}$ & $16.030\pm 0.070$s \\ \midrule
        \end{tabular}
    \end{center}
    \setlength\tabcolsep{4pt}
    \caption{Algorithm performances, toy example, $n=100$ runs. Again, the detector using the untruncated MMD is better able to detect regime changes than the truncated variant, especially when higher-order moments become more relevant.}
    \label{tab:truncvstotal}
\end{table}

%% file: section4/47nonmarkovian.tex
\subsection{Evaluation of non-Markovian data}\label{subsec:nonmarkovian}

Suppose you have a probability space $(\Omega, \mathcal{F}, (\mathcal{F}_t), \mathbb{P})$ and an $\mathcal{F}_t$-adapted stochastic process $X: [0, T] \to \mathbb{R}$. Suppose you observe a sample $x \sim \mathcal{L}(X)$, necessarily in discrete time over a grid $\Delta = \{0 = t_0 < t_1 < \dots < t_N = T\}$, so in fact one works with $\hat{\mathsf{x}} \in \mathcal{T}_\Delta([0, T]; \mathbb{R})$. As we have mentioned in Subsection \ref{subsec:partitioning}, we are interested in fusing sub-paths of the stream $\hat{\mathsf{x}}$ into collections of path ensembles. In a continuous-time setting, ensemble paths can be extracted by first defining a mesh $\Delta' = \{0 = u_0 < u_1 < \dots < u_N = T\}$ and considering the collection of paths $\mathcal{X} = \{x_{|[u_{i-1}, u_i]} : i = 1, \dots, M\}$. If $|u_i - u_{i-1}|$ is equal for all $i=1, \dots, M$, then (given certain conditions, which we will discuss) one can set $\tau := u_i-u_{i-1}$ and we have that the collection $\mathcal{X}$ defines an empirical measure $\nu$ on a compact subset $\mathcal{K}$ of path space $C([0, \tau]: \mathbb{R})$ where $\mathcal{X} \subseteq \mathcal{K}$. In particular we assume all paths are of bounded variation; as mentioned in Subsection \ref{subsec:pathsignatures}, in our setting this is a fine assumption to make.

In our work, we use $\mathcal{D}^r_{\text{sig}}$ to compare the measure $\nu$ generated by $\mathcal{X}$ with another measure $\mu \in \mathcal{P}(\mathcal{K})$, usually given by samples drawn from $\mathfrak{P}$. In order for the finite-sample (unbiased) estimator associated to $\mathcal{D}^r_{\text{sig}}$ to be asymptotically consistent, one requires that each element $x \in \mu, y \in \nu$ is 1) independent and 2) identically distributed (i.i.d), see Definition \ref{def:maximummeandiscrepancy}. In practice, the generator associated to $\mu$ is explicitly known and thus one can indiscriminately draw i.i.d samples from $\mu$. 

Let us now consider the set $\mathcal{X}$, and by extension the measure $\nu$. We want to derive conditions for when we can assume that each element of $\mathcal{X}$ \emph{could} correspond to i.i.d. samples. Even if we assume that all path segments drawn from $X$ have been generated by the same data-generating measure, we still cannot necessary conclude that this is true. To illustrate this, we consider two cases: either the underlying process $X$ is a) Markovian or b) non-Markovian. 

Recall that a stochastic process $X$ on $(\Omega, \mathcal{F}, (\mathcal{F}_t, \mathbb{P})$ is \emph{Markovian} if $\mathbb{P}(X_t|\mathcal{F}_s) = \mathbb{P}(X_t | X_s)$ for all $0 \le s \le t$. Else, it is non-Markovian, which is to say that the transition density of $X_s$ to $X_t$ depends on the entire history of $X$ up until time $s$. 

Extending the definition of Markovianity to path space, one can see that if $X$ is Markov then the path space measures $\mathbb{P}(X_{|[s, t]}\in \cdot |\mathcal{F}_s)$ and $\mathbb{P}(X_{|[s, t]}\in \cdot |X_s)$ agree. Therefore, the only piece of information required to compare measures on path space where the underlying process $x \sim \mathcal{L}(X)$ is Markovian is the initial data $x_s$. 

Returning to our original problem, each sample $x \in \mathcal{X}$ is \emph{a priori} itself a sample from a conditional distribution: $x^i \sim \mathbb{P}(X_{|[u_{i-1}, u_i]} \in\cdot | \mathcal{F}^X_{u_{i-1}})$. If the underlying process is Markovian then it is clear that (after adjusting for initial data) a) each sample is independent (Markov processes are memoryless) and b) they are identically distributed, since $$\mathbb{P}(X_{|[u_{i-1}, u_i]}\in\cdot | \mathcal{F}^X_{u_{i-1}}) =  \mathbb{P}(X_{|[u_{i-1}, u_i]} \in \cdot| X_{u_{i-1}})$$ under our initial assumptions. Therefore, by a simple path re-scaling, we can study the processes $\hat{x}^i = x^i/x^i_0$, which will comply with our stated assumptions.

Obviously if $X$ is not Markovian then the preceding analysis fails. In this case samples are not guaranteed to be independent; they are not even guaranteed to be identically distributed, since each $x^i$ is distributed conditional on the progressively enlarging filtration $\mathcal{F}^X_{u_{i-1}}$ for $i=1,\dots,M$.

If one assumes strict non-Markovianity of the observed sample path $x$, then (by definition) one cannot collect i.i.d. samples from $x$, as no matter how the path is partitioned each section will be conditional on a different filtration. This forces a path-by-path approach. In particular it is clear that we need to evaluate the whole path in order to incorporate filtration information.

A simple experiment is the following: consider two model/parameter pairs $(\mathbb{P}, \theta)$ and $(\mathbb{Q}, \phi)$. Suppose we know that $\mathbb{P}$ is non-Markovian and $\mathbb{Q}$ is Markovian; however, their marginal distributions are similar enough that naive methods will not be able to conclude that they are generated by different distributions. Suppose we observe a path $x: [0, T]\to \mathbb{R}$ over a discrete mesh $\Delta$. We know that the dynamics of $x$ are governed either by $\mathbb{P}$ or $\mathbb{Q}$, but we do not know over which intervals on $\Delta$ that this is the case. We then do the following: At given points $\Delta' = \{0 < u_1 < u_2 < \dots < u_M = T\}$, for $j=1,\dots,M$ we can study $x^j = x_{|[u_{j-1}, u_j]}$ under the similarity score function $\Sigma^{\mathbb{P}_{j}, \mathbb{Q}_{j}}$ where $\mathbb{P}_{j} = \{x_{|[u_{j-1}, u_j]}: x \sim \mathbb{P}\}$. The same notation holds for $\mathbb{Q}$. This allows us to directly evaluate our observed, conditional path sample in a like-for-like manner against our beliefs $\mathbb{P}, \mathbb{Q}$. Negative sample scores imply conformance to the model $\mathbb{P}$; positive score imply conformance to the model $\mathbb{Q}$.

\begin{remark}[Rank of signature mapping used in $\Sigma$]
    Ideally, we would like to use the rank-2 signature $S^2$ to define the constituent scoring rules of $\Sigma^{\mathbb{P}, \mathbb{Q}}$, so the filtration information over $[u_{j-1}, u_j]$ can be included when conducting inference. However, the methodology used to construct the conditional embeddings (from \cite{salvi2021higher}) cannot be used in the context of many-to-one comparisons. Thus, we must use the rank 1 signature in our experiments.
\end{remark}

In what follows we take $$(\mathbb{P}, \theta) = (\text{rBergomi}, (0.05, 0.05, -0.7, 0.4)) \text{ and } (\mathbb{Q}, \phi) = (\text{gBm}, (0, 0.25))$$ to be our non-Markovian and Markovian models respectively. These parameters were chosen so that the marginal distributions associated to either model qualitatively similar. We set $dt = 1/1764$ and generated a path $\hat{\mathsf{x}} \in \mathcal{T}_\Delta([0, 4], \mathbb{R}^2)$ where the dynamics of the simulated path is given by $(\mathbb{P}, \theta)$ till $T/2$ and $(\mathbb{Q}, \phi)$ afterwards. 

To construct our beliefs $\mathfrak{P} = (\mathfrak{P}_1, \mathfrak{P}_2)$, we first simulated two banks of paths over $[0, 4]$, where $\mathfrak{P}_1$ corresponds to $(\mathbb{P}, \theta)$ and $\mathfrak{P}_2$ corresponding to $(\mathbb{Q}, \phi)$. We then partitioned our path data into $\mathcal{SP}_h(\hat{\mathsf{x}})$, with $h=(32, 0)$. In this way, $$\Delta' = \{0 = u_0 < u_1 < \dots < u_M = T \}$$ where $u_j = j\,dt$. For each $x \in \mathcal{SP}_h(\hat{\mathsf{x}})$, we calculated $\Sigma^{\mathfrak{P}_1, \mathfrak{P}_2}(x)$ by sampling $128$ paths from both $\mathfrak{P}_1, \mathfrak{P}_2$ restricted to the interval $[u_{j-1}, u_j]$. In this way, each sample $x^i_{[u_{j-1}, u_j]} \in \mathcal{P}_i$ is conditioned on $\mathcal{F}^x_{u_{j-1}}$. We report the similarity scores associated to the rank-1 RBF-lifted kernel from eq. (\ref{eqn:rankrmmdrbf}) with $\sigma = 1$. The path transformer $\Phi$ was given by $\Phi = \varphi_{\text{norm}} \circ \varphi_{\text{time}} \circ \varphi_{\text{incr}}$. 

\begin{figure}[ht]
    \centering
    \begin{subfigure}{0.5\linewidth}
        \centering
        \includegraphics[width=\textwidth]{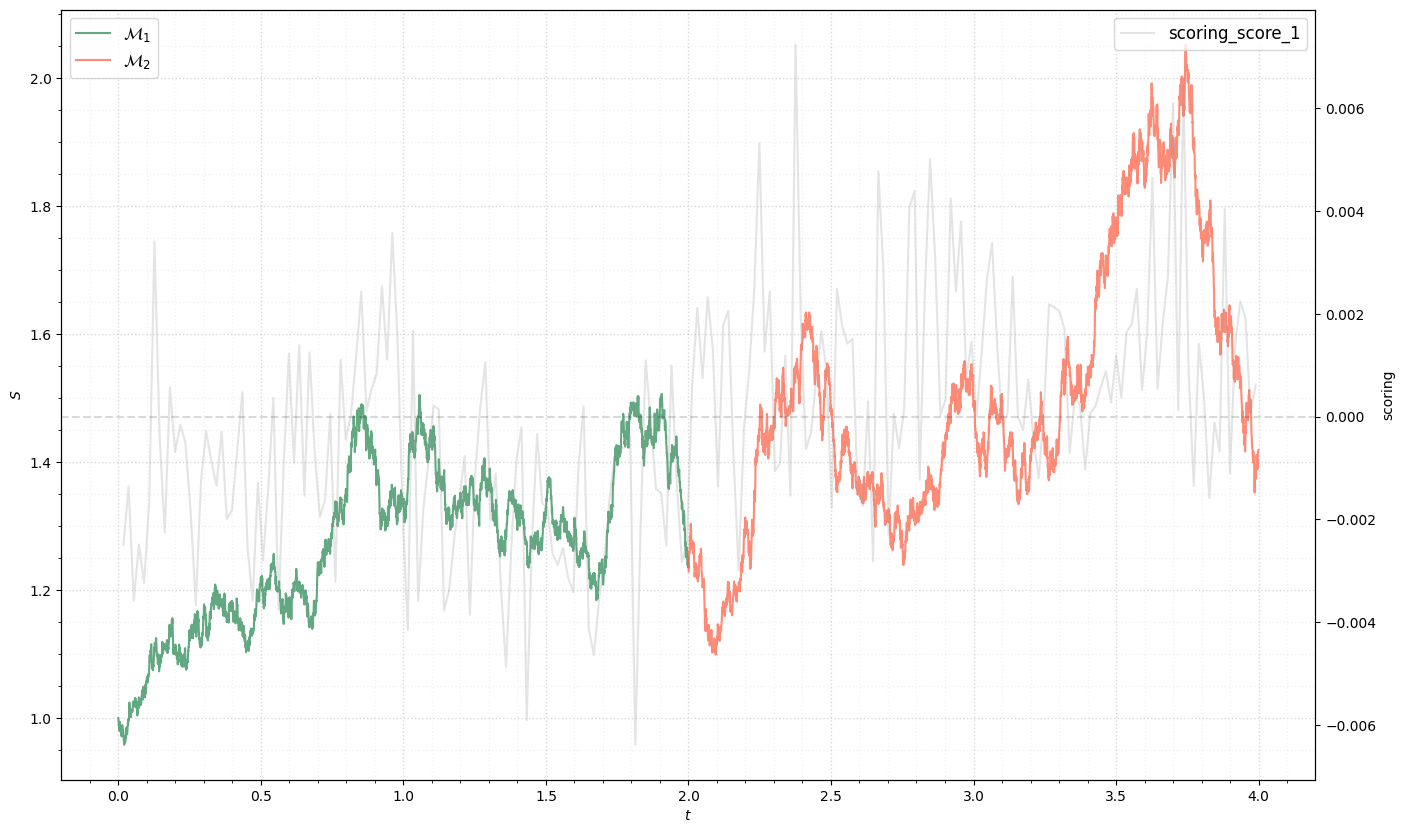}
        \caption{Regime-changed path and $\Sigma^{\mathfrak{P}_1, \mathfrak{P}_2}$.}
        \label{fig:nonmarkovianscore}
    \end{subfigure}%
    \begin{subfigure}{0.5\linewidth}
        \centering
        \includegraphics[width=\textwidth]{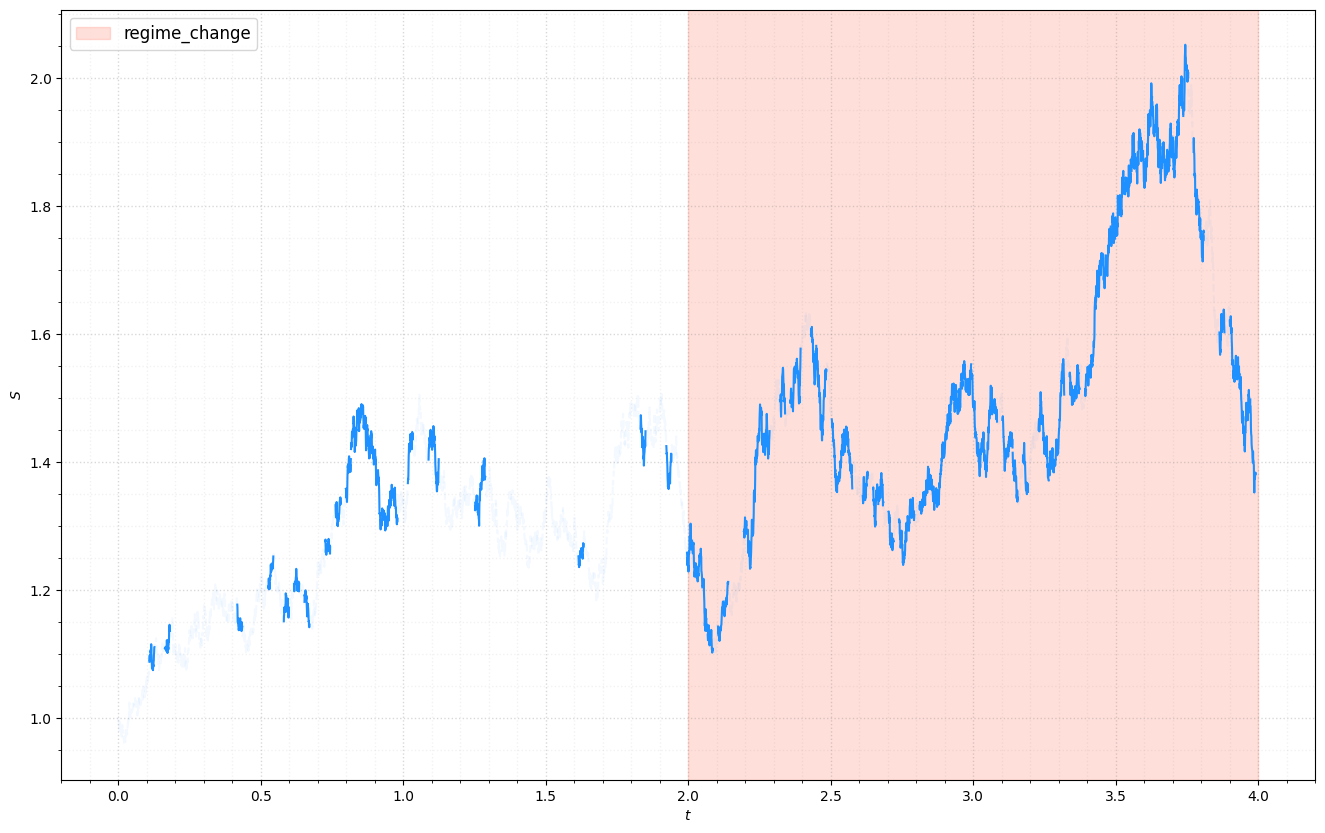}
        \caption{$\Sigma^{\mathfrak{P}_1, \mathfrak{P}_2} > 0$ plot.}
        \label{fig:nonmarkovianalphas}
    \end{subfigure}
    \caption{An example run of non-Markovian switching experiment. At $T=2$, path dynamics switch from being driven by a non-Markovian model to a Markovian one. The similarity score is able to identify this change; notice the jump in the average score after the regime change point.}
    \label{fig:nonmarkovianplot}
\end{figure}

Figure \ref{fig:nonmarkovianscore} gives an illustration of an example run with this experiment. We can see that the the conformance score is able to detect a shift in regime dynamics from a single observed path without violating any of the underlying assumptions of the kernel methods used.

%% file: section5/51toy_example.tex
In this section, we compare our clustering technique to one explored is a previous work \cite{horvath2021clustering}, in which a modification of the classical $k$-means algorithm was used to cluster regimes. For a price path $\mathsf{s} \in \mathcal{S}(\mathbb{R})$, regimes in this context were defined as empirical measures $\mu \in\mathcal{P}_p(\mathbb{R})$ extracted from the steam of log-returns associated to $\mathsf{s}$. This motivated the use of the Wasserstein distance $\mathcal{W}_p$ as the metric between elements to be clustered, and the Wasserstein barycenter was used to aggregate measures into centroids. This method (dubbed \emph{Wasserstein $k$-means}) was shown to perform strongly over other classical regime classification methodologies for streaming data over $\mathbb{R}$.

However, this method was not without drawbacks. First, it is difficult to run the Wasserstein $k$-means algorithm for $d >1$ in a manner which is computationally feasible. There do exist methodologies to make the calculation simpler by employing the (max) sliced Wasserstein distance and barycenter from \cite{rabin2011wasserstein, kolouri2019generalized}, however in practice the authors have found this approach unstable and computationally expensive. The second major drawback is that in aggregating the log-returns associated to $\mathsf{s}$, one loses information about the order in which the returns appeared. Thus, on large enough timescales, changes in auto-correlative dynamics can be lost, which may be of keen interest to a practitioner. Given that returns in financial markets exhibit such properties, an algorithm which was able to capture these effects at little to no cost would be preferable. 

The method outlined in Section \ref{subsec:mrcpmethods} is indeed able to make up for the shortcomings of the Wasserstein approach. We are able to cluster higher-dimensional paths at little computational cost, and by using the signature kernel $k_{\text{sig}}$ from eq. (\ref{eqn:signaturekernel}), the order in which path elements appear is captured. 

We begin (again) with the toy example presented in Section \ref{subsec:toyexample}. Here, we are in a simplified case of the MRCP where there are only two regimes to consider, $\mathbb{P}_{\theta_1} = \mathrm{gBm}(0, 0.2)$ and $\mathbb{P}_{\theta_2} = \mathrm{gBm}(0, 0.3)$. Again we generate a time-augmented regime-changed path $\hat{\mathsf{s}} \in \mathcal{T}_\Delta([0, T]; \mathbb{R}^d)$ with $T=2$, $d=1$ and $\left(\tau_i\right)_{i\ge 0}$ being the regime-switching times associated to $\hat{\mathsf{s}}$. 

The goal now is to assign regime labels to segments of the path $\hat{\mathsf{s}}$. We do this by apportioning $\hat{\mathsf{s}}$ into the set of (transformed) sub-paths $\mathcal{SP}^\Phi_h(\hat{\mathsf{s}})$ and corresponding set of ensemble paths $\mathcal{EP}^\Phi_h(\hat{\mathsf{s}})$ with $h=(7,10)$ and $\Phi$ as in Section \ref{subsec:toyexample}, so $\Phi = \phi_{\text{incr}} \circ \phi_{\text{time}} \circ \phi_\text{norm}$. We then classify the path $\hat{\mathsf{s}}$ under the three methods we are considering: our rank-$r$ MMD detector (MMD-DET), the Wasserstein $k$-means method, and the MMD associated to the truncated signature kernel (MMD-T). For the MMD-based methods, we calculate the pairwise distance matrix $D \in \mathbb{R}^{N_2 \times N_2}$ which (when using the signature kernel $k^r_{\text{sig}}$ is given by
\begin{equation}
    D_{ij} = \mathcal{D}_{\text{sig}}^r(\boldsymbol{s}^i, \boldsymbol{s}^j) \quad \text{for }i,j=1,\dots,N_2,
\end{equation}
where $\boldsymbol{s}^i \in \mathcal{EP}^\Phi_h(\hat{\mathsf{s}})$ for $i=1,\dots,N_2$. For both MMD-DET and MMD-T, we initially set $\sigma = 0.025$ and consider the rank $1$ MMD. We then pass $D$ to the $\mathtt{AgglomerativeClustering}$ class, from the $\mathtt{cluster}$ module within the $\mathtt{scikit}$-$\mathtt{learn}$ package. There are only two hyperparameter choices to make: the linkage criterion, which we use the average, and the number of clusters $k=2$. For the Wasserstein $k$-means approach, we first built the vector of log-returns associated to $\hat{\mathsf{s}}$, which is given by
\begin{equation*}
    \mathsf{r}_i = \log(\mathsf{s}_{i}) - \log(\mathsf{s}_{i-1}) \qquad \text{for }i=1,\dots,N,
\end{equation*}
and set $\mathsf{r}_0 := 0$. We then segmented $\mathsf{r}$ into sub-paths of length $n_1=70$, with an overlap parameter equal to $n_2=63$. Our clustering space is thus similar to the set $\mathcal{SP}_h(\hat{\mathsf{s}})$, except 1) we are considering log-returns instead of raw path values which 2) are not time-augmented. Finally we set the Wasserstein exponent $p=2$. For a more detailed explanation of the algorithm, we refer the reader to \cite{horvath2021clustering}.

Each clustering method outputs a vector of labels $\ell = (\ell_1, \dots, \ell_{N_2})$ associated to each element $\boldsymbol{s} \in \mathcal{EP}_h(\hat{\mathsf{s}})$. With this data we again calculated the average cluster label associated to each sub-path $s \in \mathcal{SP}_h(\hat{\mathsf{s}})$. These average scores are then evaluated against the vector of true labels $\ell^{\mathrm{true}} = (\ell^{\mathrm{true}}_1, \dots, \ell^{\mathrm{true}}_{N_1})$. 

\begin{figure}[ht]
    \centering
    \begin{subfigure}[t]{0.5\linewidth}
        \centering
        \includegraphics[width=\textwidth]{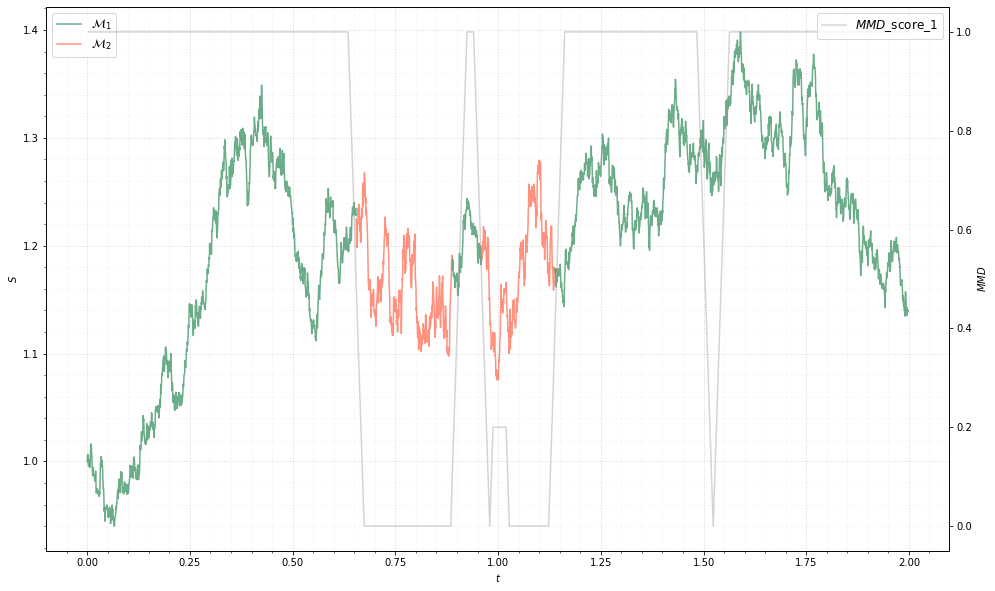}
        \caption{$\mathcal{D}^1_{\text{sig}}$}
        \label{fig:toyclusteringdet}
    \end{subfigure}%
    \begin{subfigure}[t]{0.5\linewidth}
        \centering
        \includegraphics[width=\textwidth]{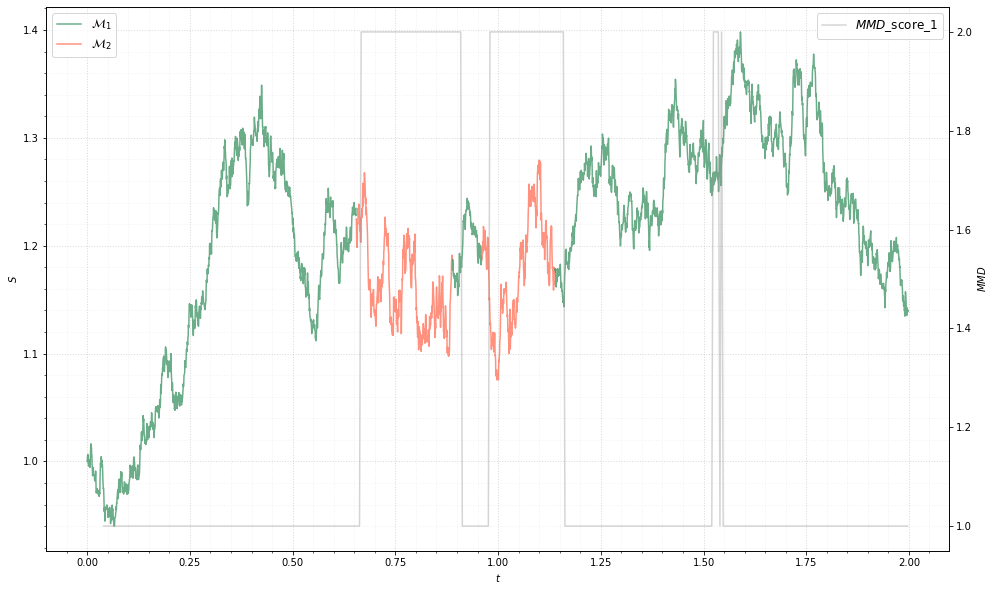}
        \caption{Wasserstein $k$-means.}
        \label{fig:toyclusteringwass}
    \end{subfigure}
    \caption{Example clustering, toy runs, $d=1$. Both algorithms perform well in this context.}
    \label{fig:toyclustering}
\end{figure}

From Figure \ref{fig:toyclustering}, we can visually confirm that the MMD-based algorithms on average perform well. The advantage of the MMD-based techniques, however, is their ability to easily study paths of higher dimensionality. To illustrate this, we constructed a regime-changed path $\hat{\mathsf{s}} \in \mathcal{T}_\Delta([0, T]; \mathbb{R}^d)$ and set $d=10$. Here, we cannot easily apply the Wasserstein algorithm. One might expect that increasing the dimensionalty renders results less stable or less accurate. However, Figure \ref{fig:toyclustering10} confirms that (in this setting) this is not the case.

\begin{figure}[ht]
    \centering
    \includegraphics[scale=0.4]{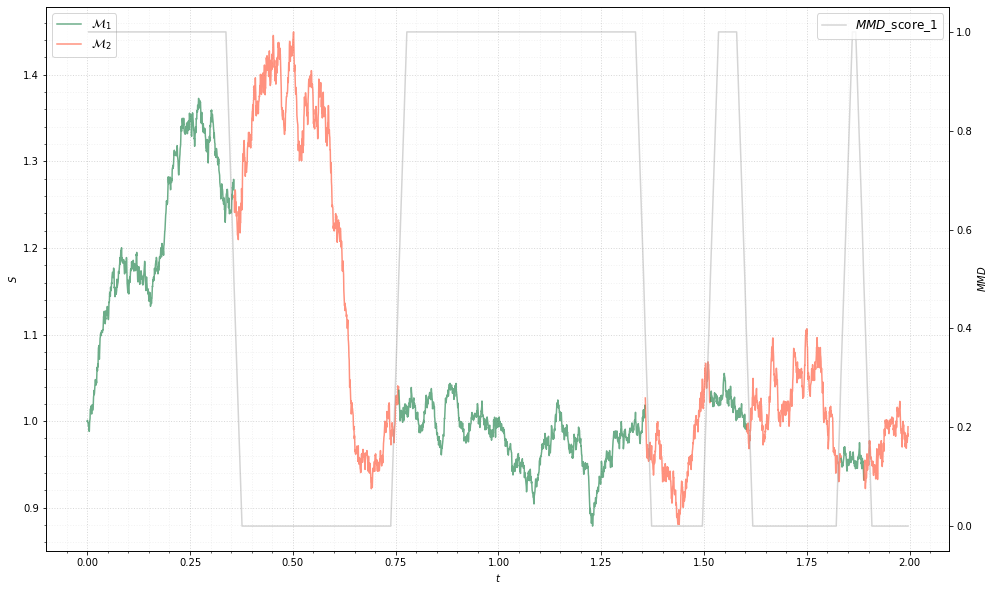}
    \caption{Example clustering, toy runs, $d=10$. The Wasserstein approach is not shown as it can only be used when $d=1$. With $\mathcal{D}^1_{\text{sig}}$, it is computationally feasible to increase dimensionality without harming performance.}
    \label{fig:toyclustering10}
\end{figure}

%% file: section6/61basket_equities.tex
\subsection{Basket of equities}\label{subsec:basket}

We begin by studying the price path associated to a basket of equities $\mathsf{s} = (\mathsf{s}_1, \dots, \mathsf{s}_d)$. Here we focused on US equities, and initially took the daily close prices from Yahoo Finance of Coca-Cola (NYSE:KO), IBM (NYSE:IBM), General Electric (NYSE:GE), Proctor and Gamble (NYSE:PG), Exxon Mobil (NYSE:XOM), JP Morgan (NYSE:JPM), and the S\&P 500 (NYSE:\^{}GSPC), giving a (time-augmented) price path $\hat{\mathsf{s}} \in \mathcal{T}_\Delta([a, b]; \mathbb{R}^{8})$. Here, $\Delta$ represents the mesh-grid associated to daily close prices. In this example we have $a=$ 1980-01-01 and $b=$ 2021-12-20. 

We apply a wholly non-parametric approach by using the auto evaluator from Subsection \ref{subsubsec:nonparametricevaluation} to calculate the score vector $A_L(\hat{\mathsf{s}})$ with $L= \{4, 8, 12\}$. We use the path transformer function $\Phi = \phi_{\text{time}} \circ \phi_{\text{norm}} \circ \phi^\lambda_{\text{scale}}$ with $\lambda = dt^{-1/2}$ (where $dt =1/252$ as we are studying daily returns). We set $h=(8, 8)$. The choice of detector is given by the rank 1 MMD associated to the RBF-lifted signature kernel with $\sigma=1$. Finally, we set $\alpha = 0.95$ and allow for the empirical prior distribution to be comprised of the previous $200$ MMD scores.

\begin{figure}[ht]
    \centering
    \begin{subfigure}[t]{0.4\linewidth}
        \centering
        \includegraphics[width=\textwidth]{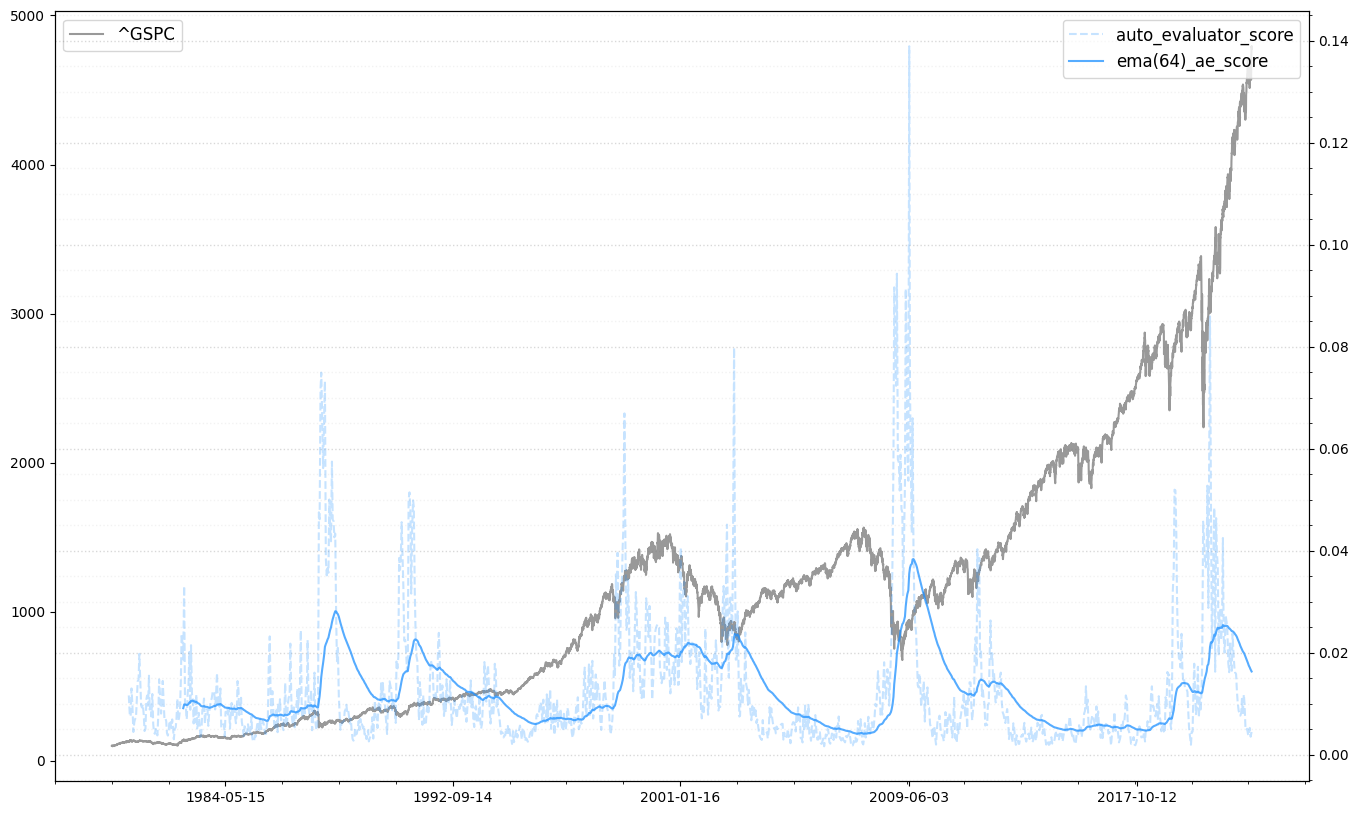}
        \caption{NYSE:SPY}
        \label{fig:spyscores}
    \end{subfigure}
    \begin{subfigure}[t]{0.4\linewidth}
        \centering
        \includegraphics[width=\textwidth]{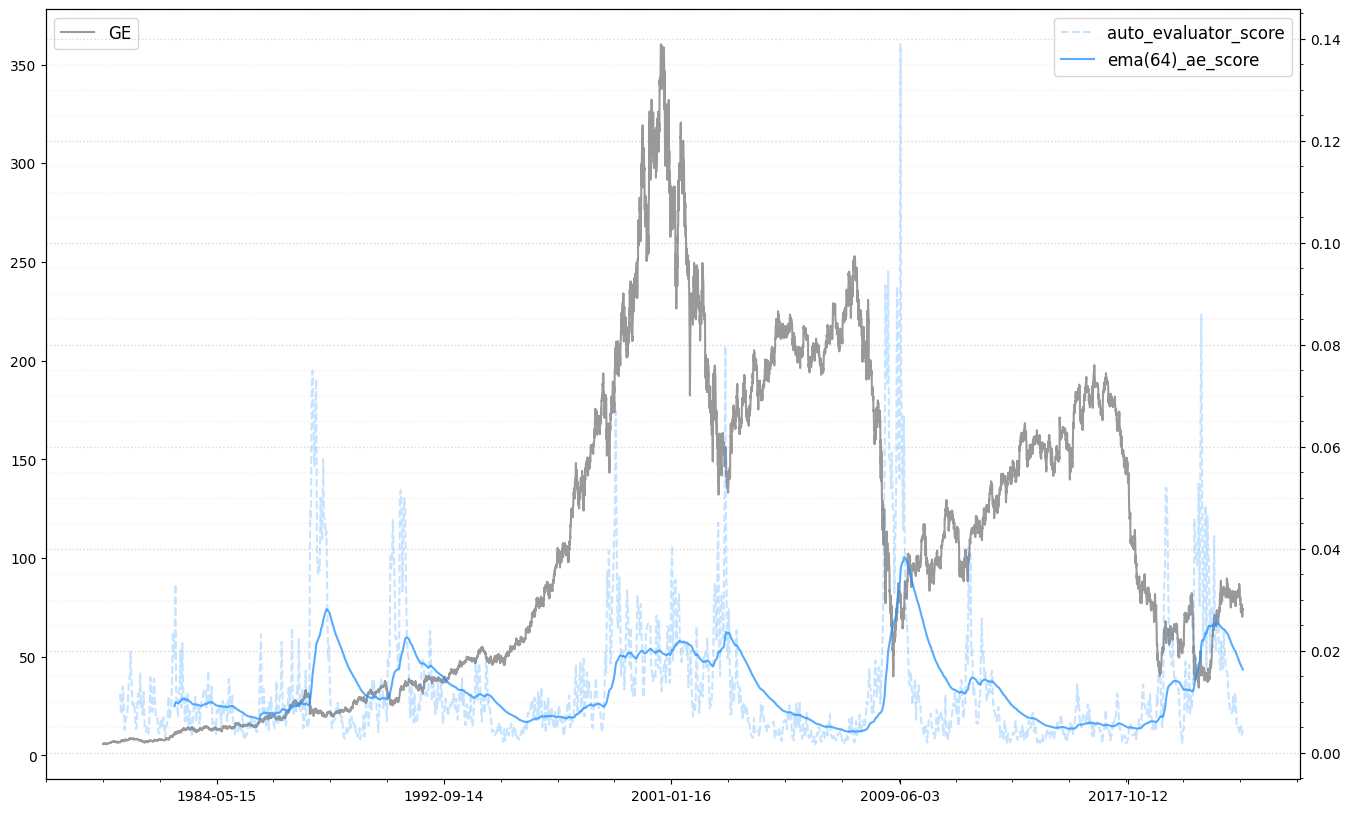}
        \caption{NYSE:GE}
        \label{fig:gescores}
    \end{subfigure}
    \medskip
    \begin{subfigure}[t]{0.4\linewidth}
        \centering
        \includegraphics[width=\textwidth]{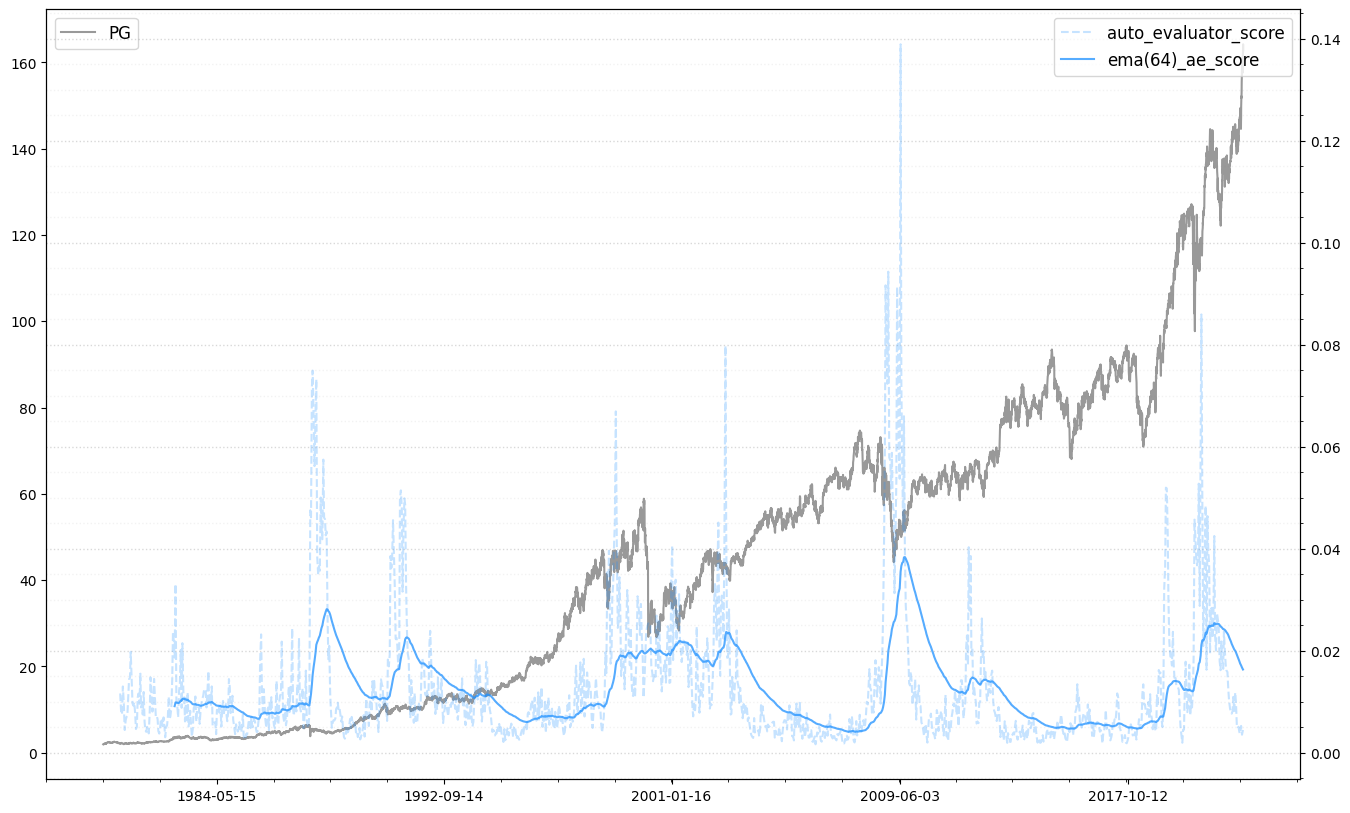}
        \caption{NYSE:PG}
        \label{fig:pgscores}
    \end{subfigure}
    \begin{subfigure}[t]{0.4\linewidth}
        \centering
        \includegraphics[width=\textwidth]{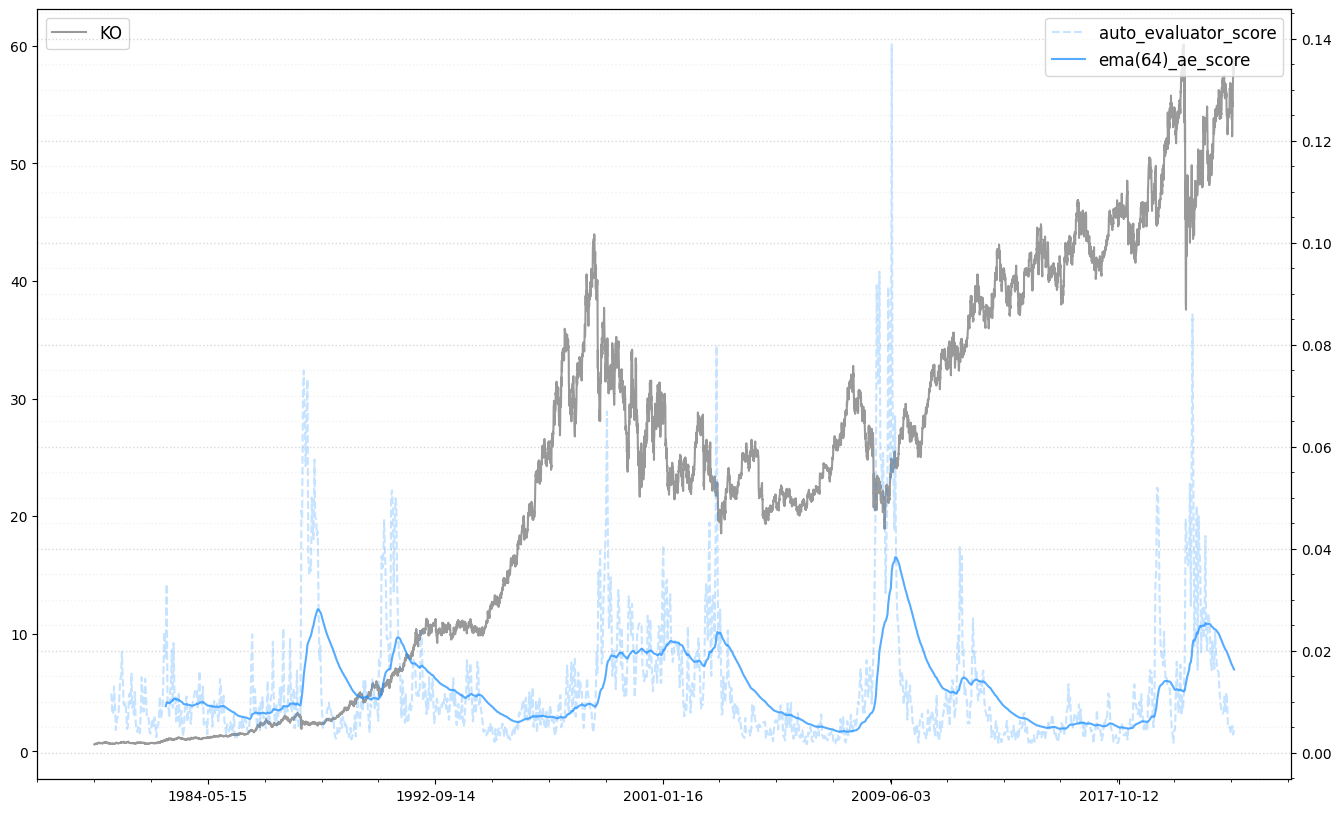}
        \caption{NYSE:KO}
        \label{fig:koscores}
    \end{subfigure}
    \caption{MMD scores projected against different equity paths contained in the basket. From Figure \ref{fig:equitiesmmd} we can see that the MMD score peaks at periods of known market turmoil: notably, the late 1980s financial crisis, the dot com bubble and subsequent crash in 2001, the Global Financial Crisis (GFC) and the period of instability in European markets afterwards, and more recently the increased volatility due to the SARS-COV-2 pandemic. }
    \label{fig:equitiesmmd}
\end{figure}

Figure \ref{fig:equitiesmmd} gives a plot of $A_L(\hat{\mathsf{s}})$ alongside daily close prices of some basket constituents. The dashed blue line is the true score vector $A_{L}(\hat{\mathsf{s}})$ for either detector, and the solid blue line is an exponential moving average. For validation purposes we give a plot of the same MMD score vector from Figure \ref{fig:equitiesmmd} against the VIX index, from January 2018 till December 2021. Up to a lag (due to the sampling frequency of path objects versus single return values) the auto-evaluator largely tracks the VIX well. Interestingly, periods of similarity in the VIX process do not necessary correlate to similar lagged MMD scores.

\begin{figure}[ht]
    \centering
    \includegraphics[scale=0.3]{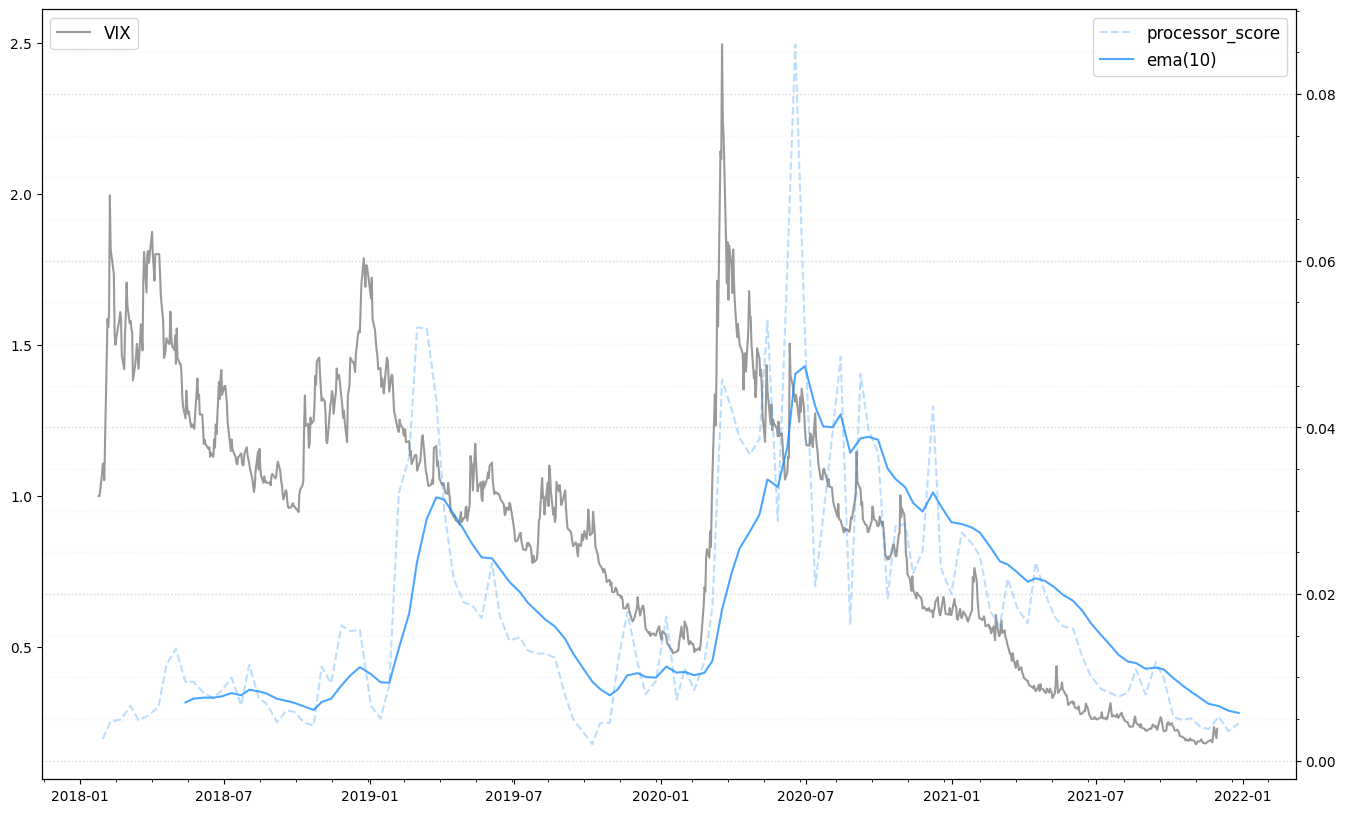}
    \caption{Auto-evaluator MMD scores (blue) against VIX (grey). The MMD score tracks the VIX well.}
    \label{fig:vixmmd}
\end{figure}

We conclude this section by clustering the price path $\hat{\mathsf{s}}$ via our method described in Section \ref{subsec:mrcpmethods}. More concretely, we apply the hierarchical clustering used in Section \ref{sec:mrcp} with maximum linkage, and initially set the number of clusters $k=2$. The results of running the clustering algorithm are presented in Figure \ref{fig:equitiesrealcluster}. Here, the S\&P 500 is plotted for reference, coloured according to the assigned cluster label. More translucent colouring means that a path segment had different labels assigned to it depending on what ensemble it was a part of. We plot the results of choosing $k=2$ or $k=4$ clusters.

\begin{figure}[ht]
    \centering
    \begin{subfigure}{0.5\linewidth}
        \includegraphics[width=\textwidth]{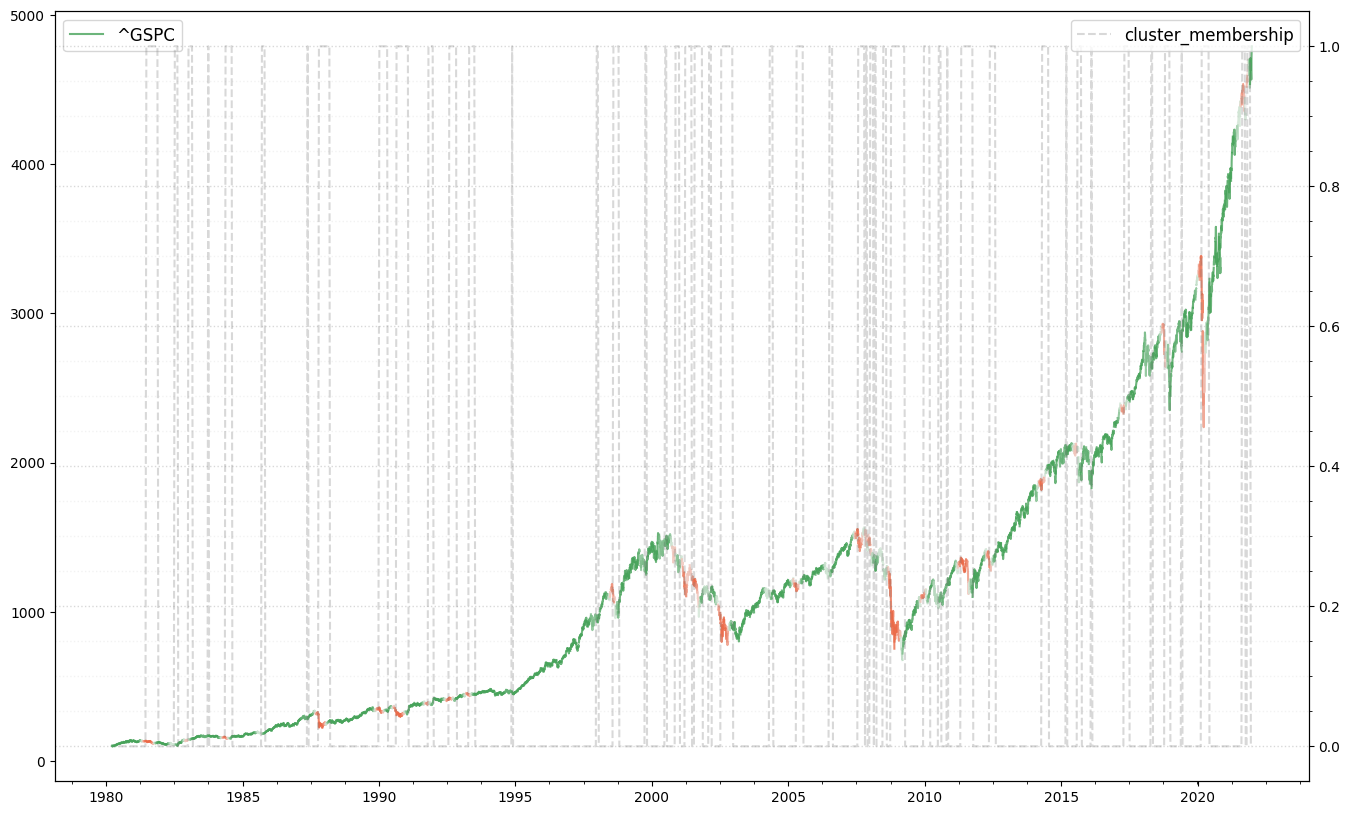}
        \caption{$k=2$.}
        \label{fig:equitiesrealcluster2}
    \end{subfigure}%
    \begin{subfigure}{0.5\linewidth}
        \centering
        \includegraphics[width=\textwidth]{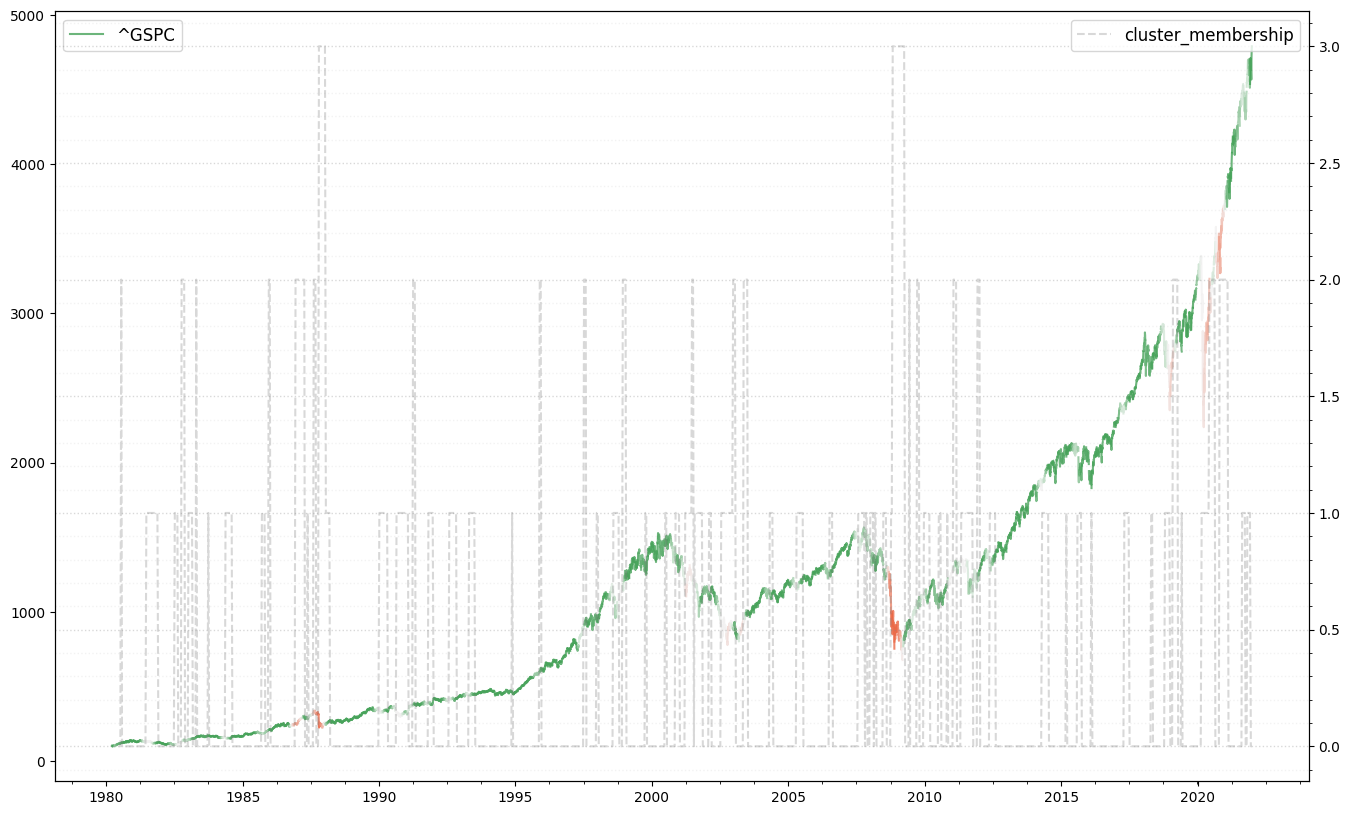}
        \caption{$k=4$.}
        \label{fig:equitiesrealcluster4}
    \end{subfigure}
    \caption{Hierarchical clustering of $A_L(\hat{\mathsf{s}})$. The average cluster membership is given by the grey line. As the number of clusters increases, the demarcations between regime states become more granular, whilst still maintaining some sense of continuity between clusters.}
    \label{fig:equitiesrealcluster}
\end{figure}

The results of either clustering algorithm are very similar to those obtained in Figure \ref{fig:equitiesmmd}. The grey line represents the average cluster membership of the given sub-path segment $s \in \mathcal{SP}_h(\hat{\mathsf{s}})$. The main observation here is that the average cluster membership does not oscillate too wildly. The same is true even with an increase in the number of clusters chosen: Figure \ref{fig:equitiesrealcluster4} gives results if we set $k=4$. Periods of market turmoil are also associated with periods of oscillating cluster membership. This is most evidenced by the periods leading up to and including the GFC, the dot-com bubble, and the recent market turbulence due to the pandemic.

%% file: section6/62crypto.tex
\subsection{Cryptocurrency data}\label{subsec:crypto}

In this section, we run our auto MMD detector over a high-dimensional path of cryptocurrency pairs, all denominated in USDT (Tether). These include Bitcoin (BTC), Ethereum (ETH), Litecoin (LTC), Ripple (XRP), Binance Coin (BNB), and Polygon (MATIC). Our path we are interested in is given by $\hat{\mathsf{s}} \in \mathcal{T}_\Delta([a, b]; \mathbb{R}^6)$ with $\Delta$ the mesh grid associated to hourly close prices, $a=$2019-06-01 and $b=$2022-12-01. Due to increased frequency of the data, we study paths of size $h=(8, 32)$, and apply the transformations $\phi_{\text{norm}} \circ \phi_{\text{time}} \circ \phi_{\text{incr}}$. We again use the auto-evaluator from Definition \ref{def:autoevaluator} equipped with the rank-1 RBF-lifted signature maximum mean discrepancy, and chose the smoothing parameter $\sigma = 1$. We calculate $A_L(\hat{\mathsf{s}})$ with the lags $L = \{8, 16\}$, representing the two most recent path ensembles that do not contain any sub-paths located in the currently evaluated sample. We set the memory of the auto-evaluator to the previous 250 MMD scores.

\begin{figure}[ht]
    \centering
    \begin{subfigure}{0.5\linewidth}
        \centering
        \includegraphics[width=\textwidth]{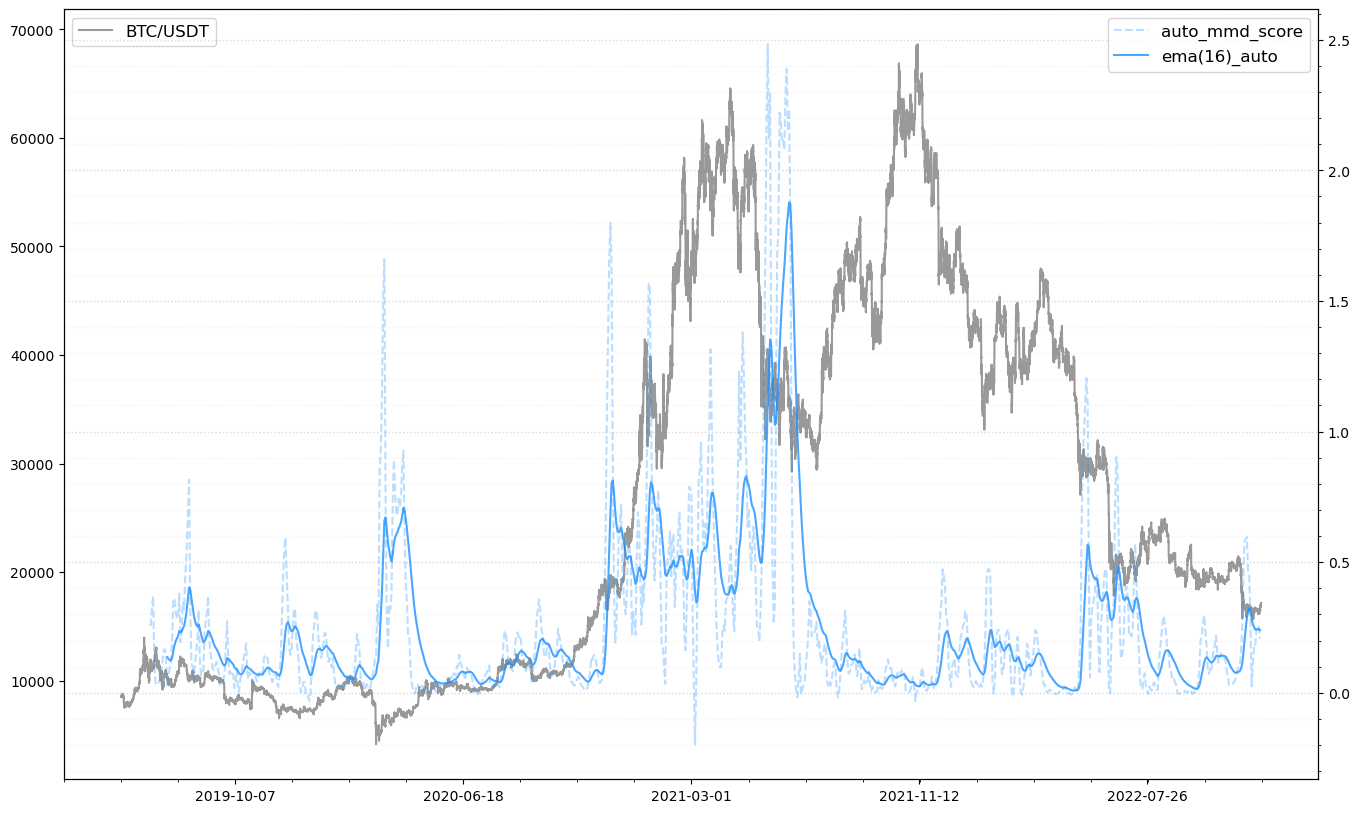}
        \caption{$A_L(\hat{\mathsf{s}})$ against BTC hourly close prices.}
        \label{fig:cryptommd}
    \end{subfigure}%
    \begin{subfigure}{0.5\linewidth}
        \centering
        \includegraphics[scale=0.225]{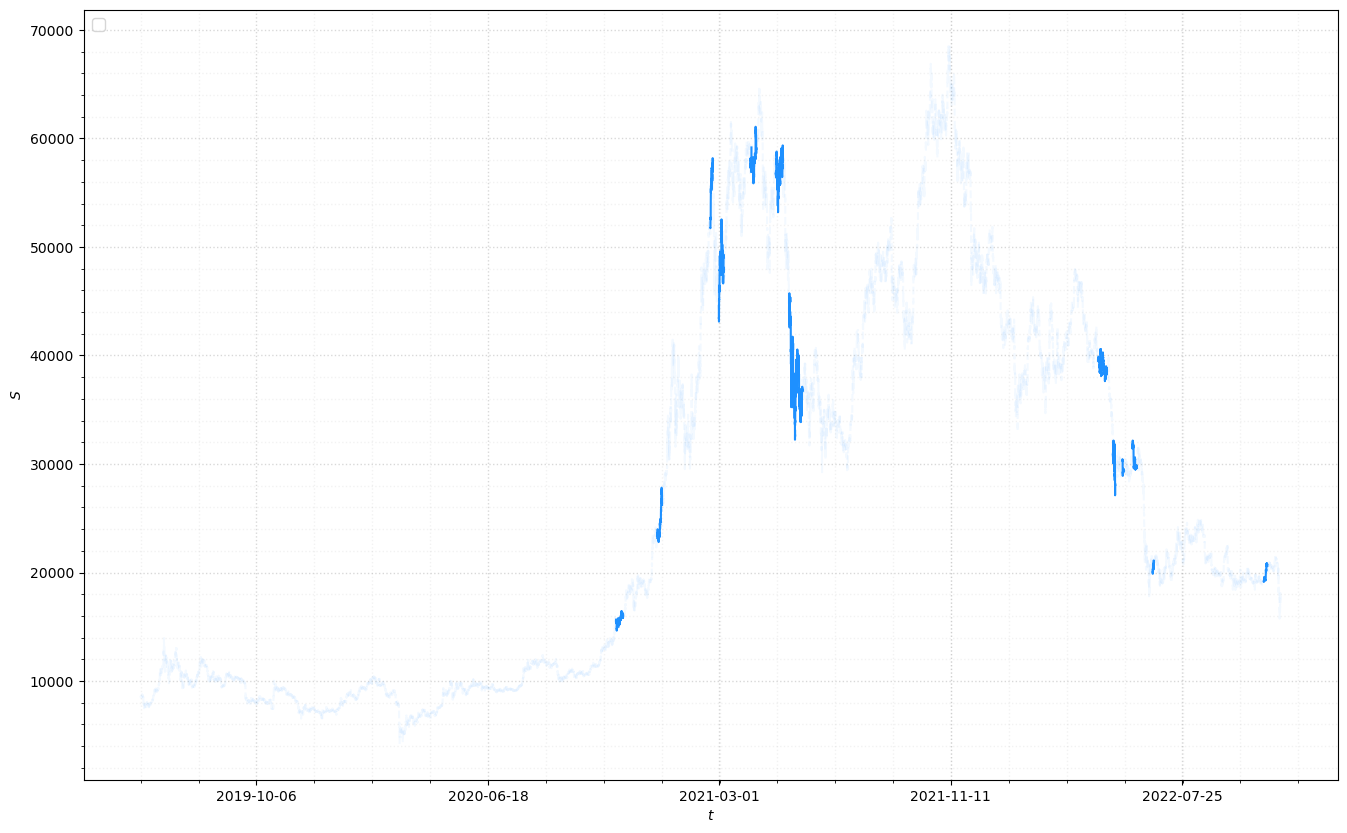}
        \caption{$\mathcal{D}^1_{\text{sig}} > c_\alpha$ threshold plot.}
        \label{fig:cryptoalphas}
    \end{subfigure}
    \caption{$A_L(\hat{\mathsf{s}})$ score for basket of cryptocurrencies and associated threshold plot. The detector identifies the beginning of the bull run and its subsequent cessation. The detector has long enough memory that it does not identify the subsequent period of volatility as anomalous. The exit from this period, however, is considered a regime change event.}
    \label{fig:mrdpcrypto}
\end{figure}

Figure \ref{fig:mrdpcrypto} shows the rank 1 MMD score over $\mathcal{EP}_h(\hat{\mathsf{s}})$ against the price of Bitcoin over the same period. We see three main regime change periods: in 2019, towards the end of 2020 and into 2021, and during mid-2022. Interestingly, the period during 2021 was not considered a regime change, which makes sense, as it closely resembles the period preceding it. Clustering results are given in Figure \ref{fig:cryptoclustering2}. Here the price plot is the BTC/USDT hourly close prices, coloured according to the average cluster membership of each $s \in \mathcal{SP}_h(\hat{\mathsf{s}})$. 

\begin{figure}[ht]
    \centering
    \begin{subfigure}{0.5\linewidth}
        \centering
        \includegraphics[width=\textwidth]{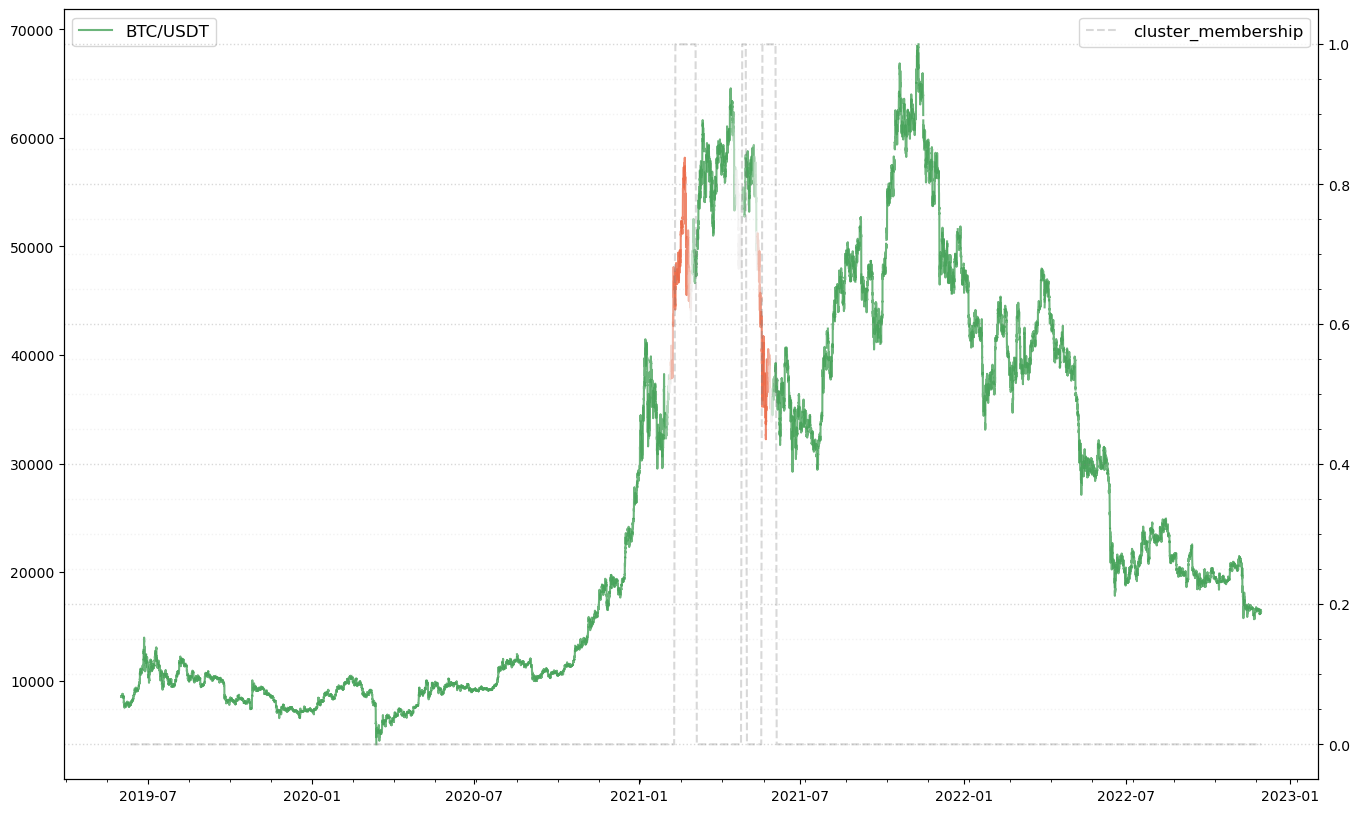}
        \caption{$k=2$.}
        \label{fig:cryptoclustering2}
    \end{subfigure}%
    \begin{subfigure}{0.5\linewidth}
        \centering
    \includegraphics[width=\textwidth]{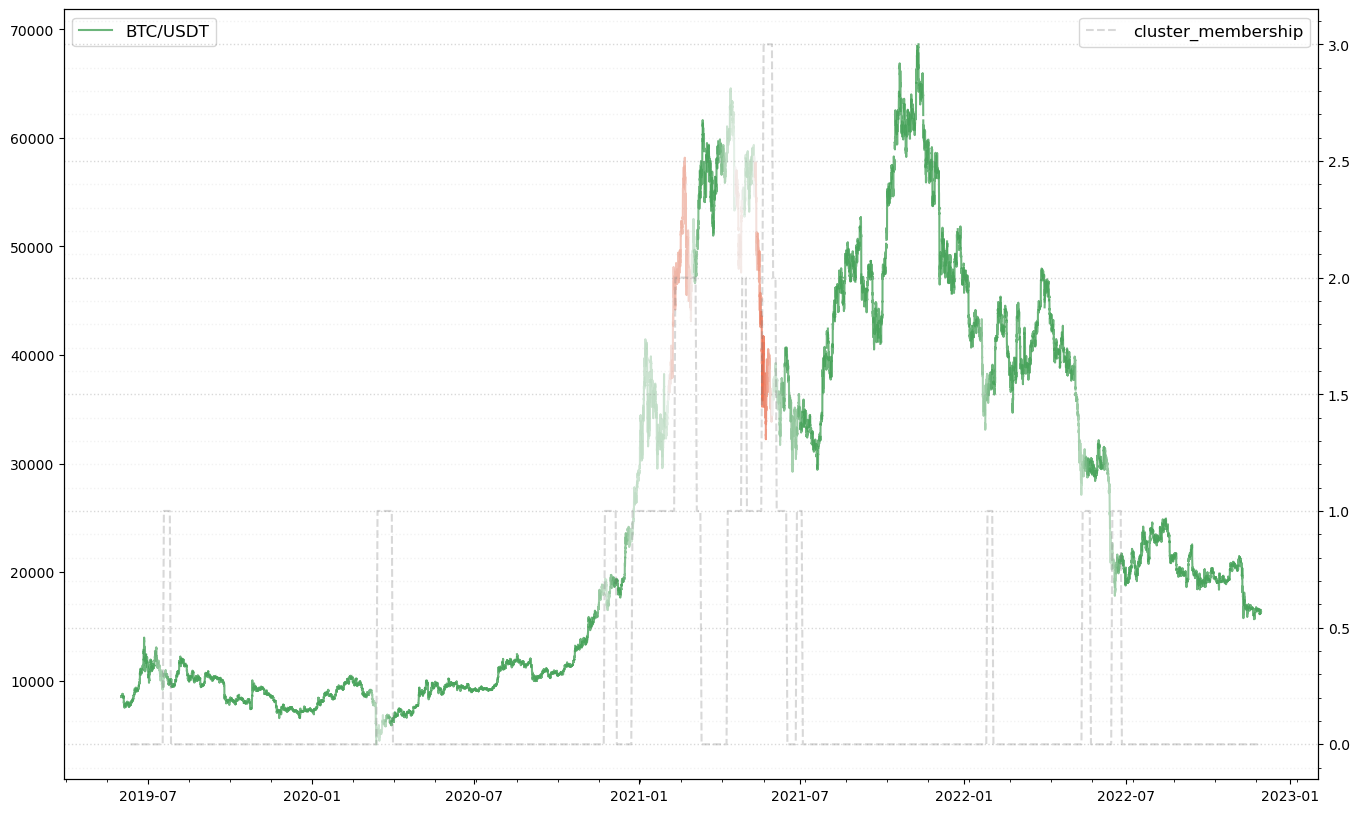}
    \caption{$k=4$.}
    \label{fig:cryptoclustering4}
    \end{subfigure}
    
    \caption{Hierarchical clustering of $A_L(\hat{\mathsf{s}})$.}
    \label{fig:cryptoclustering}
\end{figure}

With only two clusters chosen, we see that the periods of sustained increase and subsequent decrease during 2021 are the only members of the second, more volatile cluster. Interestingly, if you increase the number of clusters to $k=4$, we still see a similar segmentation.

%% file: section6/63total_pipeline.tex
\subsection{A wholly data-driven pipeline}\label{subsec:realdatapipeline}

In this final section, we showcase how both problems considered in this paper - online market regime detection and offline market regime classification - can be employed to build a wholly data-driven market regime detection pipeline, with beliefs, in a single path-by-path (non-ensemble) setting. We include here a real data example to test (in a path-by-path manner) for periods of market turmoil versus periods of relative calm. The price path under consideration is again $\hat{\mathsf{s}} \in \mathcal{T}_\Delta([a, b'], \mathbb{R}^8)$ from Section \ref{subsec:basket}, except here we include recent observations up until $b'=$ 2023-05-01. 

\begin{figure}[h]
    \centering
    \includegraphics[scale=0.4]{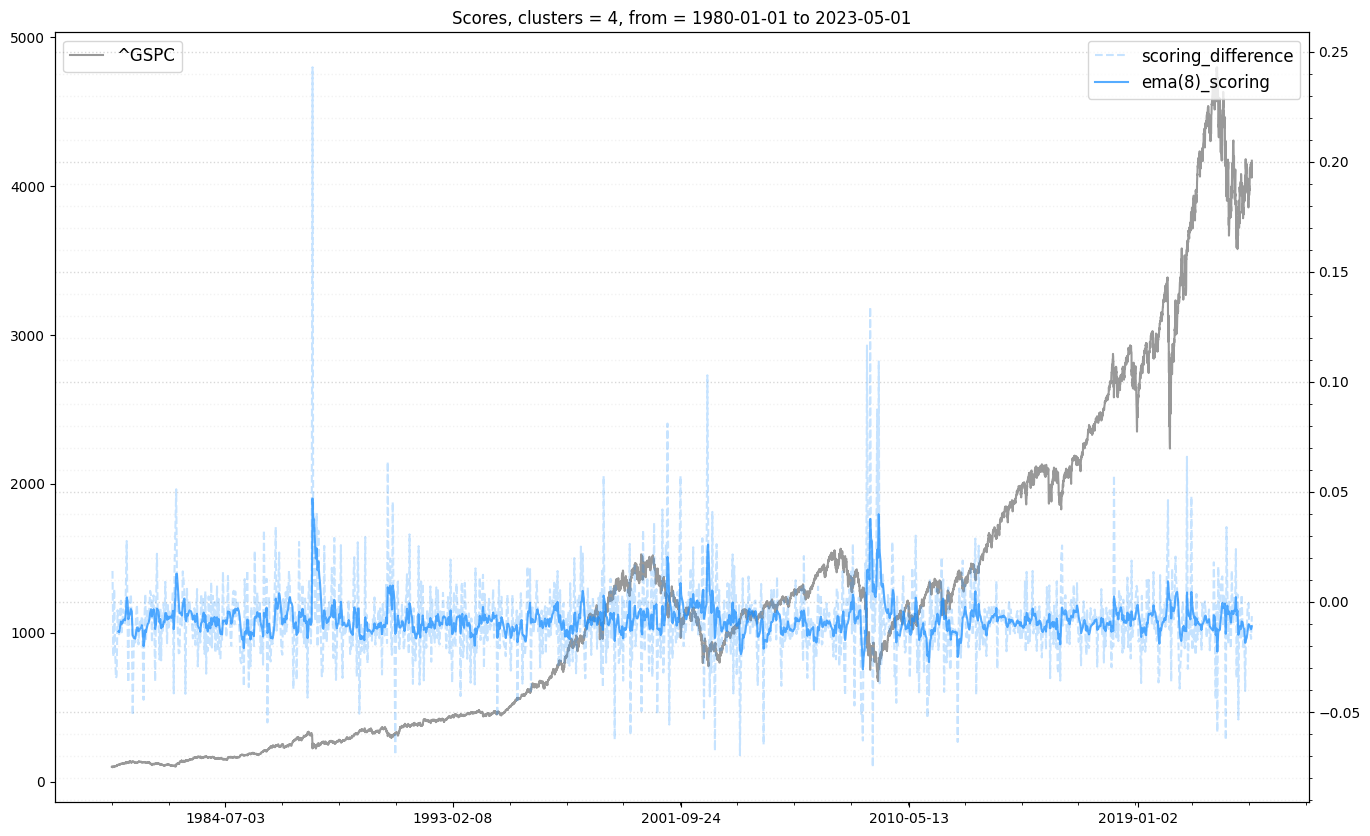}
    \caption{Conformance function evaluated from non-parametric beliefs. Even when evaluating path-by-path, we are still able to detect shifts in regime. The dashed blue line are the raw similarity scores $\Sigma^{\mathfrak{P}}(x)$. The solid blue line is an exponential moving average with $n=8$.}
    \label{fig:realdatapipelinescores}
\end{figure}

We wish to use the similarity score function $\Sigma^{\mathbb{P}, \mathbb{Q}}$ to evaluate path-by-path as opposed to taking ensembles. We calculate $\mathcal{SP}_h^\Phi(\hat{\mathsf{s}})$ with $h=(8,1)$ and take the standard time normalisation and increment transforms. To define our beliefs, we take here to be the clusters (cf. Sections \ref{subsec:mrcpmethods} and \ref{subsec:basket}) derived from the same price path over the date range $[a, b]$ where $a=$ 1980-01-01 and $b=$ 2021-12-20. We take the case where $k=4$ and chose our first set of beliefs to be all paths whose average cluster membership is less than (or equal to) $1$; these represent more stable, persistent market conditions. Naturally the second set of beliefs represent periods of greater turmoil. We then calculate the similarity score vector for each sub-path extracted from $\hat{\mathsf{s}}$, where we use 64 paths sampled from each belief to calculate the unbiased estimators of the constituent scoring rules. 

\begin{figure}[h]
    \centering
    \includegraphics[scale=0.4]{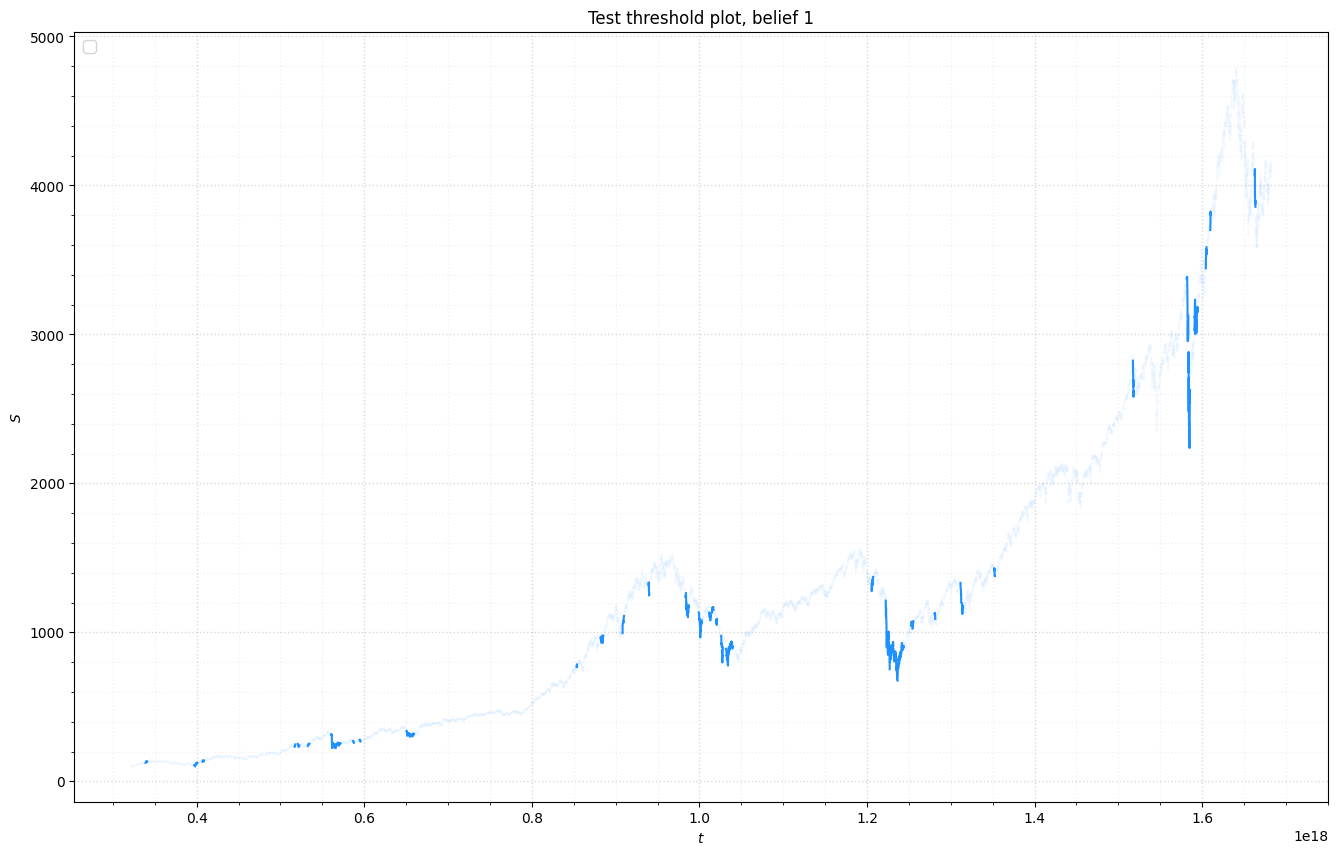}
    \caption{$\Sigma^{\mathfrak{P}}(x) \ge 0$ plot. We used the smoothed value of the similarity score to ascertain conformance to our beliefs. The classical periods of market turmoil are highlighted, indicating that these periods conformed more to the second set of beliefs (chaos).}
    \label{fig:realdatapipelinealphas}
\end{figure}

Figure \ref{fig:realdatapipelinescores} gives the value of $\Sigma^{\mathbb{P}, \mathbb{Q}}(s)$ for each sub-path extracted from $\mathcal{SP}_h^\Phi(\hat{\mathsf{s}})$. It is clear that, generally, we see similar results as per those in Section \ref{subsec:basket}, with the added advantage that we do not need to take ensembles of paths, and we are comparing our results directly to reference measures as opposed to those extracted from the path itself. With our selection of hyperparameters, we see that the conformance score often does not breach above $0$ (indicating periods of increased volatility or market turmoil). Towards the middle of last year, we see that the score was increasing. As of writing, the score is decreasing again, indicating a return to periods more similar to the standard regime. These observations are validated by the results seen in  Figure \ref{fig:realdatapipelinealphas}.

%% file: appendix.tex
\section{Signatures}\label{appendix:signatures}

In this section, we give additional details regarding the path signature. We recall the definition of the set of formal tensor series over a vector space $E$.

\begin{definition}\label{def:tensoralgebraappendix}
    Let $E = \mathbb{R}^d$ be a vector space over a field $\mathbb{F}$ with basis $\{e_1, e_2, \dots, e_d\}$. Denote by $E^{\otimes n}$ the order $n \in \mathbb{N}$ tensor power of $E$, and define $E^{\otimes 0} = \mathbb{R}$ Then, the space 
    \begin{equation*}
        T\left((E)\right) = \{(a_0, a_1, \dots, a_n, \dots) : a_k \in E^{\otimes k}, n \in \mathbb{N} \} = \prod_{n=1}^\infty E^{\otimes k}
    \end{equation*}
    
    to be the space of all formal $E$-tensor series. 
\end{definition}
The space $T((E))$ is an algebra with sum $\oplus$ and product $\cdot$. For $a, b \in T((V))$, these are defined by 
\begin{equation*}
    a \oplus b = \sum_{i \ge 0} a_i + b_i, \quad z_i = \sum_{l=0}^i a_l b_{i-l} \in E^{\otimes i},
\end{equation*}
where $a\cdot b = (z_0, z_1, \dots)$. Often we consider the subalgebra 
\begin{equation*}
    T\left(E\right) = \left\{(a_i)_{i\ge 0} : a_i \in E^{\otimes i} \text{ and there exists }N \in \mathbb{N} \text{ such that }a_i=0 \text{ for all } i\ge N\right\}.
\end{equation*}
The following algebra is also of relevance when it comes to defining the truncated signature mapping.
\begin{definition}\label{def:truncatedtensoralgebraappendix}
    Let $N \in \mathbb{N}$. Then the \emph{truncated tensor algebra} of order $N$ is given by
    \begin{equation*}
        T^N(E) = \left\{(a_i)_{i=1}^N : a_i \in E^{\otimes i} \right\}.
    \end{equation*}
    
    We have that $T^N(E)$ is a subalgebra of $T(E)$ and thus of $T((E))$. The canonical homomorphism $T((E)) \to T^N(E)$ is denoted $\pi_N$. 
\end{definition}
Finally, we recall the following notion of path regularity, which must be considered when deciding how to define the iterated integrals defining the signature mapping.
\begin{definition}[$p$-variation]\label{def:pvariationappendix}
    Let $p \ge 1$ and $X: [0, T] \to E$ be a path. Denote by $\Pi$ the set of partitions over the interval $[0, T]$. Then, the $p$-variation of $X$ is given by 
    \begin{equation*}
        \norm{X}_p := \left(\sup_{\pi \in \Pi} \sum_{[u, v] \in \pi} |X_{v} - X_{u}|^p \right)^{1/p}.
    \end{equation*}
    
    Denote by $C_p([0, T]; E)$ as the set of paths with finite $p$-variation. Equipped with the norm 
    \begin{equation*}
        \norm{X}_{\mathcal{V}^p} := \norm{X}_p + \norm{X}_\infty, 
    \end{equation*}
    
    one has that $(C_p, \norm{\cdot}_{C_p})$ is a Banach space.
\end{definition}	
With this in hand, we can make the following definition.
\begin{definition}[Signature of a path]\label{def:signatureappendix}
    Let $p\ge 1$ and $X \in C_p([0, T]; E)$. The \emph{signature} $S(X) \in T\left((E)\right)$ of the path $X$ is defined as 
    \begin{equation}\label{eqn:signatureappendix}
        S(X) := (1, \mathbb{X}^1_T, \dots, \mathbb{X}^N_T, \dots),
    \end{equation}
    
    where
    \begin{equation}\label{eqn:signatureintegralappendix}
        \mathbb{X}^k_T = \idotsint\limits_{0<t_1<\dots< t_k < T} dX_{u_1} \otimes \dotsm \otimes dX_{u_k} \in E^{\otimes k}.
    \end{equation}
    
    The \emph{truncated signature} $S^N(X) \in T^N(E)$ is similarly defined as
    \begin{equation*}
        S^N(X) := (1, \mathbb{X}^1_T, \dots, \mathbb{X}^N_T, 0, 0, \dots).
    \end{equation*}
    
\end{definition}
One can think of $S(X)$ as encoding all relevant information required to reconstruct $X$ (as stated in the body of the paper, see Proposition \ref{prop:signatureproperties}, the signature mapping is injective up to an equivalence relation on the space of paths). For example, if $X \in C_p([0, T], \mathbb{R}^d)$, the level-one terms $\mathbb{X}^1_T \in \mathbb{R}^d$ are of the form
\begin{equation*}
    \mathbb{X}^1_T = \left(\int_{0}^T dX^i_u\right)_{i=1}^d,
\end{equation*}
and since $\int_{0}^T dX^i_u = X^i_T - X^i_0$, we have that $\mathbb{X}^1_T$ corresponds to the vector of increments for each channel of $X$ over the interval $[0, T]$. The second term $\mathbb{X}^2_T$ contains terms that correspond the signed Levy area of the path $X$. Higher-order terms capture more intricate properties of the path $X$, but often do not have a nice geometric interpretation.
\begin{remark}
    The value of $p \ge 1$ determines the integration theory required to define the elements of eq. (\ref{eqn:signatureintegralappendix}). For instance, if $p \in [1, 2)$, they can be thought of in the sense of Young integration. If $X : [0, T] \to E$ is of bounded variation, then the integrals can be thought of in the Riemann-Stieltjes sense. Larger values of $p$ require rough path theory. As stated in the body, paths passed to the signature mapping are thought of as piecewise linear interpolants of discrete, time-augmented observations $\hat{\mathsf{x}} = ((t_0, x_0), \dots, (t_N, x_N))$. Thus they are always of bounded variation and the signature is always well-defined. 
\end{remark}

We now outline some some analytic and algebraic properties of the signature. The first provides an upper bound on the constituent terms of the signature at level $k$. 

\begin{proposition}[\cite{lyons1998differential}]\label{prop:factorialdecayappendix}
    Let $X$ be a path in $C_1([a, b]; E)$. Denote by $S(X) = (\mathbb{X}^k_{[a,b]})_{k\ge 0}$ the signature associated to $X$. Then, for any $k\in \mathbb{N}$ one has that 
    \begin{equation}\label{eqn:factorialdecay}
        \norm{\mathbb{X}^k_{[a, b]}}_{E^{\otimes k}} \le \frac{\norm{X}_{1, [a,b]}^k}{k!}.
    \end{equation}
\end{proposition}

Truncation of the signature is often justified due to Proposition \ref{prop:factorialdecayappendix}. However, a smaller numerical value associated to a term in the signature of $X$ does not always imply that the term is less relevant, or less important, to the reconstruction of $X$. This is why path scalings are important tools when performing inference with signatures. The subject of optimal path scalings is a topic of future research. 

An important algebraic identity is the following, which we use in practice to construct the signature of a piecewise linear path, see Appendix B for details.

\begin{theorem}[Chen's relation \cite{levin2013learning}, Theorem 2.10]\label{thm:chenappendix}
    Suppose $1 \le p < 2$ and suppose that $X \in C_p([a, b]; E)$ and $Y \in C_p([b, c]; E)$. Denote by $*$ the concatenation operation between two paths $X, Y$, whereby \begin{equation*}
        (X*Y)_t = \begin{cases}
            X_t, &\quad t\in[a, b], \\ 
            X_a + Y_t - Y_b, &\quad t \in [b, c]
        \end{cases}.
    \end{equation*}
    Then, we have that 
    \begin{equation}\label{eqn:chen}
        S(X * Y)_{[a, c]} = S(X)_{[a, b]} \cdot S(Y)_{[b, c]}.
    \end{equation}
\end{theorem}

Other important properties of the signature are outlined in Subsection \ref{subsec:pathsignatures}.

\section{Streaming data}\label{appendix:streamingdata}

The following gives our choice of embedding streaming data into a continuous path evolving over the same time interval.

\begin{definition}[Linear interpolation embedding]\label{def:pathembedding}
    Recall that $\mathcal{T}_\Delta(I, E)$ denotes the space of time-augmented streams over a grid $\Delta$. Then, let
    \begin{equation*}
        \pi: \mathcal{T}_\Delta (I, E) \to C_1(I, E)
    \end{equation*}
    be the embedding of a time-augmented path $\hat{\mathsf{x}} \in \mathcal{T}_\Delta (I, E)$ into $X\in C_1(I, E)$ via
    \begin{equation*}
        X_t := \pi(\hat{\mathsf{x}})_t =
            \frac{x_{t_{i+1}}-x_{t_i}}{t_{i+1}-t_i}(t-t_i) + x_{t_i}, \qquad t \in (t_i, t_{i+1}],
    \end{equation*}
    for $i=1,\dots, n-1$.
\end{definition}

In general, the choice of interpolation is not too important, see for instance \cite{morrill2022choice} for a comprehensive summary on different choices of interpolation in the context of neural controlled differential equations (CDEs). If $X = \pi(\hat{\mathsf{x}})$ is a piecewise linear path embedded in $C_1(I, \mathbb{R}^d)$, we can explicitly construct its signature $S(X)$ via the following algorithm. First, decompose $X$ via its knot points $X_{t_i}$ where $t_i$ is a point on the grid $\Delta$ for $i=1,\dots,N$. We know that $X_t$ is piecewise linear on $[t_{i-1}, t_i]$. Therefore, by Theroem \ref{thm:chenappendix} (Chen's relation), we know that 

\begin{equation}\label{eqn:piecewisesignature}
    S(X)_{[0, T]} = S(X)_{[0, t_1]}\cdot S(X)_{[t_1, t_2]} \cdot \dots \cdot S(X)_{[t_{N-1}, T]}.
\end{equation}

On $[t_{i-1}, t_i]$, we have that $$X_t = x_{t_{i-1}} + \frac{t-t_{i-1}}{t_i - t_{t-1}}(x_{t_{i}} - x_{t_{i-1}}),$$ and therefore $dX_t = x_{t_{i-1}, t_i}/(t_{i} - t_{i-1})$, where $x_{t_{i-1}, t_i} = (0, x_{t_i} - x_{t_{i-1}}, 0, 0, \dots) \in T((\mathbb{R}^d))$. Therefore, it follows that 
\begin{align*}
    \mathbb{X}^k_{[t_{i-1}, t_i]} &= \idotsint_{t_{i-1}<u_1 < \dots < u_k < t_i} dX_{u_1}\otimes \dotsm \otimes dX_{u_k} \\
    &= \frac{(x_{t_{i-1}, t_i})^{\otimes k}}{(b-a)^k} \idotsint_{t_{i-1}<u_1 < \dots < u_k < t_i} du_1 \otimes \dotsm \otimes du_k \\
    &= \frac{(x_{t_{i-1}, t_i})^{\otimes k}}{(b-a)^k} \frac{(b-a)^k}{k!} \\
    &= \frac{(x_{t_{i-1}, t_i})^{\otimes k}}{k!},
\end{align*}
and thus we write $S(X)_{[t_{i-1}, t_i]} = \exp(x_{t_{i-1}, t_i})$, where 
\begin{equation*}
    \exp: T((\mathbb{R}^d)) \to T_1((\mathbb{R}^d)), \quad \exp(v) = \sum_{k=0}^\infty \frac{v^{\otimes k}}{k!},
\end{equation*}
and $T_1((\mathbb{R}^d)) = \{v \in T((\mathbb{R}^d)): \pi_0 v = 1\}$. Therefore, we can calculate the signature of the full, linearly interpolated path $X$ via eqn. (\ref{eqn:piecewisesignature}), which is given by
\begin{equation*}
    S(X)_{[0, T]} = \exp(x_{t_0, t_1}) \otimes \dots \otimes \exp(x_{t_{N-1}, T}).
\end{equation*}

Before we calculate the signature of a given path $\hat{\mathsf{x}} \in \mathcal{T}_\Delta(I, E)$, we may wish to apply stream transformer(s), which can aid in future inference procedures. Here, we give a detailed overview of the stream transformers used in the paper. The first is a normalisation-type transformation which aims to remove level effects from path dynamics.

\begin{definition}[State-normalisation transform]\label{def:statenormalisationtransform}
    Suppose $\hat{\mathsf{x}} \in \mathcal{T}_\Delta (I, \mathbb{R}^d)$ is a time augmented stream of length $n \in \mathbb{N}$. Then, the transformation
    \begin{equation*}
        \phi_{\text{norm}} : \mathcal{T}_\Delta (I, \mathbb{R}^d) \to \mathcal{T}_\Delta (I, \mathbb{R}^d)
    \end{equation*}
    given by 
    \begin{equation}\label{eqn:statenormalisationtransform}
        \phi_\text{norm}(\hat{\mathsf{x}}) = \left\{(t_i, \tfrac{x_i}{x_0}) \right\}_{i=1}^n
    \end{equation}
    is called the \emph{state-normalisation transform}.
\end{definition}
\begin{remark}
    The state-space division in the expression (\ref{eqn:statenormalisationtransform}) is performed component-wise. Other candidate normalisations exist (mean-variance, minimum values, maximum values, and so on) but we found that the initial point normalisation was the easiest and gave better if not comparable results to alternatives.
\end{remark}

One also needs to take care with the sequence of time components present in each $\hat{\mathsf{x}} \in \mathcal{T}_\Delta(I, E)$. If one is only comparing data which is expected to be regularly sampled (for example, daily price values) then one should ensure that this is reflected in the time component of each associated price path. As will be explained later, minor differences (due to data corruption or re-casting of variable types when loading or saving, for example) can result in erroneous conclusions regarding the similarity of paths.  

\begin{definition}[Time-normalisation transformation]
    Let $\hat{\mathsf{x}} \in \mathcal{T}_\Delta (I, \mathbb{R}^d)$ be as in Definition \ref{def:statenormalisationtransform}. Furthermore, let
    \begin{equation*}
        \Delta^{\text{id}} = \bigcup_{i=1}^n \left[\frac{i-1}{n}, \frac{i}{n} \right)
    \end{equation*} 
    be a partition of the interval $[0, 1]$ with mesh-size $1/n$. The transformation 
    \begin{equation*}
        \phi_\text{time}: \mathcal{T}_\Delta(I, \mathbb{R}^d) \to  \mathcal{T}_{\Delta^{\text{id}}}([0, 1], E)
    \end{equation*}
    given by 
    \begin{equation}\label{eqn:timenormalisation}
        \phi_\text{time}(\hat{\mathsf{x}}) = \left\{(\tfrac{i}{n}, x_i) \right\}_{i=1}^n.
    \end{equation}
    is called the \emph{time-normalisation transform}.
\end{definition}
\begin{remark}
    Another prominent example in the literature of a time-component transformation include the \emph{time difference transform}, where one maps $(t_1, \dots, t_n)$ to $(0, t_2-t_1, \dots, t_n-t_{n-1})$. This does not absolve the issue mentioned in the prelude, however.
\end{remark}

Often, including the lagged process can be useful as a way of encoding extra information about $\hat{\mathsf{x}}$, at the cost of increasing the dimensionality of the objects one is working with. This can be achieved by applying the \emph{Hoff lead-lag transformation}. Decomposing a path into its lead and lag components is a natural way to embed a univariate sequence of prices corresponding to a financial asset. For more details, we refer the reader to \cite{Flint_2016}.

\begin{definition}[Hoff lead-lag transformation, \cite{Flint_2016}, Definition 2.1]\label{def:leadlagappendix}
    Let $\Delta = \{t_i\}_{i=0}^n$ be a partition of $[0, T]$, and $\hat{\mathsf{x}} \in \mathcal{T}_\Delta([0, T], \mathbb{R}^d)$ be a time-augmented path over $\mathbb{R}^d$. Then, the \emph{Hoff lead-lag transformation}
    \begin{equation*}
        \phi_{ll}: \mathcal{T}_\Delta([0, T], \mathbb{R}^d) \to \mathcal{S}(\mathbb{R}^{2d})
    \end{equation*}
    is given by
    \begin{equation*}
        \phi_{ll}(\hat{\mathsf{x}}) := \begin{cases}
            (x_{t_k}, x_{t_{k+1}}), &t \in \left[\tfrac{2k}{2nT}, \tfrac{2k + 1}{2nT} \right), \\
            (x_{t_k}, x_{t_{k+1}} + 2(t-(2k+1))(x_{t_{k+2}} - x_{t_{k+1}})), &t \in \left[\tfrac{2k+1}{2nT}, \tfrac{2k + 3/2}{2nT} \right),\\
            (x_{t_{k}} + 2(t-(2k+\tfrac{3}{2}))(x_{t_{k+1}} - x_{t_{k}})), x_{t_{k+2}}), &t \in \left[\tfrac{2k+3/2}{2nT}, \tfrac{2k + 2}{2nT} \right).
        \end{cases}
    \end{equation*}
\end{definition}
\begin{remark}
    If $\mathsf{x} \in \mathcal{S}(\mathbb{R}^d)$ is instead a discrete path of length $n \in \mathbb{N}$, the lead-lag transformation $\varphi_{ll}(\mathsf{x}) \in \mathcal{S}(\mathbb{R}^{2d})$ is given by (see \cite{cochrane2020anomaly}, Definition 2.3.2)
    \begin{equation*}
        \phi_{ll}(\mathsf{x})_{2i} = (x_i, x_i), \qquad\qquad \phi_{ll}(\mathsf{x})_{2i+1} = (x_i, x_{i+1})
    \end{equation*}
    for $i=1,\dots,n$.
\end{remark}

Another useful transformation is one which represents a stream of data via its absolute increments in state space. The idea behind this transformation is to emphasise the realised volatility of $\hat{\mathsf{x}}$, for reasons which we will discuss in Section \ref{sec:experiments}.

\begin{definition}[Increment transform]
    For $\hat{\mathsf{x}} \in \mathcal{T}_\Delta (I, \mathbb{R}^d)$, the \emph{increment transform}
    \begin{equation*}
        \phi_{\mathrm{inc}} : \mathcal{T}_\Delta (I, \mathbb{R}^d) \to \mathcal{T}_\Delta (I, \mathbb{R}^d)
    \end{equation*}
    is given by 
    \begin{equation*}
        \phi_{\mathrm{inc}}(\hat{\mathsf{x}})_0 = x_0, \qquad\qquad \phi_{\mathrm{inc}}(\hat{\mathsf{x}})_{t_k} = x_0 + \sum_{i=1}^k |x_i - x_{i-1}|. 
    \end{equation*}
\end{definition}

The final transform we will introduce is another kind of normalising transform, which aims to reduce the effect of the lower-order terms of the signature. More information will be provided in Section \ref{sec:experiments}.

\begin{definition}[Scaling transform]
    Let $\hat{\mathsf{x}} \in \mathcal{T}_\Delta(I, \mathbb{R}^d)$ be a time-augmented stream of data comprised of $n \in \mathbb{N}$ observations. For a given $\gamma \in \mathbb{R}$, the transformation written
    
    \begin{equation}\label{eqn:scalingtransform}
        \phi_{\mathrm{scale}}(\hat{\mathsf{x}}) = \left\{\left(t_i, \tfrac{x_i}{\gamma}\right) \right\}_{i=0}^n
    \end{equation}
    is called the \emph{scaling transform}.
\end{definition}

\section{Adapted processes and higher rank signatures}\label{appendix:adaptedprocesses}

In this section, we introduce how one can encode the information present in the filtration of a stochastic process via the signature mapping. This will be become relevant in the following section when we show that, from this higher rank mapping, one can compute a distance on path space that takes into account the filtration of a stochastic process and its law.

In what follows, let $(\Omega, \fil, \mathbb{F} = (\fil_t)_{t\in I}, \mathbb{P})$ be a filtered probability space, where $I = [0, T]$ is a time interval partitioned by $\Delta = \{0 = t_0 < t_1 < \dots < t_n = T\}$. Let $X = (X_t)_{t\in I}$ be a stochastic process on $(\Omega, \fil, \mathbb{F} = (\fil_t)_{t\in I})$ taking values in $\mathbb{R}^d$ with $\mathbb{F}$ being the right-continuous filtration generated by $X$. We call
\begin{equation*}
    \boldsymbol{X} = \left(\Omega, \fil, \mathbb{F} = (\fil_t)_{t\in T}, \mathbb{P}, X\right)
\end{equation*}
a filtered (adapted) process. Write $\mathcal{FP}_I$ for the set of all filtered processes indexed by $I$.

Theorem \ref{thm:expectedsignature} tells us that $\mathbb{E}_\mathbb{P}[S(X)]$ completely characterises the distribution of $X$, which is to say that if for another filtered process $\boldsymbol{Y}$ we have that $\mathbb{E}_\mathbb{P}[S(X)] = \mathbb{E}_\mathbb{Q}[S(Y)]$, it follows that $X$ and $Y$ have the same law, $\mathcal{L}(X) = \mathcal{L}(Y)$. However, this is not necessarily true for conditional distributions with respect to the filtration generated by the process $X$. In fact one needs to consider a different process altogether in order to conduct inference on the space of conditional distributions. The key object one needs to consider is the following which was first introduced by Aldous \cite{aldous1981weak}. 

\begin{definition}[Prediction process, \cite{aldous1981weak}, Ch.4, 13]
    Suppose $\boldsymbol{X} \in \mathcal{FP}_I$. Then, the $\mathcal{P}\left((\mathbb{R}^d)^I\right)$-valued discrete process $\hat{X} = (\hat{X}_t)_{t\in I}$ defined by 
    \begin{equation}\label{eqn:predictionprocessappendix}
        \hat{X}_t = \mathbb{P}\left[X \in \cdot \ |\mathcal{F}_t\right]
    \end{equation}
    is called the \emph{prediction process} of $X$.
\end{definition}

If one is considering problems which are dependent on the filtration generated by $X$ over $I$, then the prediction process becomes the object of relevance to study. A canonical example is the problem of pricing an American option, which is an optimal stopping problem over all $\mathbb{F}$-stopping times $\mathcal{T}_{\boldsymbol{X}}$ for a given payoff function $\gamma \in C_b(\mathbb{R} \times \mathbb{R}^d;\mathbb{R}$). It is defined by
\begin{equation*}
    v(\boldsymbol{X}) := \sup_{\tau \in \mathcal{T}_{\boldsymbol{X}}}\ex\left[\gamma(\tau, X_\tau)\right],
\end{equation*}
and under the usual (weak) topology, the value function $v$ is not continuous. However, continuity is achieved under the following extended notion of weak convergence: 
\begin{definition}[Convergence of prediction processes]\label{def:predictionprocessconvergence}
    In what follows, let $\boldsymbol{X}, \boldsymbol{Y}\in \mathcal{FP}_I$.
    \begin{enumerate}
        \item Two adapted processes $\boldsymbol{X}, \boldsymbol{Y}$ are called \emph{synonymous} if $\mathcal{L}(\hat{X}) = \mathcal{L}(\hat{Y})$.
        \item Let $(\hat{X}^n)_{n\ge 1} \subset \mathcal{FP}_I$ be a sequence of filtered processes. We say that $\hat{X}^n$ converges to $\hat{X}$ in the \emph{extended weak sense} if 
        \begin{equation*}
            \mathcal{L}(\hat{X}^n) \to \mathcal{L}(\hat{X})
        \end{equation*}
        as $n\to\infty$. 
    \end{enumerate}
\end{definition}
We reserve the more rigorous discussions of optimal stopping problems with the extended weak topology to \cite{bonnier2020adapted}. The process of iterating conditioning can be repeated and was generalized in \cite{hoover1984adapted} by the following.
\begin{definition}[Rank $r$ prediction process, \cite{hoover1984adapted}]
    For $\boldsymbol{X} \in \mathcal{FP}_I$, the rank $r$ prediction process $\hat{X}^r$ is defined recursively via 
    \begin{equation*}
        \hat{X}^r_t = \mathbb{P}\left[\hat{X}^{r-1} \in \cdot \ |\fil_t\right] \qquad \text{for }t \in I,
    \end{equation*}
    and we set $\hat{X}^0 := X$.
\end{definition}

Similarly to Definition \ref{def:predictionprocessconvergence}, we say that a sequence of filtered processes $\left(\boldsymbol{X}^n\right)_{n\ge 1} \subset \mathcal{FP}_I$ converges in the rank $r$ adapted topology to $\boldsymbol{X} \in \mathcal{FP}_I$ if $\mathcal{L}(\hat{X}^{n, r}) \to \mathcal{L}(\hat{X}^r)$ in the weak sense. If $\boldsymbol{X}, \boldsymbol{Y} \in \mathcal{FP}_I$ have the same rank $r$ adapted distribution we write $\boldsymbol{X} \sim_r \boldsymbol{Y}$. Again $\boldsymbol{X} \sim_0 \boldsymbol{Y}$ is equivalent to the standard weak convergence regarding laws of stochastic processes. 

\begin{definition}[Adapted distribution, \cite{hoover1984adapted}, Definition 2.6]
    Suppose $\boldsymbol{X}, \boldsymbol{Y} \in \mathcal{FP}_I$ are filtered processes. We say that $\boldsymbol{X}, \boldsymbol{Y}$ have the same \emph{adapted distribution} (written $\boldsymbol{X}\sim \boldsymbol{Y}$) if $\boldsymbol{X} \sim_r \boldsymbol{Y}$ for all $r \ge 0$.
\end{definition}

Again the relation that $\boldsymbol{X}\sim_0 \boldsymbol{Y}$ is equivalent to the stochastic processes $X, Y$ having the same finite dimensional distribution. It turns out that if $X$ and $Y$ are Markov processes, then this relation is enough to fully characterise the equality of all adapted distributions from the two processes. In our paper, we use the following theorem to motivate our approach to anomaly detection on path space. 

\begin{theorem}[\cite{hoover1984adapted}, Theorem 2.8]\label{thm:rankmarkov}
    Suppose that $\boldsymbol{X}, \boldsymbol{Y} \in \mathcal{FP}_I$ are Markov. Then, $\boldsymbol{X} \sim_0 \boldsymbol{Y}$ if and only if $\boldsymbol{X} \sim \boldsymbol{Y}$.
\end{theorem}

From Theorem \ref{thm:expectedsignaturelaws}, we know that for two stochastic processes $X, Y$ if $\mathcal{L}(X) = \mathcal{L}(Y)$ we have that $\mathbb{E}S(X) = \mathbb{E}S(Y)$. It follows that if $X, Y$ are Markov, then they have the same adapted distribution. If they are not, then studying their finite-dimensional distributions is not enough to conclude that other, more sophisticated properties hold for the filtered processes $\boldsymbol{X}, \boldsymbol{Y}\in \mathcal{FP}_I$. A prime example is the following. 
\begin{theorem}[\cite{aldous1981weak}]\label{thm:rankmartingale}
    Suppose that $\boldsymbol{X}, \boldsymbol{Y} \in \mathcal{FP}_I$ are filtered processes. Then, if $X$ is a $\mathcal{F}_t$-martingale and $\boldsymbol{X} \sim_1 \boldsymbol{Y}$, then $Y$ is a $\mathcal{G}_t$-martingale.  
\end{theorem}

Although the expected signature map from Definition \ref{def:expectedsignature} gives us a way of comparing rank $0$ (finite-dimensional) distributions of stochastic processes, it does not allow for comparisons between the associated adapted processes. This means that a new object is required if we are to compare (in particular) non-Markovian processes: that is, we wish to find a map $S^{2}$ analogous to the classical signature mapping that gives us
\begin{equation}\label{eqn:higherrankexpectedsignature}
    \boldsymbol{X} \sim_1 \boldsymbol{Y} \text{ if and only if }\mathbb{E}S^{2}(X) = \mathbb{E}S^{2}(Y),
\end{equation}
which would then imply that $\mathcal{L}(\hat{X}^1) = \mathcal{L}(\hat{Y}^1)$. This motivates the definition of the following.
\begin{definition}[Rank $2$ signature, \cite{bonnier2020adapted}, Definition 10]\label{def:rank2signature}
    Recall the signature map $S: \mathbb{R}^d \to T((E))$ from Definition \ref{def:signature}. Suppose that $(\mu_t)_{t\in I}$ is a measure-valued discrete time path evolving in $\mathcal{P}((\mathcal{K})^I)$, so each $\mu_t\in \mathcal{P}((\mathcal{K})^I)$ for $t \in I$. Set
    \begin{equation*}
        \overline{\mu}_t := \int S(x)\mu_t(dx) \in T((E))
    \end{equation*}
    to be the expected signature $\mathbb{E}S(\mu_t)$ of the distribution on path space $\mu_t$. Then, the \emph{rank $2$ signature} is the mapping $S^2: \mathcal{P}((\mathcal{K})^I)^I \to T^2((E))$ is given by 
    \begin{equation}\label{eqn:rank2signature}
        S^2(\mu):= S(t \mapsto \overline{\mu}_t),
    \end{equation}
    where $T^2((E)) := T((T((E))))$.
\end{definition} 	
\begin{remark}
    Although the definition of the signature map was originally for paths $X \in C_p(I, E)$, any discrete path evolving in $\mathbb{R}^d$ can be embedded into a path of finite $p$-variation, see Definition \ref{def:pathembedding}.
\end{remark}
In this way, the rank $2$ signature map is the signature of the evolving expected signature of a given path. Regarding (\ref{eqn:higherrankexpectedsignature}), we have the following.
\begin{theorem}[\cite{salvi2021higher}, Theorem 2]
    Let $\boldsymbol{X}, \boldsymbol{Y}\in \mathcal{FP}_I$ with $I=\{0=t_0 < t_1 < \dots < t_N = T\}$. Then, 
    \begin{equation*}
        \boldsymbol{X} \sim_1 \boldsymbol{Y} \iff \mathcal{L}(\hat{X}^1) = \mathcal{L}(\hat{Y}^1) \iff \mathbb{E}S^2(X) = \mathbb{E}S^2(Y).
    \end{equation*}
\end{theorem}

\begin{remark}
    For a more general discussion of path signatures for any rank $r \ge 1$ we refer the reader to \cite{bonnier2020adapted}, Section 3. 
\end{remark}

\section{Inference in reproducing kernel Hilbert spaces}\label{appendix:rkhsmaterial}

In this section we give more detailed definitions and background information on reproducing kernel Hilbert spaces, the maximum mean discrepancy, and two-sample testing. 

Recall again the definition of the MMD as an integral probability metric: for two measures $\mathbb{P}, \mathbb{Q} \in \mathcal{P}(\mathcal{X})$ and a given space of bounded, measurable functions $\mathcal{F} \subset \{f : \mathcal{X} \to \mathbb{R}\}$, we write
\begin{equation}\label{eqn:appendixmmd}
    \mathcal{D}^\mathcal{F}(\mathbb{P}, \mathbb{Q}) = \sup_{f \in \mathcal{F}} \left( \mathbb{E}_{x \sim \mathbb{P}}[f(x)] - \mathbb{E}_{y \sim \mathbb{Q}}[f(y)]\right).
\end{equation}

The question remains: how does one appropriately choose $\mathcal{F}$? At a minimum one requires that the MMD associated to  $\mathcal{F}$ is a metric on $\mathcal{P}(\mathcal{X})$. However, one would also like $\mathcal{F}$ to be tractable, in the sense that eq. (\ref{eqn:appendixmmd}) is not too difficult to calculate. 

A fine choice that balances these considerations comes from the field we will generically call \emph{kernel methods}, where kernels in this context are continuous, positive semi-definite functions $\kappa: \mathcal{X} \times \mathcal{X} \to \mathbb{R}$. Oftentimes kernels are defined through a \emph{feature map} $\varphi: \mathcal{X} \to \mathcal{H}$ which maps elements $x \in \mathcal{X}$ to an (often higher-dimensional) Hilbert space $\mathcal{H}$. Given $\varphi$, the corresponding kernel $\kappa$ is given by
\begin{equation}\label{eqn:kernelthroughfeaturemap}
    \kappa(x, y) = \langle \varphi(x), \varphi(y)\rangle_\mathcal{H}.
\end{equation}
Clearly kernels are not defined uniquely through feature maps. The same is true for pairs $(\mathcal{H}, \kappa)$; that is, every Hilbert space is not uniquely associated to a given kernel $\kappa$. However, certain pairs enjoy the following characteristic which is integral to being able to define a MMD via a given kernel.

\begin{definition}[Reproducing kernel Hilbert space, \cite{aronszajn1950theory}, Section 1.1]\label{def:appendixreproducingkernelhilbertspace}
    Suppose $\mathcal{X}$ is a non-empty set and let $(\mathcal{H}, \langle \cdot, \cdot \rangle_\mathcal{H})$ be a Hilbert space of functions $f: \mathcal{X} \to \mathbb{R}$. We call a positive definite function $\kappa: \mathcal{X} \times \mathcal{X} \to \real$ a \emph{reproducing kernel} of $\mathcal{H}$ if 
    \begin{enumerate}[label=(\roman*)]
        \item For all $x \in \mathcal{X}$, we have that $\kappa(\cdot, x) \in \mathcal{H}$, and
        \item For all $x \in \mathcal{X}$ and $f \in \mathcal{H}$, one has that 
        \begin{equation}
            f(x) = \langle f(\cdot), \kappa(\cdot, x) \rangle_\mathcal{H},
        \end{equation}
        referred to as the \emph{reproducing property}.
    \end{enumerate}
    We call the Hilbert space $\mathcal{H}$ associated to $\kappa$ a \emph{reproducing kernel Hilbert space} (RKHS). 
\end{definition}
We can associate to each RKHS $\mathcal{H}$ the \emph{canonical feature map} given by $\varphi(x) = \kappa(\cdot, x)$. Thus
\begin{equation*}
    \kappa(x,y) = \langle \kappa(\cdot, x), \kappa(\cdot, y) \rangle_\mathcal{H} = \langle \varphi(x), \varphi(y) \rangle_\mathcal{H} \qquad \text{for all }x, y \in \mathcal{X}.
\end{equation*}

Examples of reproducing kernels on $\mathcal{X} = \mathbb{R}^d$ are the Gaussian kernel $$\kappa_G(x,y) = \exp(-\norm{x-y}^2/2\sigma^2),$$ or the Laplace kernel $$\kappa_L(x,y) = \exp(-c\norm{x-y}).$$ Often it is infeasible to directly evaluate equation (\ref{eqn:kernelthroughfeaturemap}) (for instance, $\mathcal{H}$ may be infinite dimensional). If exclusively performing pairwise evaluations with a given kernel $\kappa$, one can leverage the ``kernel trick'' (if it exists), referring to the fact that evaluating $\kappa(x, y)$ directly is often much easier compared to directly evaluating the equivalent expression $\langle \varphi(x), \varphi(y) \rangle_\mathcal{H}$.

We continue to build on why kernel methods are useful tools for analysis with the following concept: namely, how probability measures on $\mathcal{X}$ can be embedded in $\mathcal{H}$. This means we can shift a given inference problem from one on $\mathcal{P}(\mathcal{X})$ to one on the simpler, linear, inner product space $\mathcal{H}$.

\begin{definition}[Mean embedding]\label{def:meanembedding}
    Let $\mathcal{X}$ be a non-empty set and $\kappa$ a kernel on $\mathcal{X} \times \mathcal{X}$. Given a measure $\mathbb{P} \in \mathcal{P}(\mathcal{X})$, we call the mapping 
    \begin{gather*}
        m : \mathcal{P}(\mathcal{X}) \to \mathcal{H} \\
        \mathbb{P} \mapsto \mathbb{E}_{X \sim \mathbb{P}}[\kappa(\cdot, X)]
    \end{gather*}
    the \emph{mean embedding} of $\mu$ in $\mathcal{H}$.
\end{definition}

Definition \ref{def:meanembedding} leads to the natural question: is $m$ injective? If so, a candidate distance between measures on $\mathcal{P}(\mathcal{X})$ becomes apparent: the distance between their mean embeddings in $\mathcal{H}$. Kernels with this property are called \emph{characteristic} (\cite{NIPS2008_d07e70ef}, Section 2). Characteristicness of the Gaussian and Laplacian kernels on $\mathbb{R}^d$ has been shown in \cite{fukumizu2007kernel}. Associated to characteristicness is the concept of \emph{universality}: that for kernels on (compact) $\mathcal{X}$, one has that the associated RKHS $\mathcal{H}$ is dense in $C_b(\mathcal{X})$ with respect to the $L_\infty$ norm. As noted in \cite{chevyrev2018signature}, these two conditions are often equivalent, and we refer the reader to \cite{simon2018kernel} for more information. We can now conclude the following.

\begin{theorem}[\cite{gretton2012kernel}, Theorem 5]\label{theroem:mmdmetric}
    Let $\mathcal{F}$ be the unit ball of a universal RKHS $(\mathcal{H}, \kappa)$ comprised of $\mathbb{R}$-functions on a compact space $\mathcal{X}$. Suppose $\mathbb{P}, \mathbb{Q} \in \mathcal{P}(\mathcal{X})$ are Borel. Then $\mathcal{D}^\mathcal{F}(\mathbb{P}, \mathbb{Q}) = 0$ if and only if $\mathbb{P} = \mathbb{Q}$.
\end{theorem}

Choosing $\mathcal{F}$ to be the unit ball in a RKHS $(\mathcal{H}, \kappa)$ ensures that the witness function $f^*$ which achieves the MMD between measures $\mathbb{P}, \mathbb{Q} \in \mathcal{P}(\mathcal{X})$ is given by 
\begin{equation}\label{eqn:witnessfmmd}
    f^* = \frac{m(\mathbb{P}) - m(\mathbb{Q})}{\norm{m(\mathbb{P}) - m(\mathbb{Q})}_\mathcal{H}},
\end{equation}
because
\begin{equation*}
    \sup_{f\in\mathcal{F}} \langle f, m(\mathbb{P})-m(\mathbb{Q}) \rangle \le \norm{f}\norm{m(\mathbb{P}) - m(\mathbb{Q})}
\end{equation*}
is maximised when $f^*$ is chosen as in (\ref{eqn:witnessfmmd}), and using the fact that $\mathcal{F}$ is a unit ball in $\mathcal{H}$. We will now drop the dependence on $\mathcal{F}$ in the notation for the MMD and instead emphasise the important on the kernel choice with the superscript. Thus, the squared MMD has the closed form 
\begin{equation*}
    \mathcal{D}^\kappa(\mathbb{P}, \mathbb{Q}) = \norm{m(\mathbb{P}) - m(\mathbb{Q})}^2_\mathcal{H}.
\end{equation*}

In practice one estimates the MMD (\ref{eqn:mmd_general}) from empirical samples. Thus, we have the following. 
\begin{proposition}[\cite{gretton2012kernel}, Lemma 6]
    Suppose $$X = (x_1, \dots, x_N) \quad \text{and } Y = (y_1, \dots, y_M)$$ are such that $x_i \sim \mathbb{P}$ and $y_j \sim \mathbb{Q}$ for $i=1,\dots,N$ and $j=1,\dots,M$ where $\mathbb{P}, \mathbb{Q} \in \mathcal{P}(\mathcal{X}$. Then, a biased estimate for the squared population MMD $\mathcal{D}^\kappa(\mathbb{P}, \mathbb{Q})^2)$ is given by
    \begin{equation}\label{eqn:samplemmd}
        \mathcal{D}_b^\kappa(X, Y)^2 = \frac{1}{N^2}\sum_{i, j = 1}^N \kappa(x_i, x_j) - \frac{2}{MN}\sum_{i=1}^N\sum_{j=1}^M\kappa(x_i, y_j) + \frac{1}{M^2}\sum_{i, j=1}^M\kappa(y_i, y_j).
    \end{equation}
    
    and an unbiased estimator is given by
    
    \begin{equation}\label{eqn:appendixunbiasedsamplemmd}
        \mathcal{D}^\kappa_u(X, Y)^2 = \frac{1}{N(N-1)}\sum_{i\ne j = 1}^N \kappa(x_i, x_j) - \frac{2}{MN}\sum_{i=1}^n\sum_{j=1}^M\kappa(x_i, y_j) + \frac{1}{M(M-1)}\sum_{i \ne j}^M\kappa(y_i, y_j).
    \end{equation}
\end{proposition}

The MMD can be used within the context of two-sample testing. Supressing the emphasis on the kernel in the notation, the Theorem which gives asymptotic consistency in the number of i.i.d samples extracted from each measure is the following. 

\begin{theorem}[\cite{gretton2012kernel}, Theorem 7]\label{thm:mmdconsistency}
    Let $\mathbb{P}, \mathbb{Q} \in \mathcal{P}(\mathcal{X})$ and suppose $X=(x_1, \dots, x_N)$ where $x_i \sim \mathbb{P}$ for $i=1,\dots, N$ and $Y = (y_1, \dots, y_M)$ where $y_j \sim \mathbb{Q}$ for $j=1,\dots, M$. Assume that $0 \le \kappa(x, y) \le K$ for $K > 0$. Then 
    \begin{equation*}
        \mathbb{P}\left[\left|\mathcal{D}_b(X,Y) - \mathcal{D}(\mathbb{P}, \mathbb{Q}) \right| > 2\left(\sqrt{\frac{K}{N}} + \sqrt{\frac{K}{M}}\right) + \varepsilon \right] \le 2\exp\left(\frac{-\varepsilon^2 MN}{2K(M+N)} \right).
    \end{equation*}
\end{theorem}

In particular when $\mathbb{P} = \mathbb{Q}$ and $N=M$ it can be shown (\cite{gretton2012kernel}, Theorem 8) that $\mathcal{D}^\mathcal{F}_b(X , Y) \le \sqrt{2K/N} + \varepsilon$ with probability at least $1-\exp(-\varepsilon^2N/4K)$.

Regarding deriving data-dependent critical threshold bounds for the sample MMD, we provide details for four methods. The first uses the eigenvalues of the sample Gram matrix. In \cite{gretton2009fast}, authors show that, under $H_0$, 
\begin{equation*}
    N\mathcal{D}^2_u \rightharpoonup \sum_{l=1}^\infty \lambda_l (z_l^2 - 2),
\end{equation*}
where $\lambda_l$ are the eigenvalues associated to the solution of the eigenvalue equation 
\begin{equation*}
    \int_\mathcal{X} \tilde{\kappa}(x, x')\psi_i(x)\,d\mu(x) = \lambda_i \psi(x'),
\end{equation*}
where $\tilde{\kappa}(x, x') = \kappa(x, x) - \mathbb{E}_x \kappa(x_i, x) - \mathbb{E}_x \kappa(x, x_j) + \mathbb{E}_{x, x'}\kappa(x, x')$, and $z_l \sim \mathcal{N}(0, 2)$. This motivates approximating the null distribution of the MMD via a Gamma distribution, as seen in the following.

\begin{definition}[Gamma approximation for null distribution of MMD, \cite{gretton2009fast}, Section 3.1]\label{def:gammaapprox}
    Suppose $X = (x_1, \dots, x_N)$ and $Y = (y_1, \dots, y_N)$ are empirical samples. Define $z_i = (x_i, y_i)$ for $i=1,\dots,N$. Then, one has under $H_0$ that
    
    \begin{equation}\label{eqn:mmdgammaapprox}
        N\mathcal{D}\sim \frac{x^{\alpha-1}e^{-x/\beta}}{\beta^\alpha \Gamma(\alpha)},
    \end{equation}
    where
    \begin{equation*}
        \alpha = \frac{\left(\mathbb{E}[\mathcal{D}(Z)]\right)^2}{\mathrm{var}\left(\mathcal{D}(Z)\right)},
    \end{equation*}
    and
    \begin{equation*}
        \beta = \frac{N\mathrm{var}\left(\mathcal{D}(Z)\right)}{\mathbb{E}[\mathcal{D}(Z)]}.
    \end{equation*}
\end{definition}

Moment estimates for the MMD can be found in \cite{gretton2009fast}. A test threshold $c_\alpha$ for given $\alpha \in (0, 1)$ can be obtained by studying the $(1-\alpha)\%$ quantile of the corresponding Gamma estimate of the null distribution of the MMD. This approach is attractive as the most expensive calculation is that of the second moment ($\mathcal{O}(N^2)$). 

Another approach to perform the two-sample test is to use a bootstrapping technique: again with $X, Y$ as in Definition \ref{def:gammaapprox}, build the vector $z = (x_1, \dots, x_N, y_1, \dots, y_M)$. Then, for $i=1,\dots, n$, calculate a permutation of $\pi_i$ of $z$ to extract vectors $$X_{\pi_i} = (\pi(z)_1, \dots, \pi(z)_N)$$ and $$Y_{\pi_i} = (\pi_i(z)_{N+1}, \dots, \pi_i(z)_{N+M}).$$ Then, one calculates the MMD between $X_{\pi_i}$ and $Y_{\pi_i}$, which is stored. This builds an empirical distribution with $n$ atoms and an associated critical threshold $c_\alpha$. One can then reject the null hypothesis if $\mathrm{MMD}_b[\mathcal{F}, X, Y] > c_\alpha$. This method is also purely data-driven, but has computational complexity $\mathcal{O}(n(p + l))$ where $p$ is the cost of computing the permutations and $l$ the cost of calculating the MMD. Thus for repeated evaluations of the MMD this approach is not ideal.

\section{Agglomerative clustering}\label{appendix:agglomerativeclustering}

Here we briefly outline the hierarchical clustering algorithm used for market regime classification problems in the paper. Recall that we are looking to cluster elements $\mathsf{x} \in \mathcal{S}(E)$ where $(E, d)$ is a metric space. The process of creating clusters via splitting (in the divisive case) or combining (agglomerative) is completed via a linkage criterion on elements of $D$. Denote by $2^V$ the power set of elements in $E$. Given two sets of elements $A, B \in 2^E$, common linkages $\ell$ include 
\begin{enumerate}
    \item \emph{Maximum linkage}, $\ell(A, B) = \max\{d(a, b): a \in A, b \in B\}$, 
    \item \emph{Minimum linkage}, $\ell(A, B) = \min\{d(a, b) : a \in A, b \in B\}$, 
    \item \emph{Average linkage}, $\ell(A, B) = \tfrac{1}{|A||B|}\sum_{a\in A, b \in B}d(a, b)$.
\end{enumerate}

Other common linkages include Ward's criterion (how much the within-cluster variance increases with the addition of new elements), the sum of inter-cluster variance, or (as in $k$-means) the distance to the central elements of the given sets $A, B$. We refer to Algorithm \ref{algorithm:hierarchicalclustering} for an outline of how either the agglomerative or divisive hierarchical algorithm works in practice.
\begin{algorithm}[h]
    \SetAlgoLined
    \KwResult{$k$ clusters}
    \textbf{calculate} $D$ given vector of observations $X$ under $d$\;
    \textbf{set} $l=|X|$\;
    \ForEach{$x \in X$}{
        \textbf{set} $\mathcal{C}_j = x_j$\;	
    }
    \textbf{initialise} linkage method $\ell$\;
    \While{$l \ge k$}{
        
    }
    \caption{Agglomerative hierarchical clustering algorithm}
    \label{algorithm:hierarchicalclustering}
\end{algorithm}	

Memory complexity for Algorithm \ref{algorithm:hierarchicalclustering} is $\Omega(n^2)$ (to store all pairwise elements) and runtime is $\mathcal{O}(n^3)$ in the number of cluster elements. In our case, since the (rank 1) MMD between sets $\boldsymbol{x} = (\hat{\mathsf{x}}_1, \dots, \hat{\mathsf{x}}_n)$ and $\boldsymbol{y} = (\hat{\mathsf{y}}_1, \dots, \hat{\mathsf{y}}_n)$ is $\mathcal{O}(n^2l^2d)$ (where $l$ is the length of the path, and $d$ is the dimension) the overall runtime is $\mathcal{O}(Nn^2l^2d)$ where $N$ is the number of cluster elements.

\section{Signature conformance}\label{appendix:conformance}

Here we give brief details of the signature conformance algorithm introduced in \cite{cochrane2020anomaly}. The authors calibrate an anomaly threshold by splitting a given corpus of paths $\mathcal{D}$ into two equally-sized portions, $\mathcal{D}_1 \cup \mathcal{D}_2 = \mathcal{D}$ and then studying the \emph{conformance score}
\begin{equation}\label{eqn:conformance}
    \mathrm{conf}(x; \mu):= \inf_{y \sim \mu} \norm{x-y}_\mu \qquad \text{for all } x \in \mathcal{D}_1,
\end{equation}
where $\mu = \mathcal{P}(\mathcal{D}_2)$ is the empirical measure associated to the compact set $\mathcal{D}_2$. Here, $\norm{\cdot}_\mu: V \to \mathbb{R}$ is the associated \emph{variance norm} given a vector space $V$ and a probability measure $\mu \in \mathcal{P}(V)$. If one considers the empirical measure associated to the distribution of the (rank 1) signature $S^N$ of order $N \in \mathbb{N}$ over any compact set $\mathcal{X} \subset \mathcal{S}(\mathbb{R}^d)$, then the associated variance norm is given by the following.
\begin{proposition}[\cite{cochrane2020anomaly}, Proposition 3.2]
    Let $I \subset [0, \infty)$. For a given compact set $\mathcal{X}\subset C(I, \mathbb{R}^d)$, let $\delta_{\mathcal{X}}$ be the empirical measure associated to the distribution of the signature $S^N$ of order $N\in \mathbb{N}$ over $\mathcal{X}$. Let $d_N = 1 + d + d^2 + \dots + d^N$, and define the matrix $A \in \mathbb{R}^{d_N \times d_N}$ by 
    \begin{equation}\label{eqn:shufflematrix}
        A_{i, j} := \left\langle e_i \Sha e_j, \mathbb{E}S^{2N}(\delta_\mathcal{X}) \right\rangle_{T^{2N}(E)} \qquad \text{for }i,j=1,\dots,d_N.
    \end{equation}
    Here, $e_i, e_j$ are basis vectors over $\mathbb{R}^{d_N}$, $\mathbb{E}S^{2N}: \mathcal{P}(\mathcal{X}) \to T^{2N}(\mathbb{R}^d)$ is the expected signature map, and $\Sha: \mathbb{R}^d \times \mathbb{R}^d \to \mathbb{R}^{2d}$ is the shuffle product. Then, the \emph{variance norm of order $N$} is given by
    \begin{equation*}
        \norm{w}_{\delta_\mathcal{X}} = \langle w, A^{-1}w\rangle \qquad \text{for all }w \in \mathbb{R}^{d_N}.
    \end{equation*}
\end{proposition}

We now recall how the authors indefine an anomalous observation $x \in C(I, E)$ relative to beliefs $\mathfrak{P}$. We refer to this anomaly detection technique as SIG-CON (signature conformance).

\begin{definition}[Anomaly threshold, \cite{cochrane2020anomaly}, Section 3.1]
    Suppose $\alpha \in [0, 1]$ and let $\mathfrak{P} \subset C(I, E)$ be a corpus of paths. Let $\mathfrak{P}^1, \mathfrak{P}^2 \subset \mathfrak{P}$ be disjoint such that $\mathfrak{P}^1 \cup \mathfrak{P}^2 =  \mathfrak{P}$ and $|\mathfrak{P}^1| = |\mathfrak{P}^2|$. Write $S^N(\mathfrak{P}^i)$ for the image of $\mathfrak{P}^i$ under the truncated signature map for $i=1,2$.
    
    Then, an \emph{anomaly threshold} $c_\alpha \in \mathbb{R}$ is the $(1-\alpha)\%$ quantile of the empirical distribution 
    \begin{equation}\label{eqn:conformancenull}
        \mathfrak{D} = \left\{\mathrm{conf}(x; \mu) : x \in S^N(\mathfrak{P}^1) \right\}
    \end{equation}
    where $\mu = \delta_{S^N(\mathfrak{P}^2)}$ is the empirical distribution of paths $p \in \mathfrak{P}^2$ under the truncated signature map $S^N$ of order $N$. 
\end{definition}

We note here that this approach is significantly more computationally expensive and memory intensive than ours outlined in Section \ref{subsec:mrdpmethods}. Computational complexity comes from calculating the matrix $A$ from (\ref{eqn:shufflematrix}), which becomes a very expensive operation as the order of the truncated signature $N$ increases: both since the matrix itself is $\mathbb{R}^{d_N \times d_N}$ and the inner product is taken over the truncated tensor algebra $T^{2N}(\mathbb{R}^d)$ which is functionally an operation in $\mathbb{R}^{d_{2N}}$. Computational complexity also arises in the building of the null distribution $\mathfrak{D}$ from (\ref{eqn:conformancenull}): for each path $p \in \mathfrak{P}^2$, in computing the infimum in (\ref{eqn:conformance}) one necessarily has to compare the variance norm to each path in $\mathfrak{P}^2$, an $\mathcal{O}(n^2)$ operation where $n = |\mathfrak{P}^2|$. Finally, the method does not scale well with dimension: terms in the signature grow exponentially in $d$ as the order $N$ increases. This means that storing and performing calculations with truncated signatures becomes infeasible. As we calculate the signature kernel from a PDE, our approach avoids this issue.

%% file: main.bbl
\newcommand{\etalchar}[1]{$^{#1}$}
\begin{thebibliography}{LYFLLC18}
\expandafter\ifx\csname url\endcsname\relax
  \def\url#1{\texttt{#1}}\fi
\expandafter\ifx\csname doi\endcsname\relax
  \def\doi#1{\burlalt{doi:#1}{http://dx.doi.org/#1}}\fi
\expandafter\ifx\csname urlprefix\endcsname\relax\def\urlprefix{URL }\fi
\expandafter\ifx\csname href\endcsname\relax
  \def\href#1#2{#2}\fi
\expandafter\ifx\csname burlalt\endcsname\relax
  \def\burlalt#1#2{\href{#2}{#1}}\fi

\bibitem[AC17]{aminikhanghahi2017survey}
S.~Aminikhanghahi and D.~J. Cook.
\newblock A survey of methods for time series change point detection.
\newblock {\em Knowledge and information systems}, 51(2):339--367, 2017.

\bibitem[ACH12]{https://doi.org/10.48550/arxiv.1202.3878}
S.~Arlot, A.~Celisse, and Z.~Harchaoui.
\newblock A kernel multiple change-point algorithm via model selection, 2012.
\newblock \doi{10.48550/ARXIV.1202.3878}.

\bibitem[ACH15]{arlot2019kernelcpd}
S.~Arlot, A.~Celisse, and Z.~Harchaoui.
\newblock A kernel multiple change-point algorithm via model selection.
\newblock {\em Journal de la Soci\'{e}t\'{e} Française de Statistique},
  156(4):133--162, 2015.

\bibitem[Ald81]{aldous1981weak}
D.~J. Aldous.
\newblock {\em Weak convergence and the general theory of processes}.
\newblock Editeur inconnu, 1981.

\bibitem[Aro50]{aronszajn1950theory}
N.~Aronszajn.
\newblock Theory of reproducing kernels.
\newblock {\em Transactions of the American mathematical society},
  68(3):337--404, 1950.

\bibitem[BBDG19]{briol2019statistical}
F.-X. Briol, A.~Barp, A.~B. Duncan, and M.~Girolami.
\newblock Statistical inference for generative models with maximum mean
  discrepancy.
\newblock {\em arXiv preprint arXiv:1906.05944}, 2019.

\bibitem[BFG16]{bayer2016pricing}
C.~Bayer, P.~Friz, and J.~Gatheral.
\newblock Pricing under rough volatility.
\newblock {\em Quantitative Finance}, 16(6):887--904, 2016.

\bibitem[BHL{\etalchar{+}}20]{buehler2020data}
H.~Buehler, B.~Horvath, T.~Lyons, I.~Perez~Arribas, and B.~Wood.
\newblock A data-driven market simulator for small data environments.
\newblock {\em Available at SSRN 3632431}, 2020.

\bibitem[BLO20]{bonnier2020adapted}
P.~Bonnier, C.~Liu, and H.~Oberhauser.
\newblock Adapted topologies and higher rank signatures.
\newblock {\em arXiv preprint arXiv:2005.08897}, 2020.

\bibitem[CAA22]{cherief2022finite}
B.-E. Ch{\'e}rief-Abdellatif and P.~Alquier.
\newblock Finite sample properties of parametric mmd estimation: robustness to
  misspecification and dependence.
\newblock {\em Bernoulli}, 28(1):181--213, 2022.

\bibitem[CFLA20]{cochrane2020anomaly}
T.~Cochrane, P.~Foster, T.~Lyons, and I.~P. Arribas.
\newblock Anomaly detection on streamed data, 2020,
  \burlalt{2006.03487}{http://arxiv.org/abs/2006.03487}.

\bibitem[Che57]{chen1957integration}
K.-T. Chen.
\newblock Integration of paths, geometric invariants and a generalized
  baker-hausdorff formula.
\newblock {\em Annals of Mathematics}, pages 163--178, 1957.

\bibitem[CL16]{chevyrev2016characteristic}
I.~Chevyrev and T.~Lyons.
\newblock Characteristic functions of measures on geometric rough paths.
\newblock {\em The Annals of Probability}, 44(6):4049--4082, 2016.

\bibitem[CLX21]{cass2021general}
T.~Cass, T.~Lyons, and X.~Xu.
\newblock General signature kernels.
\newblock {\em arXiv preprint arXiv:2107.00447}, 2021.

\bibitem[CO18]{chevyrev2018signature}
I.~Chevyrev and H.~Oberhauser.
\newblock Signature moments to characterize laws of stochastic processes.
\newblock {\em arXiv preprint arXiv:1810.10971}, 2018.

\bibitem[Con01]{cont2001empirical}
R.~Cont.
\newblock Empirical properties of asset returns: stylized facts and statistical
  issues.
\newblock {\em Quantitative finance}, 1(2):223, 2001.

\bibitem[Dun73]{dunn1973fuzzy}
J.~C. Dunn.
\newblock A fuzzy relative of the isodata process and its use in detecting
  compact well-separated clusters.
\newblock {\em Journal of Cybernetics}, 3(3):32--57, 1973.
\newblock \doi{10.1080/01969727308546046}.

\bibitem[FGSS07]{fukumizu2007kernel}
K.~Fukumizu, A.~Gretton, X.~Sun, and B.~Sch{\"o}lkopf.
\newblock Kernel measures of conditional dependence.
\newblock {\em Advances in neural information processing systems}, 20, 2007.

\bibitem[FGSS09]{NIPS2008_d07e70ef}
K.~Fukumizu, A.~Gretton, B.~Sch\"{o}lkopf, and B.~K. Sriperumbudur.
\newblock Characteristic kernels on groups and semigroups.
\newblock In D.~Koller, D.~Schuurmans, Y.~Bengio, and L.~Bottou, editors, {\em
  Advances in Neural Information Processing Systems}, volume~21. Curran
  Associates, Inc., 2009.
\newblock
  \urlprefix\url{https://proceedings.neurips.cc/paper/2008/file/d07e70efcfab08731a97e7b91be644de-Paper.pdf}.

\bibitem[FHL16]{Flint_2016}
G.~Flint, B.~Hambly, and T.~Lyons.
\newblock Discretely sampled signals and the rough hoff process.
\newblock {\em Stochastic Processes and their Applications},
  126(9):2593–2614, Sep 2016.
\newblock \doi{10.1016/j.spa.2016.02.011}.

\bibitem[Fri04]{friedman2004report}
J.~Friedman.
\newblock On multivariate goodness-of-fit and two-sample testing.
\newblock {\em Technical Report}, Stanford University, 2004.

\bibitem[GA18]{garreauarlot2018cpdkernels}
D.~Garreau and S.~Arlot.
\newblock Consistent change-point detection with kernels.
\newblock {\em Electronic Journal of Statistics}, 12(20), 2018.

\bibitem[GBR{\etalchar{+}}06]{gretton2006kernel}
A.~Gretton, K.~Borgwardt, M.~Rasch, B.~Sch{\"o}lkopf, and A.~Smola.
\newblock A kernel method for the two-sample-problem.
\newblock {\em Advances in neural information processing systems}, 19, 2006.

\bibitem[GBR{\etalchar{+}}12]{gretton2012kernel}
A.~Gretton, K.~M. Borgwardt, M.~J. Rasch, B.~Sch{\"o}lkopf, and A.~Smola.
\newblock A kernel two-sample test.
\newblock {\em Journal of Machine Learning Research}, 13(Mar):723--773, 2012.

\bibitem[GFHS09]{gretton2009fast}
A.~Gretton, K.~Fukumizu, Z.~Harchaoui, and B.~K. Sriperumbudur.
\newblock A fast, consistent kernel two-sample test.
\newblock {\em Advances in neural information processing systems}, 22, 2009.

\bibitem[GR05]{gneiting2005weather}
T.~Gneiting and A.~E. Raftery.
\newblock Weather forecasting with ensemble methods.
\newblock {\em Science}, 310(5746):248--249, 2005.

\bibitem[GR07]{gneiting2007strictly}
T.~Gneiting and A.~E. Raftery.
\newblock Strictly proper scoring rules, prediction, and estimation.
\newblock {\em Journal of the American statistical Association},
  102(477):359--378, 2007.

\bibitem[HIM21]{horvath2021clustering}
B.~Horvath, Z.~Issa, and A.~Muguruza.
\newblock Clustering market regimes using the wasserstein distance.
\newblock {\em Available at SSRN 3947905}, 2021.

\bibitem[HK84]{hoover1984adapted}
D.~N. Hoover and H.~J. Keisler.
\newblock Adapted probability distributions.
\newblock {\em Transactions of the American Mathematical Society},
  286(1):159--201, 1984.

\bibitem[HL10]{hambly2010uniqueness}
B.~Hambly and T.~Lyons.
\newblock Uniqueness for the signature of a path of bounded variation and the
  reduced path group.
\newblock {\em Annals of Mathematics}, pages 109--167, 2010.

\bibitem[HMN15]{chenzhang2015graph}
S.~Hediger, L.~Michel, and J.~N\"{a}f.
\newblock Graph-based change-point detection.
\newblock {\em The Annals of Statistics}, 43(1):139--176, 2015.

\bibitem[HMN22]{heidiger2022twosample}
S.~Hediger, L.~Michel, and J.~N\"{a}f.
\newblock On the use of random forest for two-sample testing.
\newblock {\em Computational Statistics \& Data Analysis}, 170(170435), 2022.

\bibitem[IHLS23]{issa2023non}
Z.~Issa, B.~Horvath, M.~Lemercier, and C.~Salvi.
\newblock Non-adversarial training of neural sdes with signature kernel scores.
\newblock {\em arXiv preprint arXiv:2305.16274}, 2023.

\bibitem[Joh67]{johnson1967hierarchical}
S.~C. Johnson.
\newblock Hierarchical clustering schemes.
\newblock {\em Psychometrika}, 32(3):241--254, 1967.

\bibitem[KNS{\etalchar{+}}19]{kolouri2019generalized}
S.~Kolouri, K.~Nadjahi, U.~Simsekli, R.~Badeau, and G.~K. Rohde.
\newblock Generalized sliced wasserstein distances.
\newblock {\em arXiv preprint arXiv:1902.00434}, 2019.

\bibitem[KR09]{kaufman2009finding}
L.~Kaufman and P.~J. Rousseeuw.
\newblock {\em Finding groups in data: an introduction to cluster analysis},
  volume 344.
\newblock John Wiley \& Sons, 2009.

\bibitem[KS19]{kondratyev2019market}
A.~Kondratyev and C.~Schwarz.
\newblock The market generator.
\newblock {\em Available at SSRN 3384948}, 2019.

\bibitem[LBK22]{londbuhlkov2022changepoint}
M.~Londschien, P.~B\"{u}hlmann, and S.~Kov\'{a}cs.
\newblock Random forests for change point detection.
\newblock {\em arXiv preprint arXiv:2205.04997}, 2022.

\bibitem[LCC{\etalchar{+}}17]{li2017mmd}
C.-L. Li, W.-C. Chang, Y.~Cheng, Y.~Yang, and B.~P{\'o}czos.
\newblock Mmd gan: Towards deeper understanding of moment matching network.
\newblock {\em Advances in neural information processing systems}, 30, 2017.

\bibitem[Lin72]{ling1972theory}
R.~F. Ling.
\newblock On the theory and construction of k-clusters.
\newblock {\em The computer journal}, 15(4):326--332, 1972.

\bibitem[LLN13]{levin2013learning}
D.~Levin, T.~Lyons, and H.~Ni.
\newblock Learning from the past, predicting the statistics for the future,
  learning an evolving system.
\newblock {\em arXiv preprint arXiv:1309.0260}, 2013.

\bibitem[LPO17]{lopezpaz2017oquab}
D.~Lopez-Paz and M.~Oquab.
\newblock Revisiting classifier two-sample tests.
\newblock {\em International Conference on Learning Representations}, 2017.

\bibitem[LSD{\etalchar{+}}20]{lemercier2020distribution}
M.~Lemercier, C.~Salvi, T.~Damoulas, E.~V. Bonilla, and T.~Lyons.
\newblock Distribution regression for sequential data, 2020,
  \burlalt{2006.05805}{http://arxiv.org/abs/2006.05805}.

\bibitem[LYFLLC18]{lungyutfong2015homogeneity}
A.~Lung-Yut-Fong, C.~L\'{e}vy-Leduc, and O.~Capp\'{e}.
\newblock Homogeneity and changepoint detection tests for multivariate data
  using rank statistics.
\newblock {\em Electronic Journal of Statistics}, 12(20), 2018.

\bibitem[Lyo98]{lyons1998differential}
T.~J. Lyons.
\newblock Differential equations driven by rough signals.
\newblock {\em Revista Matem{\'a}tica Iberoamericana}, 14(2):215--310, 1998.

\bibitem[M{\etalchar{+}}89]{mcdiarmid1989method}
C.~McDiarmid et~al.
\newblock On the method of bounded differences.
\newblock {\em Surveys in combinatorics}, 141(1):148--188, 1989.

\bibitem[Mac65]{macqueen1967some}
J.~MacQueen.
\newblock Some methods for classification and analysis of multivariate
  observations.
\newblock In {\em Proc. 5th Berkeley Symposium on Math., Stat., and Prob}, page
  281, 1965.

\bibitem[MFKL20]{morrill2020generalised}
J.~Morrill, A.~Fermanian, P.~Kidger, and T.~Lyons.
\newblock A generalised signature method for multivariate time series feature
  extraction.
\newblock {\em arXiv preprint arXiv:2006.00873}, 2020.

\bibitem[MJ14]{matteson2014nonparametric}
D.~S. Matteson and N.~A. James.
\newblock A nonparametric approach for multiple change point analysis of
  multivariate data.
\newblock {\em Journal of the American Statistical Association},
  109(505):334--345, 2014.

\bibitem[MKYL22]{morrill2022choice}
J.~Morrill, P.~Kidger, L.~Yang, and T.~Lyons.
\newblock On the choice of interpolation scheme for neural cdes.
\newblock {\em Transactions on Machine Learning Research}, 2022(9), 2022.

\bibitem[MN21]{mroueh2021convergence}
Y.~Mroueh and T.~Nguyen.
\newblock On the convergence of gradient descent in gans: Mmd gan as a gradient
  flow.
\newblock In {\em International Conference on Artificial Intelligence and
  Statistics}, pages 1720--1728. PMLR, 2021.

\bibitem[MS13]{merkle2013choosing}
E.~C. Merkle and M.~Steyvers.
\newblock Choosing a strictly proper scoring rule.
\newblock {\em Decision Analysis}, 10(4):292--304, 2013.

\bibitem[ND{\c{S}}B19]{nadjahi2019asymptotic}
K.~Nadjahi, A.~Durmus, U.~{\c{S}}im{\c{s}}ekli, and R.~Badeau.
\newblock Asymptotic guarantees for learning generative models with the
  sliced-wasserstein distance.
\newblock {\em arXiv preprint arXiv:1906.04516}, 2019.

\bibitem[NSV{\etalchar{+}}21]{ni2021sigwasserstein}
H.~Ni, L.~Szpruch, M.~S. Vidales, B.~Xiao, M.~Wiese, and S.~Liao.
\newblock Sig-wasserstein gans for time series generation.
\newblock {\em CoRR}, abs/2111.01207, 2021,
  \burlalt{2111.01207}{http://arxiv.org/abs/2111.01207}.
\newblock \urlprefix\url{https://arxiv.org/abs/2111.01207}.

\bibitem[PADD21]{pacchiardi2021probabilistic}
L.~Pacchiardi, R.~Adewoyin, P.~Dueben, and R.~Dutta.
\newblock Probabilistic forecasting with conditional generative networks via
  scoring rule minimization.
\newblock {\em arXiv preprint arXiv:2112.08217}, 2021.

\bibitem[RPDB11]{rabin2011wasserstein}
J.~Rabin, G.~Peyr{\'e}, J.~Delon, and M.~Bernot.
\newblock Wasserstein barycenter and its application to texture mixing.
\newblock In {\em International Conference on Scale Space and Variational
  Methods in Computer Vision}, pages 435--446. Springer, 2011.

\bibitem[SCF{\etalchar{+}}21]{salvi2021signature}
C.~Salvi, T.~Cass, J.~Foster, T.~Lyons, and W.~Yang.
\newblock The signature kernel is the solution of a goursat pde.
\newblock {\em SIAM Journal on Mathematics of Data Science}, 3(3):873--899,
  2021.

\bibitem[SGK12]{sinn2012detecting}
M.~Sinn, A.~Ghodsi, and K.~Keller.
\newblock Detecting change-points in time series by maximum mean discrepancy of
  ordinal pattern distributions.
\newblock {\em arXiv preprint arXiv:1210.4903}, 2012.

\bibitem[SGS18]{simon2018kernel}
C.-J. Simon-Gabriel and B.~Sch{\"o}lkopf.
\newblock Kernel distribution embeddings: Universal kernels, characteristic
  kernels and kernel metrics on distributions.
\newblock {\em The Journal of Machine Learning Research}, 19(1):1708--1736,
  2018.

\bibitem[SLL{\etalchar{+}}21]{salvi2021higher}
C.~Salvi, M.~Lemercier, C.~Liu, B.~Horvath, T.~Damoulas, and T.~Lyons.
\newblock Higher order kernel mean embeddings to capture filtrations of
  stochastic processes.
\newblock {\em Advances in Neural Information Processing Systems}, 34, 2021.

\bibitem[TOV20]{truong2020review}
C.~Truong, L.~Oudre, and N.~Vayatis.
\newblock Selective review of offline change point detection methods.
\newblock {\em Signal Processing}, 167:107299, 2020.

\bibitem[WBWB19]{wiese2019simulation}
M.~Wiese, L.~Bai, B.~Wood, and H.~Buehler.
\newblock Deep hedging: learning to simulate equity option markets.
\newblock {\em arXiv preprint arXiv:1911.01700}, 2019.

\bibitem[ZC21]{zhang2021chen}
Y.~Zhang and H.~Chen.
\newblock Graph-based multiple change-point detection.
\newblock {\em ArXiv Preprint}, arXiv:2110.01170, 2021.

\bibitem[ZGR21]{zhang2021understanding}
L.~Zhang, M.~Goldstein, and R.~Ranganath.
\newblock Understanding failures in out-of-distribution detection with deep
  generative models.
\newblock In {\em International Conference on Machine Learning}, pages
  12427--12436. PMLR, 2021.

\end{thebibliography}
